\documentclass{article} 
\usepackage{arxiv}

\usepackage{amsmath,amsfonts,bm}

\def\eqref#1{equation~\ref{#1}}

\def\1{\bm{1}}

\DeclareMathAlphabet{\mathsfit}{\encodingdefault}{\sfdefault}{m}{sl}
\SetMathAlphabet{\mathsfit}{bold}{\encodingdefault}{\sfdefault}{bx}{n}

\usepackage{hyperref}

\usepackage{graphicx}
\graphicspath{ {./images/} }

\usepackage{tcolorbox}
\usepackage{mathtools}

\usepackage{adjustbox}
\usepackage{color,soul}
\usepackage{amsmath}
\usepackage{amssymb}

\usepackage{bm}
\usepackage{amsfonts}       
\usepackage{url}            
\usepackage{float}

\newcommand{\St}{\text{St}^m_n}
\newcommand{\Sttan}{\mathcal{T}_U\text{St}^m_n}
\newcommand{\DeltaT}{\Delta^{\intercal}}
\newcommand{\UT}{U^{\intercal}}
\newcommand{\UTilde}{\tilde{U}}
\newcommand{\T}{\intercal}
\newcommand{\Gr}{\text{Gr}^m_n}

\usepackage{amsthm}

\newtheorem{theorem}{Theorem}[section]
\newtheorem{cor}{Corollary}[theorem]
\newtheorem{lemma}[theorem]{Lemma}

\usepackage{nicematrix}
\usepackage{arydshln}

\usepackage{tikz}

\usepackage{subfigure}

\usepackage{algorithm2e}
\RestyleAlgo{ruled}
\SetKwComment{Comment}{/* }{ */}

\usepackage[cachedir=minted-cache,frozencache=true]{minted} 
\tcbuselibrary{minted,breakable,xparse,skins}

\definecolor{bg}{gray}{0.95}
\DeclareTCBListing{mintedbox}{O{}m!O{}}{%
  breakable=true,
  listing engine=minted,
  listing only,
  minted language=#2,
  minted style=default,
  minted options={%
    linenos,
    gobble=0,
    breaklines=true,
    breakafter=,,
    fontsize=\small,
    numbersep=8pt,
    #1},
  boxsep=0pt,
  left skip=0pt,
  right skip=0pt,
  left=25pt,
  right=0pt,
  top=3pt,
  bottom=3pt,
  arc=5pt,
  leftrule=0pt,
  rightrule=0pt,
  bottomrule=2pt,
  toprule=2pt,
  colback=bg,
  colframe=orange!70,
  enhanced,
  overlay={%
    \begin{tcbclipinterior}
    \fill[orange!20!white] (frame.south west) rectangle ([xshift=20pt]frame.north west);
    \end{tcbclipinterior}},
  #3}

\usepackage{textcomp}

\usepackage{multirow}

\title{StiefelGen: A Simple, Model Agnostic Approach\\ for Time Series Data Augmentation\\ over Riemannian Manifolds}

\author{
  Prasad Cheema \\
  \textit{National Institute of Informatics}\\
  \texttt{prasad@nii.co.jp}
\And
  Mahito Sugiyama\\
  \textit{National Institute of Informatics}\\
  \textit{SOKENDAI}\\
  \texttt{mahito@nii.co.jp}
}

\begin{document}

\maketitle

\begin{abstract}

Data augmentation is an area of research which has seen active development in many machine learning fields, such as in image-based learning models, reinforcement learning for self driving vehicles, and general noise injection for point cloud data. However, convincing methods for general time series data augmentation still leaves much to be desired, especially since the methods developed for these models do not readily cross-over. Three common approaches for time series data augmentation include: (i) Constructing a physics-based model and then imbuing uncertainty over the coefficient space (for example), (ii) Adding noise to the observed data set(s), and, (iii) Having access to ample amounts of time series data sets from which a robust generative neural network model can be trained. However, for many practical problems that work with time series data in the industry: (i) One usually does not have access to a robust physical model, (ii) The addition of noise can in of itself require large or difficult assumptions (for example, what probability distribution should be used? Or, how large should the noise variance be?), and, (iii) In practice, it can be difficult to source a large representative time series data base with which to train the neural network model for the underlying problem. In this paper, we propose a methodology which attempts to simultaneously tackle all three of these previous limitations to a large extent. The method relies upon the well-studied matrix differential geometry of the Stiefel manifold, as it proposes a simple way in which time series signals can placed on, and then smoothly perturbed over the manifold. We attempt to clarify how this method works by showcasing several potential use cases which in particular work to take advantage of the unique properties of this underlying manifold.

\end{abstract}

\section{Introduction}

Deep learning has proven effective in various domains such as natural language processing (NLP), computer vision (CV), and speech recognition, requiring extensive datasets for training models with billions of parameters. Fortunately, substantial data currently exists within these sub-fields. However, one can further enhance the number of available data, through synthetic data generation which aims to create artificial samples resembling unexplored areas of the input space. For instance, in CV, common techniques include stretching, flipping, cropping, and adjusting hue in the image training data set \cite{shorten2019survey}. More recently, there have been innovations seen in \textit{mix-up} strategies for data augmentation, which is essentially a form of randomized interpolation (often linear) between pairs of input data (sometimes over the latent space) and or class labels \cite{zhang2017mixup, verma2018manifold}. Despite the recent success of the aforementioned approaches for synthetic data generation, research in the time series domain has received relatively limited attention. This lack of focus can be partially attributed to the difficulties within the inherent temporal dependency structures within the time series tasks, meaning that many existing data augmentation methods simply fail. For example, a mix-up approach between two different physical signals may result in a third generated signal that is entirely non-realistic, flipping portions of the signal may result in a system which disobeys system physics, and methods such as hue-adjustment simply do not apply. Although in practice, such methods have been used to show some level of model improvement, they still leave a lot to be desired and do not demonstrate a level of maturity comparable to those data augmentation methods within the CV field \cite{iglesias2023data,iwana2021empirical}. In the following subsections we shall briefly introduce some of the classical and modern approaches towards time series data generation.

\subsection{Classical Approaches}

Classical approaches in time series data augmentation typically involve direct modifications to the underlying signal in either the time and/or frequency domains. Common techniques include jittering (adding noise), magnitude alterations, signal warping (e.g., Dynamic Time Warping - DTW), and signal flipping \cite{iglesias2023data,iwana2021empirical}. However, these traditional methods present several drawbacks, which have prompted the shift towards modern innovations. For instance, jittering tends to limit signal perturbation to noise-like features only, neglecting other sources of novelty in time series (that is, strongly favoring an aleatoric source of uncertainty as opposed to an epistemic one) \cite{balakrishnan1962problem,zhang2007neural}. Magnitude alterations may entirely change the semantic meaning of the domain \cite{iglesias2023data}, the size and style of a desired warped signal (time or frequency) may be hard to control in practice due to difficult hyper parameter optimization \cite{iwana2021time,fujiwara2021data} , and as mentioned earlier, signal flipping (and or permuting) usually presents as a strong violation of ordering assumptions built into system physics \cite{small2005applied}. Furthermore, although these time series data augmentation methods can be combined to offer slight increases in task accuracy, the combination of doing so often requires increased care in hand engineering the hyper parameters for the task. However, if one has access to the equations governing time series, then one could traditionally also push towards justification for time series data augmentation aligned with Bayesian learning. For instance, in mathematical engineering, the Navier-Stokes (NS) equations offer a precise representation of fluid flow. Leveraging these equations allows for introducing distributions over parameters in the partial differential equation (PDE) model, revealing diverse solution pathways for the fluid flow field \cite{abucide2021data}. The difficulty however, is that such approaches are not commonly seen in the machine learning community due to the inability to have knowledge of the ground truth equations of the generative functions for the physics in general. However there exist ways to form a basis approximation of the underlying signal which sees use in GeneRAting TIme Series (GRATIS) \cite{kang2020gratis} which work by generating a mixture of autoregressive (MAR) models, from which multiple instance of non stationary time series can be generated. Similar traditional methodologies can be seen via dynamic linear modeling of the system signal \cite{fruhwirth1994data}, and in augmentation via sampling over conditional and marginal distributions \cite{meng1999seeking}.

\subsection{Modern Approaches}

Modern approaches to time series data generation have predominantly attempted to leverage the power of deep learning (DL) methods, making use of extensive time series databases and the general non-linear learning capabilities of deep generative model architectures. An exemplar is the encoder-decoder type models, which project high-dimensional structural input to a lower-dimensional latent space, facilitating the generation of new high-dimensional structural data through latent space sampling and exploration. One such example may be seen in the LSTM-AE model used to assist in skeleton-based human action recognition \cite{tu2018spatial}. Another example may be seen in \cite{devries2017dataset}, with the main difference being that an LSTM variational autoencoders (VAE) was used. Another popular DL approach sees the use of Generative Adversarial Networks (GANs) \cite{goodfellow2014generative}, which resembles AEs but with a different cost function rooted in game theoretic principles. Time series GANs exhibit several distinct architectural groupings, including fully-connected GANs \cite{lou2018one}, GANs with recurrent neural networks topologies (RNNs) constructed with respect to LSTM models \cite{haradal2018biosignal} or hidden Markov models \cite{hasibi2019augmentation}, and GANs utilizing sliding windows via convolutional neural networks (CNNs) \cite{ramponi2018t, che2017boosting}. These GANs may model over the time or spectral domain. Despite promising results in certain applications, DL-based methods do require substantial training data, prolonged training times, and face challenging tuning, which are common issues seen in the broader DL field \cite{sutskever2011generating}. In time series however, data scarcity and high generation costs compound these challenges, especially when compared to image, or language databases. Furthermore, it can be expensive to collect and or generate as seen in the examples of wind tunnel data in aerospace engineering \cite{osburg2011pricing}. As a result some methods have tried to work directly with signal and perturb it in the frequency domain to induce different noisy instances of some central reference signal. This can be seen in augmentation bank where it has been used by the authors in order to improve pre-training of unsupervised domain adaptation tasks \cite{zhang2022self}, and in Spectral and Time Augmentation (STAug), where the authors apply the empirical mode decomposition (EMD) \cite{rilling2003empirical} to split a
time series into multiple subcomponents, which can be then be individually perturbed to generate new instances of the original signal \cite{zhang2023towards}. Lastly, an important contribution to the data augmentation field as a whole may be seen in the plethora of mixup strategies which exist. Starting from its linear mixing roots \cite{zhang2017mixup}, there now exist binary \cite{beckham2019adversarial}, cut \cite{yun2019cutmix}, weighted geometric \cite{verma2021towards}, amplitude \cite{xu2021fourier}, and spectrogram \cite{kim2021specmix} instances of the mixing formulation among many others. Moreover recently, by controlling for the \textit{degree of chaos}, there have been efforts to offer a controlled form mixing for time domain specific tasks \cite{demirel2023finding} 

However, mixup methods as applied to time series data still leave a lot to be desired as the intermediary signals they generate don't necessarily have the guarantee that they produce signals which are physically realizable due to the underlying complexities which often plague dynamical data generation physics \cite{west1999evaluation}. Further, there exists difficulties in the deep learning approaches since they often require relatively large amounts of data not found in certain fields, such as engineering and physics. Lastly, the classical approaches tend to not offer enough nuances required for effective signal generation.

This paper introduces a novel approach to time series data augmentation, distinct from both classical and contemporary methods within the field. Our proposed methodology, named \textit{StiefelGen}, leverages matrix differential geometry to seamlessly provide tailored levels of aleatoric and epistemic uncertainty by traversing geodesic curves in matrix space. Futhermore, it is distinguished by its minimal hyperparameter requirements, simplicity in implementation, interpretability, and model-agnostic nature. Initially, we shall showcase its effectiveness on empirical signals, elucidating its smooth geodesic properties. Subsequently, in recognition of the distinctiveness of this time series data augmentation strategy, we shall explore some unique use cases of it. We delve into a pertinent application in structural health monitoring (SHM) and explore \textit{StiefelGen's} ability to study the robustness of SHM one-class classifier, and showcase its ability to generate adversarial time series samples. Moreover, we explore its capacity to integrate uncertainty quantification (UQ) and forecasting through the incorporation of dynamic mode decomposition (DMD), another model-agnostic, data-driven technique. Finally, fulfilling the essential role of time series data augmentation, we illustrate \textit{StiefelGen}'s proficiency in generating synthetic data, thereby enhancing the learning capabilities of a generic LSTM model.

\section{Methodology}

The algorithm posited in this paper draws its foundations from the realm of matrix differential geometry \cite{absil2008optimization}, with a specific emphasis on the Stiefel manifold—a linchpin in the broader landscape of matrix manifold theory. To impart a succinct conceptual grasp of this intricate framework, we offer a concise elucidation herein. For those seeking a more nuanced initiation into this intricate domain, we direct the reader to the Appendix (Section \ref{app_sec:Riemannian_Intro}), where a comprehensive and accessible introduction to the topic has been prepared to assist people new to this field.

\subsection{Mathematical Preliminaries} \label{sec:background_math}

Consider a smooth manifold $\mathcal{M}$. If we take $\mathcal{M}\subseteq \mathbb{R}^{m\times n}$, then $\mathcal{M}$ is what is known as a \textit{matrix manifold}. In particular, the \textit{Stiefel manifold}, $\St\subseteq \mathbb{R}^{m\times n}$ is defined as the set of matrix elements in $\mathbb{R}^{m\times n}$ which satisfy common matrix orthogonality constraints. More specifically, 
\begin{equation}
    \St \coloneqq \{U\in\mathbb{R}^{m\times n} \mid U^{\intercal}U = I_n, m\geq n\}.
\end{equation}
Notice that if $n=1$, then one is working over the geometry of a hyper sphere, and thus $\St$ can be thought of as the matrix generalization of the hyper sphere. Alternately, $\St$ can be thought of as the geometry of the set of $n$ orthonormal $m$-frames. Moreover, if one works with $\St$ considering $m=n$, then one is said to be working with the geometry of the \textit{special orthogonal group}, defined formally as follows:
 \begin{equation}
     \mathcal{O}(m)\coloneqq \{U \in \mathbb{R}^{m\times m} \mid \UT U = I_m = U \UT\}.
 \end{equation}
In other words, $\text{St}^m_m = \mathcal{O}(m)$ and similarly $\text{St}^n_n = \mathcal{O}(n)$. Due to the close connection between $\mathcal{O}(m)$ and $\St$, there exists an alternate approach to defining $\St$, in relation to $\mathcal{O}(m)$ via an equivalence relation. This involves using another matrix geometry structure known as the \textit{Grassmanian manifold} \cite{bendokat2024grassmann}, $\Gr$. Mathematically, $\St$ works as an intermediate ``jump'' or ``stepping stone'' in order to move between $\mathcal{O}(m)$ and $\Gr$. The equivalence relation structure is defined as follows,
\begin{equation}
  \Gr \coloneqq \mathcal{O}(m) / (\mathcal{O}(n)\times \mathcal{O}(m-n)),
\end{equation}
where the cross symbol should be understood as a direct product acting over groups. Further information on the intricate relationship between $\St$, $\Gr$, and $\mathcal{O}(m)$, and the reasons why it is preferable to work over $\St$ rather than $\Gr$ for the proposed methodology are clarified in Section \ref{app_sec:Understanding_Grassmann}. 

Given the manifold $\St$, consider now a point, $U \in \St\subseteq \mathbb{R}^{m\times n}$, and define the tangent space local at $U$ to be $\Sttan$. A common parameterization of this tangent (vector) space at a point $U$ is given by
\begin{equation} \label{eqn:tangent_space_background}
    \Sttan = \{\Delta \in \mathbb{R}^{m\times n} \mid \Delta^{\intercal} U + U^{\intercal}\Delta =0 \},
\end{equation}
thus implying that $U^{\intercal}\Delta \in \mathbb{R}^{n\times n}$ is skew-symmetric. Further, consider a smooth curve $\gamma: [0,1]\rightarrow \St$ such that $\gamma(0)=U$. Then, $\gamma$ is considered to be geodesic over $\St$ if $\nabla_{\dot{\gamma}} \dot{\gamma}=0$ (a condition known as auto-parallelism), where $\nabla$ is some affine connection over $\St$, as it can be shown that under this condition $\gamma$ is indeed locally length minimizing. Let $\Delta \in \Sttan$ be a tangent vector emanating from $U$ on $\St$, then there exists a geodesic $\gamma_{\Delta}: [0,1] \rightarrow \St$ such that $\gamma_{\Delta}(0)=U$, and $
\dot{\gamma}_{\Delta}(0)=\Delta$. The \textit{exponential map} corresponding to this condition is defined as the function, $\text{Exp}_U(\Delta)\coloneq\gamma_{\Delta}(1)$. The infimum radius around $U$ with which the calculation of $\text{Exp}_U(\Delta)$ exists in a diffeomorphism with the manifold $\St$ is known as the \textit{radius of injectivity}. Intuitively, it defines the largest possible distance one can travel along a geodesic starting at $U$ whilst remaining in a well-behaved, one-to-one correspondence between points on the manifold $\St$ and the tangent space generated relative to $\Sttan$. Finally, given $\Sttan$ one can define a \textit{canonical} inner product $ \langle \cdot, \cdot \rangle_w: \mathbb{R}^{m\times n} \times \mathbb{R}^{m\times n} \rightarrow \mathbb{R}$, with associated weight $w=\left(I-\frac{1}{2} U U^{T}\right)\in\mathbb{R}^{m\times m}$. The reasoning behind opting for the canonical Stiefel metric and not the one respect to the ambient space of $\mathbb{R}^{m\times n}$, is shown in Section \ref{app_sec:Riemannian_Intro}. The exact form of the canonical metric is given as follows:
\begin{align}
    \langle\Delta, \tilde{\Delta}\rangle_{U}=\text{tr}
\left(\Delta^{T}\left(I-\frac{1}{2} U U^{T}\right)
\tilde{\Delta}\right),
\end{align}
where evidently the weighting is performed with respect to $U$. Given this structure one now has a way in which tangent vectors, can be worked with in a normalized fashion as follows:
\begin{align}
    \bar{\Delta } = \frac{\Delta}{\sqrt{\langle \Delta, \Delta \rangle_w}} = \frac{\Delta}{\|\Delta\|_w}.
\end{align}
The reason for considering this in the paper is so that one can scale randomly sampled $\Delta$ vectors with respect to the \textit{radius of injectivity}, so that for example, if one desires a $\Delta$ within 0.4 times relative to the radius of injectivity, one can first normalize the sampled $\Delta$ then multiply this by 0.4. Interestingly (and rather luckily) for $\St$ the radius of injectivity is known globally to be $0.89 \pi$ (to within a tight lower bound) \cite{rentmeesters2013algorithms}.

\subsection{The StiefelGen Algorithm}

Consider a uni-variate time series, $\mathcal{T} = (t_i)_{i=1}^N$. Our proposal, the S(tiefel)Gen(ereation) algorithm, requires constructing what is known as the \textit{page matrix} \cite{damen1982approximate} relative to $\mathcal{T}$. This involves choosing a particular window size, $m$, and reshaping it as:
\begin{align} \label{eqn:page_matrix}
\mathcal{T}_{\text{mat}} = 
    \begin{bmatrix}
    t_1 & \hdots & t_m\\
    t_{m+1} &    & t_{2m} \\
    \vdots  & \ddots & \vdots \\
    t_{m(n-1)+1} & \hdots & t_{N}
    \end{bmatrix},
\end{align}
where we must naturally require that $m \mid N$ ($m$ divides $N$). In the event that $m \nmid N$, the resolution entails either (i) implementing a rounding operation, (ii) applying padding, or (iii) introducing a slight signal overlap.

The origin of page matrices can be traced back to their utilization as a  method to approximate Hankel matrices \cite{damen1982approximate}. This technique has garnered historical significance within the domain of singular spectrum analysis (SSA). In particular, SSA can be construed as a form of spectrum-based principal components analysis (PCA) tailored for time series \cite{schoellhamer2001singular}. Thus, the intertwining narratives of page matrix construction and time series analysis delineate a rich history. More recently, page matrices have played a pivotal role in Koopman-based data-driven simulation and control \cite{lian2021koopman}. And contemporaneously, they have been instrumental in a body of work which shares some motivational similarity with this paper: to conduct a model-agnostic analysis of time series signals using page matrices \cite{agarwal2018model}. However, it should be noted that while the authors in \cite{agarwal2018model} concentrate on time series forecasting and imputation within a specific subclass of time series signals, our trajectory diverges from this point. Specifically, we focus on the innovative utilization of page matrices to generate synthetic data through the matrix geometry.

A natural question then, is what constitutes a good choice of $m$ (and by consequence, $n$). Evidently, these act as hyper parameters. However, an argument will be made later on that in practice, the exact choice of these do not matter greatly. In fact, one's choice of $m$ and $n$ will tend to dictate whether or not one is looking for synthetic signals which have more aleatoric or epistemic sources of uncertainty. Moreover, notice that if one is working with a set of multivariate signals, then there is no need to even select an $m$ or $n$ as this can be pre-determined by the stacking of the multivariate signals (meaning no page matrix is required to be constructed). The one caveat is that for an interpretable augmentation to occur, the multivariate signals should ideally be produced via the same underlying generating source. For example, one shouldn't stack random, unrelated signals, which are generated from a different set of physics.  Examples of valid signal stacking will be shown in this paper, as they pertain to structural health monitoring, and the Japanese vowel data set. Lastly, it should be noted that another manner in which page matrix construction can be avoided is simply to perform  \textit{StiefelGen} in the degenerate base case of $\text{St}_1^m\in\mathbb{R}^m$. This is a perfectly valid Stiefel manifold which has collapsed to hyper sphere geometry. However this would necessitate ``performing the SVD for a 1D'' matrix which is an unusual corner case. Due to this, we call this by a new name, \textit{SphereGen}, and explore its construction and intuitions in Appendix \ref{app_sec:quirks_lims}. Although \textit{SphereGen} has some limitations compared to its generalization, \textit{StiefelGen}, its main benefit comes in its linear $O(m)$ time complexity (as opposed to the $O(mn^2)$ complexity of \textit{StiefelGen} due to matrix exponentiation. Or one can take it as $O(mn^2)$ complexity, and consider $n=1$, to arrive at $O(m)$ anyway).

Upon construction of the page matrix, \textit{StiefelGen} requires performing a singular value decomposition (SVD) over the input matrix, $\mathcal{T}_{\text{mat},1}$ in the way of: $\mathcal{T}_{\text{mat},1} = U_1 \Sigma V^{\intercal}_1$. In particular, notice that in this factorization, geometric constraints are placed upon $U_1$ and $V_1$, as they are unitary by definition, implying that $U_1 U_1^{\intercal} = I_m$, and $V_1 V_1^{T} = I_n$. Now, given that both $U_1, V_1 \in \St$\footnote{More correctly $U_1\in \text{St}_m^m = \mathcal{O}(m)$, and $V_1\in \text{St}_n^n = \mathcal{O}(n)$. However we will use $\St$ in general for notational expedience.}, it is possible to perturb these matrices to a new set of matrices so that $U_1 \rightarrow U_2$, and similarly, $V_1 \rightarrow V_2$. If one were to na\"ively perform a linear perturbation in the way of, $U_2 = U_1 + \varepsilon$ instead, $U_2$ would naturally leave $\St$ and be pushed into the ambient $\mathbb{R}^{m\times n}$ space for large enough choice of $\varepsilon$, which by implication would mean that $U_2$ is no longer a rotation (and by extension, unitary) matrix. The key part to \textit{StiefelGen} is the fundamental awareness that since $U_1, V_1 \in \St$, then one can perturb these matrices without leaving the manifold since the geodesics of the manifold are well known \cite{absil2008optimization, edelman1998geometry}. Therefore one is able to obtain a valid $U_2$ and $V_2$ as unitary rotation matrices, which form a new page matrix $\mathcal{T}_{\text{mat},2} = U_2 \Sigma V_2^{\intercal}$, and then with a reshape operation, we can recover a newly generated time series $\mathcal{T}_{2}$. Notice that throughout this entire process the singular values of the time series remain unmodified. This is the key which remains to be invariant and which ensures that the original signal physics are always referenced. Only new components of the rotation matrices are being smoothly generated, but the amount of energy being contributed towards each singular vector remains unchanged. An intuition of the impact of this can also be viewed from the perspective of the dyadic expansion: 
\begin{align}
\mathcal{T}_{\text{mat}} = \Sigma_{i=1}^n \sigma_i u_i v_i^{\intercal},
\end{align}
where $\sigma_i = \Sigma_{ii}$\footnote{Be aware of the overloaded use of $\Sigma$ as both, a summation symbol and that of the SVD singular value matrix.}, and $u_i$ and $v_i$ are the $i$-th columns of matrices $U$ and $V$. According to this expansion by changing only the $u_i$ and $v_i$ vectors, but leaving the $\sigma_i$ as invariant, one is in smoothly perturbing the basis representations of the signal, but leaving the magnitude of the expansion coefficients unchanged. 

However, the issue of how best to perform this perturbation remains, since randomly following geodesics over a manifold requires solving a set of differential equations which can be expensive, and this will also scale with the distance travelled over the manifold. The proposed solution is to make use of the known formulation for the \textit{exponential map} of $\St$, which allows one to efficiently project a random vector $\Delta \in \Sttan$ directly onto $\St$ \cite{edelman1998geometry}. This is far more efficient than having to numerically integrate a set of differential equations over a high dimensional space, as instead the exponential map provides a direct computation of the end point of the geodesic pathway. The one caveat is that one must remain within the \textit{radius of injectivity} to ensure the exponential map calculation remains valid. Beyond the radius of injectivity, (i) The exponential map may no longer be injective, meaning that different tangent vectors can map to the same point on the manifold, (ii) The requirement for the exponential map to be a diffeomorphism onto its image is no longer satisfied outside the radius of injectivity, implying a loss of invertibility and smoothness, and (iii) Beyond the radius of injectivity, geodesics may encounter conjugate points \cite{spivak1978comprehensive, do2016differential}. However, for most practical use cases of \textit{StiefelGen}, one does not need to venture near the radius of injectivity for time series data augmentation. Luckily, the radius of injectivity over $\St$ is globally known to be 0.89$\pi$ to within a tight lower bound (Section \ref{sec:background_math}), and thus scaling to be within this radius remains globally simple. 
Given this, a simple four step process may be envisioned for generating new $U$ and $V$ matrices, as follows: (1) Considering $U_1$ (or $V_1$) as an origin point, randomly sample a matrix. This matrix does not need to lie on $\St$ at the moment. From here, (2) Project the matrix onto the tangent space generated by base point, $U_1$, $\Sttan$, which we may call $\Delta_1\in \Sttan$. (3) Scale $\Delta_1$ to be within the radius of injectivity, in that $\Delta_2 = \Delta_1 \ \|\Delta_1\| \times 0.89\pi \beta$, where $\beta\in [0,1]$. Naturally, if $\beta=0$ you will end up back at $U_1$, and if $\beta=1$ you will end up on the edge of the injectivity radius. (4) Now calculate the matrix exponential map, $\exp(\Delta_2)$, which will lead to a geometrically consistent $U_2$ to be calculated. This can be similarly repeated for $V_1$ to get a new $V_2$. Each of $U_1$ and $V_1$ will naturally lie on their own manifolds. These four steps are summarized in Figure \ref{fig:Stiefel_Projection}.

\begin{figure}[htbp]
  \centering
 {\includegraphics[width=1\textwidth]
  {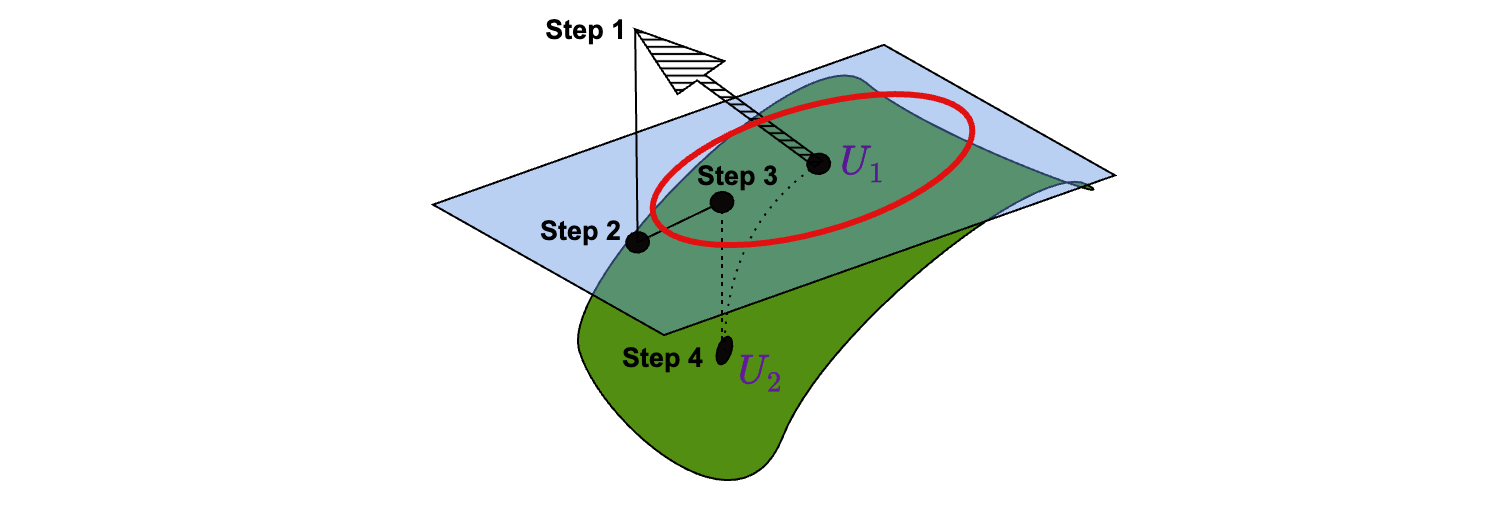}}
  \caption{Summary of the basic steps taken in the \textit{StiefelGen} algorithm.}
  \label{fig:Stiefel_Projection}
\end{figure}

This forms the essential \textit{core} of the \textit{StiefelGen} algorithm. Practically speaking, it is  recommended to usually apply a small smoothing filter to $\mathcal{T}_2$. This is because the vast majority of points over $\St$ tend to be unstructured and lack meaningful information (in other words, they appear as pure noise relative to the original signal), as the only technical requirement is that these matrix points be orthonormal, of which an infinitude of matrices are possible. Although meaningful points \textit{do exist} (as must be the case because our inputs $U_1$ and $V_1$ have desirable structural properties, relative to the unchanged $\Sigma$ in particular), the manifold $\St$ encompasses such a vast array of matrices that actual meaningful data points can become remarkably rare, especially as one moves further away from $U_1$ and $V_1$. This is the sole reason why an additional smoothing filter is often recommended. For small perturbations however, smoothing often does not need to be performed. The pseudocode for the overall algorithmic implementation is shown in Algorithm \ref{alg:one}. The overall time complexity of \textit{StiefelGen} is $O(mn^2)$\footnote{If $m < n$, then it will be $O(nm^2)$.} due to efficiency boosts in the calculation of the matrix exponential for skew-symmetric systems in the retraction stage (Step 3 $\rightarrow$ Step 4) \cite{zimmermann2022computing,edelman1998geometry}. Lastly, to emphasize its terseness and simplicity of implementation, the end-to-end Python code is shown in Appendix Section \ref{app:alg} to complement the pseudocode in Algorithm \ref{alg:two}.

\begin{center}
\begin{algorithm}[htbp]
\caption{S(tiefel)Gen(eration) MAIN}\label{alg:one}
\KwData{Input: $(\mathcal{T}_1, m\geq 2, \beta\in[0,1], \ell \geq 1)$}
\KwResult{Output: $(\mathcal{T}_2)$}
$n\gets \text{round}(\text{len}(\mathcal{T}_1)/m)$\;
$\mathcal{T}_{\text{mat}, 1} \gets \text{reshape}(\mathcal{T}_1$) \Comment*[r]{Reshape to $m\times n$ size.}
$U_1, \Sigma, V_1 \gets \text{SVD}(\mathcal{T}_{\text{mat}, 1})$\;
$U_2 \gets \text{StiefelGen}(U_1,\beta)$\;
$V_2 \gets \text{StiefelGen}(V_1,\beta)$\;
$\mathcal{T}_{\text{mat}, 2} \gets \text{reconstruct}(U_2,\Sigma,V_2)$ \Comment*[r]{Build matrix using SVD elements.}
$\mathcal{T}_2 \gets \text{reshape}(\mathcal{T}_{\text{mat}, 2})$\;
$\mathcal{T}_2 \gets \text{smooth}(\mathcal{T}_2, \ell)$ \Comment*[r]{Set $\ell=1$ if no smoothing required.}
\end{algorithm}
\end{center}

\begin{center}
\begin{algorithm}[htbp]
\caption{StiefelGen}\label{alg:two}
\KwData{Input: $(U_1, \beta\in[0,1])$}
\KwResult{Output: $(U_2)$}
$\Delta = \text{rand\textunderscore tan}(U_1)$ \Comment*[r]{Random tangent space vector centered at $U_1$}
$\tilde{\Delta} = (\Delta / \|\Delta\|_w) \times \beta \times 0.89 \pi$ \Comment*[r]{Scale wrt injectivity radius} 
$U_2 = \text{Exp}(\tilde{\Delta})$ \Comment*[r]{Exponential retraction} 
\end{algorithm}
\end{center}

We again draw attention to the fact that \textit{StiefelGen} requires \textit{minimal} hyper parameters usage, and zero training. The only technically required hyper parameter is $\beta$, which describes how far towards the injectivity radius one intends to move. For generated signals which require a strong ``outlier'' or ``novelty'' flavour, one should opt for larger values of $\beta$, and for a standard data augmentation procedure, one would opt for a smaller value of $\beta$. Further, the $m$ and $n$ reshaping is only required if the signal arrives in a uni-variate fashion, necessitating the construction of the page matrix (recall if a multivariate stack of time series signals is input, then $m$ and $n$ are already known, and no page matrix construction needs to occur). Also, the hyper parameter for smoothing, $\ell$, is completely optional, but recommended if larger perturbations are performed. Finally we point out that in Algorithm \ref{alg:one} we are \textit{technically} following the geometry of the underlying orthogonal groups, which are $\mathcal{O}(m)$ for $U$, and $\mathcal{O}(n)$ for $V$. This is because we are not taking the additional step of considering the first $d < \min(m,n)$ columns of either $U$ or $V$ in the algorithm. For most experiments in this paper we consider working with all the columns of the input $U$ and $V$ matrices, unless otherwise stated. Lastly, we clarify that the step one of Figure \ref{fig:Stiefel_Projection} is shown because the first line in Algorithm \ref{alg:two} \texttt{rand\textunderscore tan(U\textunderscore 1)} is implemented using functions written in the \textit{Geomstats} library \cite{geomstats}, wherein, inspired by \textit{PyManOpt} \cite{townsend2016pymanopt} they first generate a random normal matrix in ambient $\mathbb{R}^{m\times  n}$ space, then apply a vector projection onto $\Sttan$. It should be noted that in principle one can generate elements of $\Sttan$ directly, by the multiplication $\Delta  = UA$, where $U$ is the reference base point, and $A\in\mathbb{R}^{n\times n}$ is a random skew-symmetric matrix. This is because the necessary condition for elements to be in $Sttan$ is that $\Delta^{\intercal} U +  U^{\intercal}\Delta=0$ (Equation \ref{eqn:tangent_space_background}), and thus via substitution we see,
\begin{align*}
    (UA)^{\intercal}U + U^{\intercal}(UA) &= A^{\intercal} U^{\intercal} U + U^{\intercal}UA\\
    &= A^{\intercal} + A\\
    &=0,
\end{align*}
where $U^{\intercal}U=I_n$ due to the orthogonality condition, and $A^{\intercal}+A=0$ due to skew-symmetry. This approach is similarly taken by Zimmerman \cite{zimmermann2017matrix}, except he adds an additional level of randomness by projecting an additional matrix onto $\Sttan$ as follows:
\begin{align}
   \Delta = UA + (I_{m} - UU^{\intercal } )T,
\end{align}
where $T\in\mathbb{R}^{m\times n}$ is any random matrix, $UU^{\intercal }$ is the orthogonal projector matrix onto $\St$ (notice that $UU^{\intercal }\neq U^{\intercal }U$ except in the case that the manifold is $\mathcal{O}(m)$, in which case $UU^{\intercal }= U^{\intercal }U=I_{m}$), and thus $(I_{m} - UU^{\intercal } )$ represents the orthogonal projection matrix onto the orthogonal complement of the column space (or range) of $U\in \St$. The reason this operation lands onto $\Sttan$ once again comes back to the vital condition induced by orthogonality in Equation \ref{eqn:tangent_space_background}. Briefly,
\begin{align*}
    &(UA + (I_{m} - UU^{\intercal } )T)^{\intercal}U + U^{\intercal}(UA + (I_{m} - UU^{\intercal } )T)\\ =\ &(A^{\intercal} U^{\intercal} + T^{\intercal}(I_{m} - U^{\intercal }U)) U + U^{\intercal}(UA + (I_{m} - UU^{\intercal } )T)\\
    =\ &A^{\intercal} U^{\intercal}U + T^{\intercal}(I_{m} - I_{m})U + U^{\intercal}UA + (U^{\intercal} - U^{\intercal}UU^{\intercal })T\\
    =\ &A^{\intercal} + A + (U^{\intercal} - U^{\intercal})T\\
    =\ &0, 
\end{align*}
thus implying the validity of this form of projection on $\Sttan$, as per \cite{zimmermann2017matrix}. Moreover, in this framework the projector matrix, $UU^{\intercal}=\Gr$, can be considered to act as the ``projector viewpoint'' of a Grassmann manifold previously alluded to in \ref{sec:background_math}. Some further information about the Grassman manifold and its relation to the Stiefel manifold is provided in Section \ref{app_sec:Understanding_Grassmann}.

\section{Results and Analysis} \label{sec:ResultandAnalysis}

Throughout the following sections, we shall conduct a diverse set of experiments to exemplify the versatile capabilities of \textit{StiefelGen}. These applications span from perturbing time series signals along geodesics, showcasing a seamless transition from the conventional time series data augmentation problem to outlier signal generation, which has particular significance in synthetic structural health monitoring (SHM) scenarios. Additionally, we integrate this approach into a dynamic mode decomposition (DMD) problem, illustrating how it can generate multiple solution pathways in a model-agnostic manner and facilitate a form of functional uncertainty quantification (UQ) analysis. Finally, we demonstrate its efficacy in addressing its intended purpose: generating synthetic data to enhance the performance of an LSTM classification task.

\subsection{Data Sets}\label{sec:datasets}

\textbf{SteamGen: }The SteamGen dataset contains the output steam flow telemetry generated using fuzzy models to mimic a steam generator at the Abbott Power IL \cite{pellegrinetti1996nonlinear, Stumpy_zenodo}. The dataset features the output steam flow telemetry, which has units of kg/s, and is sampled every three seconds leading to a total of 9600 data points. It has been used for various time series analysis and data mining tasks, such as in the discovery of time series motifs \cite{law2019stumpy}. 

\textbf{New York Taxi: } The New York Taxi dataset represents a half-hourly average of the average number of passengers recorded using NYC taxi services, as recorded by the Government of NYC \cite{Stumpy_zenodo}. We shall focus in on a 75 day period over the Fall of 2014 which is consistent with other times series studies based on this dataset \cite{law2019stumpy, cheema2023use}. This results in a dataset with 3600 data points. This time period contains three notable anomalies in the form of (i) Columbus Day, (ii) Daylight Savings, and (iii) Thanksgiving. The exploration of this dataset is detailed in Appendix \ref{app:Taxi}, distinct from the main paper. Since its analysis parallels that of SteamGen, it solely serves to demonstrate the efficacy of the proposed algorithm on a different dataset.

\textbf{Synthetic SHM: }To showcase further practical applications of Stiefel geodesics in a domain of unused in conventional machine learning (academic) datasets, we use a synthetic Structural Health Monitoring (SHM) dataset. This dataset emulates the customary array stacking procedure (often in 2D or 3D), which is often employed in the multivariate (multi-sensor) SHM context \cite{cheema2016structural, anaissi2018tensor}. Our scenario considers a simple bridge structure with five sensors, each recording 50 sinusoidal responses embedded with white noise and a constant bias. The recording frequency is set at 50Hz, spanning a duration of 9 seconds. Additional details about this dataset are presented in Section \ref{sec:SHM}.

\textbf{Spatio-Temporal DMD: }The geometric properties of Stiefel manifolds find another application in predicting future states of time series data (under certain assumptions), establishing a significant connection with the dynamic mode decomposition (DMD) algorithm \cite{schmid2010dynamic}. Our attention will be directed towards a well studied synthetic dataset, which consists of a blend of two spatio-temporal signals outlined in Chapter One of the textbook entitled \textit{Dynamic Mode Decomposition: Data-Driven Modeling of Complex Systems} \cite{kutz2016dynamic}. Further insights into this synthetic dataset are expounded upon in Section \ref{sec:UQ_DMD}.

\textbf{Japanese Vowels: } 
The Japanese Vowels dataset comprises multivariate time series data capturing the speech patterns of nine male speakers articulating the vowels 'a' and 'e' as a diphthong. Available on the UCI Machine Learning Repository \cite{misc_japanese_vowels_128}, it consists of 640 time series, each featuring 12 LPC (Linear Prediction Cepstrum) coefficients derived from speech signals. These coefficients convey the spectral characteristics of vowel pronunciation \cite{kudo1999multidimensional}. Recorded at 10kHz, the dataset is multivariate and time-series in nature, with 270 training and 370 testing instances. Aimed at conventional time series data augmentation, we shall use it to enhance the overarching accuracy of an under-capacity LSTM model, trying to classify nine speakers.

\subsection{StiefelGen: An Overview} \label{sec:overview}

We shall delve into the impact of generating augmentation instances—marked by moderate perturbations over the Stiefel manifold—as well as outlier instances—characterized by substantial perturbations over the same manifold. We aim to offer insight into the distinct nature of these perturbations, particularly as they apply separately to matrices $U$ and $V$, showcasing the diverse sets of information carried by these matrices. Our examination in this subsection shall center on the SteamGen dataset, with a concise analysis similarly applied to the Taxi dataset available in the Section \ref{app:Taxi}.

\textbf{Moderate Perturbation}

In our initial examination, we apply a ``moderate perturbation factor'' ($\beta=0.4$) as well as a smoothing factor ($\ell=5$) to the first 2000 points of the SteamGen dataset. After reshaping the time series signal into a 50 $\times$ 40 page matrix, and applying a \textit{StiefelGen} perturbation, we arrive at Figure \ref{fig:intro_small}. Notably, Subfigure \ref{fig:intro_small}a presents a global view where the newly generated time series seamlessly integrates with the original signal by a process of appending. To characterise the nature of this signal change a little better, a plot of the residuals in Subfigure \ref{fig:intro_small}d unveils novel intricate noise patterns, deviating from standard probabilistic models. In particular Subfigure \ref{fig:intro_small}d is a histogram plot of the $\Delta$ values generated in Subfigure \ref{fig:intro_small}c by subtracting the newly generated signal from the original signal. We also note that from observing Subfigure \ref{fig:intro_small}b, the newly generated time series has exhibited both shifts with respect to an aleatoric sense (noise), as well as an epistemic sense (functional). The reasons for this will be explored later in this subsection.

\begin{figure}[htbp]
  \centering
  \subfigure[A global view of the signal with the generated signal appended to the end.]{\includegraphics[width=0.45\textwidth]{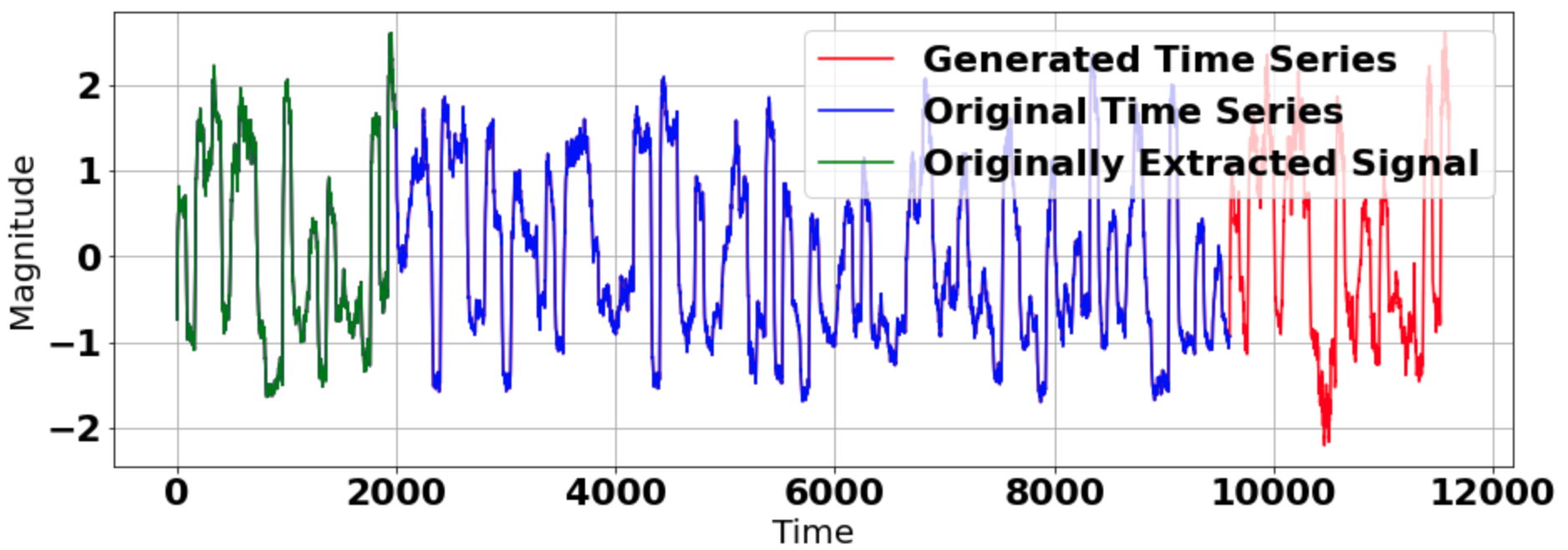}}
  \hspace{0.01\textwidth} 
  \subfigure[An over plot of the original signal with the newly generated signal.]{\includegraphics[width=0.45\textwidth]{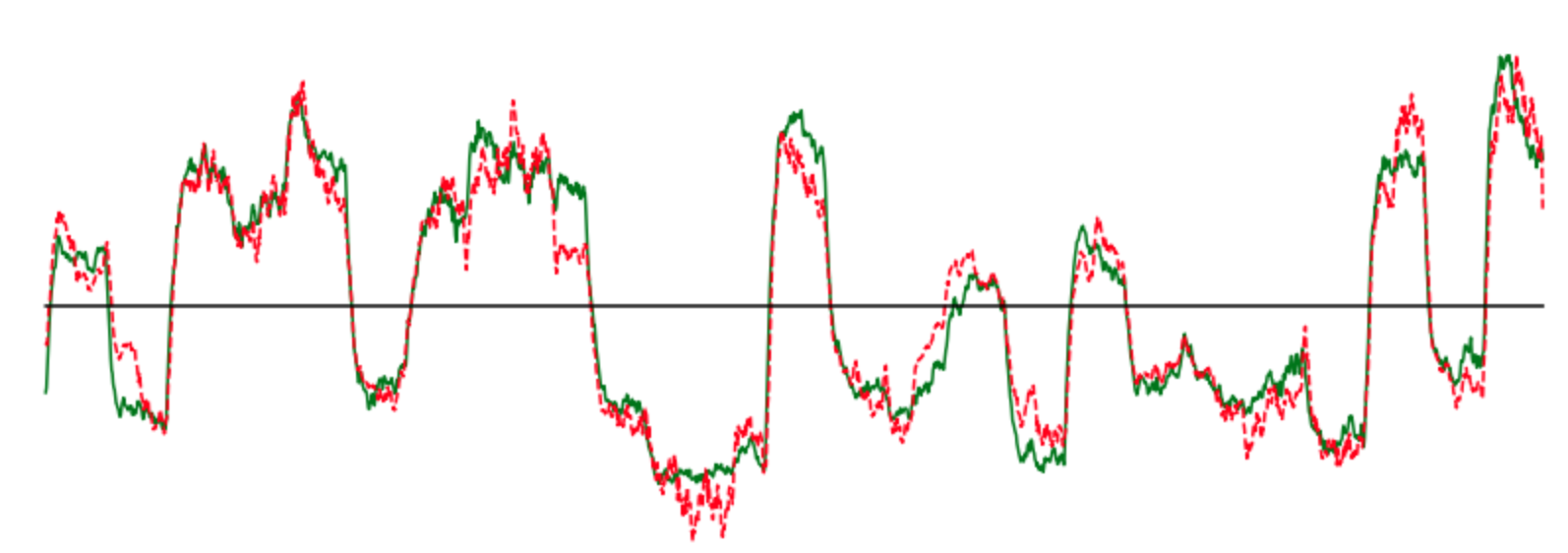}}

  \subfigure[A residual plot of the original and generated signals.]{\includegraphics[width=0.45\textwidth]{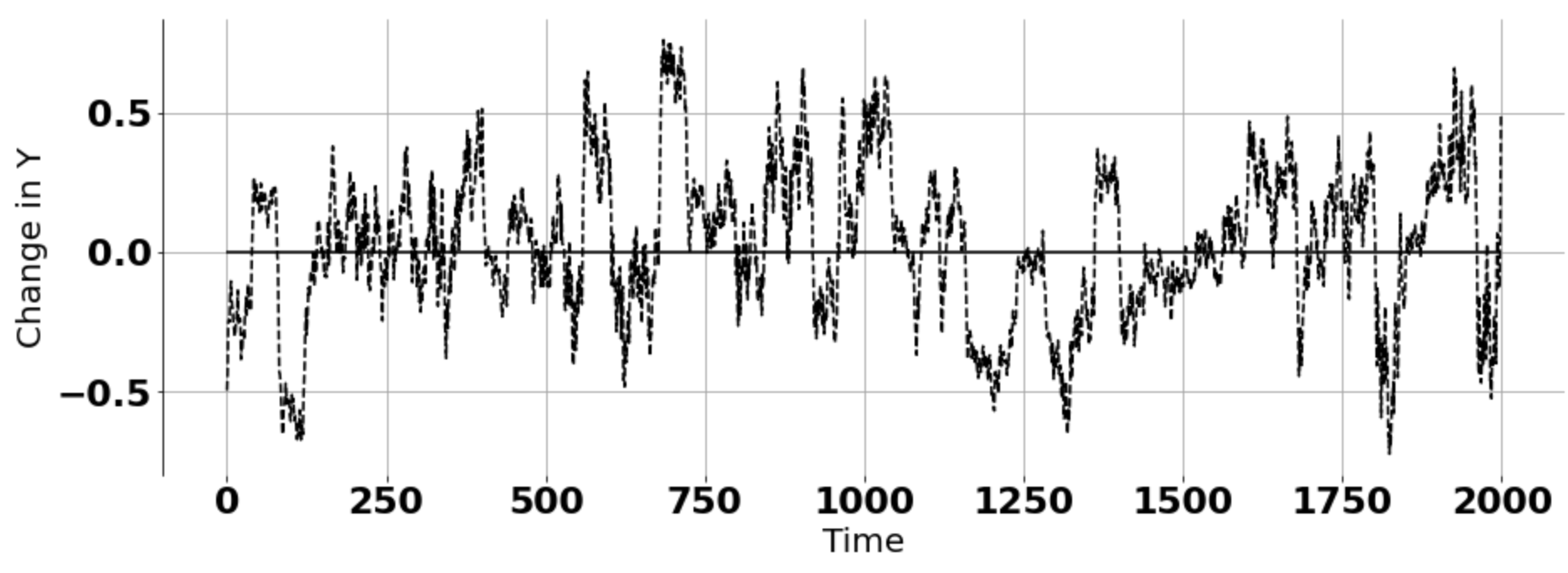}}
  \hspace{0.01\textwidth} 
  \subfigure[A histogram over the residuals.]{\includegraphics[width=0.45\textwidth]{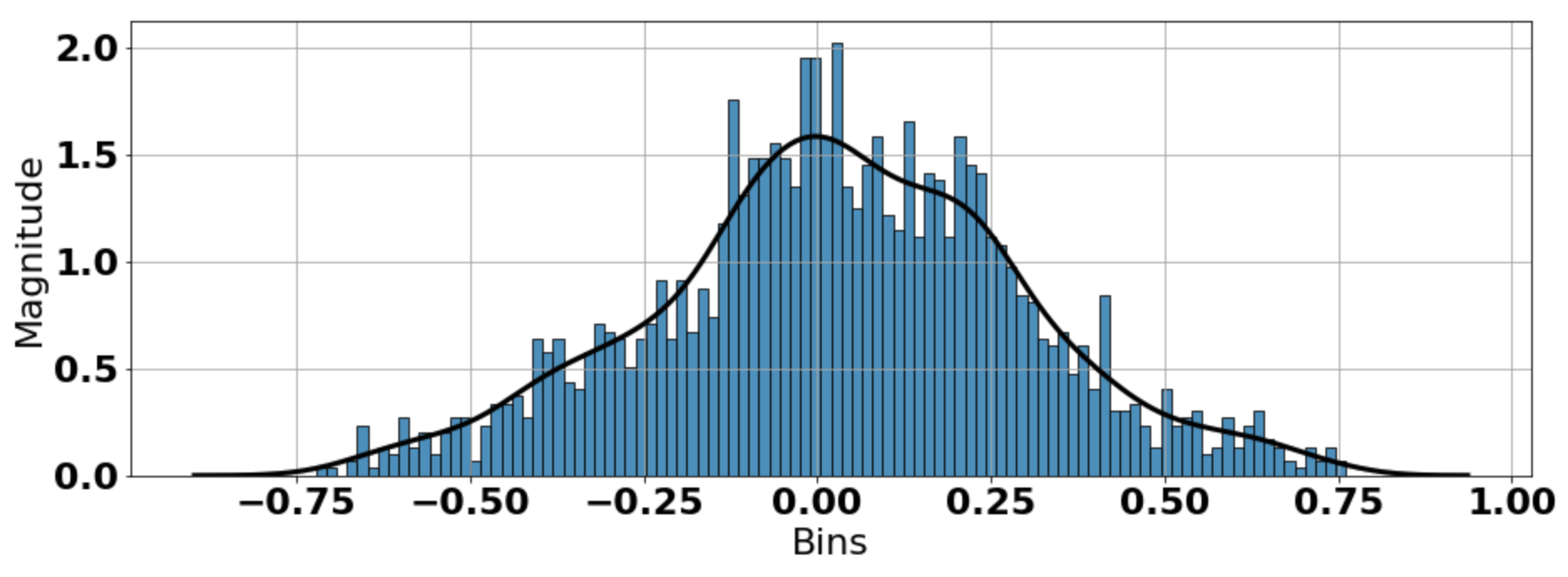}}
  \caption{An overview of the effect of applying a moderate perturbation factor to generate a new signal based on the first 2000 points of input for the SteamGen data set.}
  \label{fig:intro_small}
\end{figure}

\textbf{Large Perturbation}

Now, we turn our focus to the impact of a larger perturbation factor ($\beta=0.9$) on the underlying signal. In this scenario, the aim would be to generat seemingly anomalous looking signals which are not representative of catastrophic failure or triviality (e.g. if the signal degrades simply to white noise). This contrasts the previous subsection on ``moderate perturabation'' where the aim there would be to ``augment'' the original signal. A slightly larger smoothing factor ($\ell=9$) was employed to mitigate the impact of unwanted noise due to the larger perturbation factor. Subfigure \ref{fig:intro_large}a immediately reveals a stark contrast in the generated signal when viewed globally—both in magnitude and noise—compared to the preceding portions of the signal. In particular, Subfigure \ref{fig:intro_large}b showcases a pronounced basis change with significant alterations in noise. Such a signal could indeed serve as valuable input for testing the efficacy of outlier detection systems. Examining the residuals in Subfigure \ref{fig:intro_large}d, we once again notice a large deviation from standard probabilistic models. As anticipated, for the larger $\beta$ value, a much broader dispersion of residual values with more pronounced tail behavior is evident, implying that the generated signal much further away than the base reference signal. 

\begin{figure}[htbp]
  \centering
  \subfigure[A global view of the signal with the generated signal appended to the end.]{\includegraphics[width=0.45\textwidth]{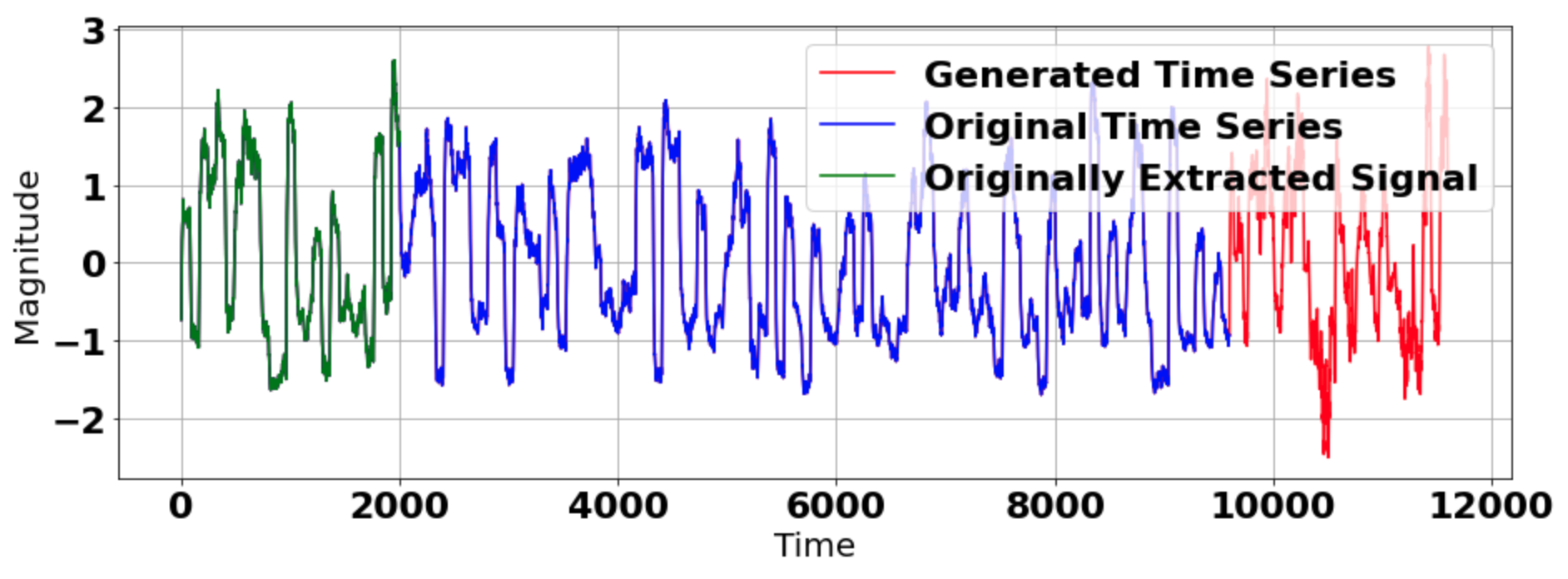}}
    \hspace{0.01\textwidth}  
  \subfigure[An over plot of the original signal with the newly generated signal.]{\includegraphics[width=0.45\textwidth]{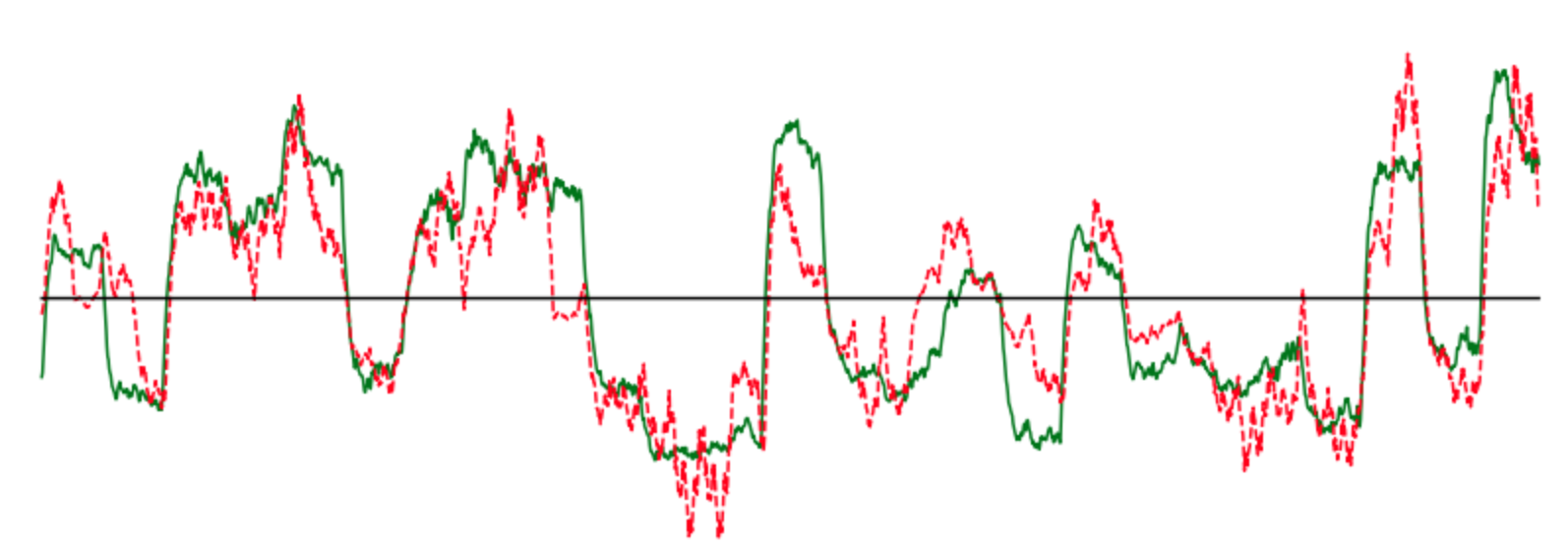}}

    \subfigure[A residual plot of the original and generated signals. ]{\includegraphics[width=0.45\textwidth]{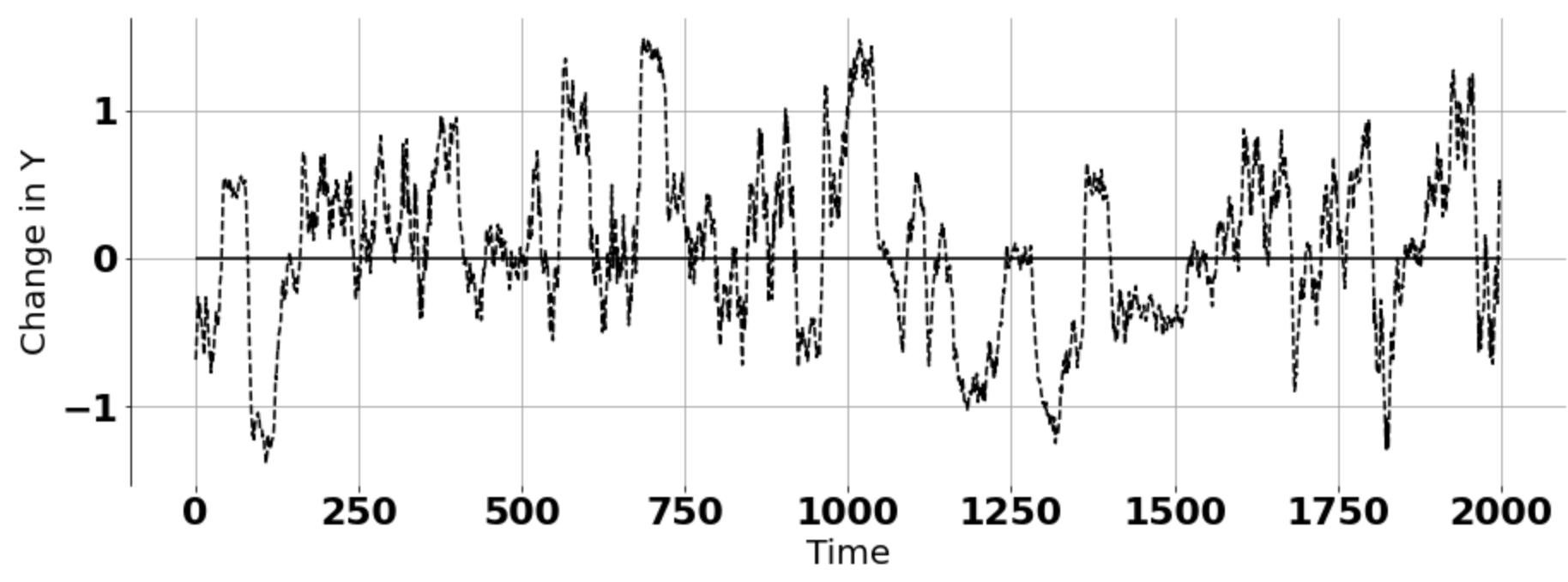}}
     \hspace{0.01\textwidth} 
      \subfigure[A histogram over the residuals.]{\includegraphics[width=0.45\textwidth]{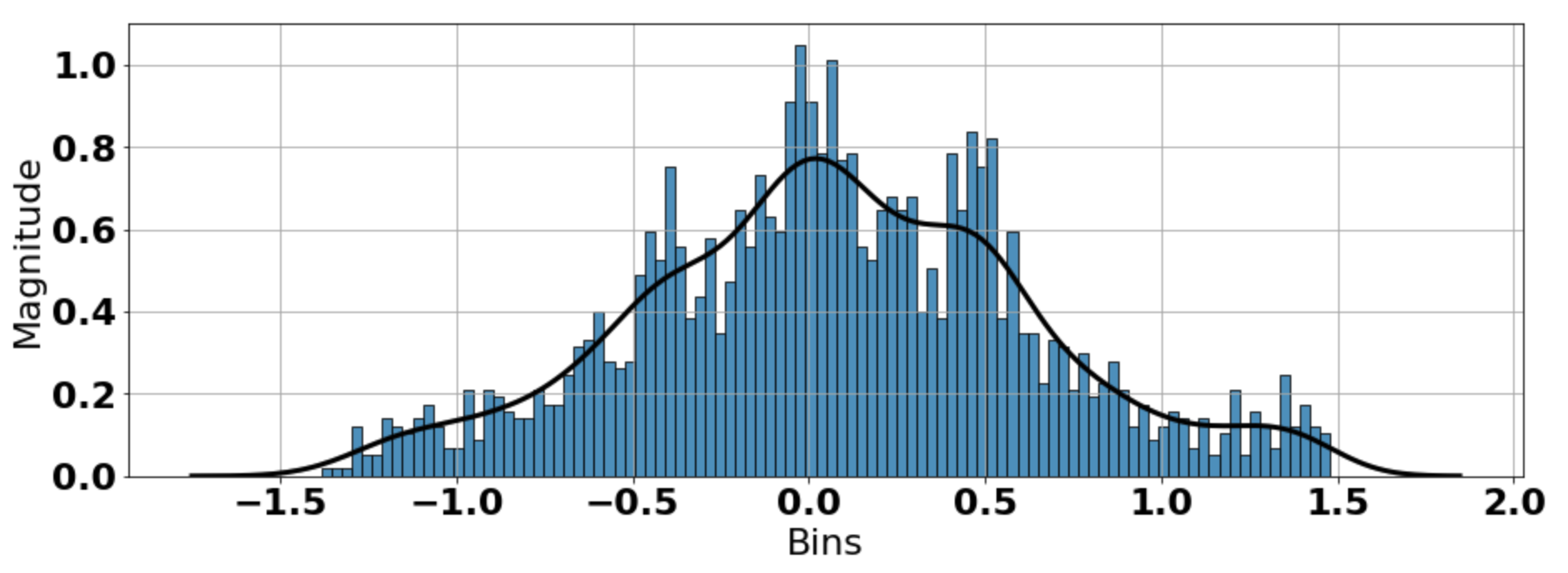}}
  \caption{An overview of the effect of applying a large perturbation factor to generate a new signal based on the first 2000 points of input for the SteamGen data set.}
  \label{fig:intro_large}
\end{figure}

A similar yet brief analysis on a different time series data set (the NYC Taxi data set) is shown in Appendix \ref{app:Taxi}. 

\textbf{Effect of Page Matrix Dimensions}

In the preceding investigations it was assumed that $m=50$ and $n=40$, given the input signal's dimensions of $m\times n = 2000$ units long. Now, maintaining the same $\beta$ and $\ell$ hyper parameters, we shall opt for more extreme values of $m$ and $n$ to highlight the impact of working with tall-skinny page matrices or short-fat page matrices. The purpose of this will be to demonstrate how the reshaped dimensions of the page matrix impact the generated time signal. Our emphasis here will be on the moderate setting; a $\beta$ factor of 0.4 (and thus $\ell=5$), thereby forming a direct comparison with Subfigure \ref{fig:intro_small}b. The comparative results of these scenarios are illustrated in Figure \ref{fig:mn_comparison}.

\begin{figure}[htbp]
  \centering
  \subfigure[The time series generated if a tall-skinny ($m=100,n=20$) page matrix were constructed and perturbed.]{\includegraphics[width=0.45\textwidth]{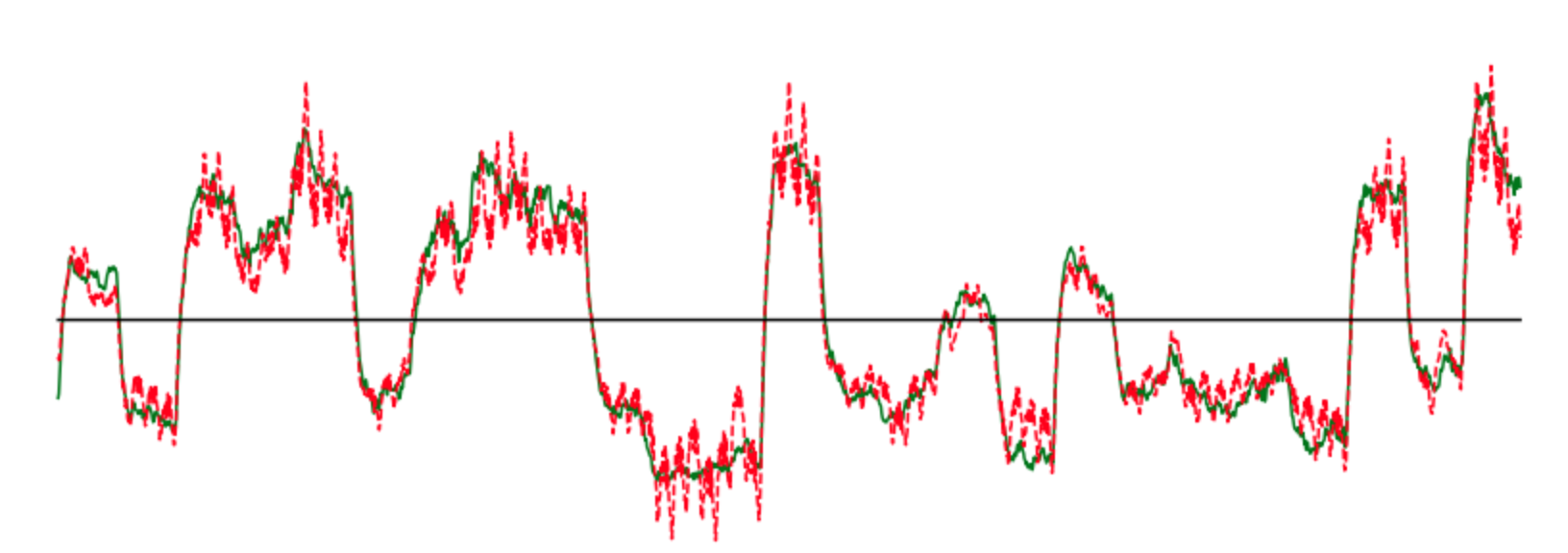}}
  \hspace{0.01\textwidth} 
  \subfigure[The time series generated if a short-fat ($m=20,n=100$) page matrix were constructed and perturbed.]{\includegraphics[width=0.45\textwidth]{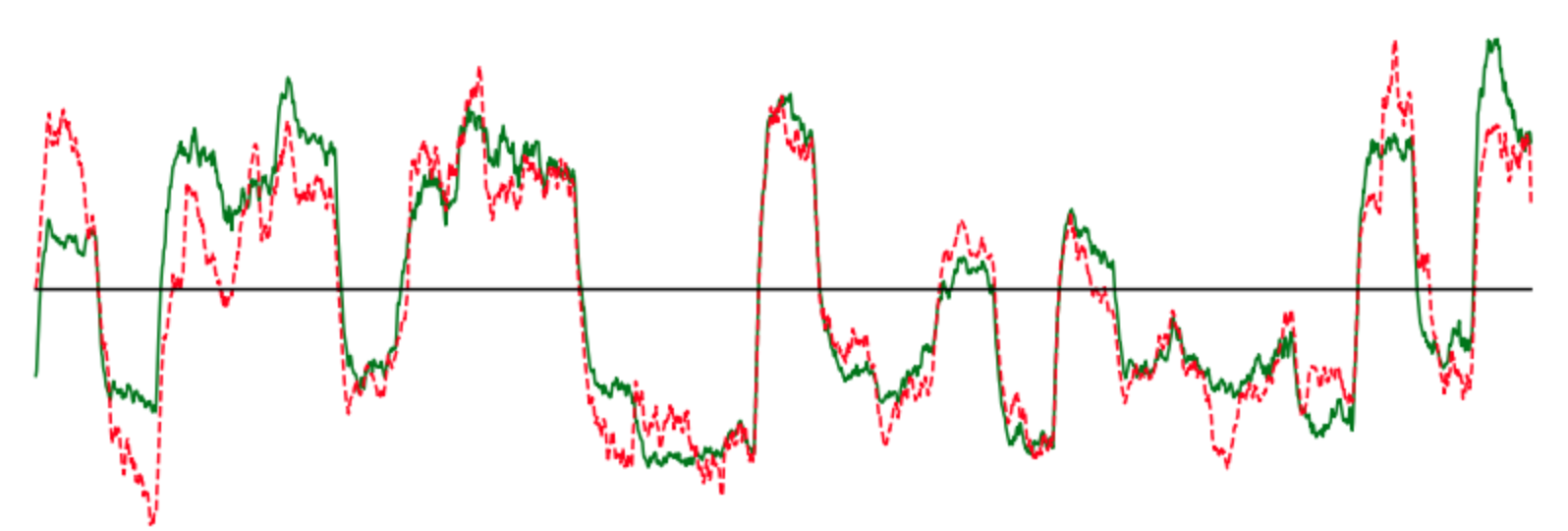}}
  \caption{An overview of the effect of applying a moderate perturbation factor on the SteamGen data set for the two opposite cases of a tall-skinny, and short-fat shape respectively.}
  \label{fig:mn_comparison}
\end{figure}

As can be seen, in contrast to the nearly square page matrix reshape scenario ($m=50, n=40$), a more extreme rearrangement yields noticeably distinct generated signals. Notably, one type of reshaping seems inclined to generate a new time series predominantly driven by noise (Subfigure \ref{fig:mn_comparison}a). On the other hand, the alternative reshaping structure introduces minimal additional noise but instead exhibits a substantial \textit{basis change} relative to the input time series signal (Subfigure \ref{fig:mn_comparison}b). A clearer understanding of this phenomenon may arise when we re-visit the prior, more square ($m=50$, $n=40$) page matrix case. Particularly, visualizing the $U$ and $V$ matrices in Figure \ref{fig:UV_mat} reveals that in fact, the $U$ matrix predominantly contains a more \textit{noisy} structure, whereas the $V$ matrix (especially towards the top portion of the matrix), appears to exhibit much more structured patterns. 

\begin{figure}[htbp]
  \centering
  \subfigure[A visualization of the $U$ matrix.]{\includegraphics[width=0.4\textwidth]{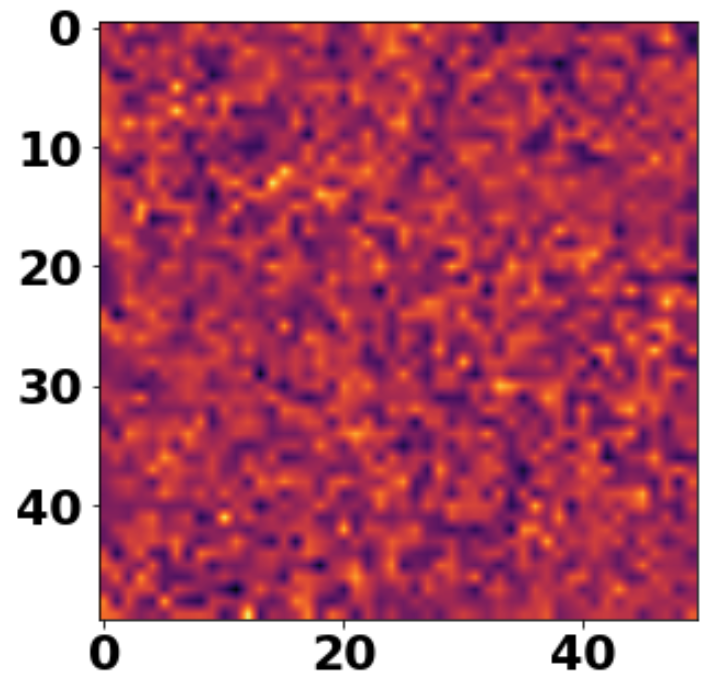}}
  \hspace{0.1\textwidth} 
  \subfigure[A visualization of the $V$ matrix.]{\includegraphics[width=0.4\textwidth]{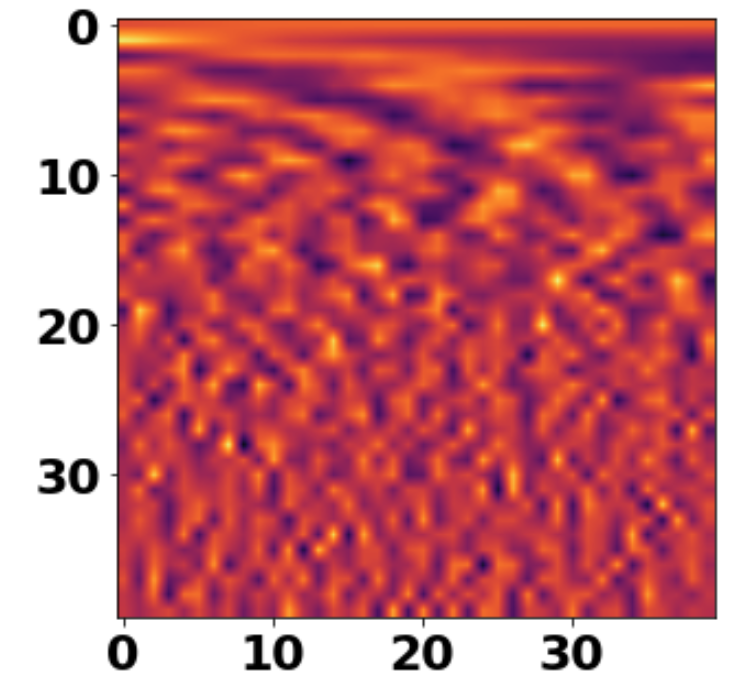}}
  \caption{A comparison between the $U$ and $V$ matrices as a result of an SVD of the $\mathcal{T}_{\text{mat}}$ Page matrix of the SteamGen input when $m=50$ and $n=40$.}
  \label{fig:UV_mat}
\end{figure}

This insight is further reinforced by Figure \ref{fig:UV_basis}, where indeed an empirical basis-like structure is evident in Figure \ref{fig:UV_basis}b, when considering the $V$ matrix. Consequently, the choice of reshaping with respect to the rows ($m$) and columns ($n$) one chooses when constructing the page matrix ultimately serves to emphasize whether the end goal is to generate a time series signal biased towards more noise (aleatoric) or towards basis representation changes (epistemic).

\begin{figure}[htbp]
  \centering
  \subfigure[A plot of the first four rows of the $U$ matrix.]{\includegraphics[width=0.49\textwidth]{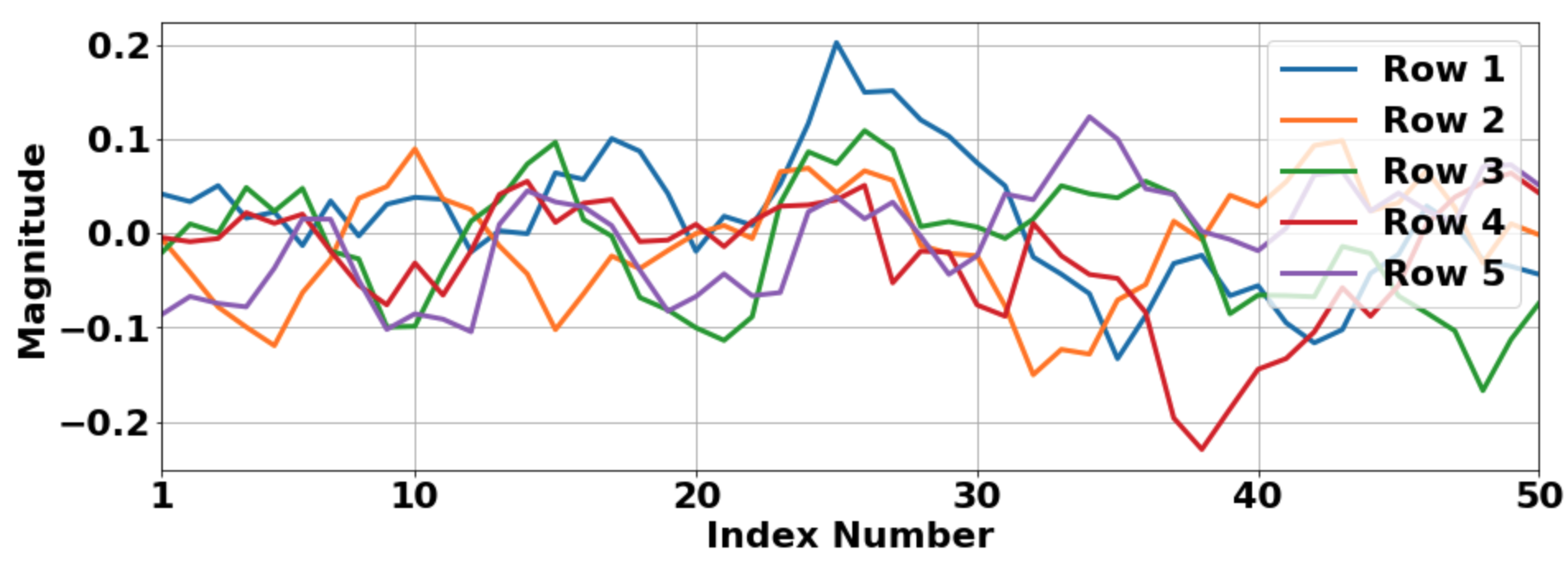}}
  \hfill  
  \subfigure[A plot of the first four rows of the $V$ matrix.]{\includegraphics[width=0.49\textwidth]{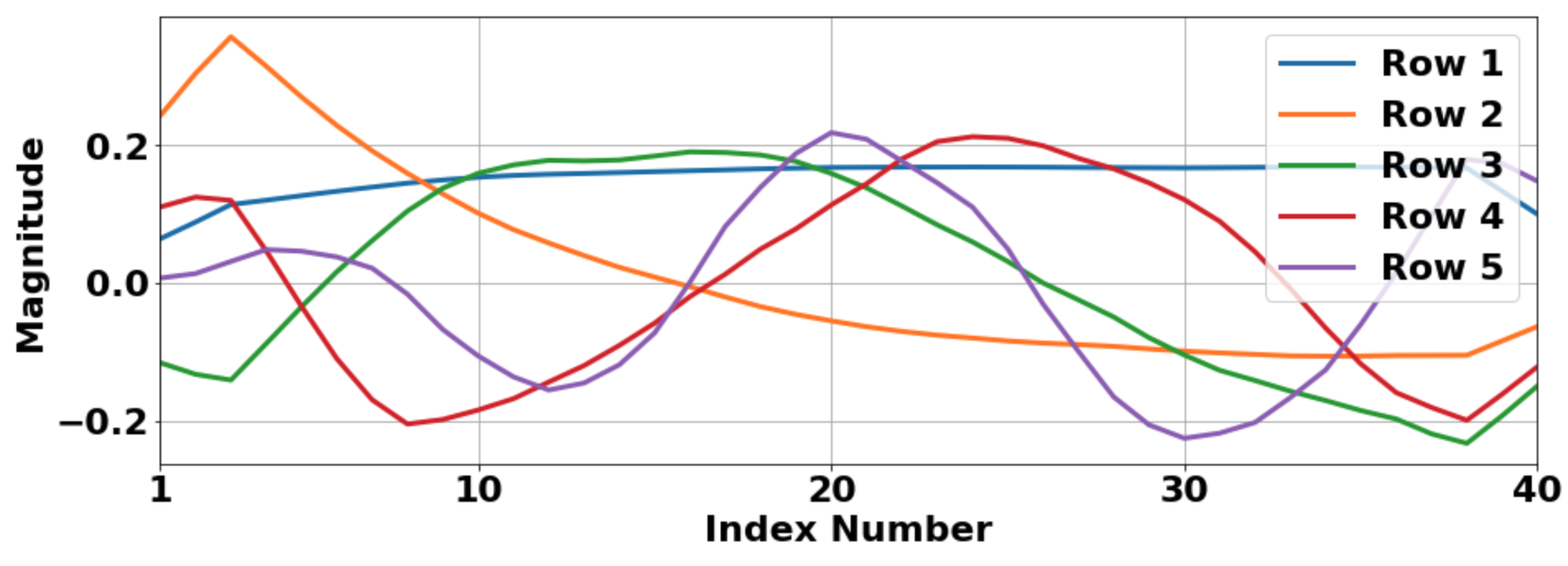}}
  \caption{A plot of the first four rows of the $U$ and $V$ matrices of the SVD of the SteamGen data, with smoothing factor set to $\ell=3$ to clarify the presence of the basis information.}
  \label{fig:UV_basis}
\end{figure}

The distinct treatment of noise and basis information arises from the construction of the page matrix. In Equation \ref{eqn:page_matrix}, $\mathcal{T}_{\text{mat}}$ is formed by \textit{stacking} portions of the uni-variate signal vertically. Reading $\mathcal{T}_{\text{mat}}$ from left to right reveals portions of the entire signal, while inspecting it top-to-bottom emphasizes noise variation as the signal compares itself every 50 data points ahead. This concept is illustrated in Figure \ref{fig:page_matrix_demo}.

\begin{figure}[htbp]
  \centering
{\includegraphics[width=0.9\textwidth]{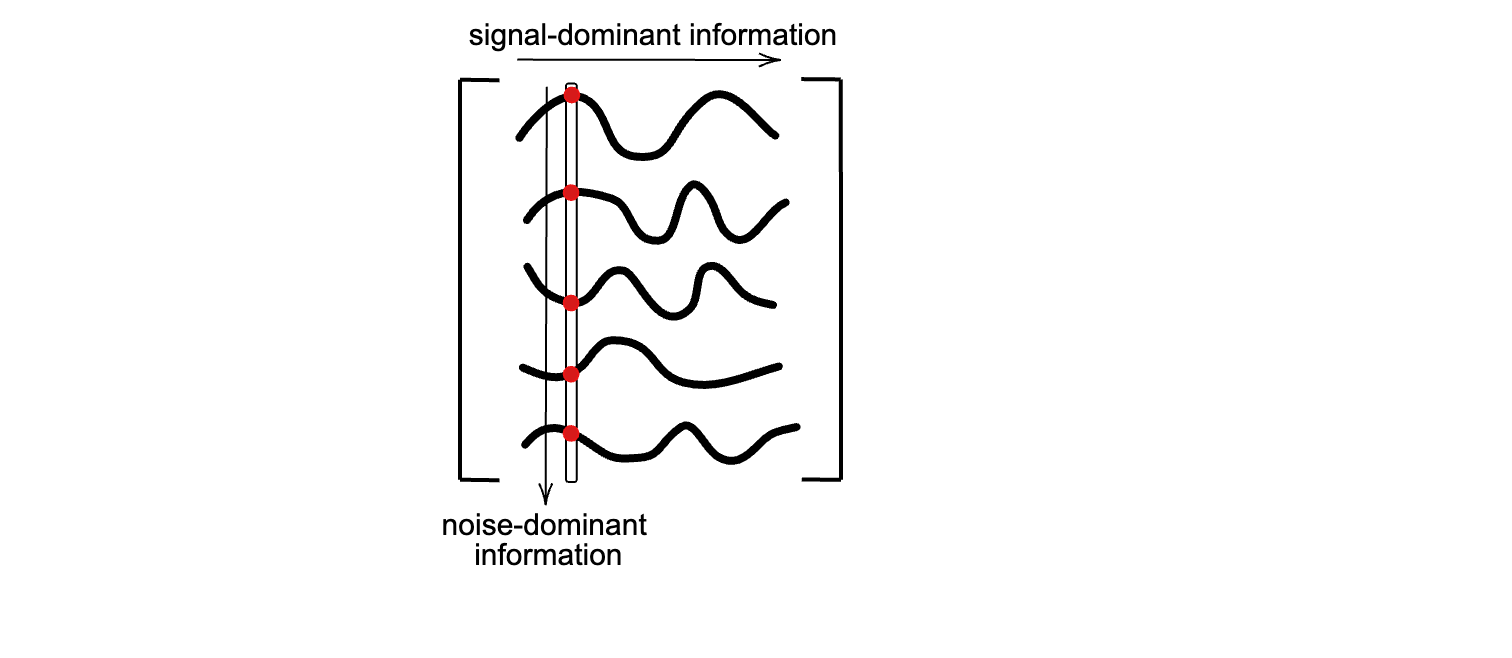}}
  \caption{A demonstration of the stacking scheme used in the Page matrix construction of $\mathcal{T}_{\text{mat}}$.}
  \label{fig:page_matrix_demo}
\end{figure}

Further, noting that a specific property of the SVD (and linear algebra in general) is that, 
\begin{align*}
    \mathcal{T}_{\text{mat},i}v_i &= \sigma_i u_i, \quad \forall i=1,2,\ldots,n, \\
\mathcal{T}_{\text{mat},i}^{\intercal}u_i &= \sigma_i v_i, \quad \forall i=1,2,\ldots,m,
\end{align*}
we see that each $u_i$ acts on $C(\mathcal{T}_{\text{mat}})$, the column space of $\mathcal{T}_{\text{mat}}$, and each $v_i$ acts on $R(\mathcal{T}_{\text{mat}})$, the row space of $\mathcal{T}_{\text{mat}}$. And since signal dominant information is being stored row-wise, it is evident that this particular stacking scheme means short-fat stacking scheme ($V$-dominant) hyper-emphasises those perturbations which lead to a change in basis, whereas the tall-skinny stacking scheme ($U$-dominant) leads to perturbation-drive changes in signal information. 

Returning to an initial premise, while $m$ and $n$ do indeed serve as hyper parameters, their significance lies in emphasizing the type of signal one desires to generate — whether more noise-oriented or more basis-oriented. Combined now, with the choice of the scaling factor $\beta$ (smaller $\beta$ for increased data augmentation examples and larger $\beta$ for more novel signals), this framework offers a spectrum of highly customizable signal generation options tailored to the task at hand. Remarkably, this approach requires minimal assumptions and zero training. The simplicity and interpretability of the method are further underscored by the ease of inspecting the $U$ and $V$ matrices, allowing for a direct examination of the actual (empirical) basis functions used to reconstruct the new signal.

Finally on the topic basis functions, it should be clarified that the ones which are learned (shown in \ref{fig:UV_basis}b) \textit{are not} the same as the common wavelet or Fourier bases, although they provide an apt analogy to the situation (especially when viewed through the lens of the dyadic expansion). The basis functions found for the $U$ and $V$ matrices are purely data-driven and point along the directions of \textit{maximal variation} within the input data.

\subsection{Working over Geodesics}

In this subsection, we delve into the intriguing properties that emerge when working over $\St$, driven by the inherent \textit{smoothness} features of Riemann manifolds. Specifically, we can traverse a geodesic path from a point $U_1$ (or $V_1$) and smoothly adjust deformations until reaching $U_2$ (or $V_2$). This provides a nuanced means to precisely control the degree of signal perturbation. Expanding on this capability, we can seamlessly shift the problem from signal augmentation to novel signal generation, identifying the point where one transforms into the other. This property proves pivotal when exploring potential applications in the structural health monitoring (SHM) field.

For the creation of Figure \ref{fig:geo_full}, we revisit the scenario explored in Figure \ref{fig:intro_large}. However, this time, we dissect the geodesic steps into one-tenth increments, allowing for a gradual progression of signal deformation. The advantage of this approach with the \textit{StiefelGen} algorithm lies in the intuitive notion that \textit{there is no such thing as a bad signal}. In simpler terms, if a generated time series signal exhibits excessive outlier behavior, one can easily "deform the signal back" by a few steps to achieve the desired magnitude changes.

\begin{figure}[htbp]
  \centering
  \subfigure[The first 5 incremental geodesic steps.]{\includegraphics[width=0.48\textwidth]{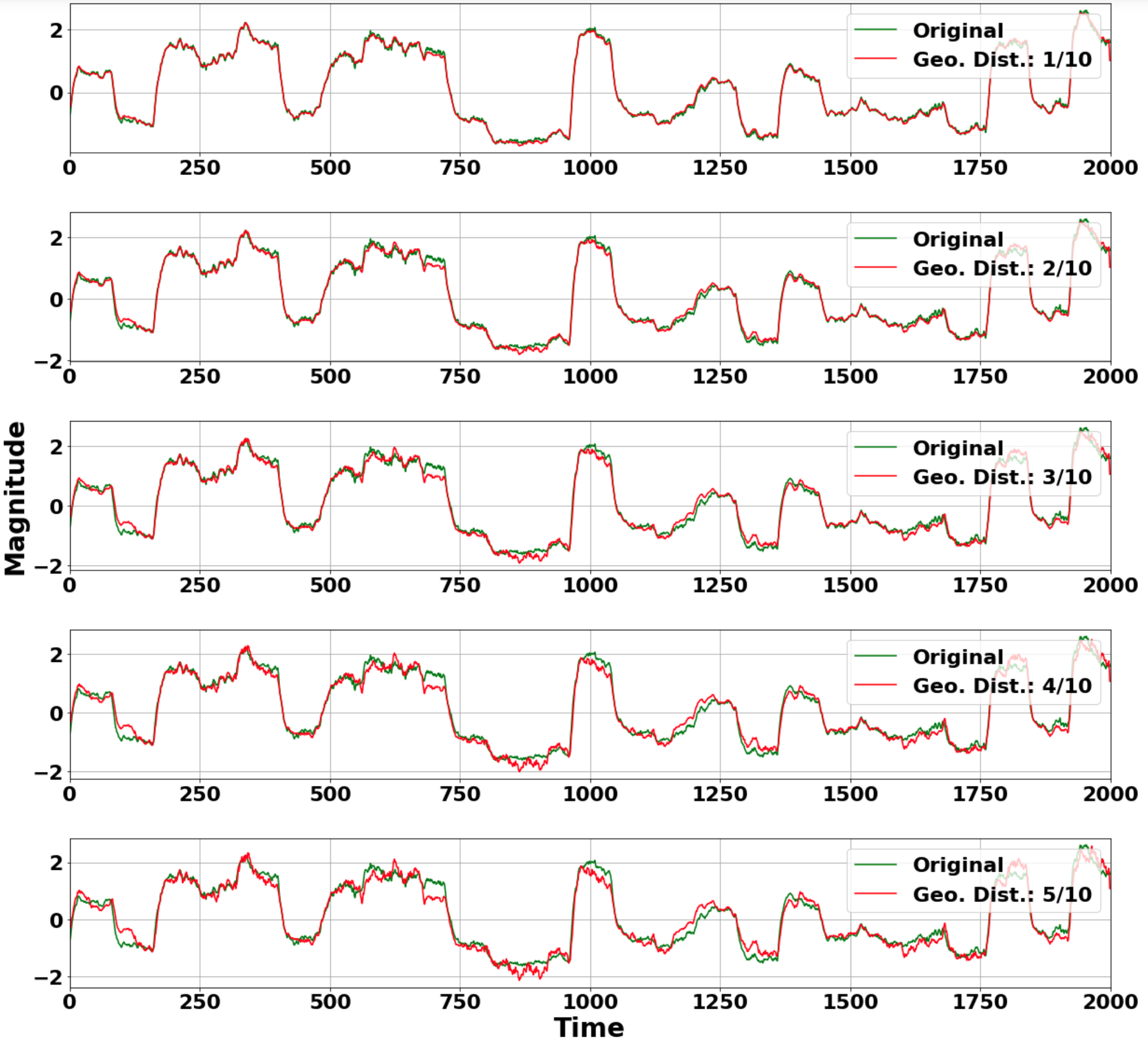}}
  \hspace{0.01\textwidth} 
  \subfigure[The last 5 incremental geodesic steps.]{\includegraphics[width=0.48\textwidth]{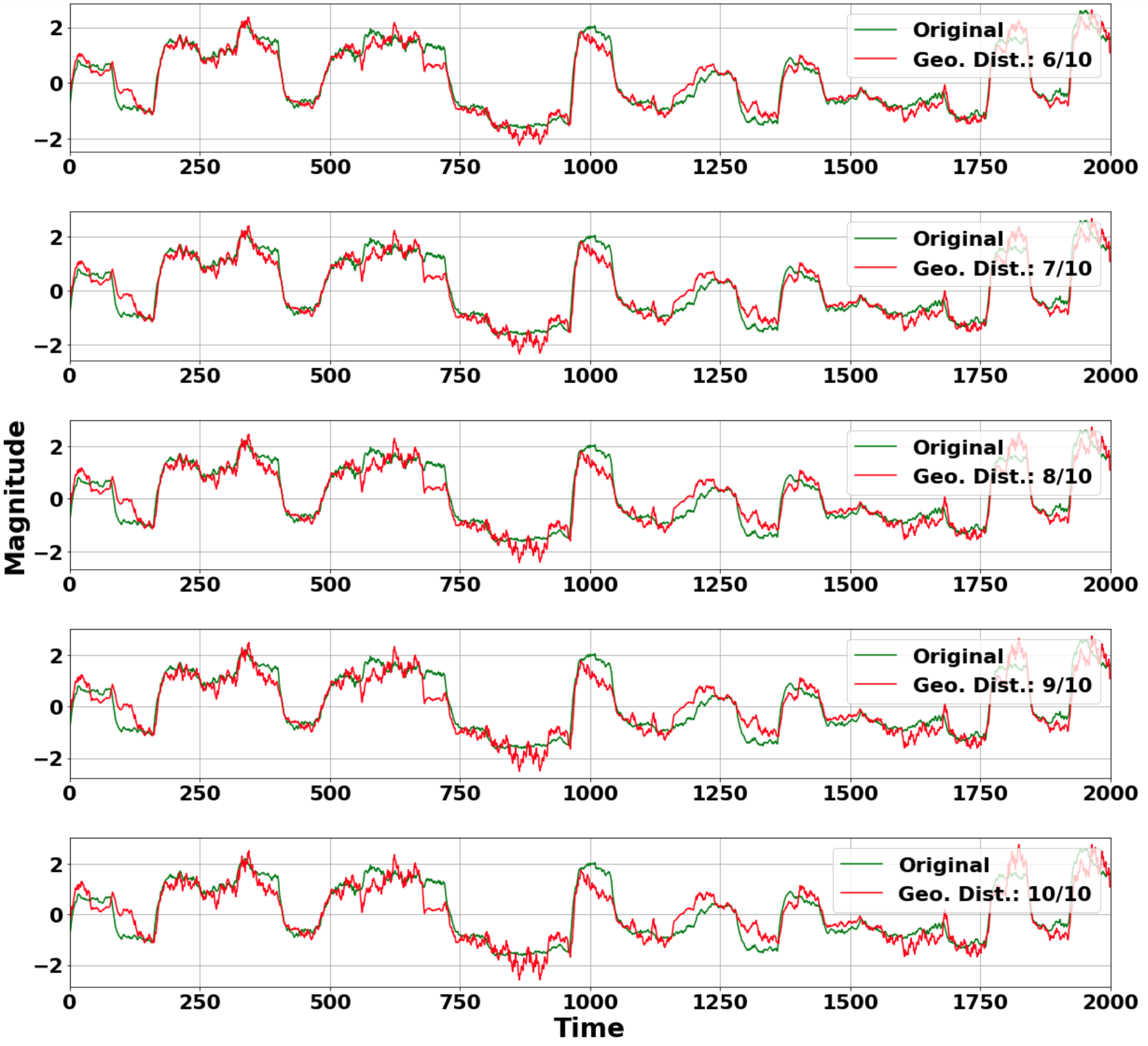}}
  \caption{Incrementally deforming the generated signal until it arrives at the final position investigated previously in Figure \ref{fig:intro_large}.}
  \label{fig:geo_full}
\end{figure}

Recall that the \textit{StiefelGen} signal generation relies upon separating noise and basis-driven deformations through the row and column separation. Given that both the $U$ and $V$ matrices exist on their respective Stiefel manifolds, perturbing either matrix independently becomes feasible. This flexibility allows the generation of time series signals that are either ``more noise-driven'' or ``more basis function-driven''. While a perfect separation along a "noise axis" orthogonal to a "basis axis" is not mathematically possible, practical outcomes tend to exhibit substantial separation, primarily resulting from the row-column space separation seen in the SVD. Consequently, one can choose to emphasize either noise deformations or basis function deformations depending on the chosen approach to stack the page matrix (as illustrated in Figure \ref{fig:mn_comparison}). Thus, instead of uniformly perturbing both $U$ and $V$, one can opt for different magnitudes of perturbation factors for $U$ and/or $V.$ This is demonstrated in Figures \ref{fig:subplot_example_basis} and \ref{fig:subplot_example_noise}, where geodesics are traversed in one-tenth increments, specifically with respect to either $U$ or $V$ matrices. This approach enables smooth basis function deformation or smooth noise deformation, depending on the chosen matrix.

\begin{figure}[htbp]
  \centering
  \subfigure[The first 5 incremental geodesic steps.]{\includegraphics[width=0.48\textwidth]{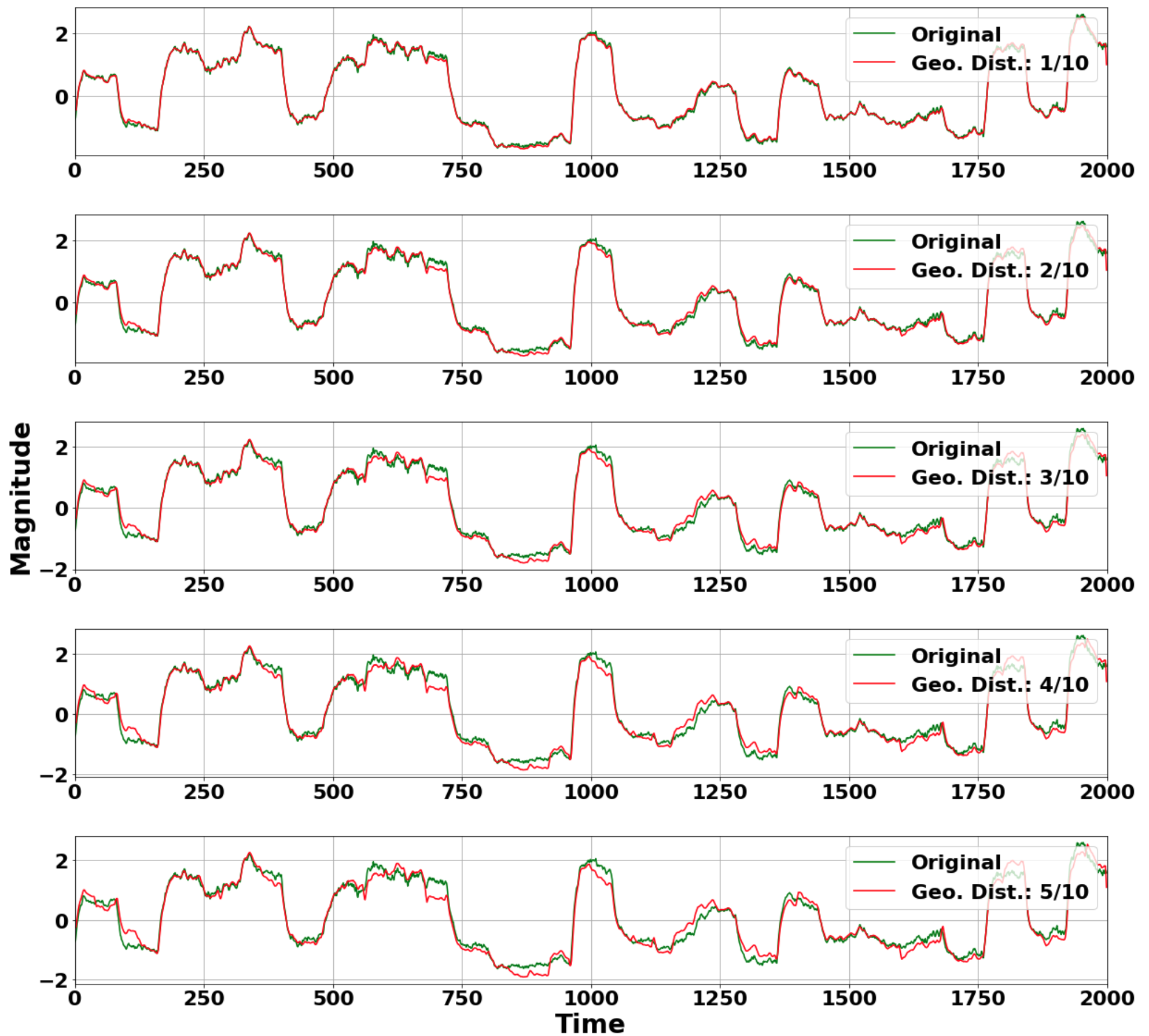}}
  \hspace{0.01\textwidth} 
  \subfigure[The last 5 incremental geodesic steps.]{\includegraphics[width=0.48\textwidth]{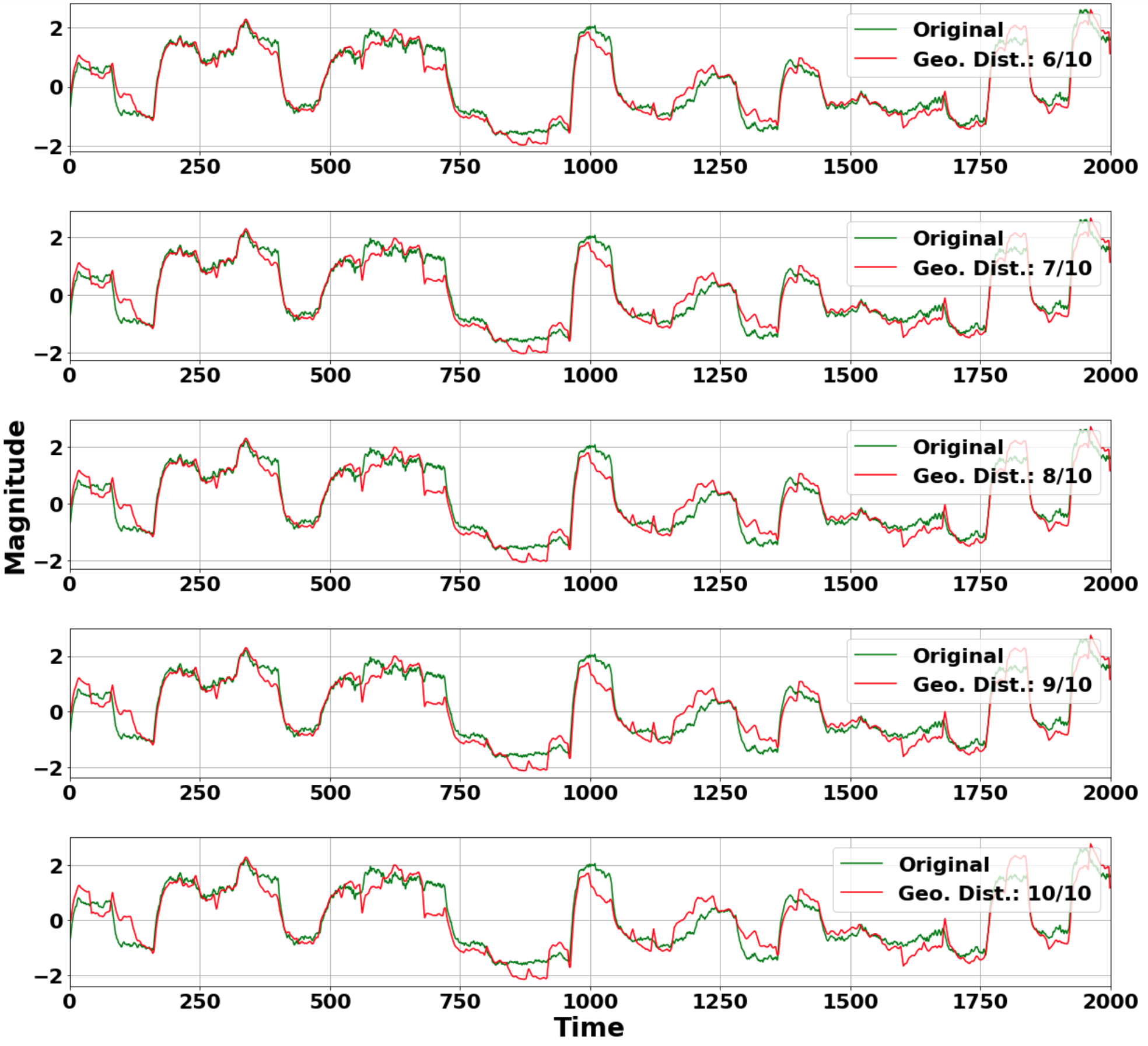}}
  \caption{Incrementally deforming the generated signal with respect to a Stiefel manifold that holds more basis information of the original signal.}
  \label{fig:subplot_example_basis}
\end{figure}

\begin{figure}[htbp]
  \centering
  \subfigure[The first 5 incremental geodesic steps.]{\includegraphics[width=0.48\textwidth]{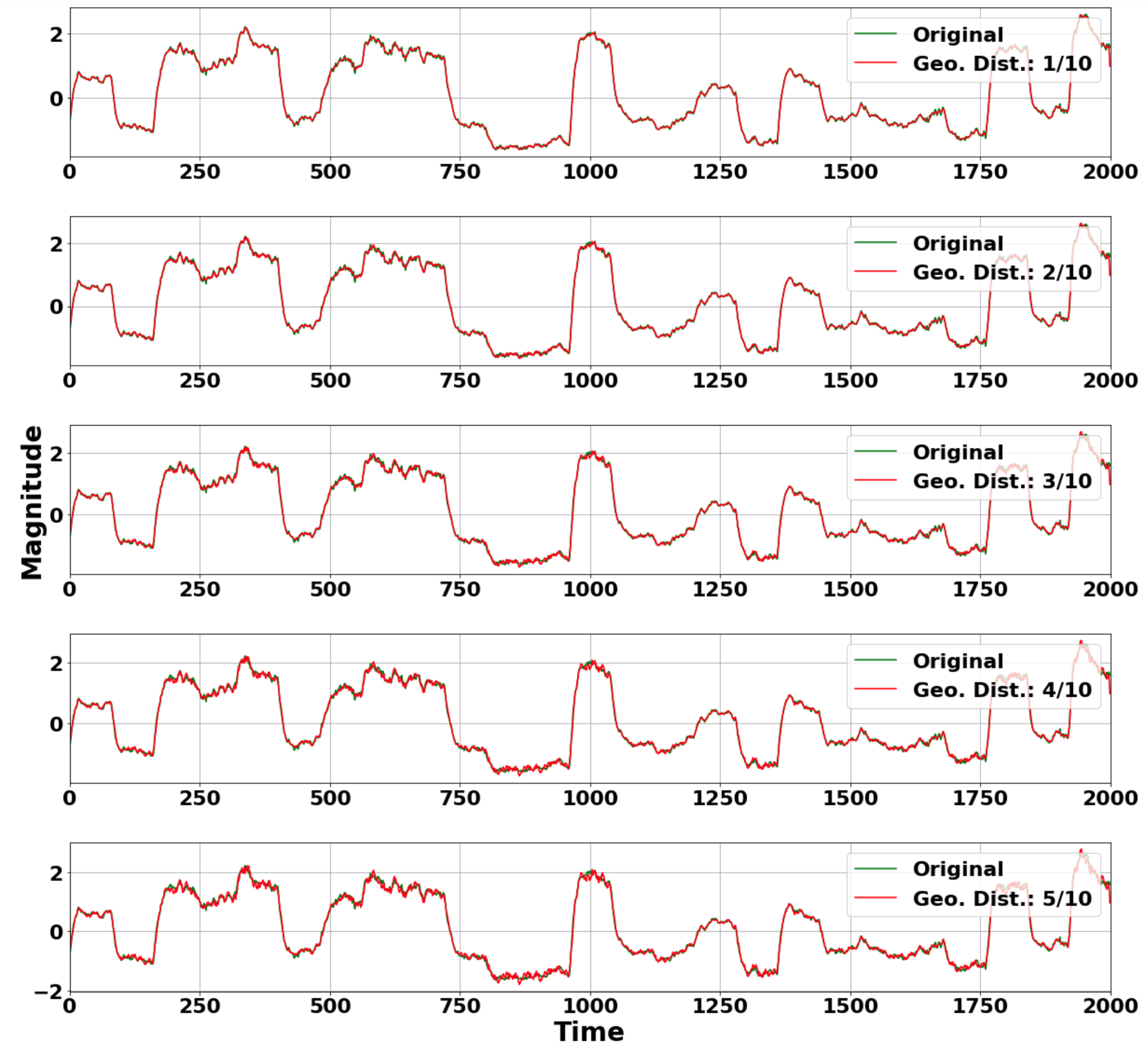}}
  \hspace{0.01\textwidth} 
  \subfigure[The last 5 incremental geodesic steps.]{\includegraphics[width=0.48\textwidth]{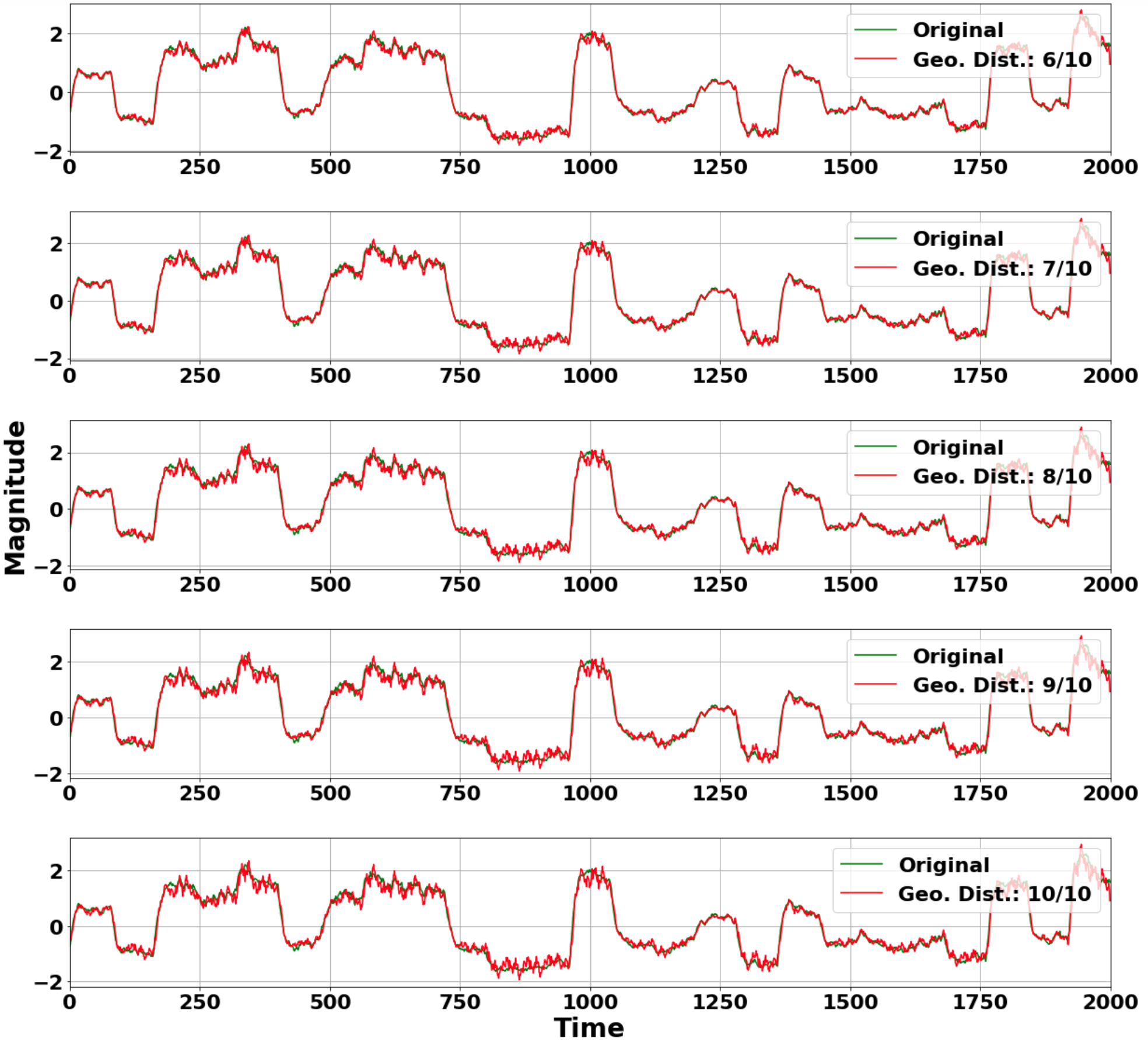}}
  \caption{Incrementally deforming the generated signal with respect to a Stiefel manifold that holds more noisy information of the original signal.}
  \label{fig:subplot_example_noise}
\end{figure}

In summary of this subsection, it has been shown that the \textit{StiefelGen} framework offers a flexible and interpretable approach for time series signal generation. By leveraging properties of the Stiefel manifold and the SVD, noise and basis deformations can be separately controlled to emphasize either signal generation for the purpose of data augmentation or for novel outlier detection training as needed. Further, the incremental geodesic steps enable precise adjustment of the signal, avoiding being stuck with either trivial or catastrophic deformations. The empirical basis functions learned directly from the data require minimal assumptions, and the impact of the page matrix dimensions provides further customization. In the following section, we shall discuss some unique applications of \textit{StiefelGen} which result due to these aforementioned properties.

\section{Applications of \textit{StiefelGen}}

While the conventional application of time series data augmentation involves generating additional data to enhance the training dataset and improve the generalization performance of underlying machine learning models, we now shift our focus to alternative uses stemming from the distinctive properties of \textit{StiefelGen}.

\subsection{Structural Health Monitoring} \label{sec:SHM}

This subsection delves into the application of \textit{StiefelGen} for studying both robustness and adversarial data generation in the context of structural health monitoring (SHM). We shall focus on a classic problem and approach in SHM involving stacking sensor data into a matrix format. Subsequently, dimensionality reduction (PCA) is applied to this matrix, followed by training a one-class support vector machine (OCSVM). The workflow for this problem is exemplified in Figure \ref{fig:SHM_demo}.

\begin{figure}[htbp]
  \centering
{\includegraphics[width=1\textwidth]{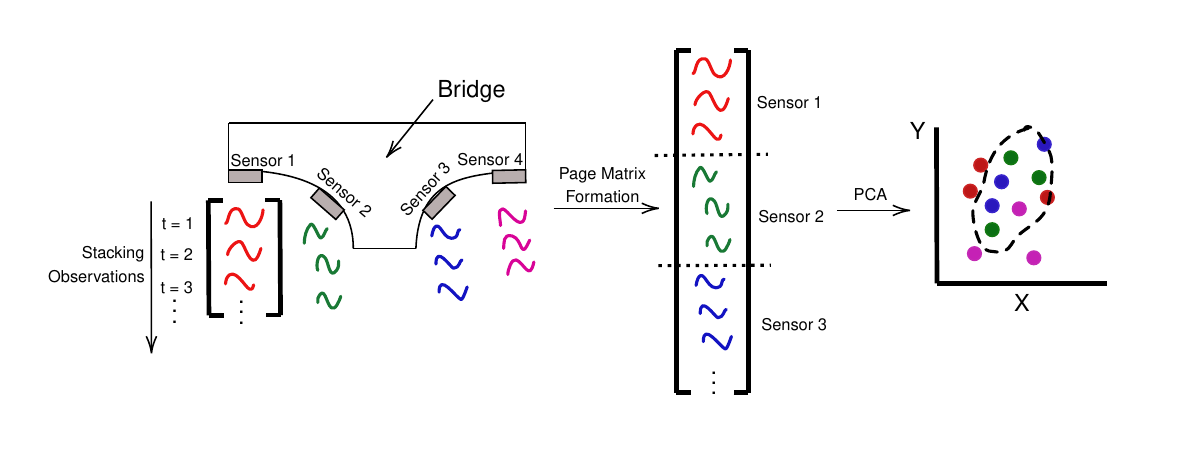}}
  \caption{An example of the conventional data-drive approach in SHM which involves collecting data across a sensor array network, stacking the data into a matrix-like form, and then projecting the data for analysis in lower dimensional spaces. }
  \label{fig:SHM_demo}
\end{figure}

Figure \ref{fig:SHM_demo} illustrates the collection of multiple observations at each sensor location over time, organized into a structured data format. Typically, all the data is consolidated into a larger matrix structure, followed by projection to lower-dimensional spaces \cite{cheema2022bridge, cheema2022drive}. In these spaces, the conventional approach involves training a one-class support vector machine (OCSVM) algorithm. This is due to the ethical constraint that prevents engineers from intentionally damaging deployed structures for the purpose of simply obtaining "damage data" in order to facilitate a two-class SVM analysis.

The use of \textit{StiefelGen} is particularly well-suited for this scenario because the collected data is often stored in a structured matrix form, interpretable as an ``already formed'' page matrix. Given the structured nature of the data, it can be treated as a page matrix with pre-selected values for $m$ and $n$, eliminating the need for reshaping as required for uni-variate time series signals. Therefore one set of hyper parameters (selection of $m$ and $n$) can be completely ignored in practice (if desired). For the SHM problem, without loss of generality, we shall assume identical frequency and duration of measurements across sensors for simple stacking. If there is non-uniformity in the sensor array network, \textit{StiefelGen} analysis can be performed on a per-sensor basis.

In the context of civil engineering, two critical questions arise: (i) How much signal deviation in the original measurement space can be accommodated before official ``damage detection'' is triggered? This poses a model robustness problem, specifically gauging the tolerance for deviation before the OCSVM detects damage. (ii) Are there model signals, severely perturbed to the extent that they should be recognized as damage, but remain undetected by the OCSVM model? Such cases represent adversarial signals for the model. We posit that \textit{StiefelGen} can simultaneously generate sets of time series signals to address both tasks, providing engineers with a nuanced understanding of their chosen outlier detection model, which would be otherwise challenging to achieve in practice.

For this investigation, we employ a simplified model of a SHM problem. The model consists of a toy bridge structure equipped with five sensors, with fifty observations recorded per sensor. The sensor readings have a frequency of 50Hz, and each observation spans a total duration of nine seconds. We assume a small bridge size with symmetrically placed sensors, ensuring similar modal information per sensor. Consequently, variations in observations across different sensors primarily stem from noise, indicating a predominantly aleatorically imposed uncertainty distribution in sensor space. The data generation model applied across the sensor space is as follows,
\begin{equation}
    S = 4\sin\left(6\pi t^{0.5}\right) + \sin\left(15\pi t\right)+ \mathcal{N}(1, 0.5),
\end{equation}
where $t\in\mathbb{R}_{\geq 0}$ represents the time variable, and the mean of 1 used in the Gaussian term represents an arbitrary bias term. As mentioned earlier, if the observed noise model (or underlying epistemic uncertainty) varies greatly between sensors, then \textit{StiefelGen} should be used on a per-sensor basis instead. Lastly, the OCSVM model used in these analysis is based on that found in the Scikit-learn package with $\nu=0.1$, and $\gamma=10^{-3}$ \cite{pedregosa2011scikit}.

\subsubsection{StiefelGen: Robustness in SHM}

For the \textit{StiefelGen} analysis in this context, a single large perturbation is applied to each data point once, taking $\varepsilon=1$. Subsequently, the final outputs are examined in a 2D PCA plot. The expectation is that most, if not all, input data points in the 2D projected space will have undergone a significant movement, potentially surpassing the trained OCSVM boundary. To analyze the impact of this shift, the data point with the largest deviation (with respect to the Euclidean norm) is identified by comparing the norms of the projected data points before and after perturbation, and then considering the top $K$ data points changes. In this case, we shall just focus on $K=1$, that is, the signal which experienced the most substantial change for demonstration purposes. While analyzing a single data point that underwent considerable deviation might not seem advantageous on the surface, because these perturbations occurred over a smooth manifold within/on the radius of injectivity, there then exists a unique geodesic path from the starting position to the end position. This concept allows for the gradual tracking of signal deviations over the data manifold, determining the point at which the original signal's deviations become substantial enough to reach the edge of the OCSVM boundary. This approach facilitates the inspection of the OCSVM model's efficacy and addresses questions such as: ``what levels of deviation are acceptable until the OCSVM boundary is breached?'' This is a crucial consideration for model robustness (``what exactly is happening on the OCSVM boundary and what data points cause the model to fail?''), and the ability to follow a data geodesic is essential for providing meaningful answers to this question. The impact of this approach in the context of structural health monitoring (SHM) is illustrated in Figure \ref{fig:shm_intro_example_b}. 
  
\begin{figure}[htbp]
  \centering
  \subfigure[Plot of the initial projected data samples (green), and the consequence of perturbing to the input data to radius of injectivity (blue). The red points refer to the outlier points used in the training of the OCSVM of the original dataset (green), given the OCSVM hyper parameters.]{\includegraphics[width=0.45\textwidth]{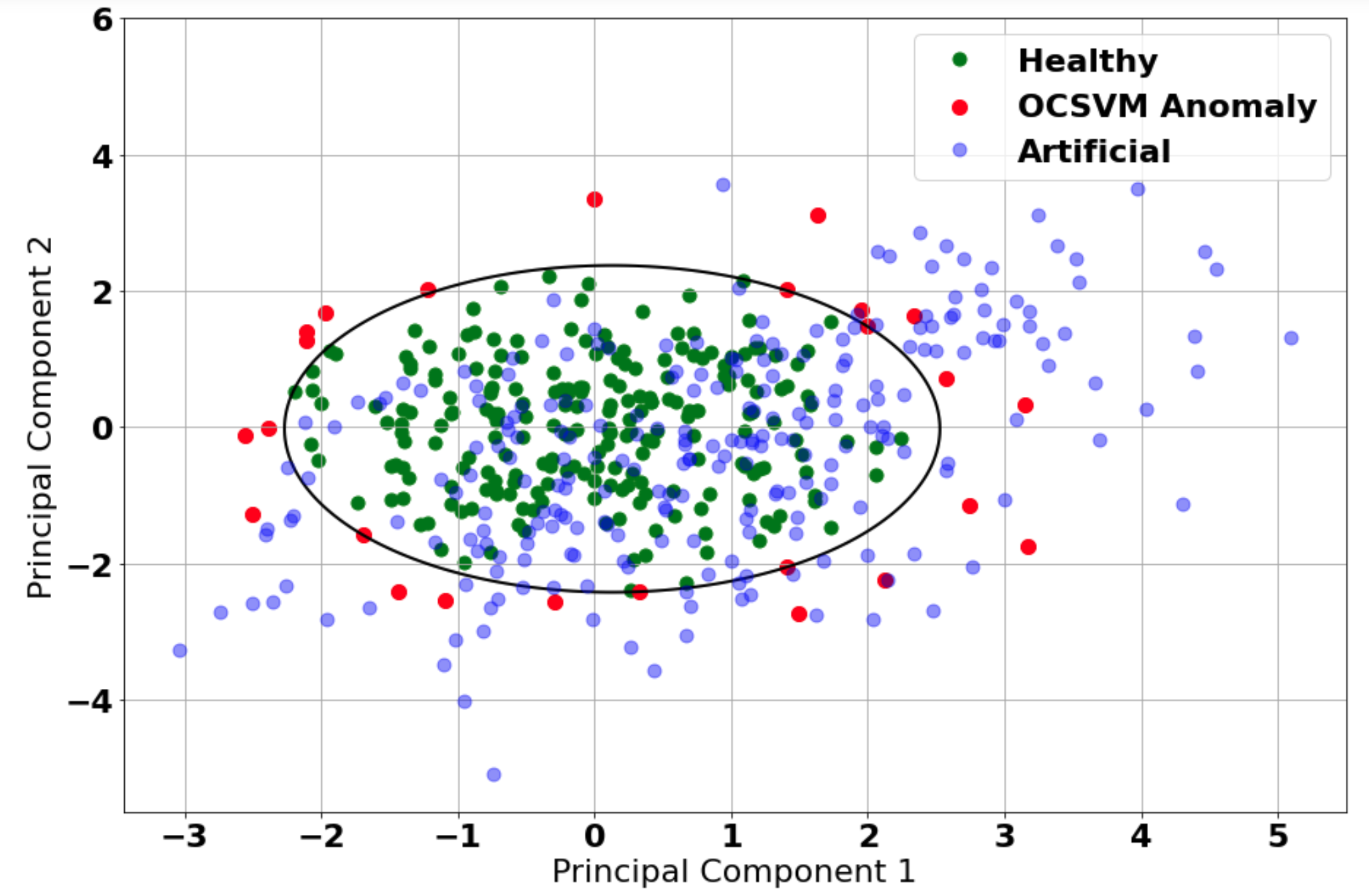}\label{fig:shm_intro_example_a}}
    \hspace{0.01\textwidth} 
     \subfigure[A histogram of the $L_2$ norms before and after the perturbation (that is comparing the healthy and the perturbed datapoints).\label{fig:norms_wrt_origin}]{\includegraphics[width=0.47\textwidth]{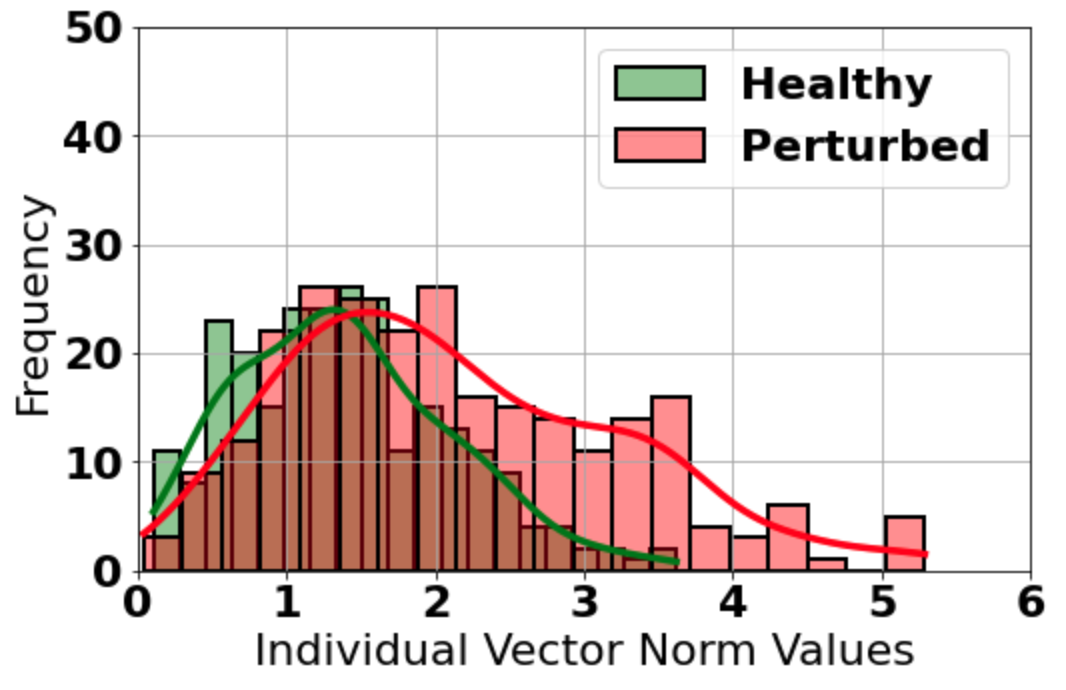}}
  
  \subfigure[Showcase of following a geodesic between a point which is initially approximately in the centre of the OCSVM boundary, and its resultant point. The point at which the OCSVM boundary is to be crossed over is given by the blue star.]{\includegraphics[width=0.48\textwidth]{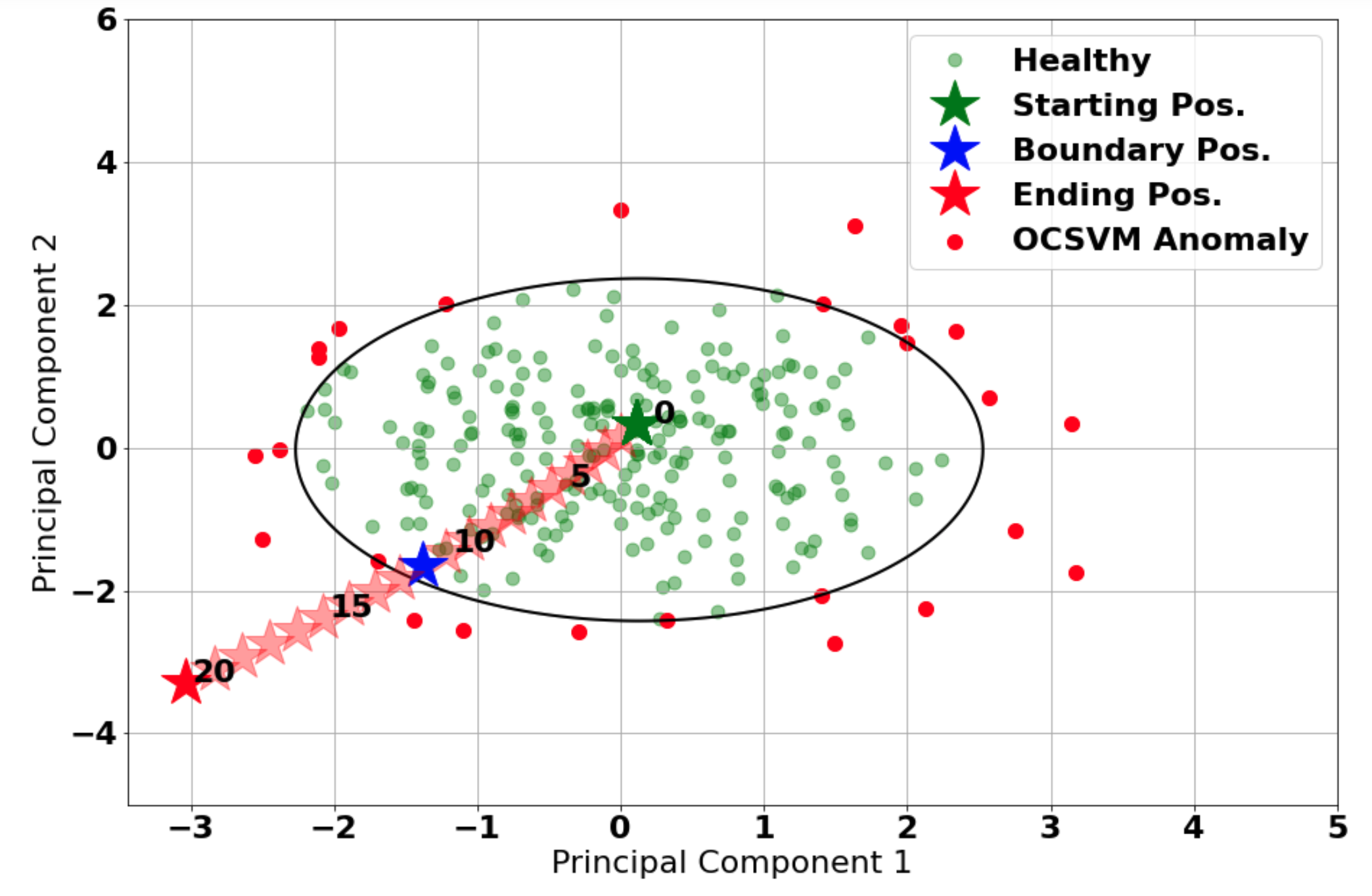}\label{fig:shm_intro_example_b}}
\hspace{0.01\textwidth} 
    \subfigure[Outline of the difference between either following the $U$ perturbation only, and the $V$ perturbation only, in relation to when the $U-V$ perturbation is jointly performed.]{\includegraphics[width=0.45\textwidth]{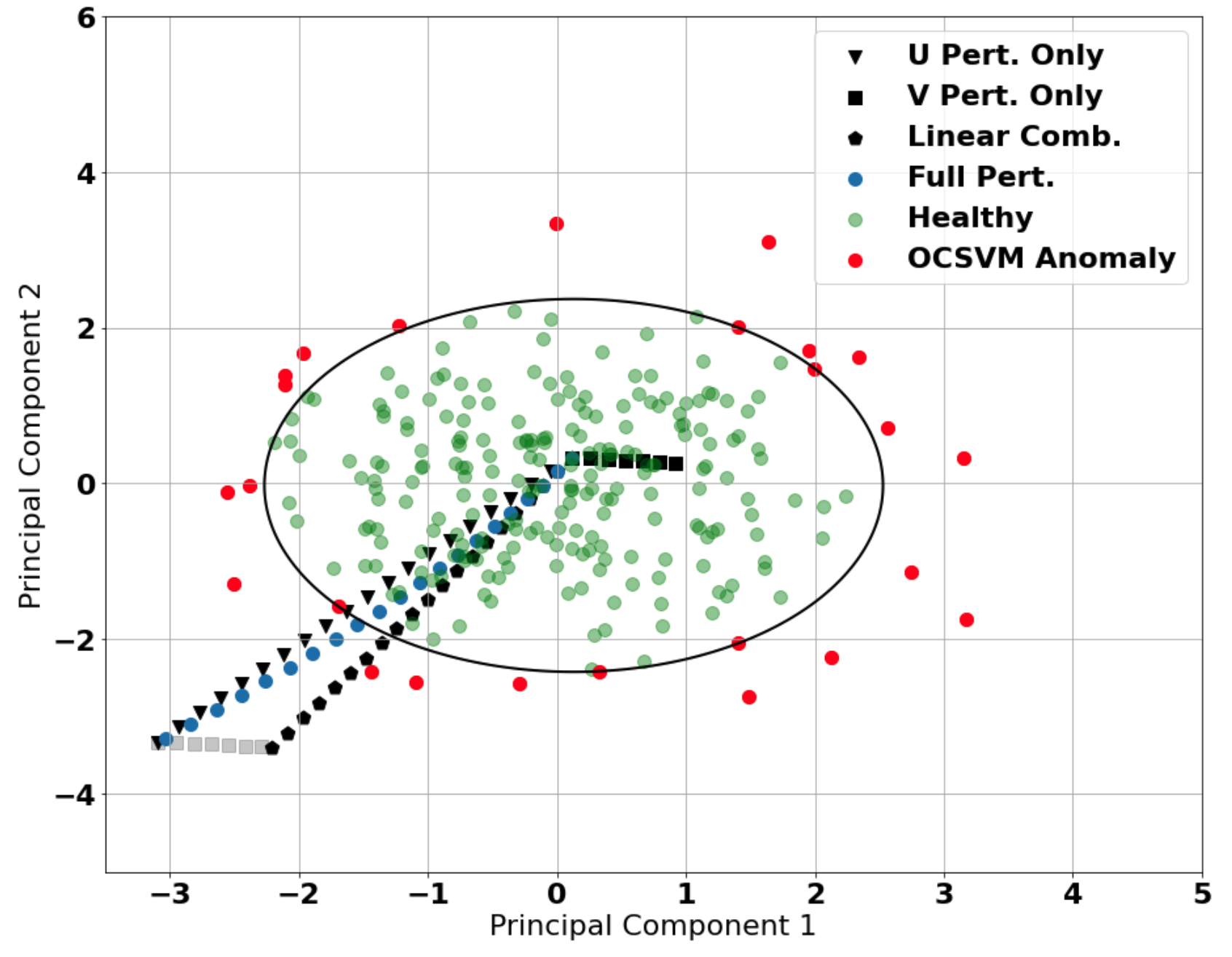}\label{fig:shm_intro_example_c}}

      \subfigure[Plots of the original input data space in relation to positions along the geodesic (every 5-th position is shown in black-dashed line). Signals have been shown with the applied smoothing window of three to assist with visual clarity. These times series signals are based on the path followed by Subfigure \ref{fig:shm_intro_example_c}.\label{fig:shm_intro_example_d}]{\includegraphics[width=0.46\textwidth]{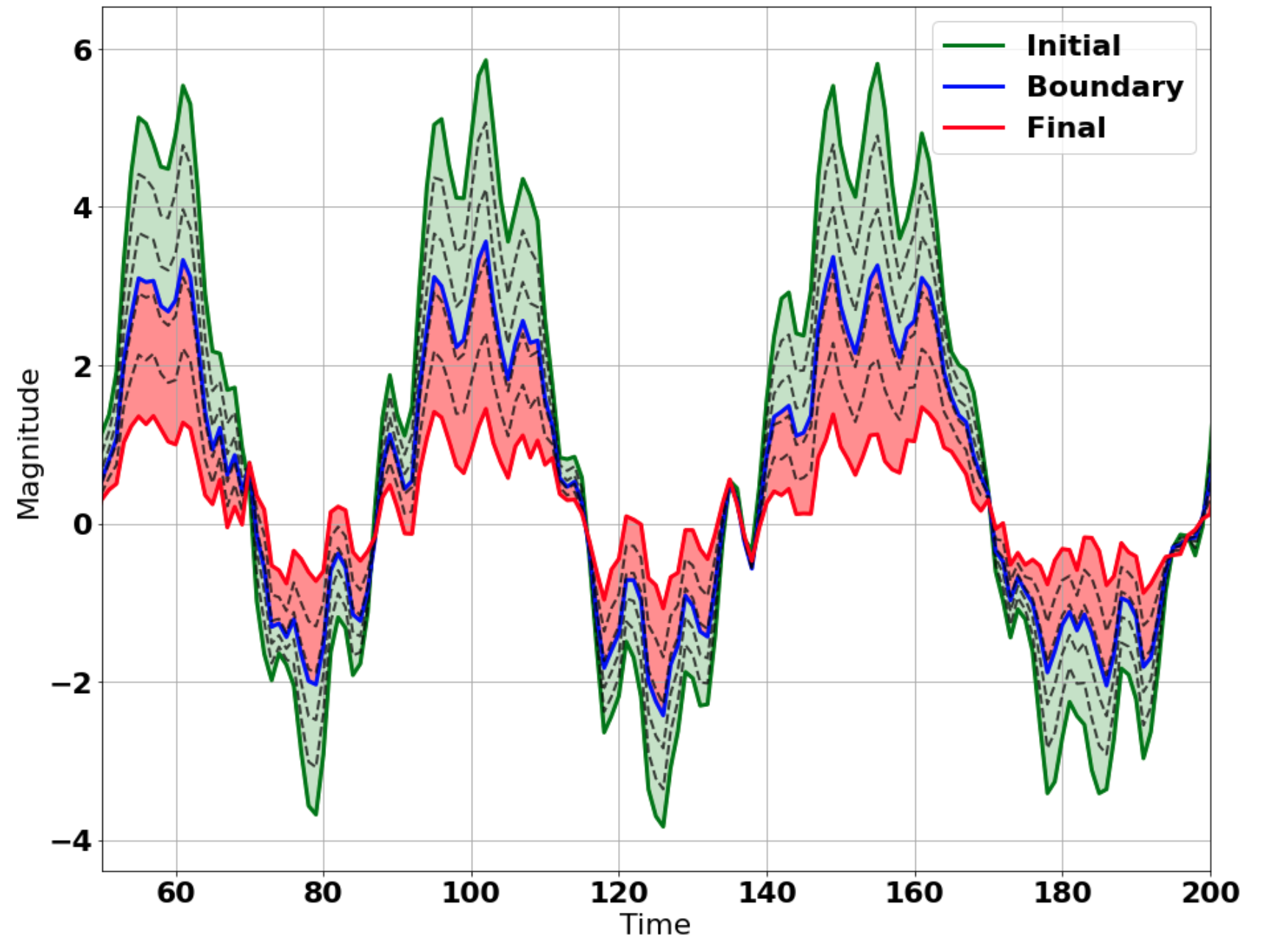}}
  \caption{An example of the effect of exploiting the geodesics in the StiefelGen algorithm.}
  \label{fig:shm_intro_example}
\end{figure}

In Figure \ref{fig:shm_intro_example_b}, the initial projected point is denoted by the dark green star, and its final position after a large perturbation is represented by the solid red star. Intermediate points, totaling eighteen stars (there are twenty stars in total considering the initial and final points), are depicted as light-red stars. Leveraging the properties of \textit{StiefelGen}, we can pinpoint the moment when perturbations become significant enough to push the initial point just outside the OCSVM boundary. This capability allows SHM engineers to critically assess the model's efficacy by gauging how much signal change is allowable at the point of reaching this boundary. Consequently, they can comment on the overall robustness of the model concerning permissible levels of perturbations before triggering an alarm for potential damage. Figure \ref{fig:shm_intro_example_d} illustrates the process of following these geodesics, plotting every fifth time series signal along the geodesic. Notably, the signal corresponding to the solid red star exhibits a considerable deviation from its original waveform, which expectedly represents a damage state which sits firmly outside of the learned OCSVM boundary.

However, a key signal to inspect for assessing the OCSVM model's robustness is that of the blue star. Upon plotting its corresponding time series signal, it is apparent that it has also shifted significantly from its original position yet remains within the OCSVM boundary. In principle, it should technically be classified as a healthy signal according to the OCSVM model, but for the SHM engineer this signal has deviated so much that it should invariably represent a damage state for the structure \cite{cheema2023use}. This paradox underscores the insufficiency of the learned OCSVM model, as expected in this somewhat na\"ive implementation. Drawing this conclusion would be challenging without the ability to smoothly deform a signal from its starting to its ending position, showcasing an application benefit of \textit{StiefelGen}. These results are particularly positive, given that no model assumptions were made except for the diffeomorphic nature of signal changes, and that the set of allowable deformations shall lie on the Stiefel manifold with respect to its corresponding SVD decomposition. Moreover, in this case, the $m$ and $n$ values typically needed to reshape a 1D signal were not required, as they were implicitly chosen when stacking signals from each sensor. Moreover, for this example, the only reason a smoothing window was ever needed was to enhance the visual clarity of Figure \ref{fig:shm_intro_example_d}. Thus, the sole hyper parameter required was that of the percentage of perturbation. However, this choice is easily justified by simply considering the we desire the maximum possible perturbation of 100\%, placing us at the edge of the radius of injectivity, because then we were able to smoothly follow along the geodesic until reaching the OCSVM boundary. In principle then, no parameters, no hyper parameters, or any model training was required. Further, minimal data assumptions had to be made in practice. This approach to studying the robustness properties of the OCSVM boundary of allowable signals, is significantly more straightforward than the alternative: which would be attempting to estimate the level of perturbation required to push the initial point onto the OCSVM boundary by varying the perturbation levels and ``shooting'' the signal forward.

Subfigure \ref{fig:shm_intro_example_c} illustrates the distinct effects of following individual $U$ or $V$ geodesics. As established earlier, depending on the stacking procedure, either the $U$ (column space) or $V$ (row space) geodesic tends to influence a basis or noise change more prominently. In Subfigure \ref{fig:shm_intro_example_d}, a noticeable and sizable basis change in the original signal is evident. Therefore, we anticipate that explicitly moving along the noise direction would result in far less movement in the projected space compared to altering the basis representation. The consequences of traversing these directions concerning the original signal space are elucidated in Figure \ref{fig:shm_proof_noise_basis}.

\begin{figure}[htbp]
  \centering
  
  \subfigure[Noise deviation direction ($V$).]{\includegraphics[width=0.45\textwidth]
  {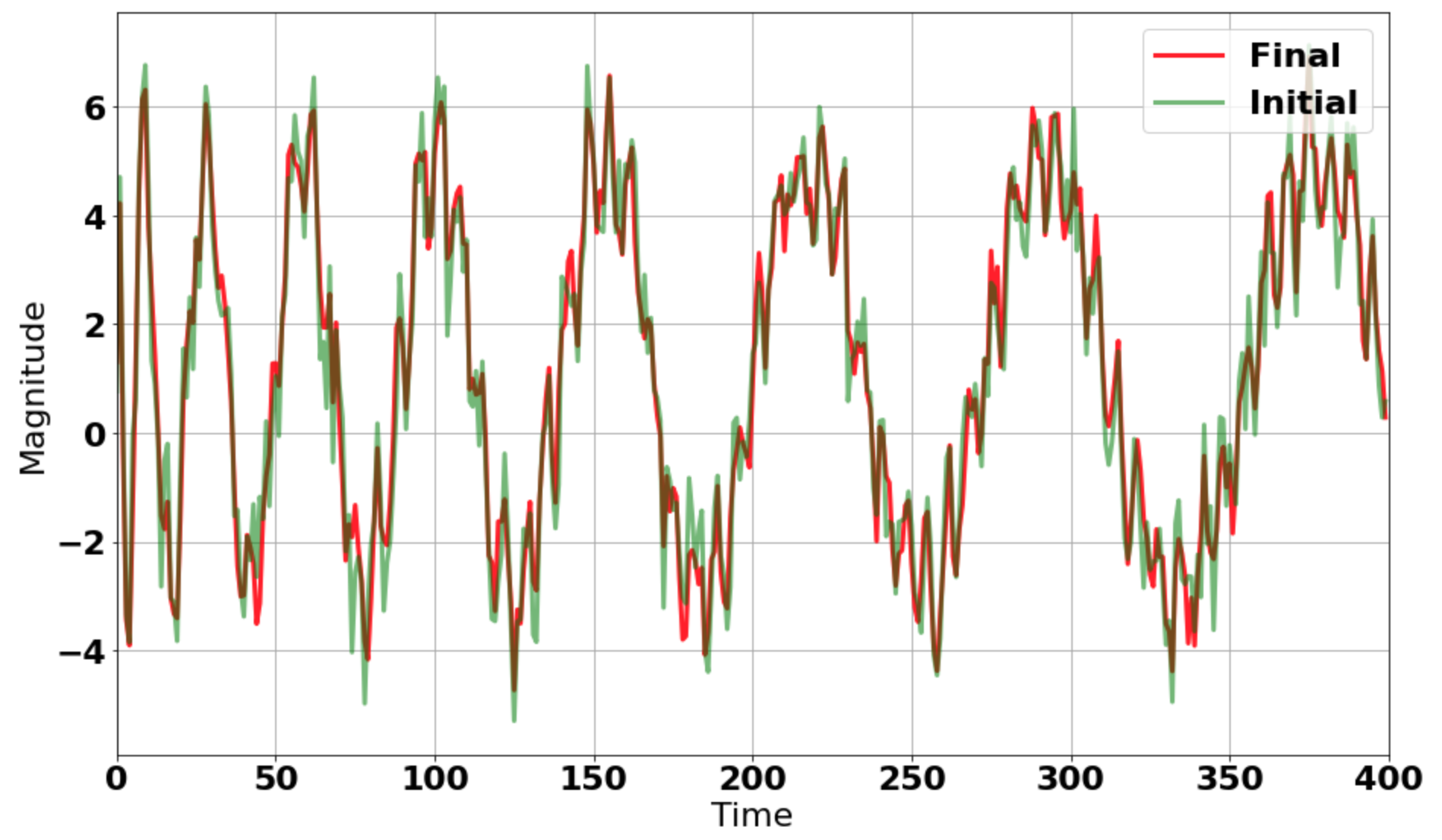}\label{fig:shm_proof_noise_basis_V}}
 \hspace{0.01\textwidth} 
  \subfigure[Basis deviation direction ($U$).]{\includegraphics[width=0.44\textwidth]
  {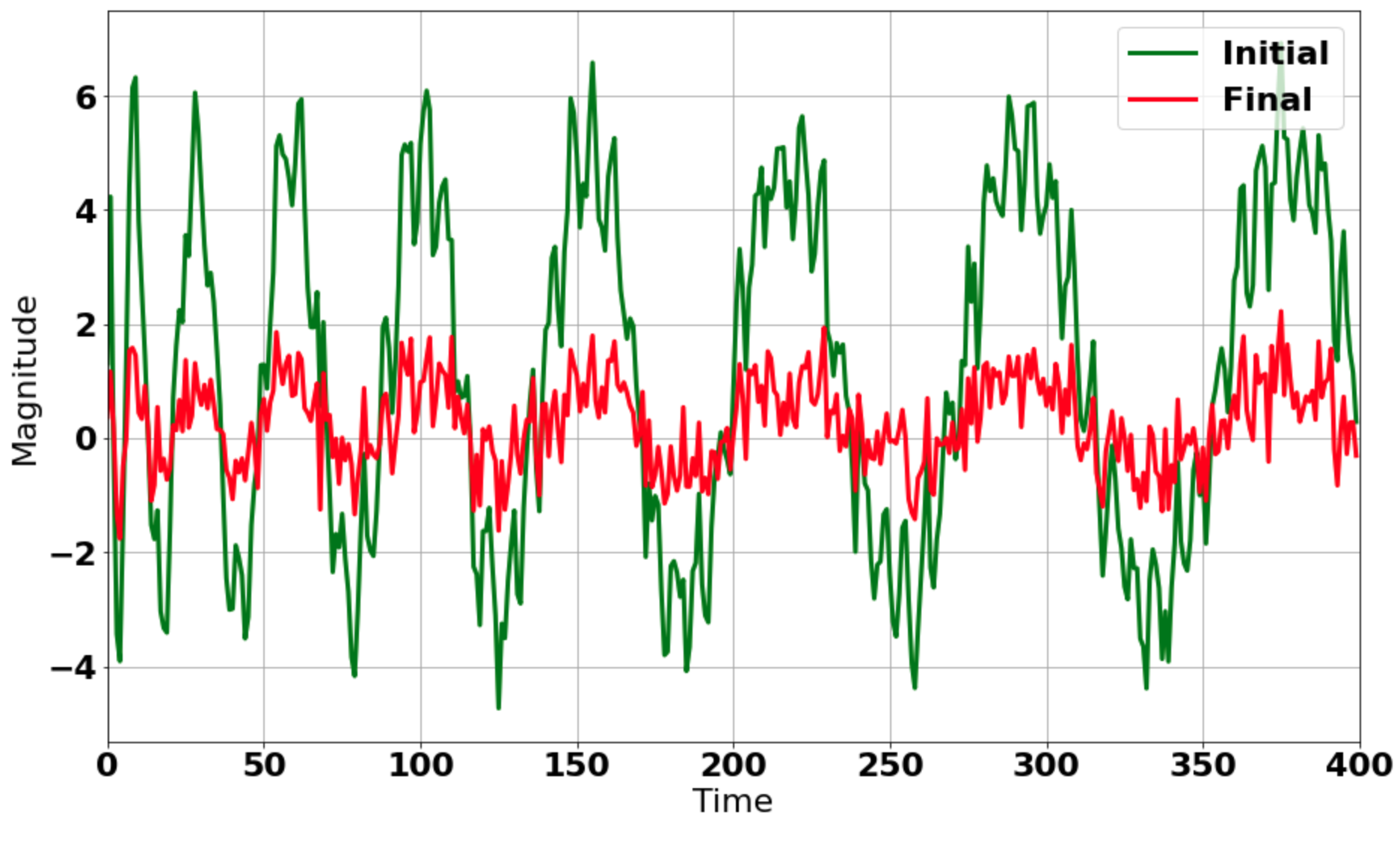}\label{fig:shm_proof_noise_basis_U}}
  \caption{The difference between following along either $U$ and $V$ directions individually, or in a combined fashion, in relation to the original signal space.}
  \label{fig:shm_proof_noise_basis}
\end{figure}

Here, we see that indeed there has been changes to the variation in the noise pattern for Subfigure \ref{fig:shm_proof_noise_basis_V}, and indeed a change in the basis functions in Figure \ref{fig:shm_proof_noise_basis_U}. However, it is crucial to reiterate that the terms "noise direction" and "basis direction" are empirical approximations. Finally, a noteworthy insight from Figure \ref{fig:shm_intro_example_c} is that when applying either a $U$-only or $V$-only perturbation, the projected geodesics onto 2D space ``appear'' linear, while the combined $U-V$ perturbation path exhibits clear curvature. A simplified explanation of this phenomenon can be seen as follows. Consider small perturbations over $U$ and $V$ to be taken as $U + f_U(\delta U)$ and $V + f_V(\delta V)$ respectively, where $f_U, f_V$ are arbitrary matrix-valued non-linear functions which arise due to curvature which may exist on the Stiefel manifold. Consider applying the perturbation in the $U$ direction without loss of generality,
\begin{align*}
   \mathcal{T}_{\text{mat}, 2} &= (U + f_U(\delta U)) \Sigma  V^{\intercal}\\
   &= U \Sigma V^{\intercal} + f_U(\delta U) \Sigma V^{\intercal} \\
   &= \mathcal{T}_{\text{mat}, 1} + \alpha f_U(\delta U) \\
   &\approx \mathcal{T}_{\text{mat}, 1} + \alpha f_U(0) + \alpha f_U'(0) \delta U + \alpha \frac{1}{2} f_U''(0) (\delta U)^2
\end{align*}

where $\alpha= \Sigma V^{\intercal}$ is constant, and we have expanded up to a second order Taylor expansion. Thus if the curvature is minimal at point $U$, we may consider $f_U''(0)\rightarrow 0$ in approximation, meaning that if \textit{StiefelGen} is applied along one of $U$ or $V$, we may expect the ensuing augmentations (and therefore projections) to occur in approximately linear increments. Consider now applying both, $\delta U$ and $\delta V$ perturbations,
\begin{align*}
    \mathcal{T}_{\text{mat}, 2} &= (U + f_U(\delta U)) \Sigma  (V^{\intercal} + f_V(\delta V)^{\intercal})\\
    &= U \Sigma V^{\intercal} + U \Sigma f_V(\delta V)^{\intercal} + f_U(\delta U) \Sigma V^{\intercal} + f_U(\delta U) f_V(\delta V)^{\intercal} \\
    &= U \Sigma V^{\intercal} + \alpha f_U(\delta U) + \beta f_V(\delta V) + f_U(\delta U) f_V(\delta V)^{\intercal} \\
    &\approx U \Sigma V^{\intercal} + \alpha f_U(0) + \alpha f_U'(0) \delta U + \alpha \frac{1}{2} f_U''(0) (\delta U)^2 \\
    &\qquad+ \beta f_V(0) + \beta f_V'(0) \delta V + \beta \frac{1}{2} f_V''(0) (\delta V)^2 +  f_U''(\delta)f_V''(0)^{\intercal} + f_U(\delta U) f_V(\delta V)^{\intercal},
\end{align*}
where,
\begin{align*}
f_U(\delta U) f_V(\delta V)^T \approx
&\quad f_U(0) f_V(0)^T + f_U(0) f_V'(0) \delta V^T + f_U'(0) \delta U f_V(0)^T + f_U'(0) \delta U f_V'(0) \delta V^T \\
&\quad + \frac{1}{2} f_U(0) f_V''(0)(\delta V)^2 + \frac{1}{2} f_U''(0)(\delta U)^2 f_V(0)^T + \frac{1}{2} f_U'(0) \delta U f_V''(0)(\delta V)^2 \\
&\quad + \frac{1}{2} f_U''(0)(\delta U)^2 f_V'(0) \delta V^T + \frac{1}{4} f_U''(0) (\delta U)^2 f_V''(0)(\delta V)^2.
\end{align*}

Thus, by simultaneously perturbing both the $U$ and $V$ matrices, the underlying complexity of the geodesic path in the projected PCA space significantly increases, as evidenced by the interplay between the first and second derivatives in the above equations. In other words, minor curvatures in the Stiefel manifolds of either $U$ or $V$ can lead to substantial downstream coordinate changes. Figure \ref{fig:SHM_curve} illustrates this concept, by comparing the $(x,y)$ coordinate paths of the basis-based $U$-only perturbation path, along with the joint $U-V$ perturbation path. The pronounced curvature in the $x$ coordinate along the geodesic path of the joint $U-V$ perturbation (Subfigure \ref{fig:SHM_curve_x_coord}) is evident, highlighting a major difference when compared to a $U$-only (or $V$-only) approach when using \textit{StiefelGen}. 

\begin{figure}[htbp]
  \centering
  
  \subfigure[Deviation of the $x$ co-ordinate from step 0 (start) through to step 20 (end) of the geodesic path when both $U$ and $V$ are perturbed jointly, and when only $U$ is perturbed. The linear interpolation from step 0 to step 20 is also shown to clarify the deviation which occurs from a simple straight line path.]{\includegraphics[width=0.45\textwidth]
  {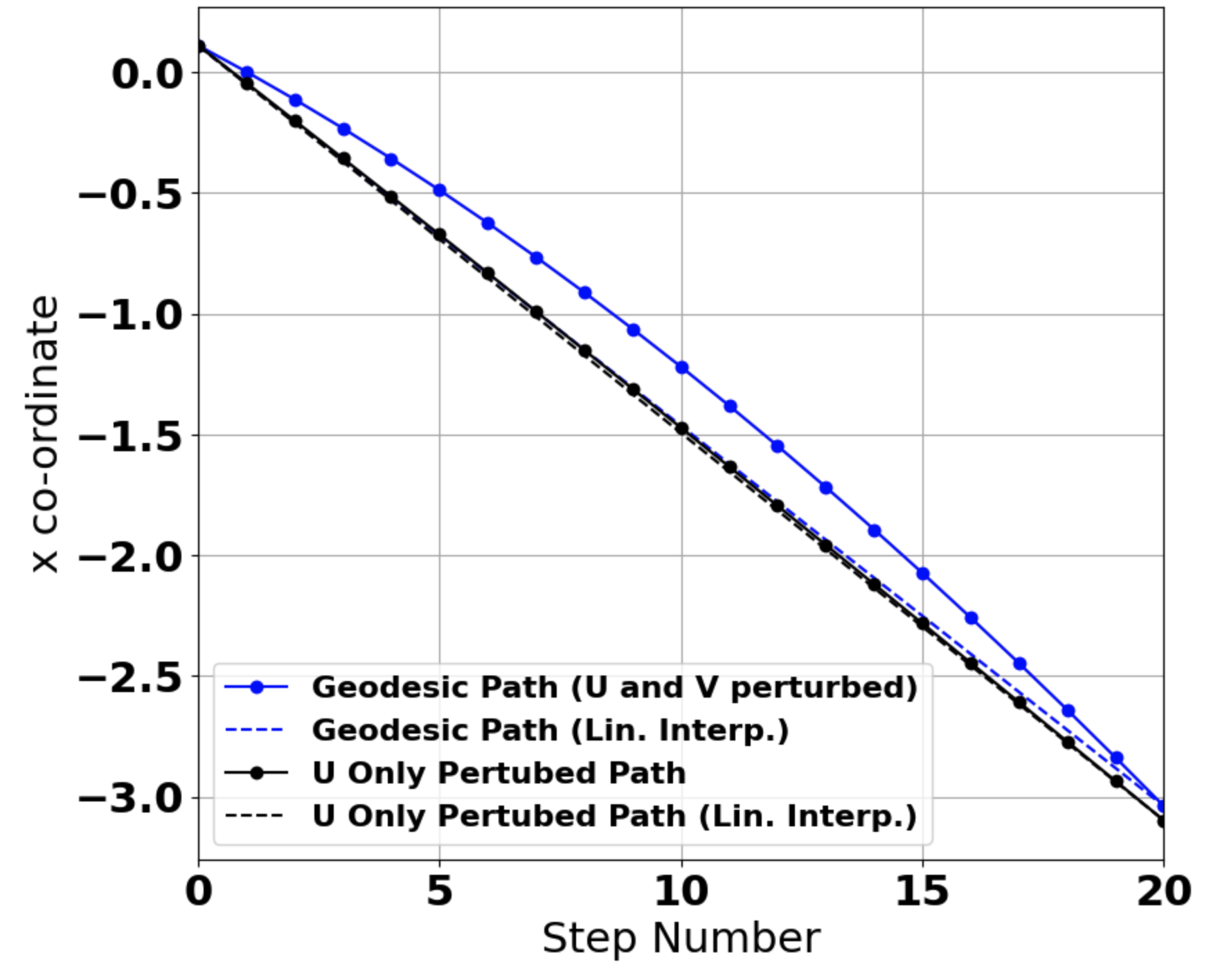}\label{fig:SHM_curve_x_coord}}
 \hspace{0.01\textwidth} 
  \subfigure[Deviation of the $y$ co-ordinate from step 0 (start) through to step 20 (end) of the geodesic path when both $U$ and $V$ are perturbed jointly, and when only $U$ is perturbed. The linear interpolation from step 0 to step 20 is also shown to clarify the deviation which occurs from a simple straight line path.]{\includegraphics[width=0.44\textwidth]
  {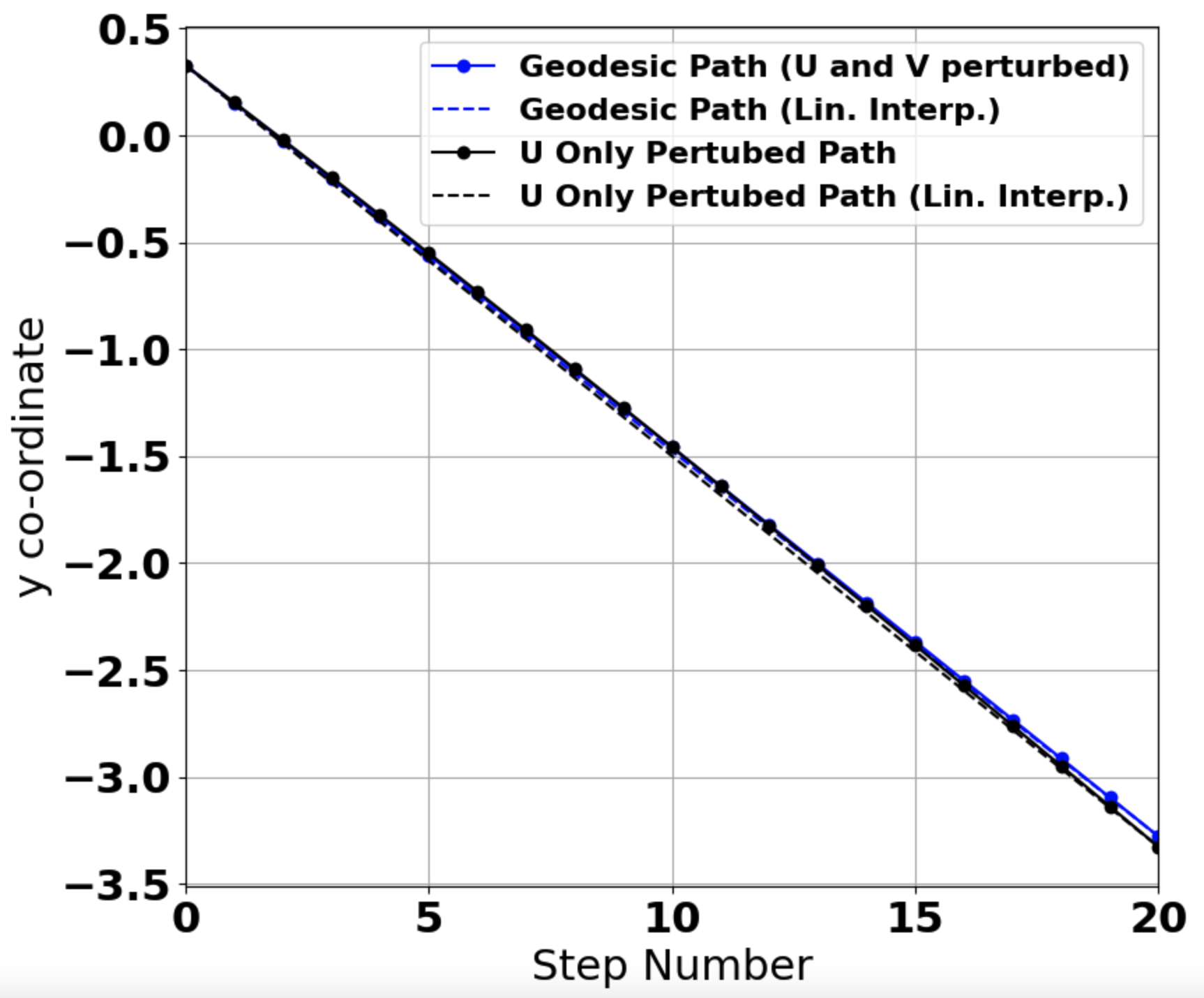}\label{fig:SHM_curve_y_coord}}
  \caption{Visual clarification of the deviation of the co-ordinates from a simple straight line path that two approaches take in the projected space. The first being when $U$ and $V$ are jointly perturbed, the second being only when $U$ is perturbed.}
  \label{fig:SHM_curve}
\end{figure}

\subsubsection{StiefelGen: Adversarial Data Generation in SHM} 

In the SHM domain, addressing the adversarial data generation problem involves answering a critical question: \textit{What does it take for damage signals to pass through my model and be misclassified as healthy?} To leverage \textit{StiefelGen} for this purpose, we follow a simple two-step procedure as a means to generate and search for potential adversarial examples. First, we calculate the Euclidean norm of the difference between the maximally perturbed ($\varepsilon=1$) and non-perturbed ($\varepsilon=0$) data points in the projected space. For each data point $x_1$ in its healthy state and $x_2$ in its perturbed state, the quantity $|x_2 - x_1|_2$ is computed and sorted, as illustrated in Subfigure \ref{fig:shm_adversarial_a}. Upon inspecting the plot, notable features indicative of adversarial vectors are identified. The plot exhibits a sigmoidal shape with two inflection points, roughly dividing it into three distinct portions. The first region corresponds to minimal norm changes, suggesting little alteration in the underlying time series. The final portion represents time series that have undergone significant shifts, likely resulting in the data point leaving the OCSVM boundary. In order to explore for adversarial samples focus is placed upon the intersection between large norm shifts (indicating the presence of substantial data movement) and moderate norm shifts (suggesting that whilst the data exhibited large movement, there is a reasonable chance it did not leave the OCSVM boundary). While not mathematically rigorous, this heuristic approach parallels the elbow method used in $k$-means clustering, albeit instead of looking for diminishing returns, we look for the point of rapidly increasing (norm) returns. In this study, the 85th percentile, centered with the second inflection point in Figure \ref{fig:shm_adversarial_a}, leading to the selection of the first data point after the 85th percentile as the adversarial example for exploration.

This data point exhibits perturbation geodesics in the projected 2D space as shown in Subfigure \ref{fig:shm_adversarial_b}, accompanied by the corresponding time series signals visualized in Subfigure \ref{fig:shm_adversarial_c}. Notably, two key observations emerge: (i) Throughout its geodesic path, the perturbed data point consistently remains within the healthy boundary defined by the OCSVM, and (ii) The change in the time series function escalate fairly rapidly, indicating an early manifestation of damage well before reaching its final position. Such an analysis serves as a valuable case study for the comprehensive examination of the structural health monitoring (SHM) problem. Therefore, \textit{StiefelGen} not only facilitates the exploration of model robustness by scrutinizing the boundary behavior of the OCSVM model concerning potential signals (prior subsection) but also enables the generation of adversarial examples. These nuanced investigations are challenging to conduct with other time series data augmentation methods, as (to the best of the author's knowledge) they tend to lack the combined capability of smoothly deforming signals from an augmentation state to a novelty state, without the need for pre-training or relying on large data sets.

\begin{figure}[htbp]
  \centering
    \subfigure[The change in $L_2$ norms before and after perturbation is applied. Here $x_1$ refers to the healthy state, and $x_2$ the perturbed (damage) state.\label{fig:shm_adversarial_a}]{\includegraphics[width=0.47\textwidth]
  {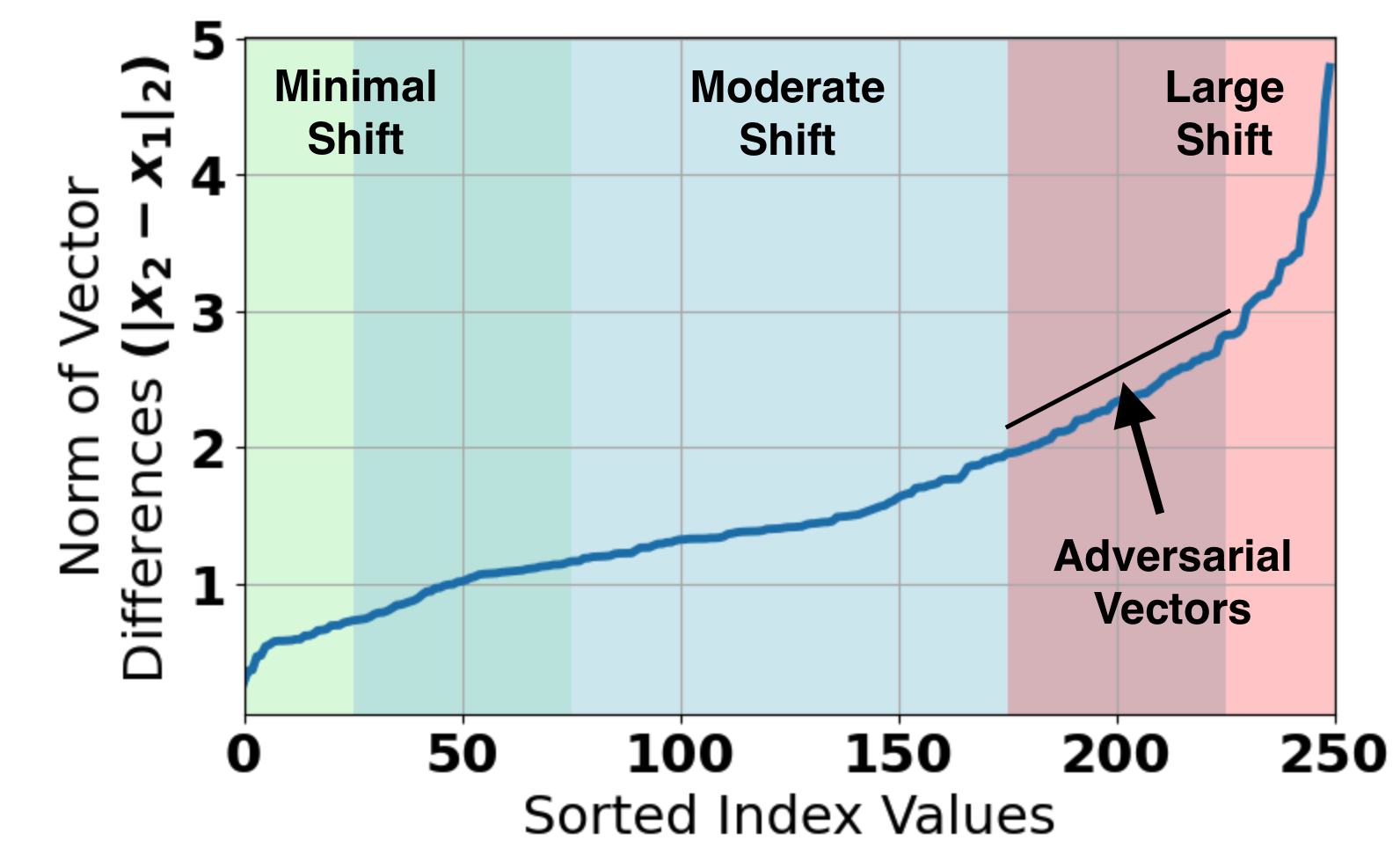}}
  \hspace{0.01\textwidth} 
  \subfigure[The geodesic path of the first data point which satisfied the condition that $\|x_2-x_1\|$ is over the 85-th percentile.\label{fig:shm_adversarial_b}]{\includegraphics[width=0.43\textwidth]
  {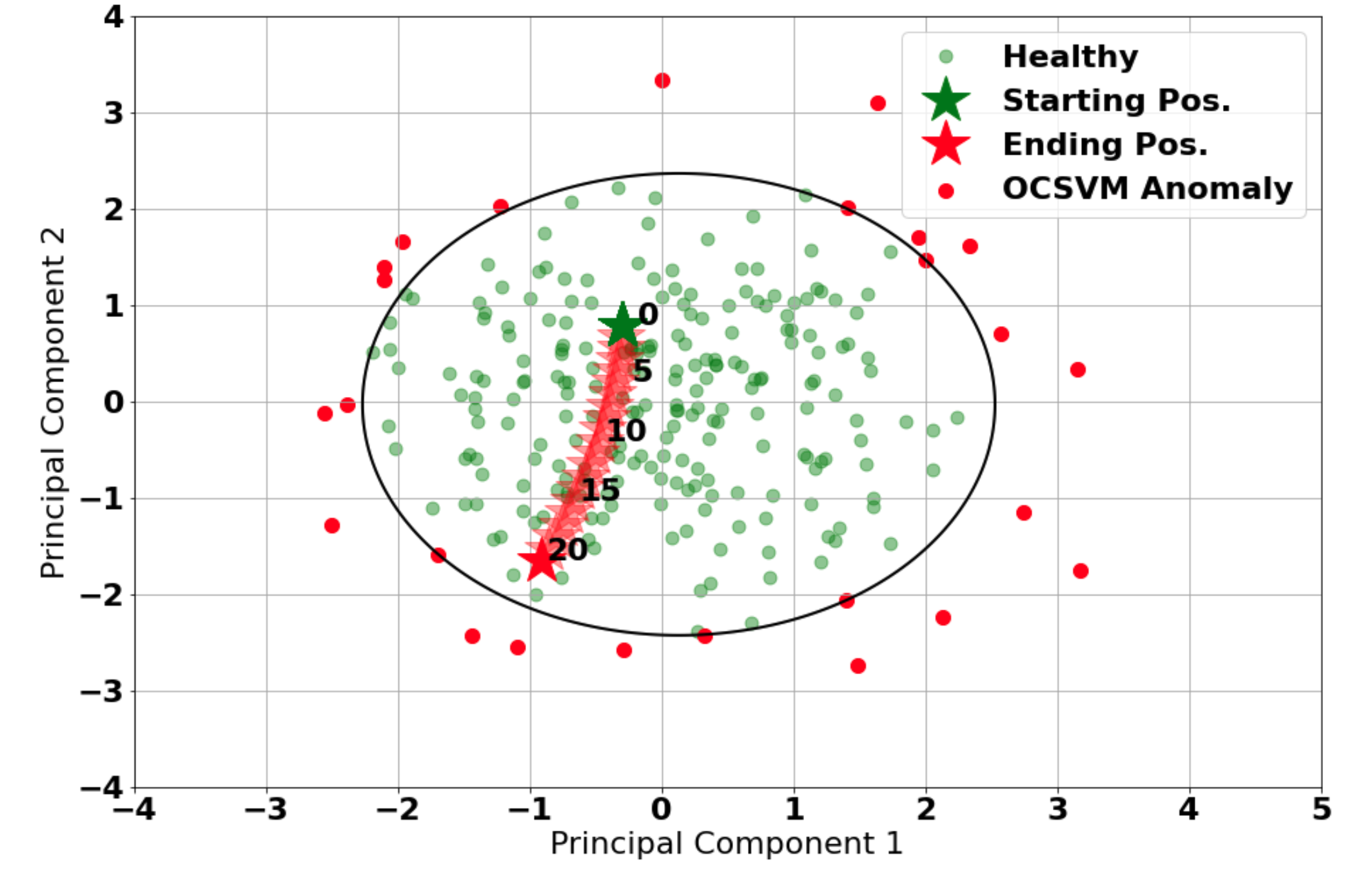}}
 
  \subfigure[The signal as it moved from the healthy to the maximally perturbed state. The black dashed lines represent every 5-th step along the geodesic.\label{fig:shm_adversarial_c}]{\includegraphics[width=0.45\textwidth]
  {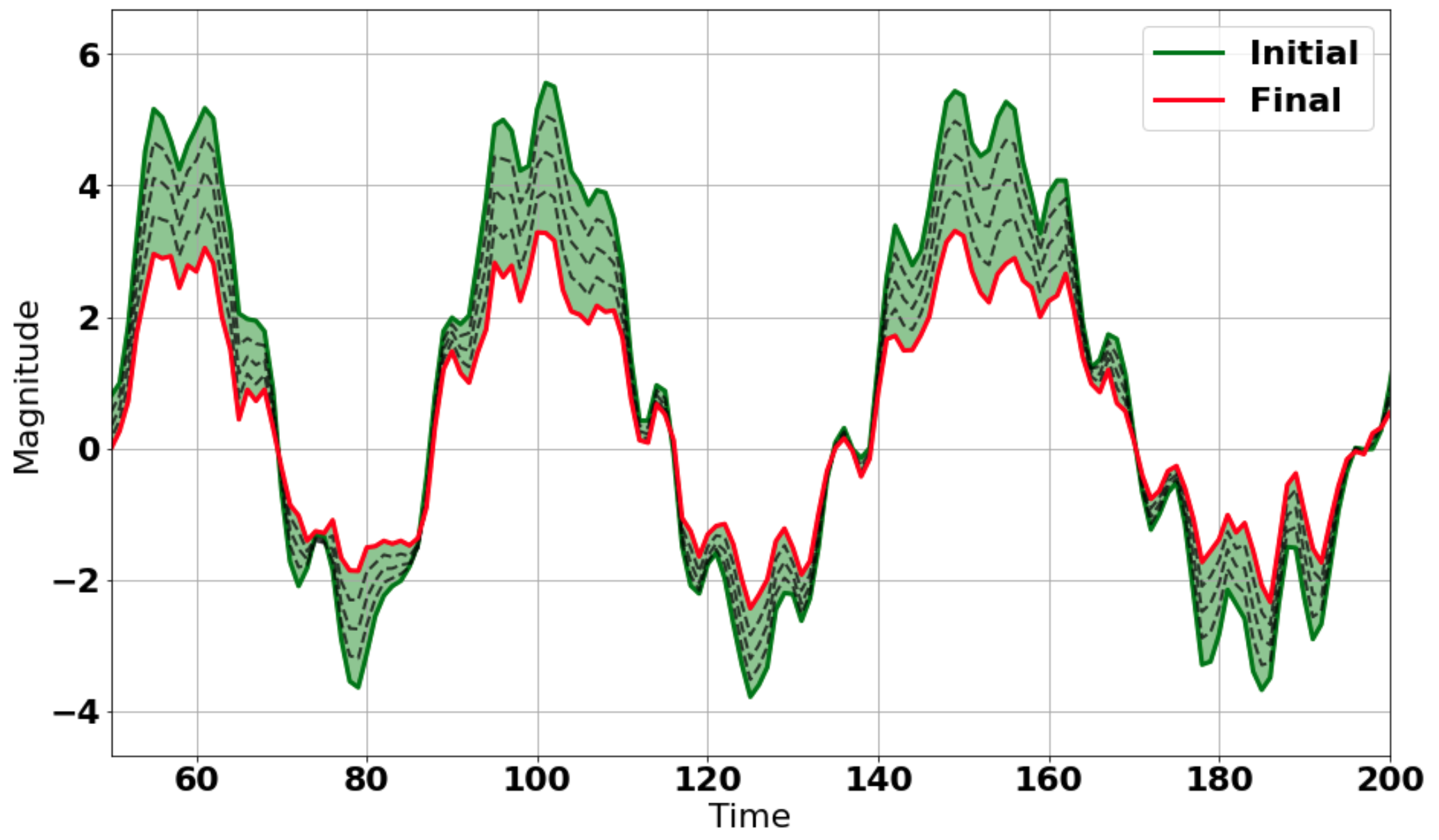}}
   \hspace{0.01\textwidth} 
  \subfigure[Several signal plots which aim to make sure that the final perturbed signal cannot be misconstrued in any way as being healthy.\label{fig:shm_adversarial_d}]{\includegraphics[width=0.45\textwidth]
  {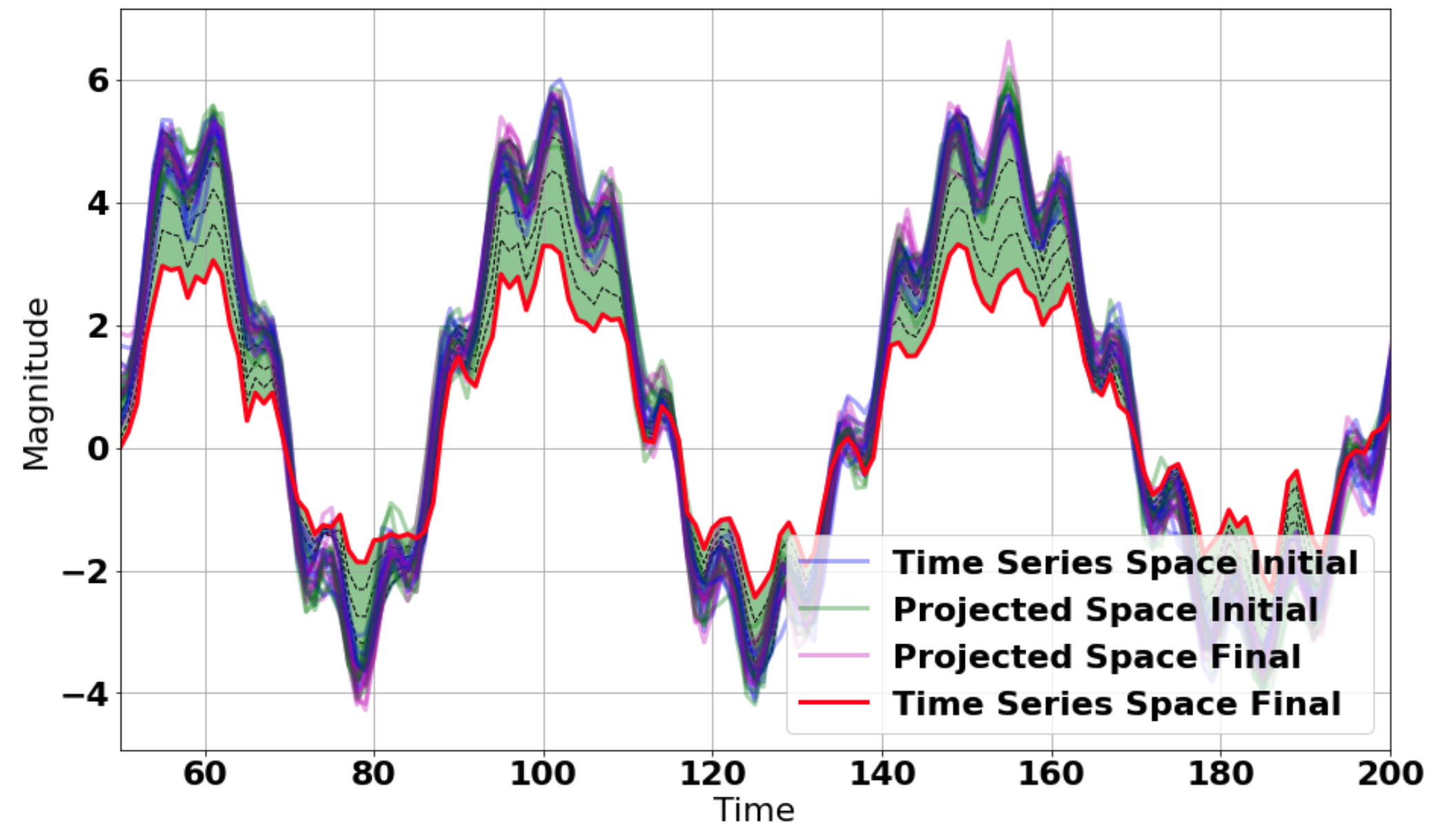}}
  \caption{Plots which show how to identify, generate, and analyse adversarial time series signal using StiefelGen for the SHM case study.}
  \label{fig:shm_adversarial}
\end{figure}

Finally, in Figure \ref{fig:shm_adversarial_d}, we aim to conclusively demonstrate the adversarial nature of the generated signal by presenting, (i) \textit{Time Series Space Initial}: The initial appearance of the signal before perturbation, (ii) \textit{Projected Space Initial}: An overlay of the ten closest time series signals relative to the projected data point under study, and (iii) \textit{Projected Space Final}: An overlay of the ten closest time series signals corresponding to the data point under study in projected space \textit{after} it has completed its full geodesic path. Figure \ref{fig:shm_adversarial_d} ensures there are no data points, whether in the original signal space or in the projected space, at the initial or final geodesic positions, resembling the generated adversarial signal. This reaffirms that the red time series signal in Figure \ref{fig:shm_adversarial_d} labeled \textit{Time Series Space Final}, is genuinely adversarial.

\subsection{\textit{StiefelGen}: Uncertainty Quantification and Linear Time Series Dynamics} \label{sec:UQ_DMD}

In this subsection, we shall delve into the compatibility of \textit{StiefelGen} for a joint application of uncertainty quantification (UQ) and dynamic mode decomposition (DMD) over spatio-temporal signals. Traditional UQ approaches in time series often involve creating prediction intervals that expand over time \cite{chatfield2001prediction, xiao2012time}. These methods often incorporate Bayesian analysis by providing a probabilistic framework over the signal by leveraging prior knowledge of the signal generating process \cite{kumar2008bayesian}. Monte Carlo methods are also frequently employed \cite{maybank2020mcmc}, and more recently, conformal predictions have gained attention due to its theoretical guarantees and distribution-free claims \cite{stankeviciute2021conformal}.

\subsubsection{Funtional Data Analysis}

At first glance, \textit{StiefelGen} may not seem to work well with conventional UQ methods. However, its strength lies in its ability to swiftly generate novel time series augmentations, with the required level of signal novelty scaled against the sampling distance away from the injectivity radius. Consequently, \textit{StiefelGen} is well-suited for a UQ-based approach to time series analysis through functional data analysis (FDA) \cite{wang2016functional}. FDA encompasses a suite of statistical methods designed for analyzing data that varies over a curve, surface, or continuum. In this context, observations are treated as \textit{functions}, emphasizing the analysis of datasets where the primary units of measurement are curves or functions, often observed over a continuous domain such as time. Mathematically we shall consider that \textit{StiefelGen} has generated a set of time series augmentations: $\bm{\mathcal{T}}(t) = \{\mathcal{T}_i(t)\}_i^K$, for $K$ possible input time series, defined over a closed interval $t\in [0,1]$. Using this construction, $\{\mathcal{T}_i(t)\}_i$ may be interpreted as a collection of stochastic random variables. Then, assuming a finite square integrability condition as: $\mathbb{E}\left(\int_0^1 |\bm{\mathcal{T}}(t)|^2 \,dt\right) < \infty$, one can generalize typical statistical quantities over function spaces. For example, the mean function of $\bm{\mathcal{T}}$ can be taken to be as $\mathbb{E}\langle \bm{\mathcal{T}},h\rangle = \langle\mu, h\rangle$, where $\mu$ is unique, and $\mu, h\in \mathcal{H}$, are separable Hilbert spaces of square-integrable functions \cite{kokoszka2017introduction}.

In employing FDA for UQ, we extend the conventional statistical box plot to \textit{functional box plots} \cite{sun2011functional} as part of our approach. Specifically, we leverage \textit{StiefelGen} to swiftly generate diverse instances of the provided reference uni-variate time signal. Subsequently, a functional box plot is constructed to encapsulate key descriptive statistics, such as the median signal and the envelope of the 50\% central region. This methodology becomes evident when applied to the SteamGen dataset. Figure \ref{fig:UQ_partone_a} illustrates the rapid generation of 500 time series signals, employing a scaling factor of $\beta=0.3$, and identical $m$ (50) and $n$ (40) reshape values as outlined in Section \ref{sec:overview}. The resulting functional box plot is illustrated in Figure \ref{fig:UQ_partone_b}.

\begin{figure}[htbp]
  \centering
  \subfigure[Over plotting of multiple generated signals based on the original reference signal. \label{fig:UQ_partone_a}]{\includegraphics[width=0.43\textwidth]
  {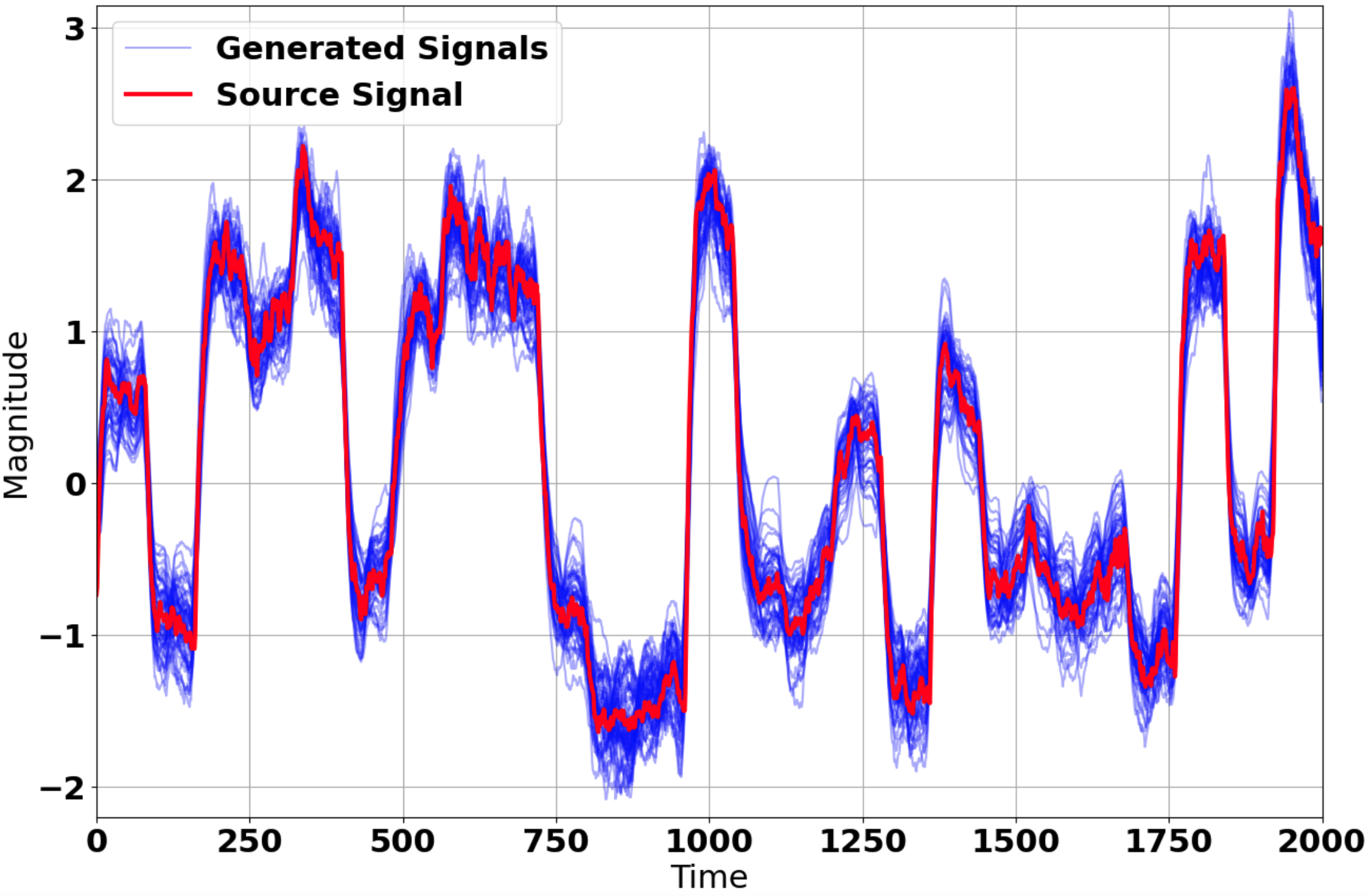}}
 \hspace{0.01\textwidth} 
  \subfigure[An example of plotting the 50\% central envelope of the functional box plot around the median signal.\label{fig:UQ_partone_b}]{\includegraphics[width=0.45\textwidth]
  {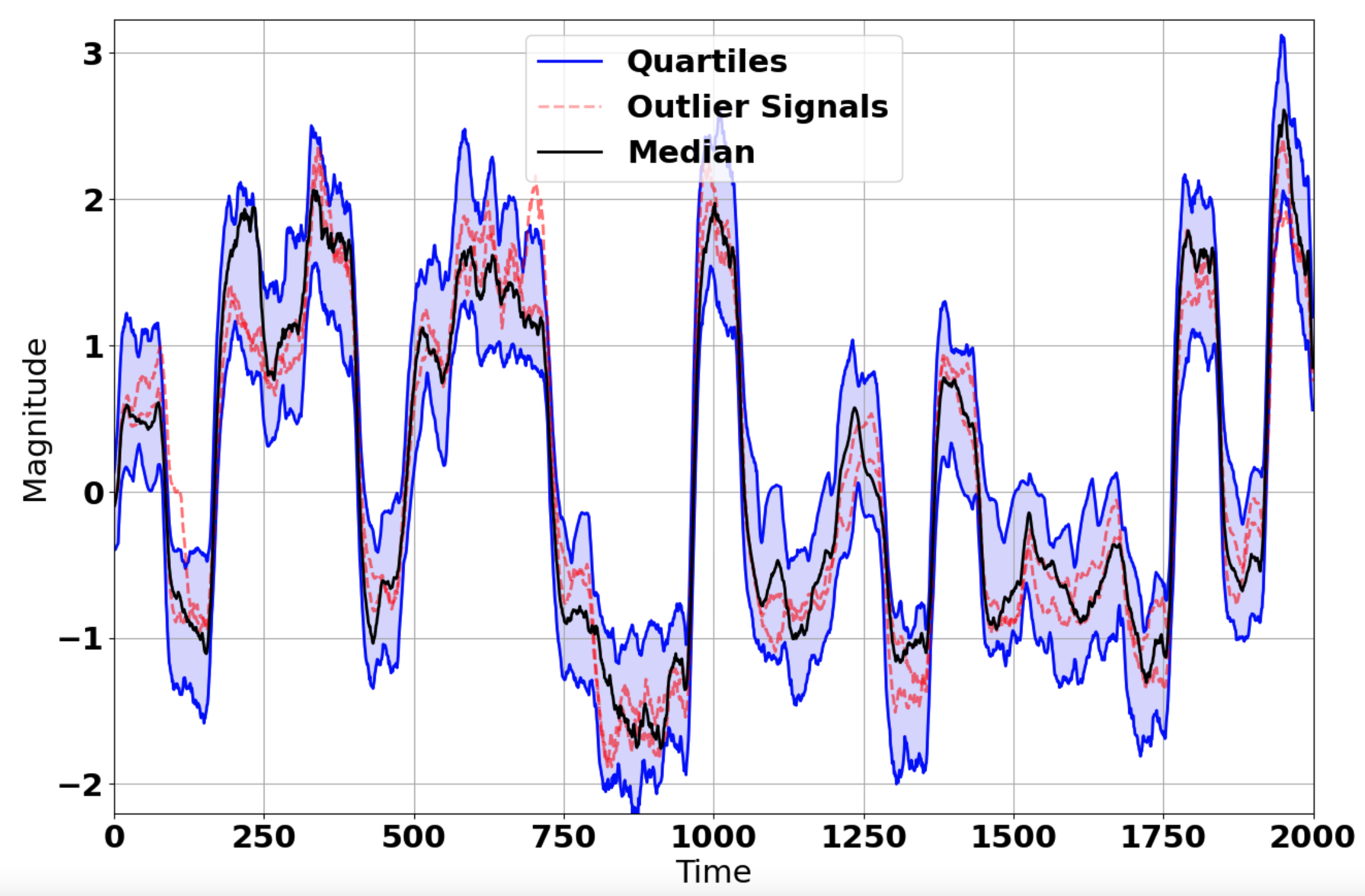}}
  \caption{An example of applying FDA to allow for \textit{StiefelGen} to be used for the purpose of UQ.}
  \label{fig:FDA_example}
\end{figure}

In order to create Figure \ref{fig:UQ_partone_b}, the sci-kit fda library was utilized \cite{ramos2022scikit}, and adopted the modified band depth (MBD) approach for functional box plot generation. This method was chosen for its simplicity, interpretability, and parameter-free nature. Briefly, the MBD computation relies on the arrangement of points concerning a reference point or curve in the data set. Typically, the median function serves as the reference point when aiming to calculate the functional box plot, offering a natural choice for assessing depth centrality within a distribution \cite{lopez2009concept}. Given this reference point, MBD then calculates the proportion of curves inside bands surrounding the target curve, providing a depth measure that reflects its position relative to the distribution of functional data in the dataset. Finally, it is crucial to emphasize that functional data lacks a universally accepted definition for outliers. This ambiguity arises because a curve can be identified as an outlier based on different criteria, including but not limited to substantial distance from the mean (magnitude outlier) or a distinctive pattern compared to other curves (shape outlier). In essence, a curve is typically classified as an outlier if it stems from a distinct underlying process \cite{lopez2009concept}.

The ability for \textit{StiefelGen} to perform UQ in this manner will be shown to be highly beneficial when implemented along side dynamic mode decomposition (DMD) for the purpose of perturbing system dynamics to obtain a multitude of different spatio-temporal signal evolutions.

\subsubsection{Dynamic Mode Decomposition}

Dynamic mode decomposition (DMD) originated within the fluid dynamics community as a means to break down intricate flows into a simplified representation based on spatio-temporal coherent structures \cite{schmid2010dynamic}. The increasing success of DMD over time can be attributed to its equation-free, data-driven nature, allowing for an accurate deconstruction of complex systems into spatio-temporal coherent structures \cite{kutz2016dynamic}, as the DMD algorithm can be leveraged for short-term future-state prediction and control \cite{lu2020prediction, yuan2021flow, grosek2014dynamic}. Algorithmically, DMD relies upon collecting a set of data snapshots, denoted as $x_k$, from a dynamical system at various time points, where $k \in \mathbb{N}.$ DMD then performs a regression of these snapshots over locally linear dynamics, represented as, $x_{k+1} = Ax_k$, where the $A$ matrix holds the locally linear system physics, and is chosen to minimize $\|x_{k+1} - Ax_k\|_2$. DMD's  advantages lie in its simplicity of execution and the minimal assumptions it imposes on the underlying system, which parallels the approach taken by \textit{StiefelGen}. It has also seen a recent re-surge of interest due to its close relationship to the \textit{Koopman operator} theory \cite{kutz2016dynamic}. The reason why \textit{SteifelGen} readily works with the DMD framework can be seen if one inspects its mathematics. Assume one has two separate snapshots of input data as:
\begin{align*}
    \bm{\mathcal{T}_1} &= [\mathcal{T}_1, ..., \mathcal{T}_{N-1}]\\
    \bm{\mathcal{T}_2} &= [\mathcal{T}_2, ..., \mathcal{T}_{N}]
\end{align*}
As made clear, these snapshots are assumed to vary according to a \textit{linear} dynamical system with aleatoric uncertainty orthogonal :
\begin{align*}
    \mathcal{T}_{1,i+1} = A \mathcal{T}_{1,i} 
\end{align*}
for some state space matrix, $A$. Written in its matrix form and adding aleatoric uncertainty, $\bm{\mathcal{T}_2} = A \bm{\mathcal{T}_1} + \varepsilon^{\intercal}$. Now the manner in which \textit{StiefelGen} arises, is that the base DMD algorithm requires calculating the SVD over $\bm{\mathcal{T}_1}$, leading to: $\bm{\mathcal{T}_2} = A U \Sigma V + \varepsilon^{\intercal}$. Assuming that the residuals $\varepsilon$, remain orthogonal to the minimizing basis found through DMD (often termed as the proper orthogonal decomposition modes \cite{chatterjee2000introduction}), one can left-right multiply the matrix, $A$, in order to discover the quintessential DMD relationship: $U^{\intercal} A U = U^{\intercal} \bm{\mathcal{T}_2} V \Sigma^{-1} = \tilde{S}$. Now, since $A$ and $\tilde{S}$ are related through a \textit{similarity transform} \cite{kutz2016dynamic}, one can find the eigenbasis for $\tilde{S}$, in order to obtain precisely the required eigenbasis for $A$. If one were to use \textit{StiefelGen} to \textit{perturb} the $U$ and $V$ matrices in this similarity transform, then what one is doing is \textit{perturbing} the linear dynamics of the system (encapsulated in the $A$ matrix), thereby projecting the spatio-temporal signal ahead in time through a multitude of pathways.

As $A$ and $\tilde{S}$ are connected through a similarity transform, finding the eigenbasis for $\tilde{S}$ allows one to precisely determine the required eigenbasis for $A$.  Thus information regarding the linear state evolution matrix $A$, can be directly obtained simply be observing the empirical relationship between $\bm{\mathcal{T}_1}$, and $\bm{\mathcal{T}_2}$. Ultimately, since the SVD is used to determine the evolutionary dynamics of the spatio-temporal system, \textit{StiefelGen} may be used to perturb the $U$ and $V$ matrices within this similarity transform, leading to different spatio-temporal path ways of the signal. In effect this imbues the dynamics with a non-trivial form of epistemic uncertainty. This integration presents a noteworthy extension of \textit{StiefelGen}'s utility beyond simply conventional time series data augmentation.

In order to demonstrate how this works we utilize the example shown in Section 1.4 of the textbook by Kutz et al. \cite{kutz2016dynamic} which involves mixing two spatio-temporal signals over the complex domain, clarified in Equation \ref{eqn:spatiotemporal}.
\begin{align}\label{eqn:spatiotemporal}
    f(x,t) &= f_1(x,t) + f_2(x,t) \\ \nonumber
    &= \text{sech}(x+3)\exp(i2.3t) + 2\text{sech}(x)\text{tanh}(x) \exp(i2.8t),
\end{align}
where the use of two separate frequencies, $\omega=2.3$ and $\omega=2.8$ allows for distinct spatial structures to arise \cite{kutz2016dynamic}. It is crucial to emphasize that Equation \ref{eqn:spatiotemporal} functions solely as the ground truth observation, initiating the iterative data-driven process of DMD. Furthermore, in alignment with the methodology outlined in \cite{kutz2016dynamic}, we adopt the reduced rank versions of matrices $U$ and $V^{\intercal}$, specifically retaining only the first two columns after performing the manifold perturbation. Lastly, it's worth observing that the integration of \textit{StiefelGen} into DMD does not necessitate any smoothing or rearrangement of the page matrix, re-emphasizing the lack of need for hyper-parameter selection. The only hyper parameter technically required would once again be that of $\beta$ which was chosen to be 0.2 without loss of generality.

To illustrate the impact of \textit{StiefelGen} on the spatio-temporal evolution, we generated a visual representation by plotting the ground truth signal against a diverse set of solution pathways. These pathways correspond to different perturbed $A$ matrices, as depicted in Figure \ref{fig:DMD_Frames_no_FDA}. The real component of the spatio-temporal signal is presented over a 4$\pi$ seconds interval for clarity. To ensure a varied degree of spacing across the 200 iterations, an exponential spacing scheme was employed for the ``frames''. This choice results in closely spaced initial frames, gradually growing apart, effectively highlighting the increasingly deviant paths taken by the perturbed signals.

As observed in Figure \ref{fig:DMD_Frames_no_FDA}, the perturbed signals consistently align with the overall modal behavior and exhibit smoothness properties akin to those of the original signal. This smoothness arises naturally from the inherent properties of the Stiefel manifold, thereby offering an empirical validation of the influence exerted by \textit{StiefelGen} on the spatio-temporal signal evolution.

\begin{figure*}[htbp]
  \centering
  {\includegraphics[width=0.9\textwidth]{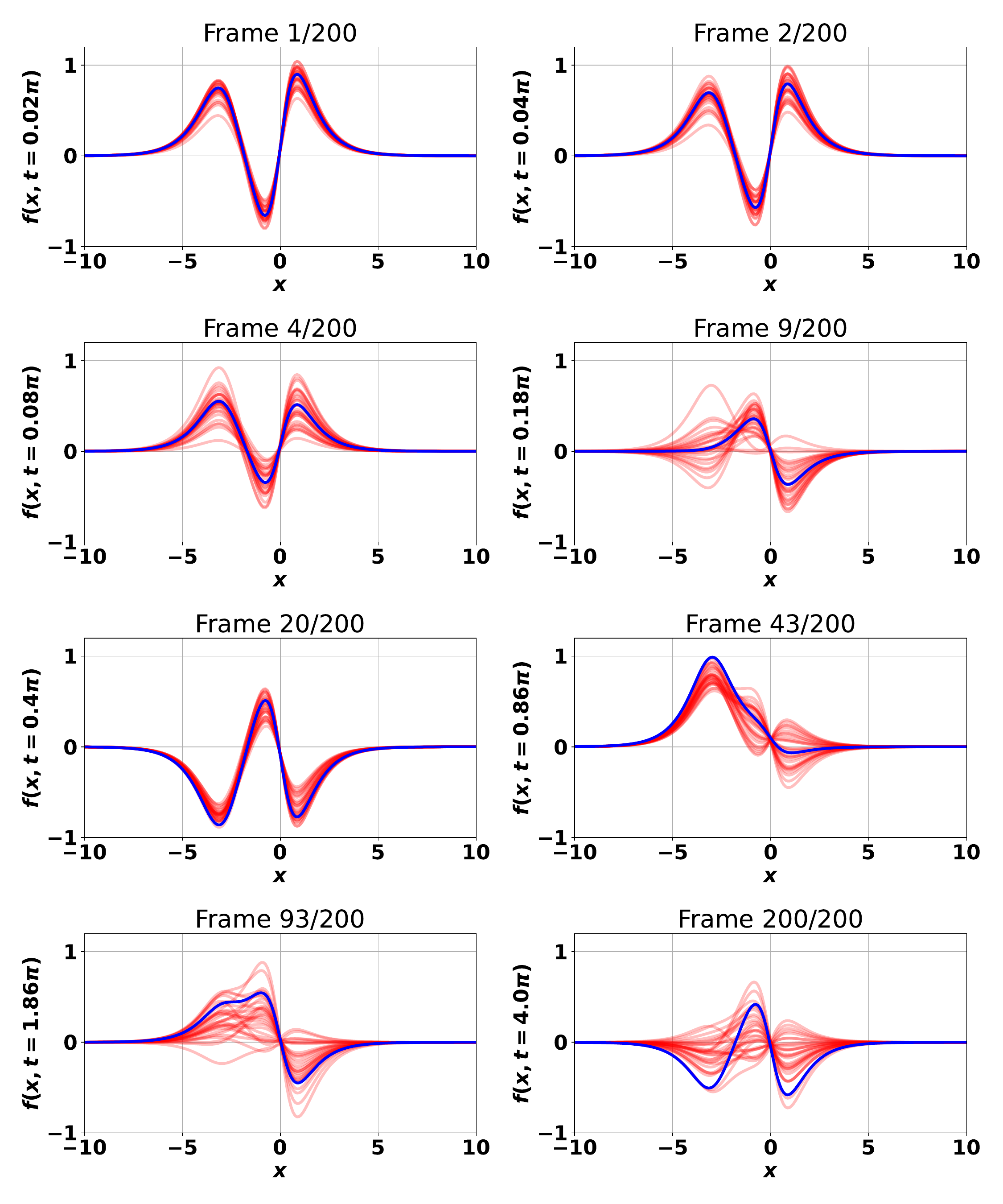}}
  \caption{A plot of different spatio-temporal states of the ground truth signal (blue), against thirty separate solution path ways for perturbed $A$ matrix with $\beta=0.2$.\label{fig:DMD_Frames_no_FDA}}
\end{figure*}

Expanding on this analysis to include a temporal functional box plot for the spatio-temporal signal, Figure \ref{fig:DMD_Frames_FDA} illustrates the effective tracking of the 50\% central region over time. To showcase the versatility of the approach and highlight the dynamic emergence of functional outlier data, an additional outer envelope representing the 75\% central region was chosen. This envelope not only underscores the flexibility of the method but also demonstrates the dynamic nature in which functional outlier data may manifest, as defined relative to the aforementioned Modified Band Depth (MBD) method. 

\begin{figure}[htbp]
  \centering
  {\includegraphics[width=0.9\textwidth]{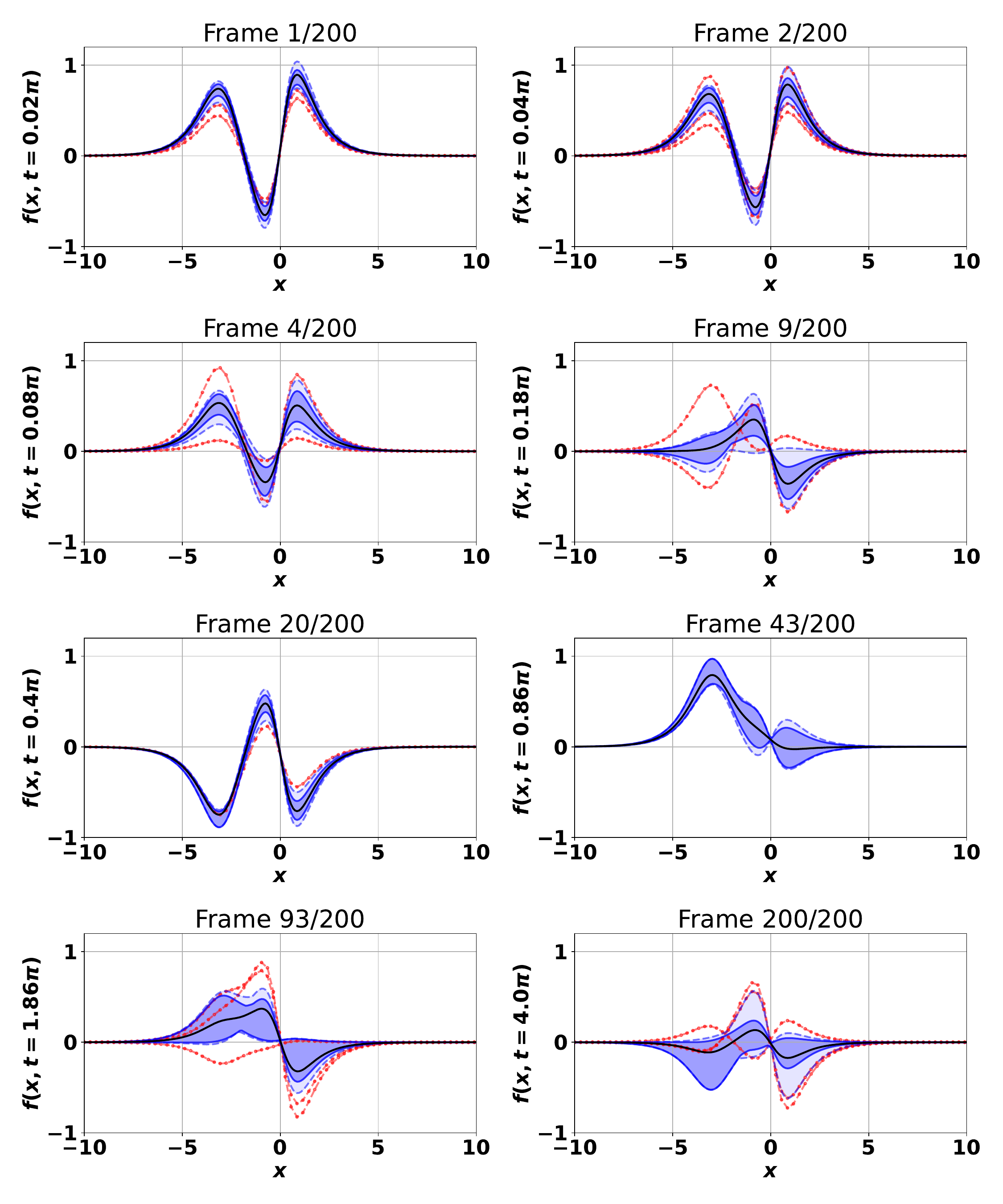}}
\caption{A plot of different spatio-temporal states of the functional box plot. The median signal (black), is encompassed by the 50\% central region (dark blue shading with solid blue boundary), as well as the 75\% envelope (ligh blue shading with dashed blue lines). The outlier signals are shown in red, and are defined to be those signals outside of the inter-quartile range.\label{fig:DMD_Frames_FDA}}
\end{figure}

Ultimately, the \textit{StiefelGen} algorithm appears to seamlessly integrate into the DMD framework, facilitating the exploration of induced epistemic uncertainty levels within linear dynamical models in a manner which respects the underlying geometry of its SVD step. This integration holds promise as a valuable data augmentation tool for projecting time series forward. Nevertheless, it's crucial to acknowledge the current limitations of DMD in projecting forward sets of signals with pronounced nonlinear behavior, excessive chaos, or non-stationary properties \cite{kutz2016dynamic}. A notable effort within the DMD community to address some of these limitations include multi-resolution DMD, a recursive, hierarchical extension of the conventional DMD algorithm, designed to mitigate drawbacks by separating microscale and macroscale effects \cite{kutz2016multiresolution}. 

\subsection{On Neural Networks}

Finally, in addition to the unique ways in which \textit{StiefelGen} can be used in many applications, we explore its conventional, implied usage as a synthetic data generator for the purpose of increasing the amount of overall training data of a deep, network model. In order to investigate this, we constructed an LSTM model for the purpose of sequence data classification for the \textit{Japanese Vowels} dataset from the UCI machine learning repository \cite{kudo1999multidimensional}. This dataset is a multidimensional dataset in which nine male speakers uttered two vowels in sequence ('a' and 'e') reflective of a phonetic dipthong, /ae/ that is common in Japanese. On each utterance of the dipthong, a ``12 degree linear prediction analysis'' \cite{atal1971speech} was performed by splitting each utterance waveform into a 12 dimensional feature vector (12 separate coefficients)  input per utterance,  where each individual dimension (coefficient) may be anywhere from 7 to 29 units long. The total number of time series (experimental observations) used in this study was 640, with 270 time series for training (30 measurement instances spread amongst the 9 speaker classes) and the other set of 370 time series for testing (24 - 88 measurement instances for each class, spread this time non-equally amongst the same 9 speaker classes). Further information such as sampling rate, frame lengths, and shift lengths (which are not relevant to the current discussion) may be found online \cite{misc_japanese_vowels_128}.

Expanding on the versatile applications of \textit{StiefelGen}, we delve into its more traditional yet inherently valuable role as a synthetic data generator. To explore this aspect, we employed a Long Short-Term Memory (LSTM) model designed for sequence data classification, utilizing the \textit{Japanese Vowels} dataset from the UCI machine learning repository \cite{kudo1999multidimensional}.

The \textit{Japanese Vowels} dataset comprises multidimensional features extracted from the utterances of nine male speakers producing the phonetic diphthong /ae/. The dataset involved performing a ``12 degree linear prediction analysis'' \cite{atal1971speech}, breaking down each utterance waveform into a 12-dimensional feature vector. Further, the length of this 12-dimensional feature vector varied in from 7 to 29 units depending on the recording. Mathematically, each observation, $\bm{x}_i$ has matrix dimension $\bm{x}_i\in\mathbb{R}^{12 \times k}$, where $k\in \mathbb{N}_{[7,29]}$. Further clarification is available at \cite{matlab_LSTM}. In order to combat the varying $k$, each $\bm{x}_i$ was padded to the same total length, so that $\bm{x}_i=\bm{x}_j$ for all $i,j$ in the training and testing sets respectively. In total, 640 time series observations were used in the experiment. Of these, 270 time series were allocated for training, with 30 measurement instances each distributed among the 9 speaker classes. The remaining 370 time series were designated for testing, with 24 to 88 measurement instances for each class (thus in the testing data the experimental observations were distributed non-uniformly across the same 9 speaker classes). Additional details, such as sampling rate, frame lengths, and shift lengths (though not central to the current discussion), can be found online \cite{misc_japanese_vowels_128}.

Given that \textit{StiefelGen} operates primarily through a perturbation factor/percentage (where $\beta=1\rightarrow 100\%$ implies lying directly on the radius of injectivity), a comprehensive assessment of \textit{StiefelGen} necessitates exploring the impact of varying this percentage value. The goal is to generate additional data that is "similar enough" to the input data for augmentation. To achieve this, we focus on smaller to moderate percentage ranges, avoiding training on strong outlier behavior. Hence, we selected percentage perturbation factors ($\beta\times 100)$ of: $[0\%, 5\%, 10\%, 15\%, 30\%]$, where 0\% signifies the default operating condition (applied perturbation). The increasing percentages reflect an escalation in variance within the augmented datasets. It's noteworthy that, since we are working with a multidimensional dataset, there is no need for any \textit{reshape} operation on the input data, as it naturally stacks in a matrix format (each $\bm{x}_i$ is already a matrix ). Additionally, we did not introduce any additional smoothing, as the data sequences are relatively short on a per-experiment/observation basis and exhibit minimal noise. 

The nested structure of this experiment unfolds as follows:
(i) Investigating the effects of taking the first $n$ samples per class, which allows the experiment to explore low data cardinality settings ($n=\{5,10,15,20,25,30\}$ where $n=30$ implies utilizing the entire training dataset consisting of 30 observations across 9 classes). (ii) Selecting the amount of additional data to generate, acting as the synthetic data multiplier over the $n$ samples ($\text{gen} = \{5,10,15,20,25,30\}$). For instance, choosing $n=10$ samples per training class and opting to generate $\text{gen}=15$ times more data results in working with $n×\text{gen}=10×15=150$ synthetic data samples, across each of the 9 classes (The original $n=15$ data samples used for generating synthetic data are entirely discarded for simplicity so that one doesn't need to work with $15+150=165$ augmented data samples per class). (iii) Testing various levels of percentage perturbations whilst iterating through points (i) and (ii) concurrently. For each new percentage perturbation factor, an LSTM model was trained and its testing accuracy scores were grouped 20 random seed iterations. Naturally, this allows for the generation of reasonable standard error uncertainty bar estimates and mean value estimates per experiment. The overall nested structure of the experiment's flow is visually represented in Figure \ref{fig:big_flow_chart}.

\begin{figure}[htbp]
  \centering
 {\includegraphics[width=1\textwidth]
  {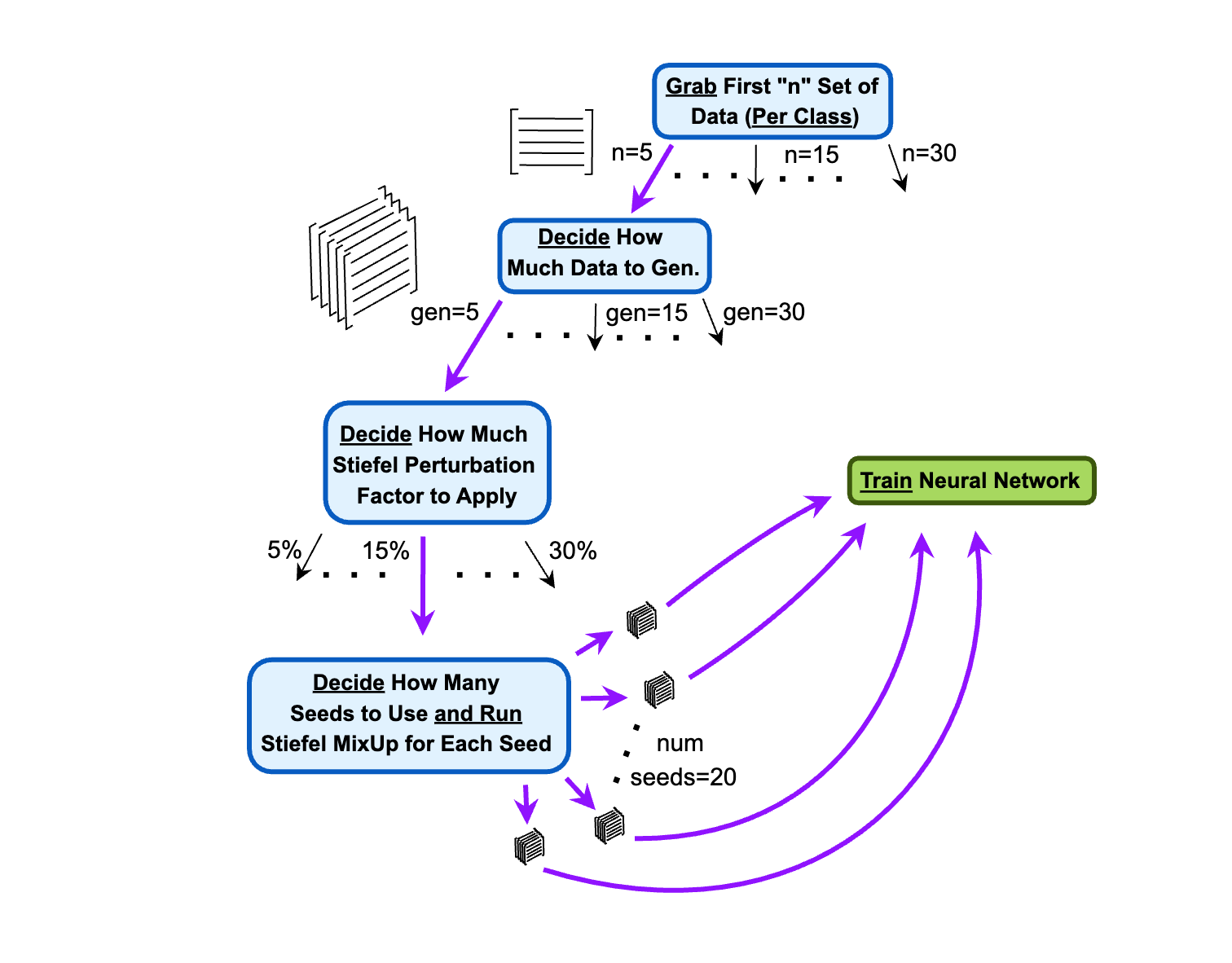}}
  \caption{A diagram showcasing the nested workflow of experiments designed to validate the effectiveness of \textit{StiefelGen} for a numerical time series classification problem. The primary objective is to address the necessity for data augmentation in this context, ultimately enhancing the learned capacity of the LSTM model for improved performance.}
  \label{fig:big_flow_chart}
\end{figure}

Concerning the architecture of the complete end-to-end LSTM model, the initial layer features a bidirectional LSTM with 100 hidden units in total \cite{graves2005framewise}. This was followed by a fully connected layer with 9 output units (matching the number of classes) and a subsequent softmax layer. The chosen loss function was cross-entropy, and a mini-batch size of 64 was selected without loss of generality, and the training time spanned 50 epochs in total. These hyper parameters were selected as empirical evidence suggests they can lead to an LSTM model with ample learning capacity to achieve high accuracy scores on the given dataset (with testing accuracy scores above that of 97\%) \cite{matlab_LSTM}.

However, since the primary aim is to illustrate that the proposed data augmentation method through \textit{StiefelGen} can significantly boost training and testing accuracy scores, the learning rate for teh ADAM optimizer in \textit{PyTorch} \cite{kingma2014adam,paszke2019pytorch} was selected to be 0.001. This choice ensures that the \textit{PyTorch} LSTM model does not attain its full learning capacity. Although a comparable learning rate was used in \cite{matlab_LSTM} in its MATLAB\textsuperscript{\textregistered} implementation, it is essential to note that the default hyper parameter inputs of the ADAM function differ significantly across the \textit{PyTorch} and MATLAB\textsuperscript{\textregistered} implementations, which leads to vastly different end states of the 50 epoch learning process across both implementations. Consequently, equivalent ADAM learning rates do not necessarily imply reaching an equivalent capacity LSTM model within the specified 50 epochs, which in the case of \textit{PyTorch} for this experiment arrives at an under capacity LSTM model given the chosen random seed. The empirical results of performing this nested set of experiments is show in Figure \ref{fig:NN_results}. To facilitate discussion, the Table \ref{tab:times1_main} is shown to summarize the quantitative results of the vanilla, un-augmentated data case. Furthermore, Table \ref{tab:30times_main} (corresponding to Figure \ref{fig:NN_results_f}) is presented which contains the best possible training and testing scores that approximately doubles the scores of the vanilla model. All of the quantitative results have been presented for further reading in Appendix \ref{app_sec:LSTM_Table}. 

\begin{table}[htbp]
\centering
\caption{The result of using the first $N$ elements from the each class of the data set, where $N=[5,10,15,20,25,30]$. This is the base reference set of training and testing accuracies which have not received any data augmentation.}
\label{tab:times1_main}
\resizebox{\textwidth}{!}{%
\begin{tabular}{cccccccc}
\hline
\textbf{Pert. Level {[}\%{]}} & \textbf{Accuracy {[}\%{]}} & \textbf{$5(\times 1)$} & \textbf{$10(\times 1)$} & \textbf{$15(\times 1)$} & \textbf{$20(\times 1)$} & \textbf{$25(\times 1)$} & \textbf{$30(\times 1)$} \\ \hline
\multirow{2}{*}{\textbf{0}} & \textbf{Train} & 32.78$\pm$1.45 & 38.67$\pm$1.83 & 42.26$\pm$2.43 & 49.17$\pm$1.95 & 49.13$\pm$2.40 & \textbf{50.02$\pm$2.22} \\
 & \textbf{Test} & 22.42$\pm$1.33 & 32.64$\pm$1.39 & 38.39$\pm$1.87 & 40.77$\pm$1.86 & 39.76$\pm$2.03 & \textbf{41.53$\pm$2.33} \\ \hline
\end{tabular}%
}
\end{table}

\begin{table}[htbp]
\centering
\caption{Results of applying \textit{StiefelGen} with various perturbation levels ($[5\%, 10\%, 15\%, 30\%]$, across various amounts of data taken from the dataset for generation $[5,10,15,20,25,30]$ for augmentation, assuming 30 times the amount of data is generated.}
\label{tab:30times_main}
\resizebox{\textwidth}{!}{%
\begin{tabular}{cccccccc}
\hline
\textbf{Pert. Level {[}\%{]}} & \textbf{Accuracy {[}\%{]}} & \textbf{$5(\times 30)$} & \textbf{$10(\times 30)$} & \textbf{$15(\times 30)$} & \textbf{$20(\times 30)$} & \textbf{$25(\times 30)$} & \textbf{$30(\times 30)$} \\ \hline
\multirow{2}{*}{\textbf{5}} & \textbf{Train} & 94.79 $\pm$1.70 & 93.94$\pm$3.04 & \textbf{97.14$\pm$1.08} & 87.02$\pm$4.49 & 93.76$\pm$2.43 & 88.88$\pm$2.76 \\
 & \textbf{Test} & 61.12$\pm$2.12 & 76.53$\pm$3.06 & \textbf{88.49$\pm$1.43} & 80.62$\pm$4.29 & 88.05$\pm$2.02 & 85.20$\pm$2.61 \\ \hline
\multirow{2}{*}{\textbf{10}} & \textbf{Train} & 93.73$\pm$2.11 & 94.98$\pm$1.46 & 95.60$\pm$1.58 & 89.96$\pm$3.15 & 94.28$\pm$1.84 & \textbf{96.35$\pm$1.01} \\
 & \textbf{Test} & 61.39$\pm$2.58 & 78.50$\pm$1.74 & 88.74$\pm$1.63 & 84.23$\pm$3.71 & 90.61$\pm$1.76 & \textbf{93.55$\pm$0.77} \\ \hline
\multirow{2}{*}{\textbf{15}} & \textbf{Train} & 92.28$\pm$2.22 & 96.79$\pm$1.27 & 94.24$\pm$1.71 & 91.29$\pm$4.44 & \textbf{94.09$\pm$1.43} & 93.16$\pm$1.49 \\
 & \textbf{Test} & 61.97$\pm$2.56 & 80.49$\pm$2.10 & 89.46$\pm$1.13 & 87.35$\pm$4.25 & \textbf{91.43$\pm$1.48} & 90.49$\pm$1.72 \\ \hline
\multirow{2}{*}{\textbf{30}} & \textbf{Train} & 93.67$\pm$1.57 & 91.20$\pm$1.38 & 88.14$\pm$1.65 & 88.01$\pm$2.29 & 87.68$\pm$1.57 & \textbf{90.47$\pm$0.88} \\
 & \textbf{Test} & 59.42$\pm$2.22 & 73.43$\pm$1.73 & 85.14$\pm$1.32 & 87.61$\pm$2.61 & 88.09$\pm$1.31 & \textbf{91.14$\pm$0.83} \\ \hline
\end{tabular}%
}
\end{table}

\begin{figure}[htbp]
  \centering
  \subfigure[Augmenting with $\text{gen}=5$ times the number of observed points. \label{fig:NN_results_a}]{\includegraphics[width=0.43\textwidth]
  {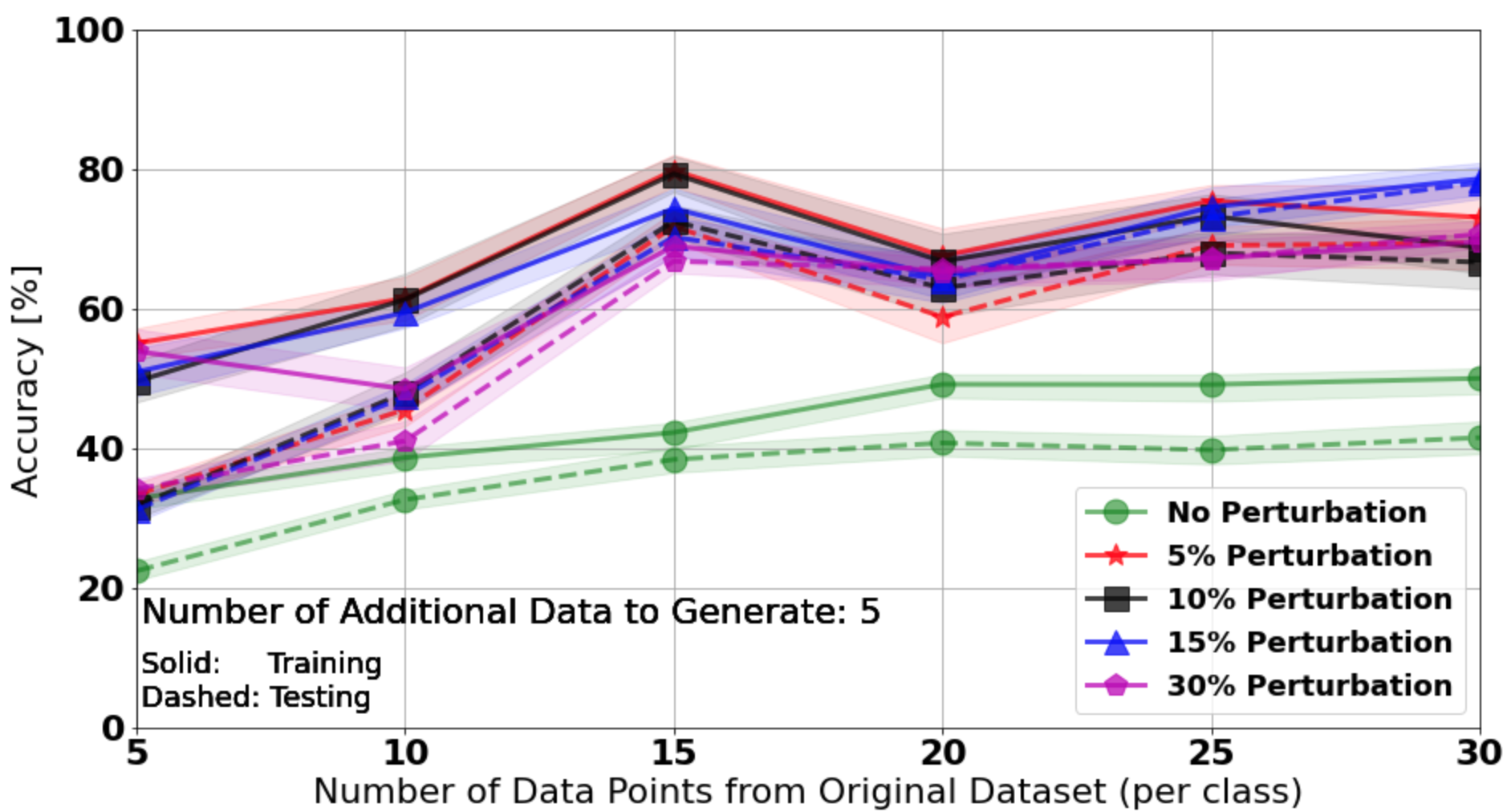}}
 \hspace{0.01\textwidth} 
  \subfigure[Augmenting with $\text{gen}=10$ times the number of observed points.\label{fig:NN_results_b}]{\includegraphics[width=0.45\textwidth]
  {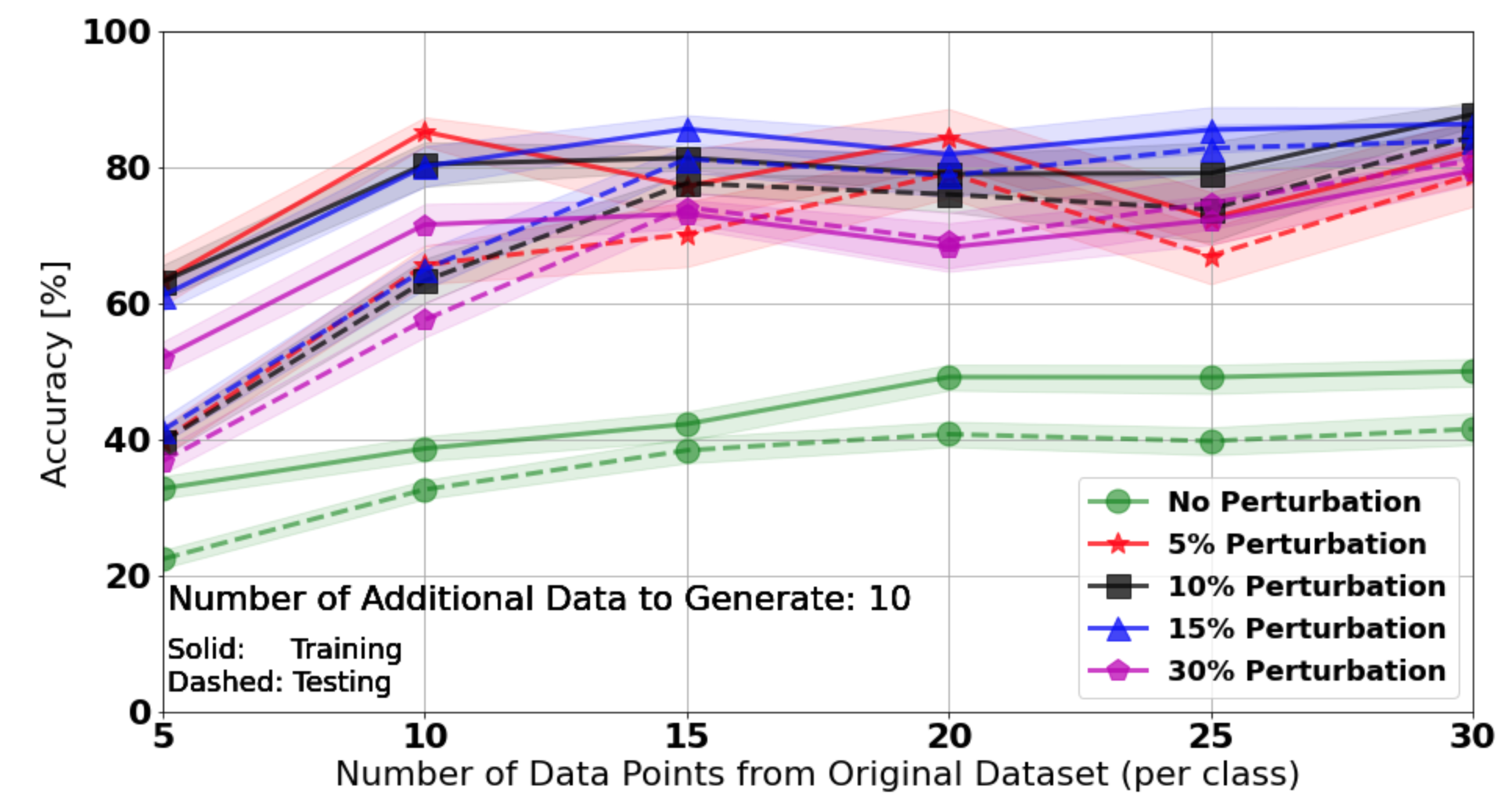}}

  \subfigure[Augmenting with $\text{gen}=15$ times the number of observed points.\label{fig:NN_results_c}]{\includegraphics[width=0.45\textwidth]
  {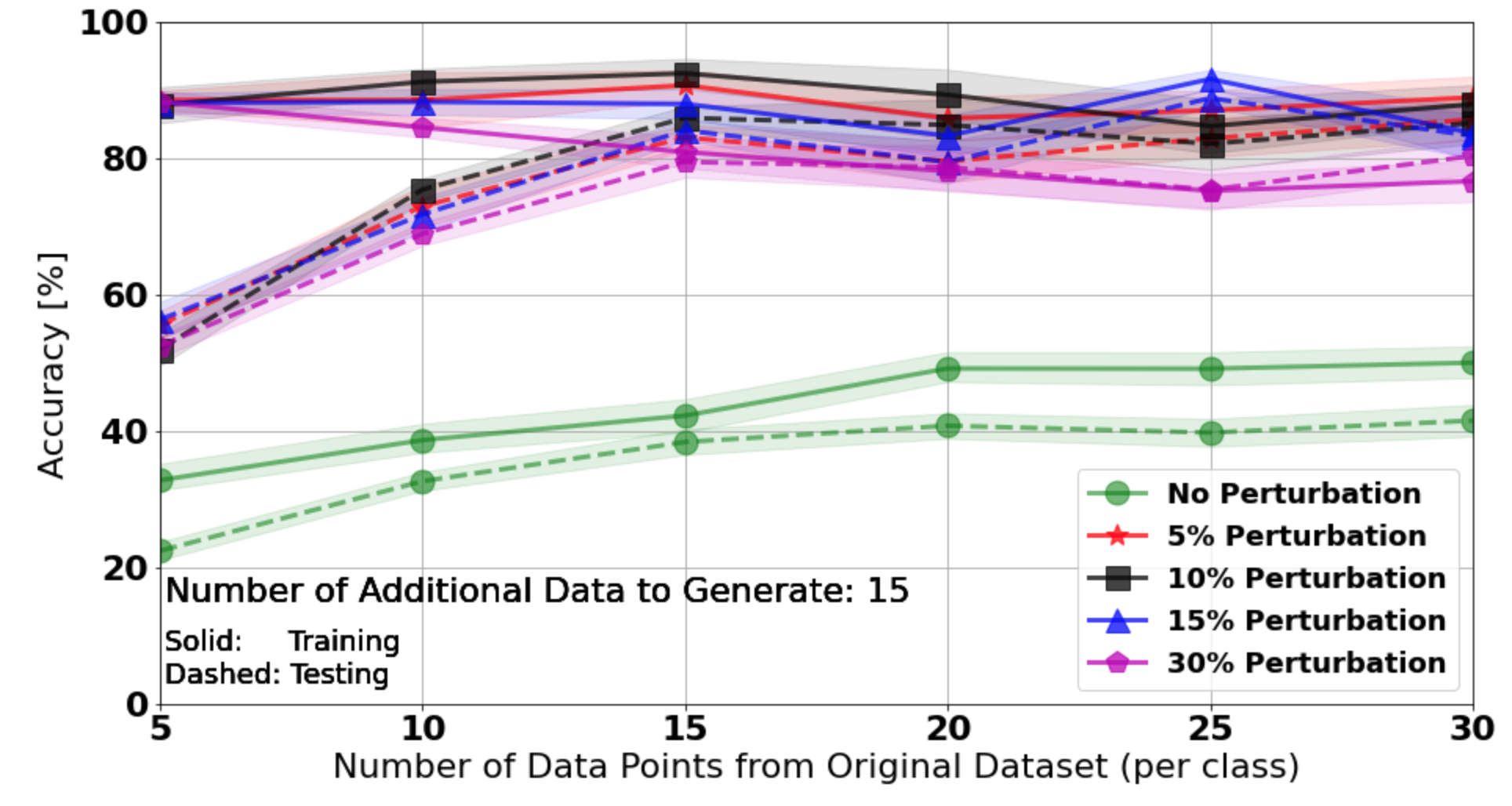}}
   \hspace{0.01\textwidth} 
  \subfigure[Augmenting with $\text{gen}=20$ times the number of observed points.\label{fig:NN_results_d}]{\includegraphics[width=0.45\textwidth]
  {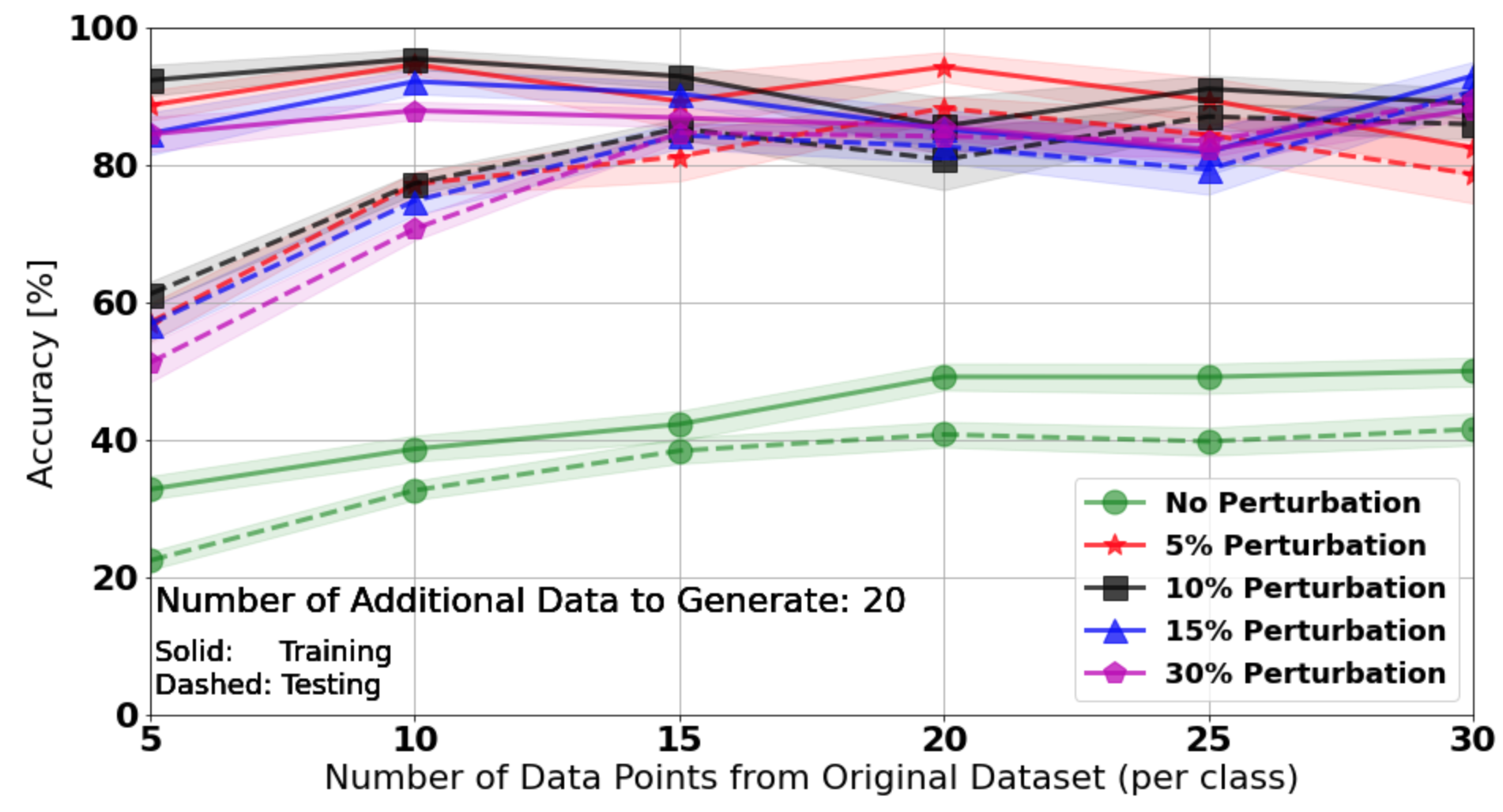}}

    \subfigure[Augmenting with $\text{gen}=25$ times the number of observed points.\label{fig:NN_results_e}]{\includegraphics[width=0.45\textwidth]
  {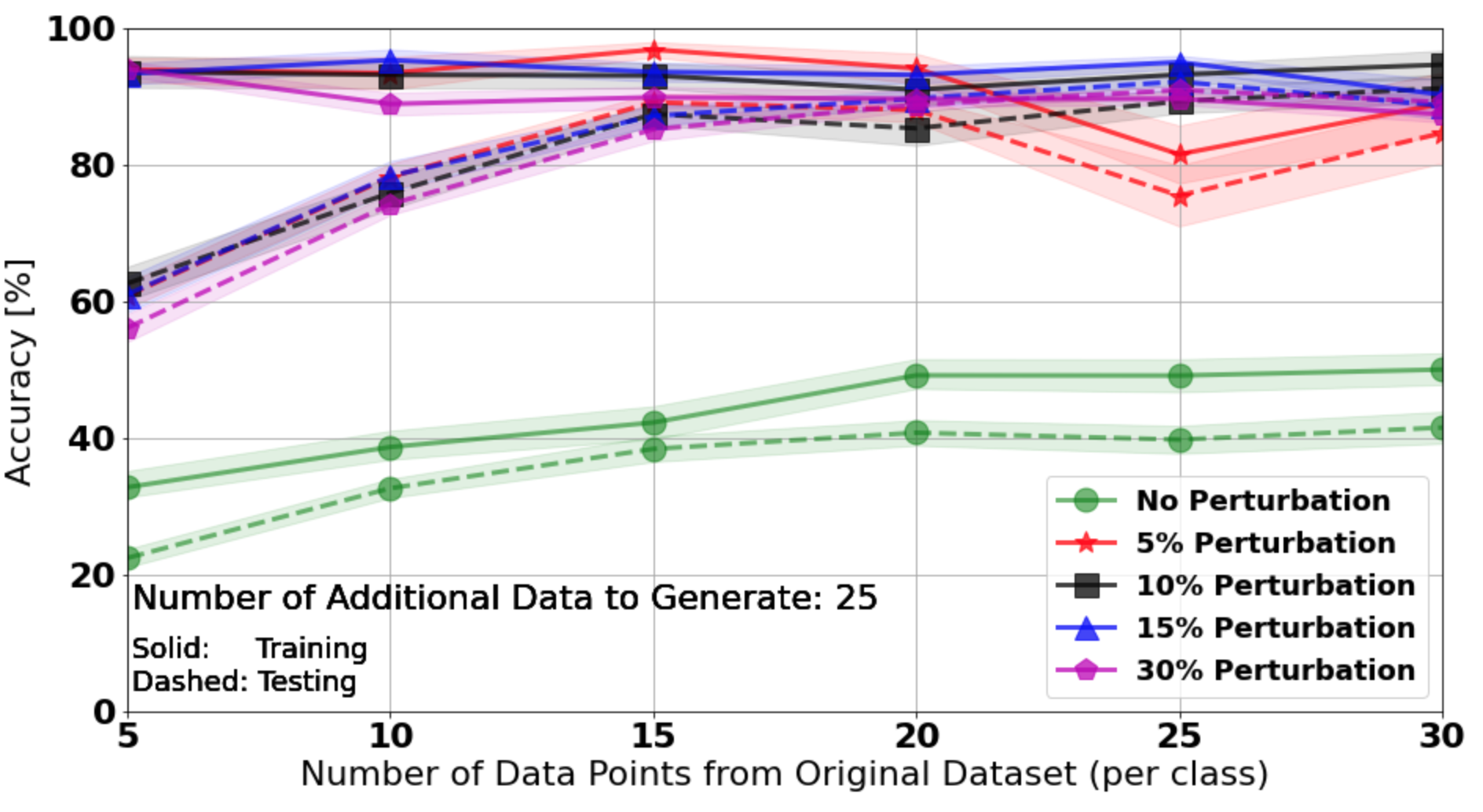}}
   \hspace{0.01\textwidth} 
  \subfigure[Augmenting with $\text{gen}=30$ times the number of observed points.\label{fig:NN_results_f}]{\includegraphics[width=0.45\textwidth]
  {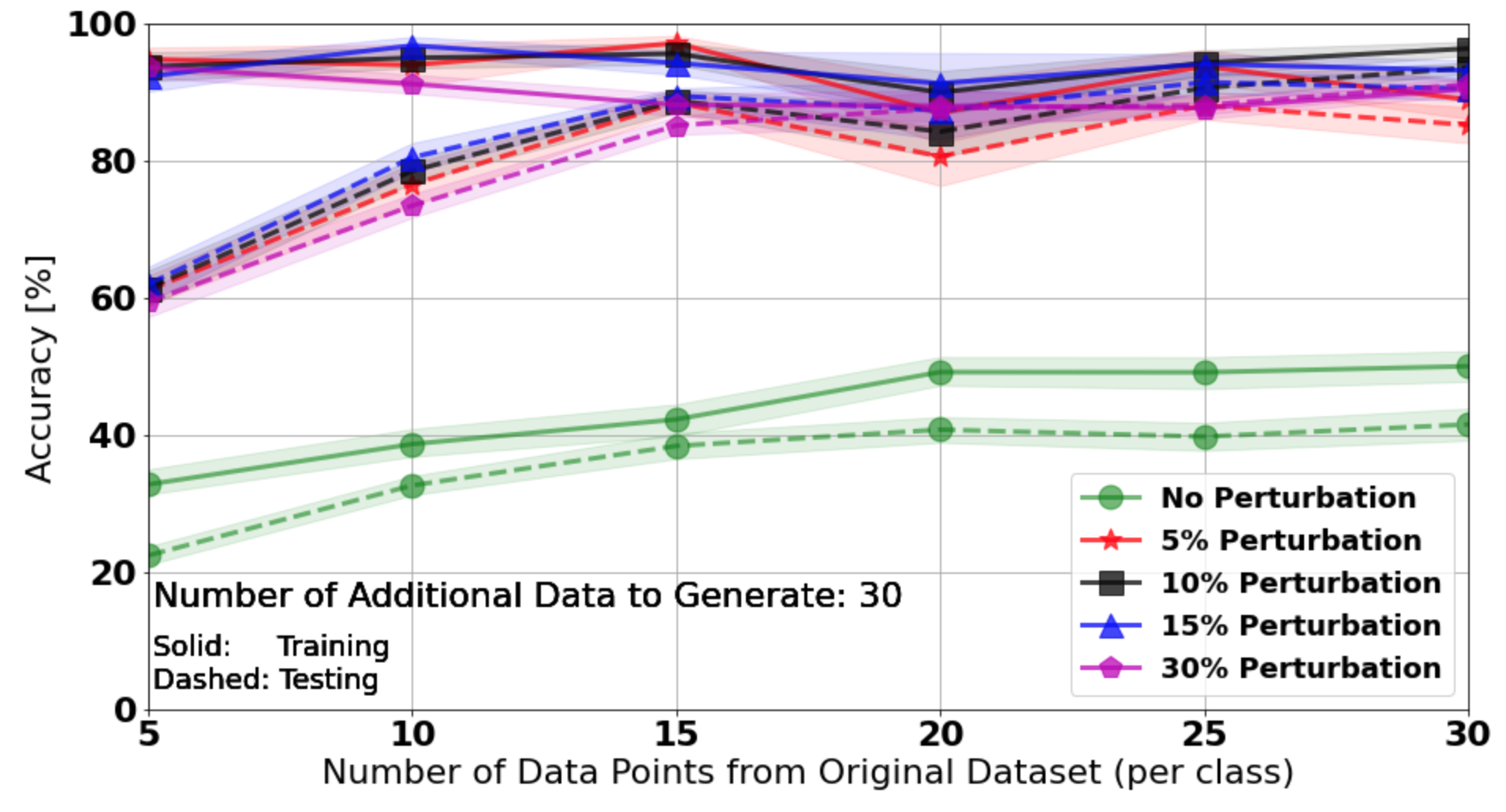}}
  \caption{Plots of the empirical results as the number of augmented datasets ($\text{gen}$) increases. The $x$-axis of each plot represents, $n$, the number of observed data taken from the training data set. Different levels of perturbation are shown on each plot, as well as training and testing accuracy results.}
  \label{fig:NN_results}
\end{figure}

Firstly, it should be noted that in Figure \ref{fig:NN_results} each subplot's data augmentation level ($\text{gen}$) does not apply to the green curves representing the original dataset. For instance, at $n=10$ on the $x$-axis in Figure \ref{fig:NN_results_a}, the green curve will only see $n=10$ observations per class (resulting in $10 \times 9 = 90$ data points in total) for LSTM model training. This will be the same for every subplot in Figure \ref{fig:NN_results}, as this curve represents the non-augmented, vanilla, reference model control, and is plotted across all plots for the sake of comparison. Conversely, each \textit{perturbed} curve in Figure \ref{fig:NN_results_a} will involve the generation of $n×\text{gen}=10×5=50$ augmented data points \textit{per class} for each perturbation level in Figure \ref{fig:NN_results_a}, resulting in $10 \times 5 \times 9 = 450$ data points in total for LSTM model training. Note that we do not have $450 + 15 = 465$ data points because the original data set was discarded after augmentation (for simplicity).

Looking at Figure \ref{fig:NN_results} holistically, it's evident that by applying the \textit{StiefelGen} augmentation procedure, one consistently outperforms the training and testing scores of the original dataset across all perturbation levels, for every augmentation generation level. Also, increasing the number of augmented datasets ($\text{gen}$) for a fixed $n$ tends to lead to a gradual rise in accuracy scores for both training and testing sets. However, this effect seems to plateau around $gen=15$, as the augmentation cases ($\text{gen}=\{15,20,25,30\}$) tend to display similar score magnitudes and overall shapes (from Figures \ref{fig:NN_results_c} to \ref{fig:NN_results_f})), in that they all seemingly share an elbow-like effect at $n=15$. Beyond $n=15$, testing scores appear to approximately level out with a steady, albeit slightly noisy increase in the test score as $n$ grows. This behavior is shared by both the original non-augmented dataset and the augmented datasets, suggesting that from $n=15$ onwards, the provided dataset perhaps offers limited additional variety for modeling out-of-sample scenarios which may be present in the testing set. Despite this, the \textit{StiefelGen} methodology, with its more rapid gain in percentage testing accuracy per unit $n$, from $n=5$ up to the elbow point of $n=15$, suggests it is performing an effective generation of out-of-sample data even in the event of working with a severe sub sample of the original data set, which suggets that \textit{StiefelGen} is augmenting the data set with effective (and or realistic) signal data.

\section{Conclusion}

This paper has introduced \textit{StiefelGen}, a novel model-agnostic approach for time series data augmentation that leverages the geometry of Stiefel manifolds. Key advantages of \textit{StiefelGen} include its simplicity, flexibility, lack of need for model training, and its interpretability (due to the well-understood nature of the SVD factorization in general).

The algorithm requires minimal hyper parameters, with the perturbation scale, $\beta$, acting as the sole hyper parameter in most cases. By smoothly perturbing signals along geodesic paths on the Stiefel manifold, \textit{StiefelGen} can generate both augmented data instances for enhancing model training as well as outlier examples for robustness analysis. The effect of the page matrix dimensions provides further control over the type of synthetic signals produced. We have demonstrated the capabilities of \textit{StiefelGen} across a diverse set of experiments important to the domain of empirical time series analysis, including structural health monitoring, uncertainty quantification, dynamic mode decomposition for spatio-temporal forecasting, and improving neural network training via data augmentation. Across all cases, \textit{StiefelGen} enabled a nuanced and often fine-tuned investigation to be carried out.

Overall, \textit{StiefelGen} pioneers an innovative, matrix-geometric perspective for time series augmentation. Its model-agnostic nature ensures wide applicability across a range of problem domains. The primary direction of ongoing work for extending \textit{StiefelGen} will be to those problems in machine learning which currently rely upon the SVD factorization in a strong, and algorithmic sense. 

\section*{Acknowledgments}
This work was supported by JSPS, KAKENHI Grant Number JP21H03503, Japan and JST, CREST
Grant Number JPMJCR22D3, Japan.

\bibliographystyle{unsrt}  
\bibliography{biblo}

\newpage
\appendix
\onecolumn
\section{Appendix}

\subsection{StiefelGen in Python} \label{app:alg}

The code provided in Algorithm \ref{alg_app:StiefelGen_python} includes all the essential \textit{Python} code necessary to execute \textit{StiefelGen}. It is intentionally designed to be minimal and straightforward, requiring no training. The variable \texttt{perc} corresponds to the $\beta$ variable, representing the ``percentage perturbation'' scaled to the range $[0,1]$. The variable $U$ denotes the reshaped page matrix, or in a multivariate context, a stack of time series data. Once applying \textit{StiefelGen} to the signal, it is occasionally recommended to apply a smoothing filter to the output signal, especially if working in the domain of moderate to large perturbations (for outlier signal generation). Note that the code relies upon \textit{geomstats} version 2.5.0.

\label{alg_app:StiefelGen_python}
\begin{mintedbox}{python} 

import numpy as np
from geomstats.geometry.stiefel import Stiefel, StiefelCanonicalMetric

def StiefelGen(U, perc):
    '''
    U - Numpy array (2D) matrix
    perc - Percentage perturbation.
            Values in [0,1] (0=no perturbation, 1=maximal)
    '''
    
    # Constants 
    INJ_RADIUS = 0.89*np.pi
    
    # Geomstats Instantiation
    dim1, dim2 = U.shape  
    st = Stiefel(dim1,dim2)  
    st_metric = StiefelCanonicalMetric(dim1, dim2)  
    
    ########### StiefelGen Algorithm ############
    # Steps 1 and Step 2 -- Random Matrix Projected Onto Tangent Space 
    tan_plane_vec = st.random_tangent_vec(U,1) # It randomly generates from a normal distribution then uses pymanopt projection
    
    # Step 3 -- Scale wrt Radius of Injectivity
    canonical_ip = st_metric.inner_product(tan_plane_vec,tan_plane_vec,U)
    scaled_tan = tan_plane_vec / np.sqrt(canonical_ip) * perc * INJ_RADIUS
    
    # Step 4 -- Exponential Retraction 
    St_mat = st_metric.exp(scaled_tan, U)
    #############################################
    
    return St_mat
\end{mintedbox}

\newpage
\subsection{\textit{StiefelGen}: Taxi Data Analysis} \label{app:Taxi}

Here, we delve into the application of \textit{StiefelGen} using the NYC Taxi dataset mentioned in Section \ref{sec:datasets}. In the main paper, while working with the SteamGen dataset, we considered a moderate perturbation factor at the level of $\beta=0.4$ and a large perturbation factor at $\beta=0.9$. These factors remain unchanged here. However, the smoothing window size and the length of the univariate signal to be reshaped into the page matrix will be adjusted for the Taxi dataset. Specifically, we will extract the signal from indices 1300 to 1700, allowing \textit{StiefelGen} to focus on a sharp outlier (refer to Figures \ref{fig:outlier_small} and \ref{fig:outlier_large}). Consequently, the reshaped page matrix will have different dimensions compared to SteamGen, set at 20 rows by 20 columns.

The outcomes of applying the moderate perturbation factor are depicted in Figure \ref{fig:outlier_small}. Figures \ref{fig:outlier_small_a} and \ref{fig:outlier_small_b} highlight the impact of \textit{StiefelGen} on the original signal. Generally, the generated signal appears as a slightly noisier version of the original (as holistically clarified in Figure \ref{fig:outlier_small_a}). Notably, \textit{StiefelGen} successfully perturbs the Taxi dataset signal while preserving the impact of the point outlier (Figure \ref{fig:outlier_small_b}). Achieving this property through a simple noise addition would be challenging since clearly the noise levels added elsewhere in the generated signal are relatively high. Hand-engineering noise addition to minimize or eliminate noise at the outlier location would definitely be required in that case. Furthermore, the minimal smoothing window length played a crucial role in this property for \textit{StiefelGen}; a larger smoothing window would tend to ``smooth out'' the impact of the outlier, retaining only lower-frequency structures. Due to the highly repetitive pattern of the signal, the histogram of the residuals (Figure \ref{fig:outlier_small_d}) appears much less unusual than that of the SteamGen dataset, exhibiting a definite ``bell-shape'' reminiscent of a Gaussian distribution, albeit slightly extended here and with moderately unusual tail behaviour.

\begin{figure}[H]
  \centering
  \subfigure[A global view of the signal with the generated signal appended to the end.\label{fig:outlier_small_a}]{\includegraphics[width=0.45\textwidth]{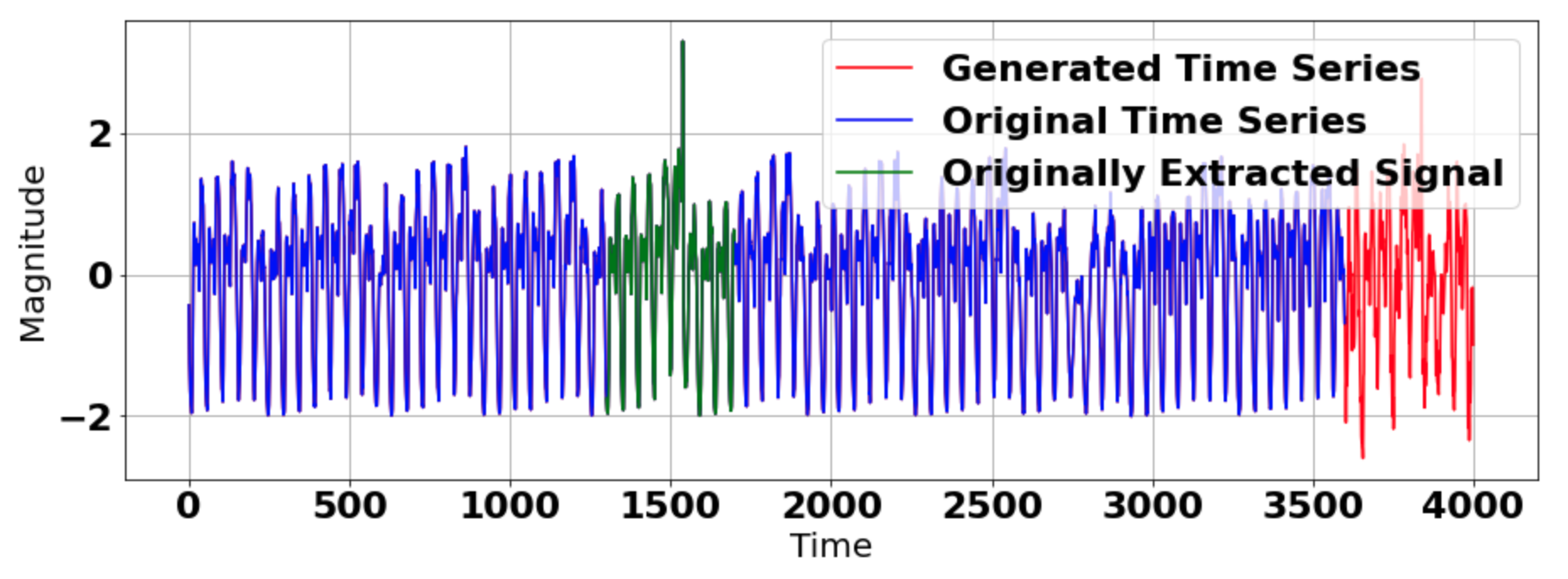}}
    \hspace{0.01\textwidth}  
  \subfigure[An over plot of the original signal with the newly generated signal.\label{fig:outlier_small_b}]{\includegraphics[width=0.45\textwidth]{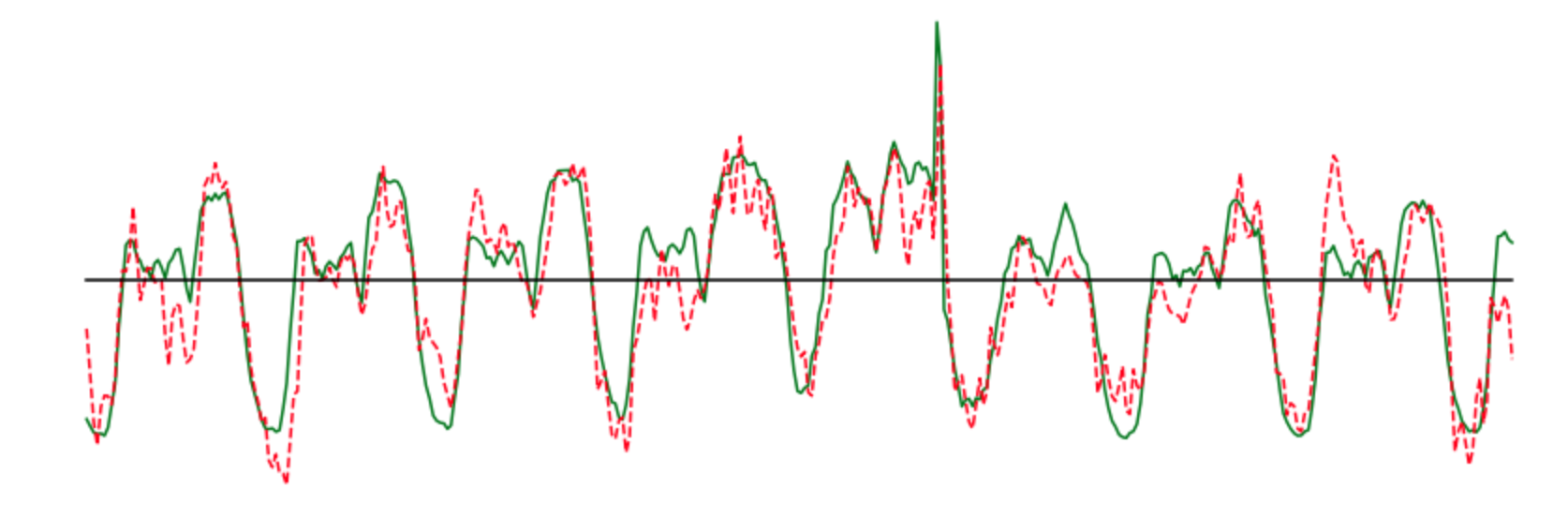}}

    \subfigure[A residual plot of the original and generated signals.\label{fig:outlier_small_c}]{\includegraphics[width=0.45\textwidth]{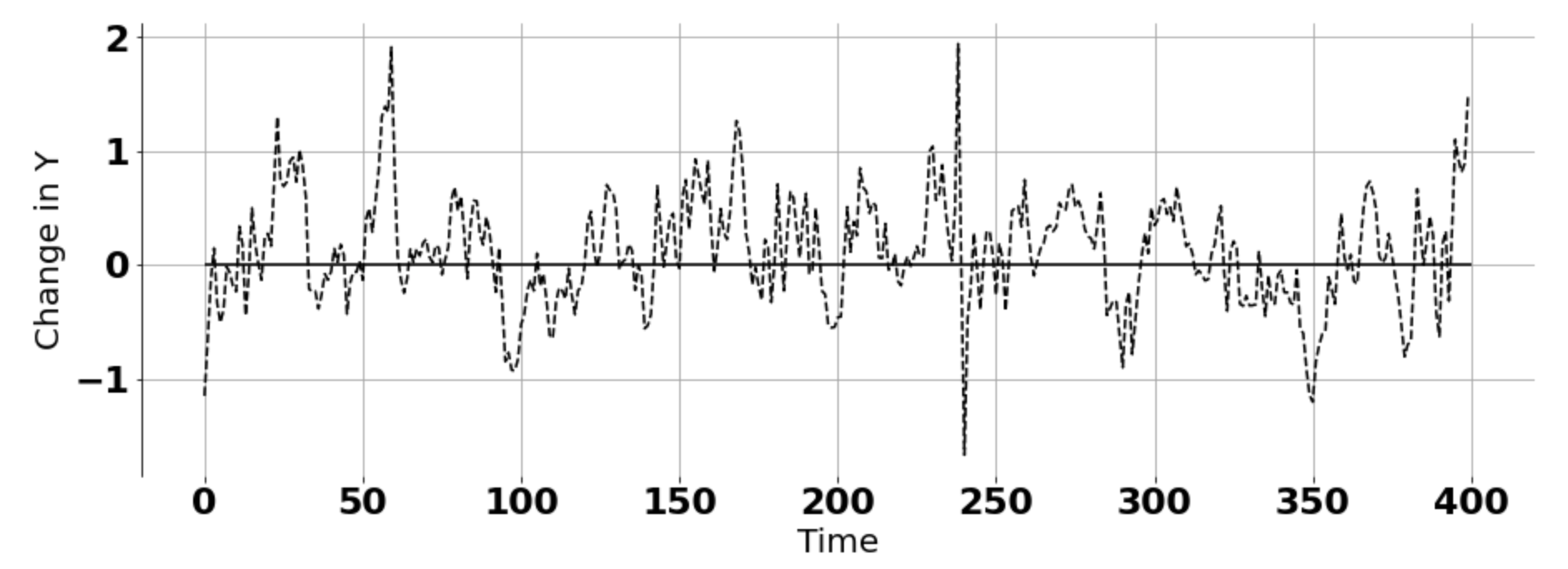}}
     \hspace{0.01\textwidth} 
      \subfigure[A histogram over the residuals.\label{fig:outlier_small_d}]{\includegraphics[width=0.45\textwidth]{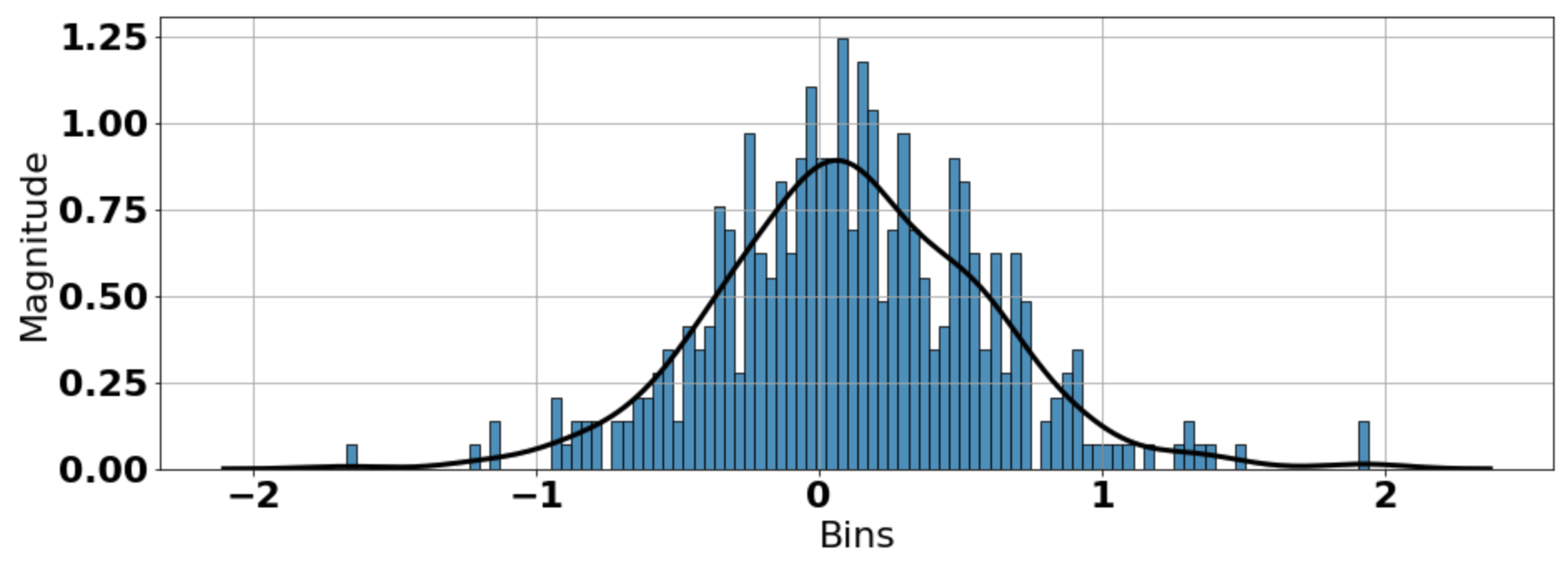}}
  \caption{An overview of the effect of applying a moderate perturbation factor to generate a new signal based on indices 1300 - 1700 of the Taxis data set.}
  \label{fig:outlier_small}
\end{figure}

Examining the effects of a large perturbation (thereby shifting the purpose of \textit{StiefelGen} from augmentation to outlier generation), Figures \ref{fig:outlier_large_a} and \ref{fig:outlier_large_b} illustrate that the generated signal significantly differs from the original reference signal. Excessive noise has clearly been generated, which also implicitly underscores the importance of applying a post-analysis smoothing filter when performing large perturbations nearing the radius of injectivity. However, for this level of perturbation and signal type, it might be challenging to preserve the structure of the original reference signal through smoothing alone. Similar to the plot for moderate perturbation, the histogram of the residuals (Figure \ref{fig:outlier_large_d}) appears roughly ``bell-shaped'', but with a notable difference — the distribution has been substantially widened and exhibits much stronger tail behavior. This widening and enhanced tail behavior are expected, given the elevated noise levels (increased variance) in the generated signal.

\begin{figure}[H]
  \centering
  \subfigure[A global view of the signal with the generated signal appended to the end.\label{fig:outlier_large_a}]{\includegraphics[width=0.45\textwidth]{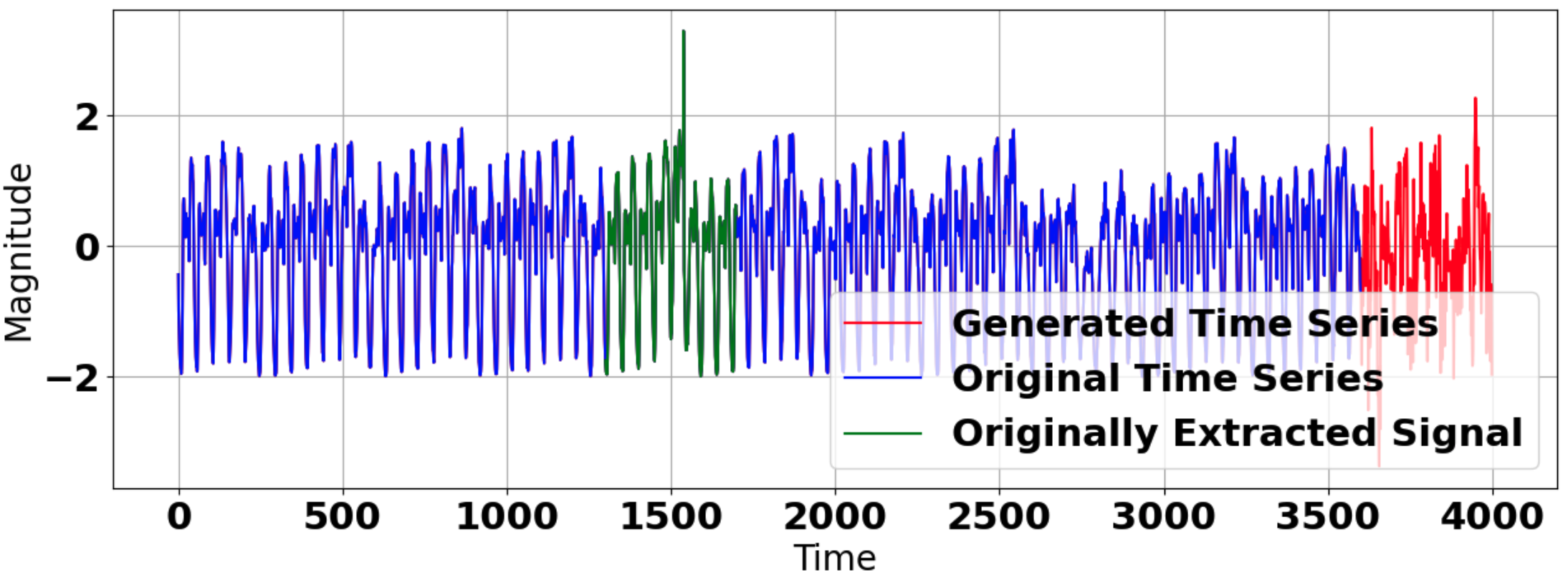}}
    \hspace{0.01\textwidth} 
  \subfigure[An over plot of the original signal with the newly generated signal.\label{fig:outlier_large_b}]{\includegraphics[width=0.45\textwidth]{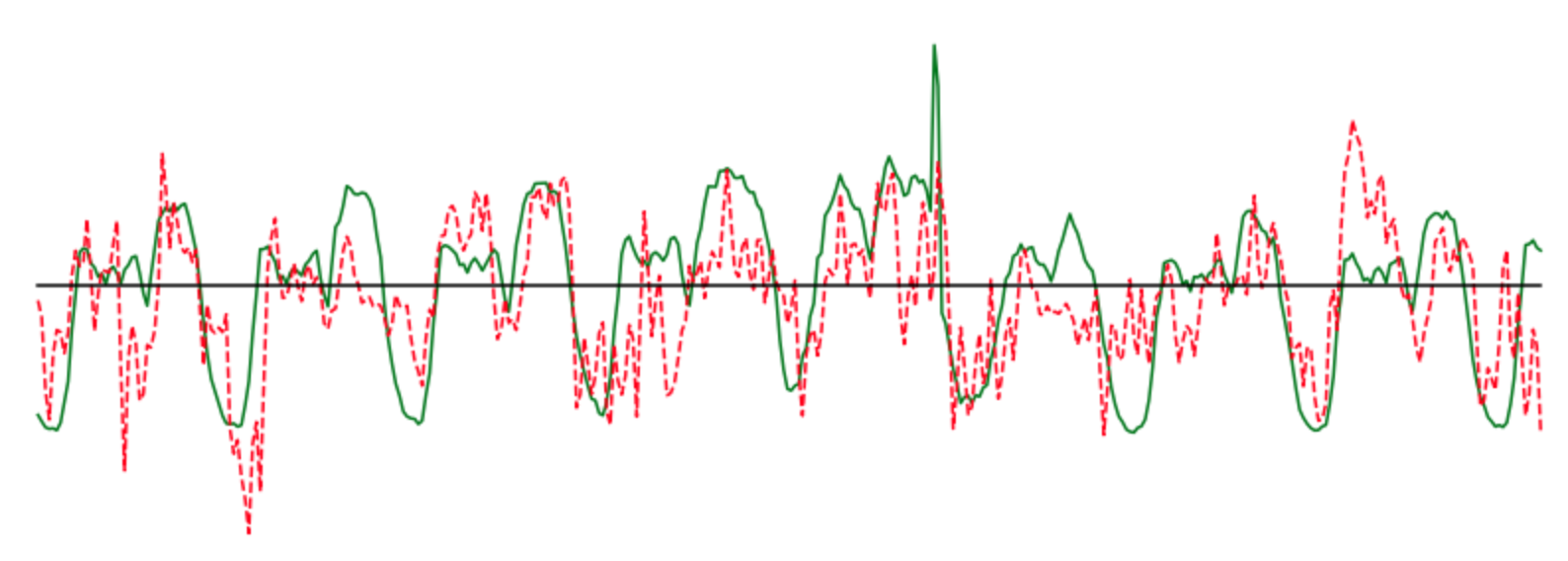}}

    \subfigure[A residual plot of the original and generated signals.\label{fig:outlier_large_c}]{\includegraphics[width=0.45\textwidth]{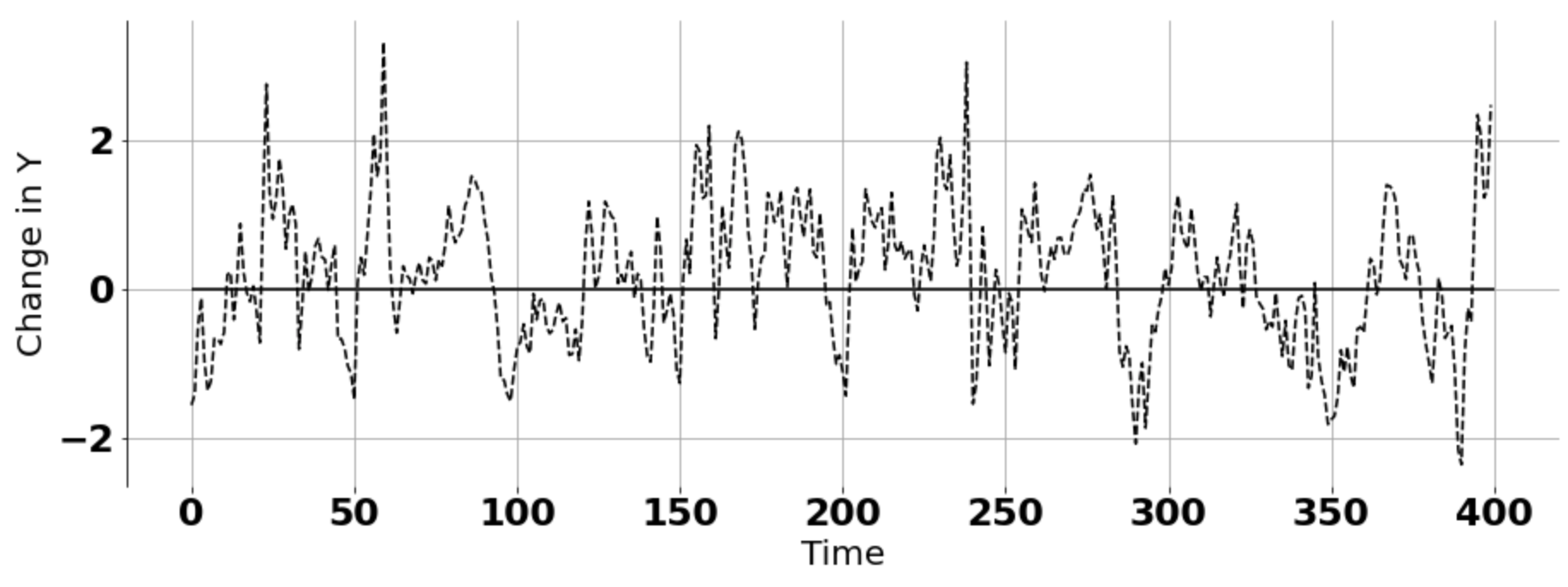}}
      \hspace{0.01\textwidth} 
      \subfigure[A histogram over the residuals.\label{fig:outlier_large_d}]{\includegraphics[width=0.45\textwidth]{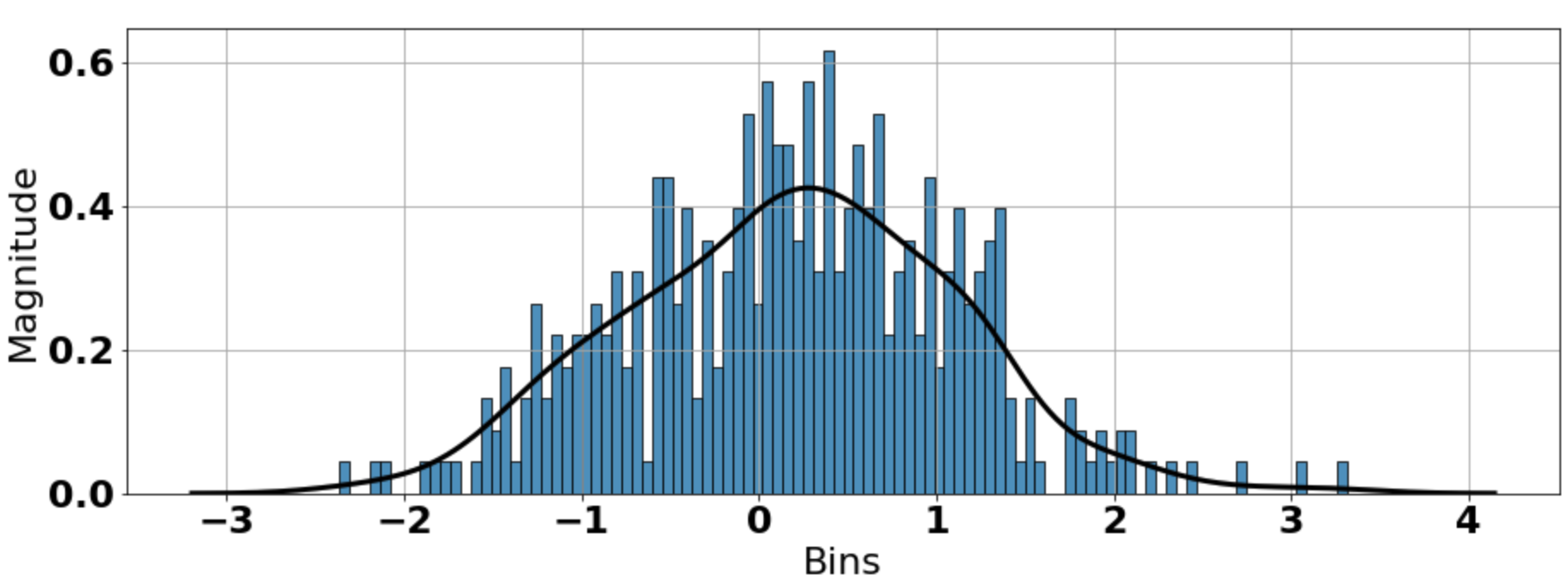}}
  \caption{An overview of the effect of applying a large perturbation factor to generate a new signal based on indices 1300 - 1700 of the Taxis data set.}
  \label{fig:outlier_large}
\end{figure}


\newpage
\subsection{\textit{StiefelGen}: Limitations and Quirks} \label{app_sec:quirks_lims}

The subsequent discussions endeavor to elucidate certain properties of \textit{StiefelGen} that we believe users of the algorithm should be aware of, and intuit, taking into account its limitations. While the main paper serves  to present a lot of the favorable and novel properties of \textit{StiefelGen}, by showcasing it through versatile applications, it is imperative for machine learning practitioners to acknowledge the absence of a universal solution —ultimately adhering to the principle that ``there is no such thing as a free lunch'' in machine learning.

\textbf{How Low Can You Go?}

Since the sole requirement of \textit{StiefelGen} is the ability to reshape a single uni-variate signal into a page matrix, the smallest matrix that can undergo a non-trivial SVD factorization is a 2 $\times$ 2 matrix. Consequently, \textit{StiefelGen} is capable of generating new signals as short as four units. This is exemplified in Figures \ref{fig:four_data_points_a} and \ref{fig:four_data_points_b}, where the smaller perturbation factor yields a novel signal that closely follows the shape of the original reference signal. In contrast, the larger perturbation leads to the endpoint of the generated signal moving in the opposite direction to its reference.

\begin{figure}[htbp]
  \centering
  \subfigure[A small perturbation ($\beta=0.3$) applied to a signal that is four units long.\label{fig:four_data_points_a}]{\includegraphics[width=0.4\textwidth]{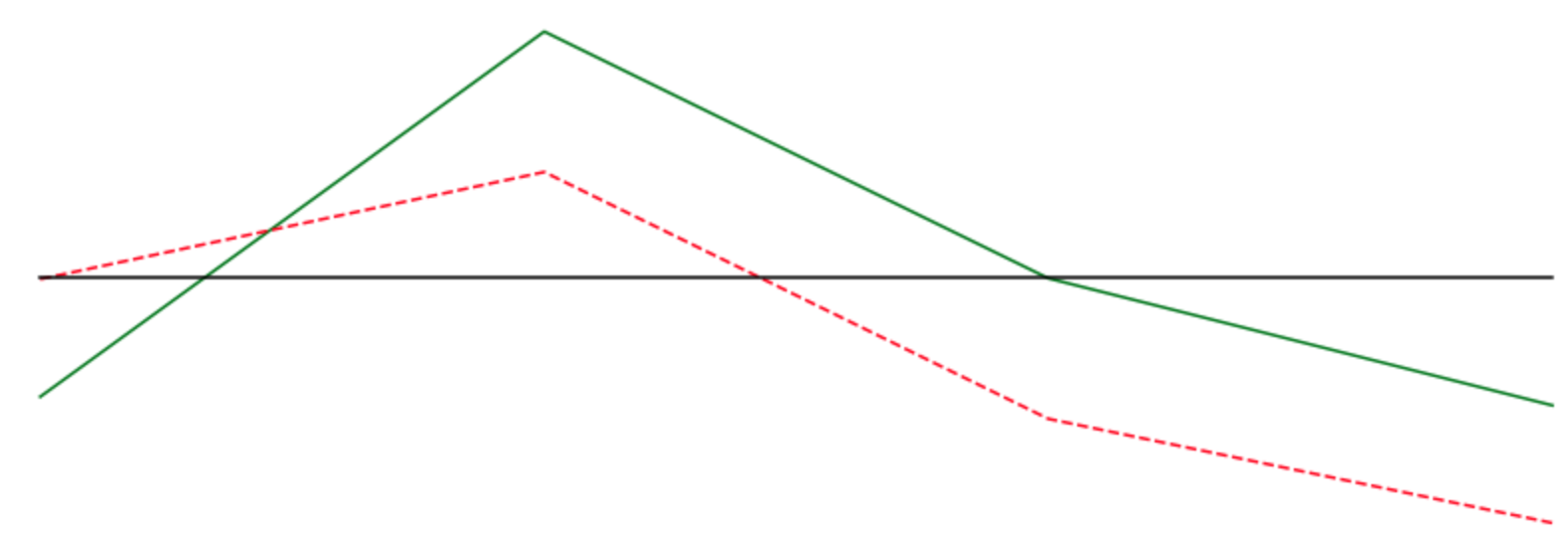}}
    \hspace{0.01\textwidth} 
  \subfigure[A small perturbation ($\beta=0.8$) applied to a signal that is four units long.\label{fig:four_data_points_b}]{\includegraphics[width=0.4\textwidth]{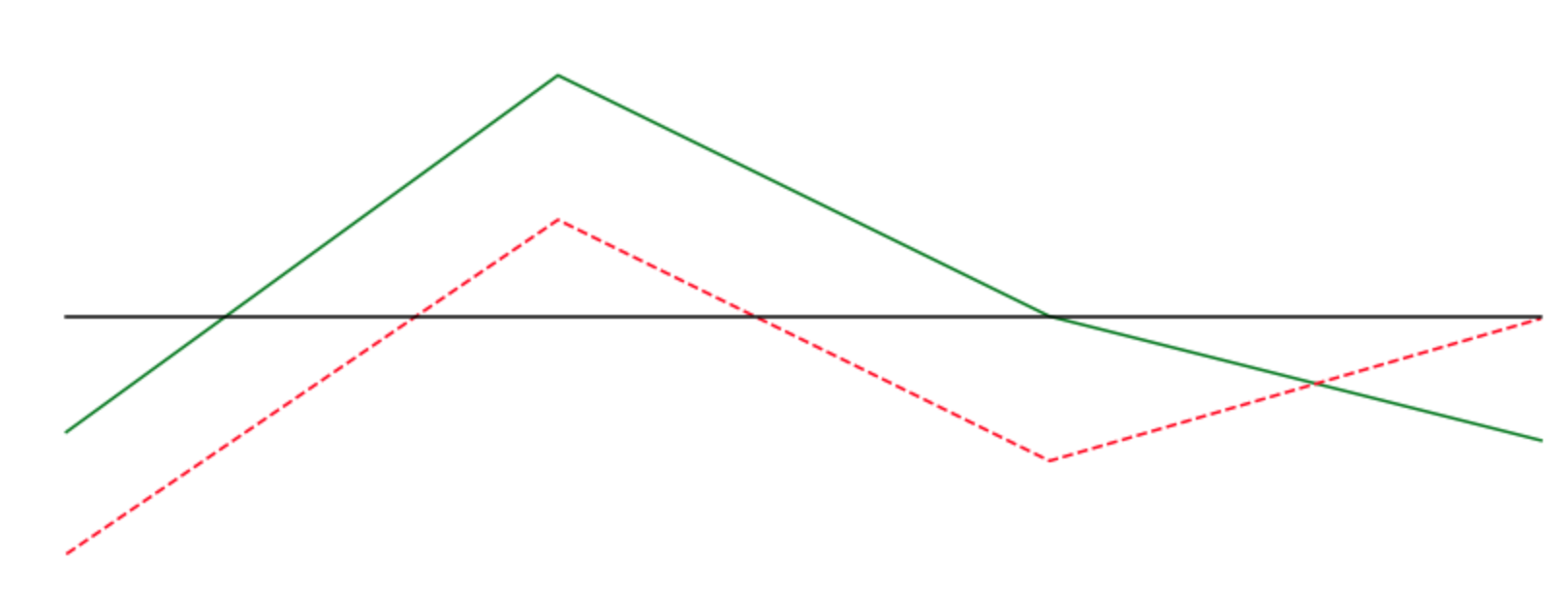}}
  \caption{Applying \textit{StiefelGen} to the shortest possible signal it can perturb}
  \label{fig:four_data_points}
\end{figure}

\textbf{Not Every Signal is Reachable}

While \textit{StiefelGen} exhibits the capability to generate a diverse array of signals, ranging from those closely aligned with the original (for augmentation purposes) to those resembling outliers, and traversing the geodesic path in between, it is crucial not to view \textit{StiefelGen} as a universal default for time series data augmentation. The tool has inherent limitations; it cannot generate every conceivable type of augmentation or outlier that might manifest in practical scenarios. This restriction stems from its construction, as \textit{StiefelGen} is specifically designed to produce signals that smoothly adhere to the orthogonality constraints of the $U$ and $V$ matrices following the SVD. In practice, \textit{StiefelGen} is effective only in generating a subset of signals within the broader set of possible signals encountered in real-world applications. This concept is illustrated in Figure \ref{fig:mathcha_venn_diagram}, where $\mathcal{H}$ denotes the entire hypothesis space of reachable signals.

\begin{figure}[htbp]
  \centering
 {\includegraphics[width=0.85\textwidth]
  {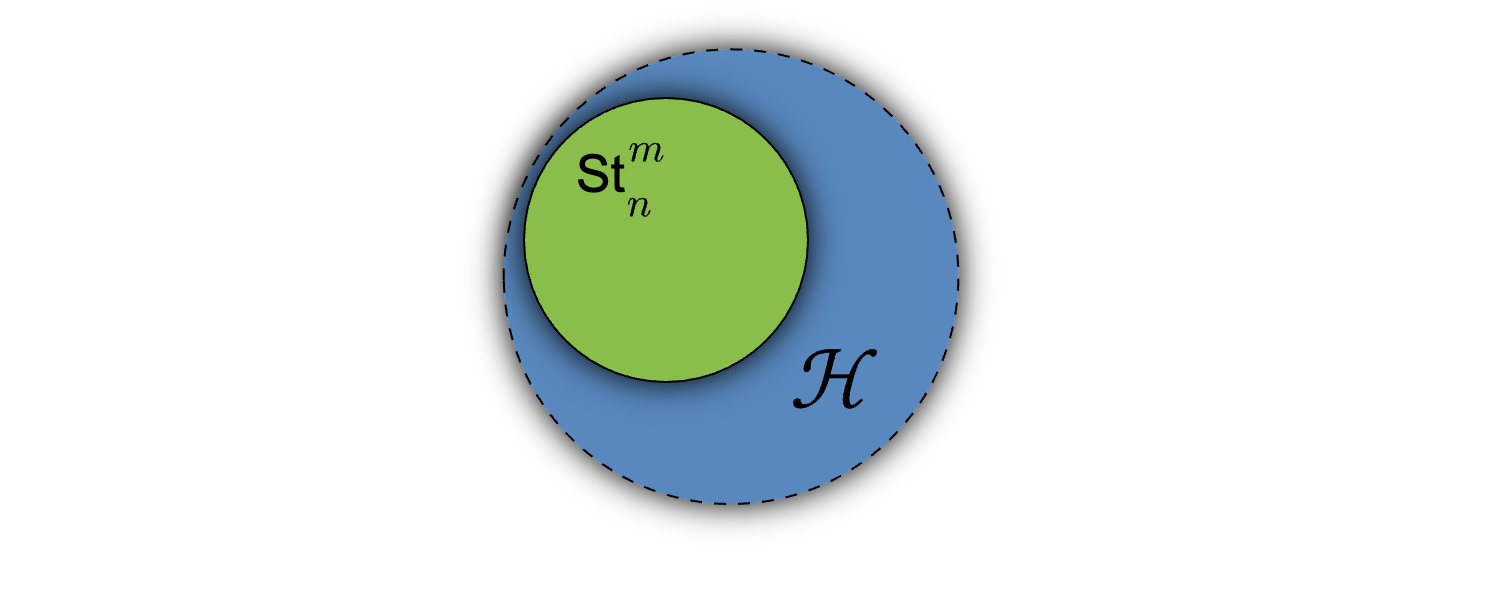}}
  \caption{Generating novel signals using \textit{StiefelGen} cannot generate every type of possible signal. \label{fig:mathcha_venn_diagram}}
\end{figure}

\textit{What types of signals or unobtainable via \textit{StiefelGen} then?} An obvious answer is to look for cases which will lead to the  breaking the orthogonality constraints of the $U$ and $V$ matrices. For reference a sinusoidal signal with random Gaussian noise is shown in Figure \ref{fig:unreachable_stiefel}, as well as an example \textit{StiefelGen} augmentation signal (generated from $\beta=0.1$, and a page matrix with $(m,n)=(20,10)$).

\begin{figure}[htbp]
  \centering
 {\includegraphics[width=0.5\textwidth]
  {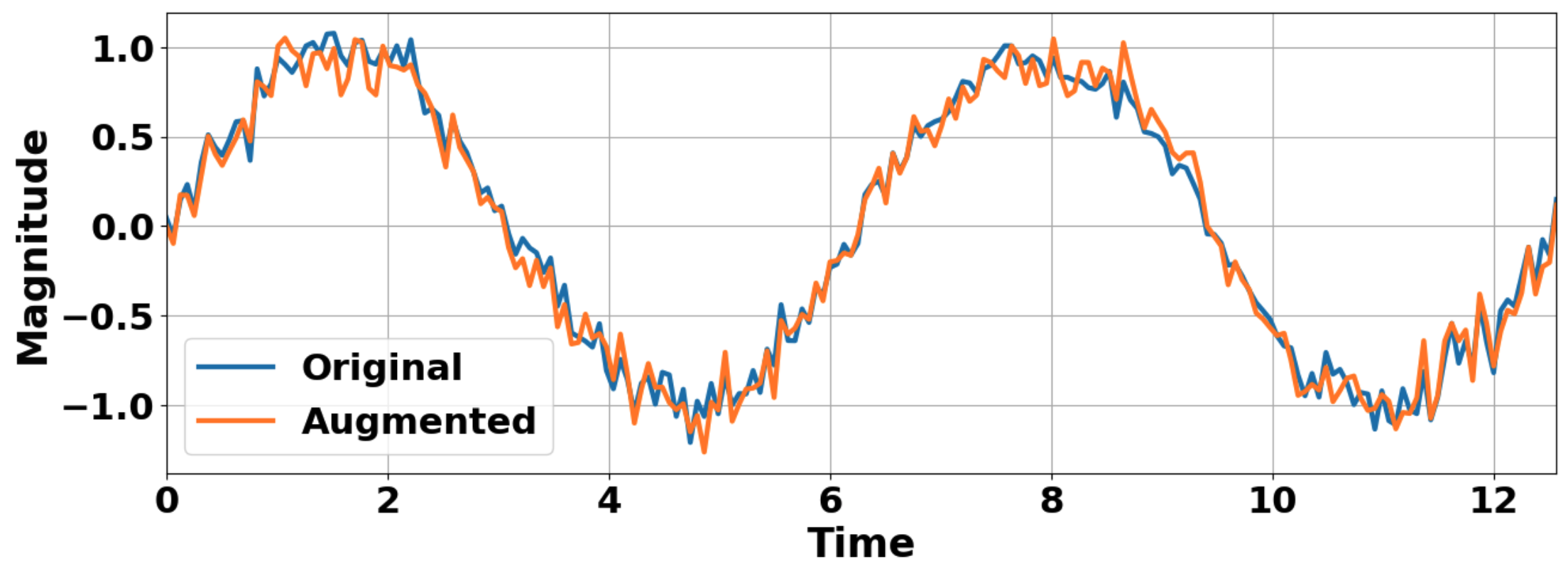}}
  \caption{A generic sinudsoidal signal with noise, as well as a \textit{StiefelGen} augmentation signal. \label{fig:unreachable_stiefel}}
\end{figure}

Based on using Figure \ref{fig:unreachable_stiefel} as a reference we may generate some example signals which \textit{StiefelGen} will effectively never be able to generate:

\begin{itemize}
    \item In order to break the orthogonality constraints placed on $U$ and $V$, a simple solution would be to zero-out rows (or columns) of either matrix. The results of this are made clear in Figure \ref{fig:unreachable_U0} which has had the first five columns of the $U$ matrix zeroed out, and Figure \ref{fig:unreachable_V0} which has had the first five rows of the $V$ matrix zeroed out. The reason for the unusual pattern Figure \ref{fig:unreachable_V0} may be intuited by noticing that the page matrix as stacked is size $(m,n)=(20,10)$, and by noticing that the signal has been zeroed out in 20 subsequent locations. 
    \item Another obvious way to break orthogonality is to perturb either the $U$ and $V$ matrices without regards to the underlying Stiefel geometry. This way $U+\varepsilon =  U'$ will lie in the ambient $\mathbb{R}^{m \times m}$ space (or $\mathbb{R}^{m \times k}$ space if one opts to work with the first orthonormal $k$ frames). The result of this is shown in Figure \ref{fig:unreachable_ambient}. Evidently even with small perturbations ($\sigma^2=0.04$) the act of leaving the Stiefel manifold and randomly perturbing the $U$ matrix results in quite large deviations from the original source signal. More information about this will be provided later under the heading ``Faster Than \textit{StiefelGen}?''
    \item Finally, it is possible to generate signals which are reachable with probability zero. We term these as being ``technically unreachable'' signals, an example of which is shown Figure \ref{fig:unreachable_change_point}, which has a change point. The reason that this is technically unreachable is that in order to generate the signal change point at a specified time location, one would need an overall perturbation that is zero everywhere except at that one specific location, which is ultimately a measure zero event.
\end{itemize}

\begin{figure}[htbp]
  \centering
  \subfigure[Resulting signal if first five columns of $U$ are zeroed out.\label{fig:unreachable_U0}]{\includegraphics[width=0.4\textwidth]{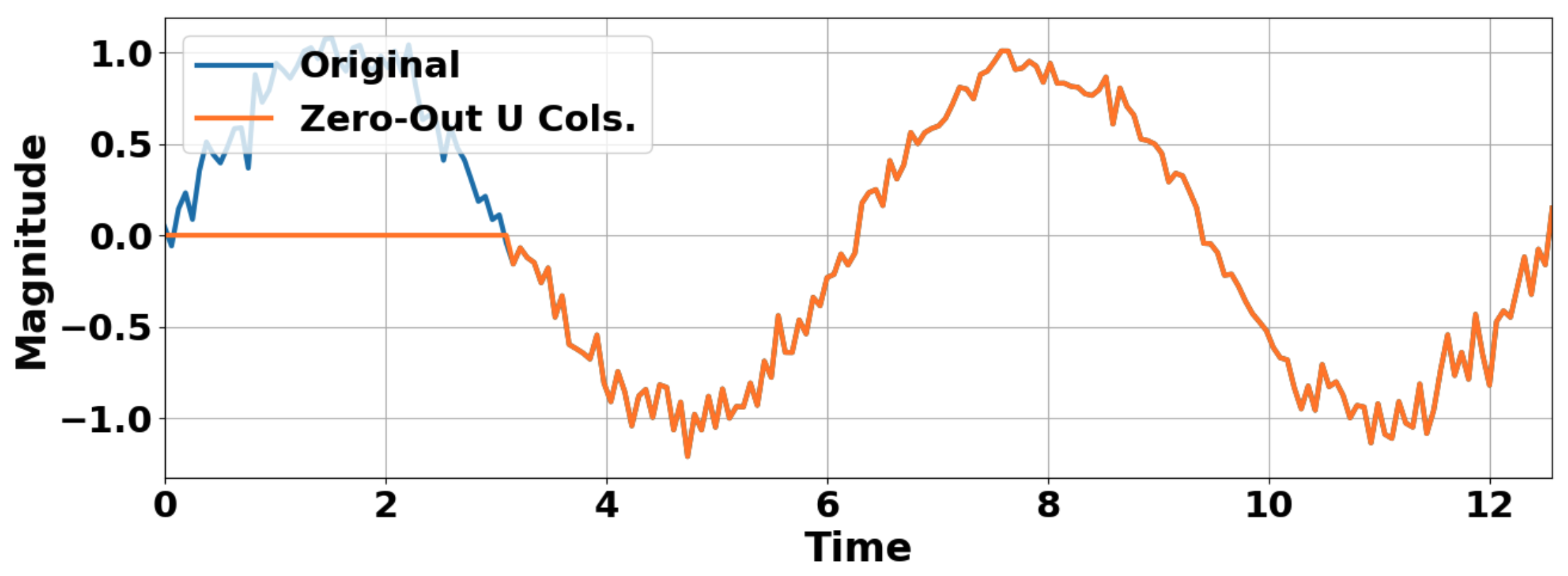}}
    \hspace{0.01\textwidth} 
  \subfigure[Resulting signal if first five rows of $V$ are zeroed out.\label{fig:unreachable_V0}]{\includegraphics[width=0.4\textwidth]{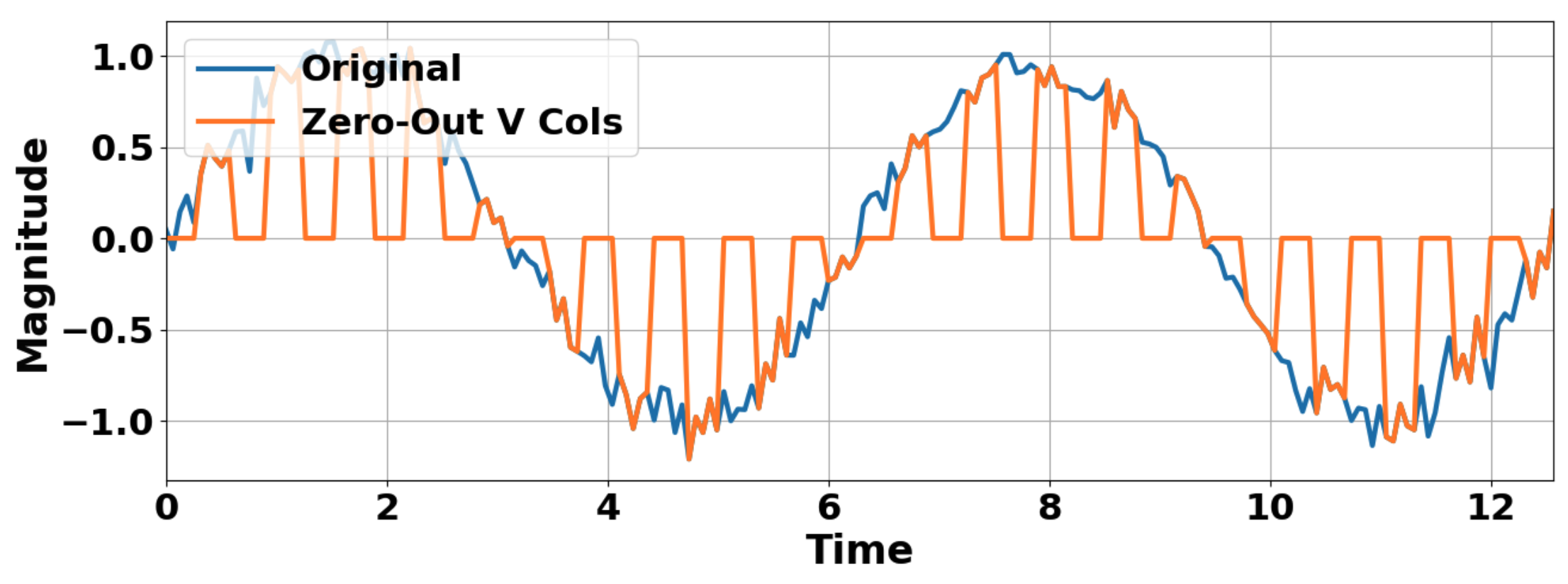}}

    \subfigure[Resulting signal if $U$ is perturbed by a random normal matrix into ambient $\mathbb{R}^{m \times m}$ space with $(\mu,\sigma^2)=(0,0.04)$.\label{fig:unreachable_ambient}]{\includegraphics[width=0.4\textwidth]{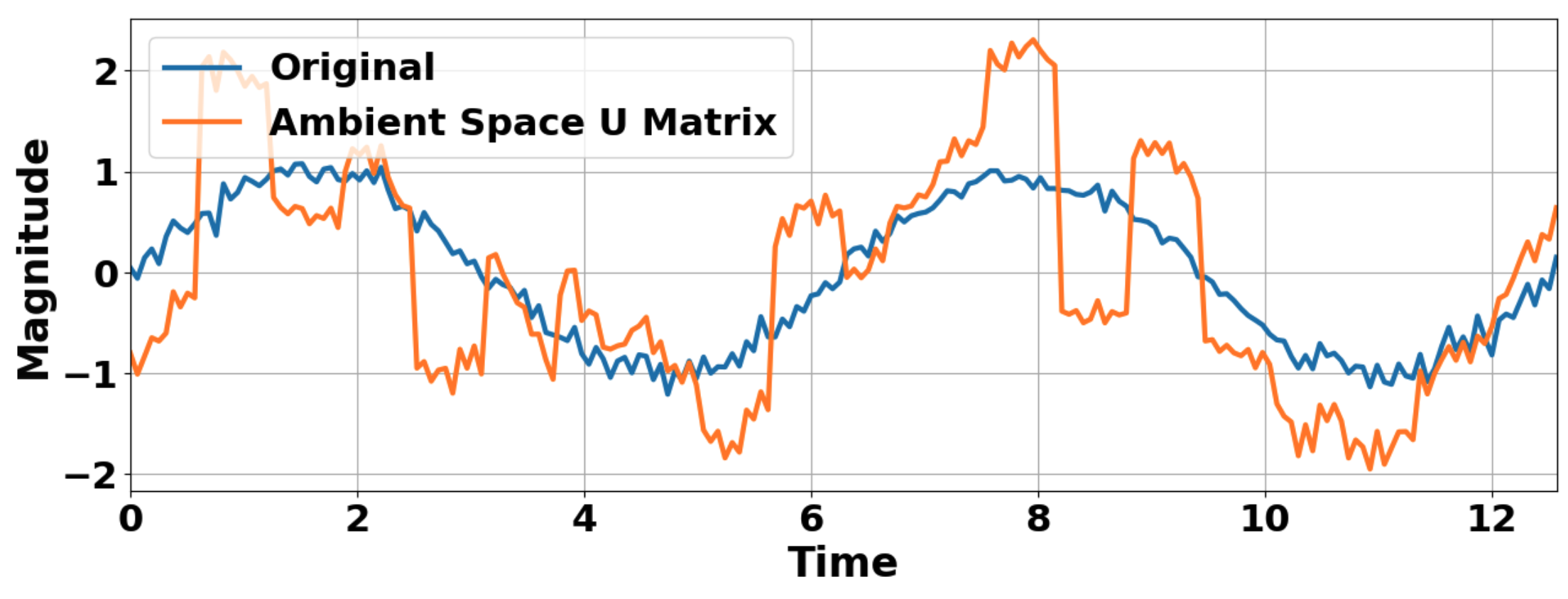}}
      \hspace{0.01\textwidth} 
      \subfigure[Resulting signal if a change point occurs half way through signal, pushing the second half of the signal to a mean value of five.\label{fig:unreachable_change_point}]{\includegraphics[width=0.4\textwidth]{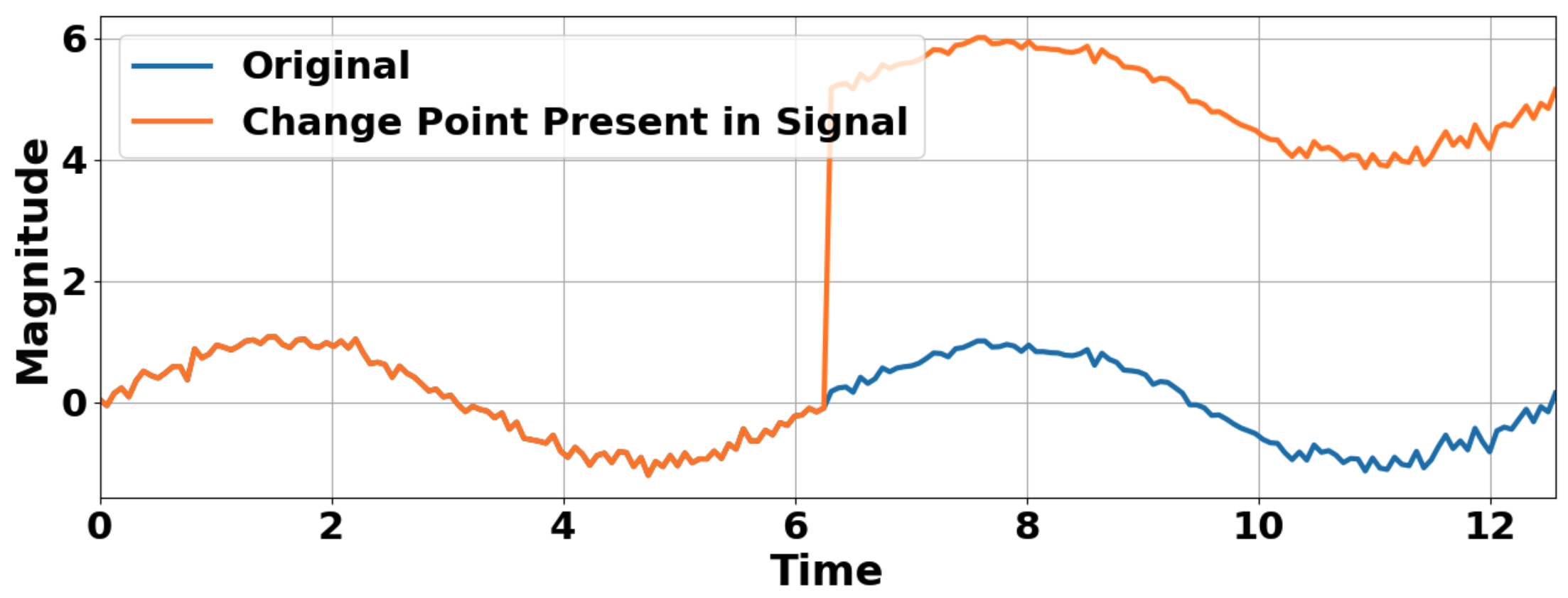}}
  \caption{A number of valid time series signals which would be completely unreachable via \textit{StiefelGen}.}
  \label{fig:unreachable}
\end{figure}


\textbf{Faster Than \textit{StiefelGen}?}

The inquiry prompted by Figure \ref{fig:unreachable_ambient} naturally leads to the question of ``how much flexibility is there in perturbing the $U$ and $V$ matrices directly, provided the perturbation remains within reasonable bounds?'' To delve into this, a set of 20 signals has been generated by introducing perturbations based on the addition of random normal matrices, as illustrated in Figure \ref{fig:stiefel_direct}.

\begin{figure}[htbp]
\centering
\subfigure[A series of signals resulting from perturbing a combination of $U$ and $V$ with a random normal matrix, where the elements are drawn from $\mathcal{N}(0,0.1^2)$. \label{fig:stiefel_direct_a}]{\includegraphics[width=0.48\textwidth]{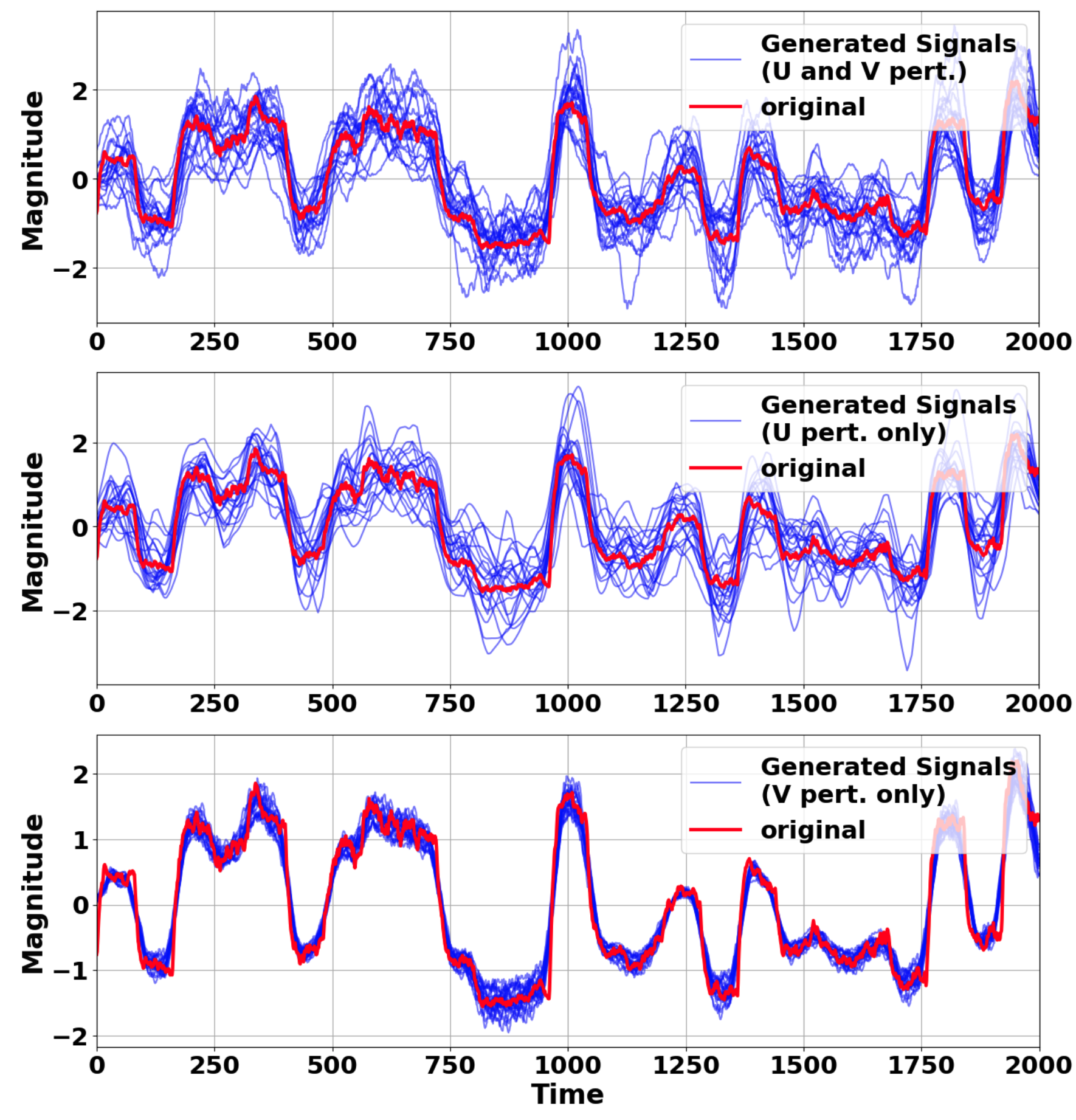}}
\hspace{0.01\textwidth}
\subfigure[A series of signals resulting from perturbing a combination of $U$ and $V$ with a random normal matrix, where the elements are drawn from $\mathcal{N}(0,0.05^2)$. \label{fig:stiefel_direct_b}]{\includegraphics[width=0.48\textwidth]{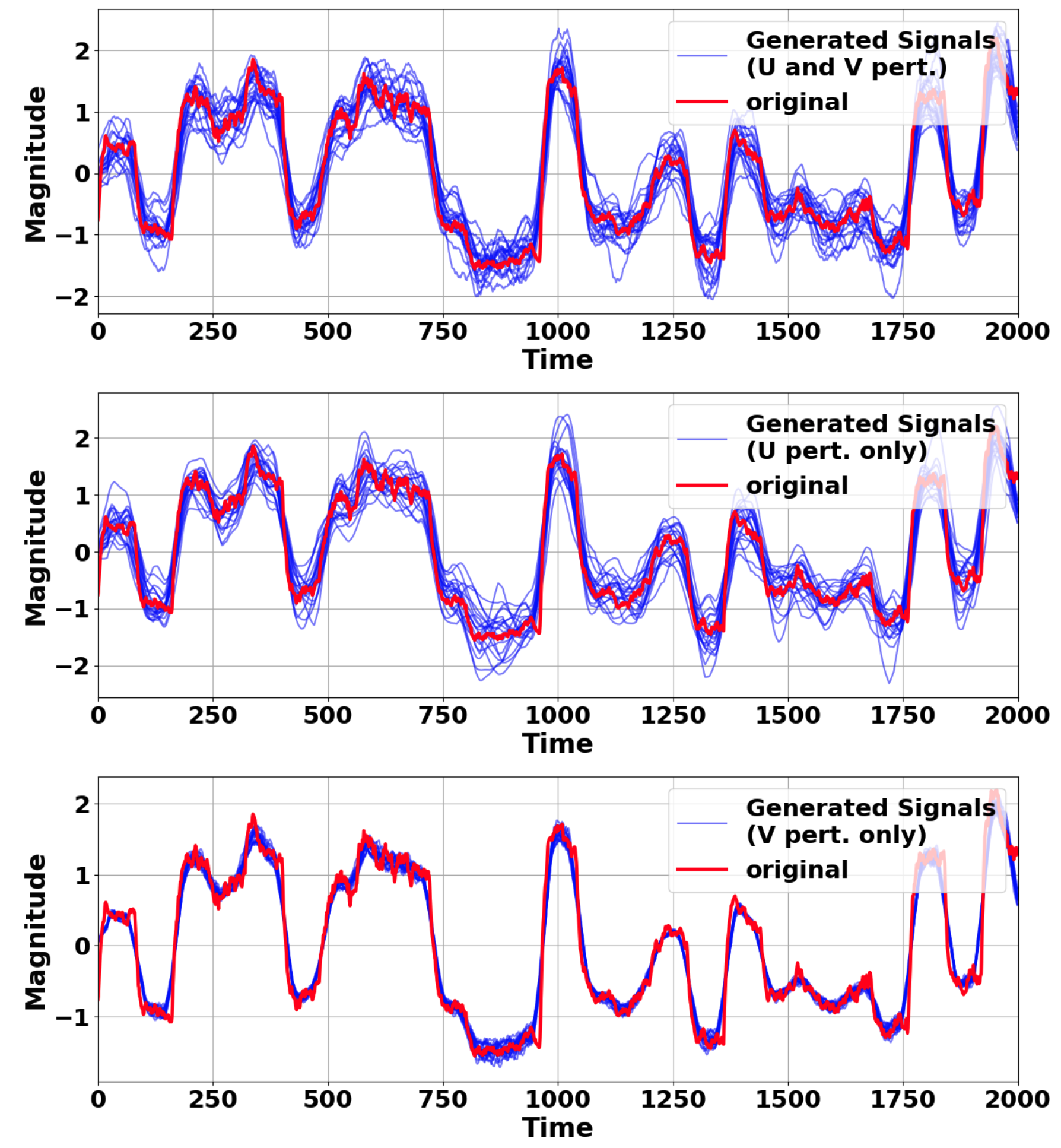}}
\caption{Comparison of the effects of perturbing the SVD orthonormal matrices directly, thereby departing from the Stiefel manifold.}
\label{fig:stiefel_direct}
\end{figure}

Clearly, the quality of signals deteriorates rapidly with increasing perturbation levels, even at modest magnitudes. However, as depicted in Figure \ref{fig:stiefel_direct_b}, there exists a threshold where ``small enough'' perturbations yield signals that closely resemble the original reference signal. In essence, an algorithm following this approach would be akin to implementing only Step 1 of the \textit{StiefelGen} algorithm. Such an algorithm could operate considerably faster (linear time after the initial SVD), yet it sacrifices the nuanced control afforded by constraining the geometry to the Stiefel manifold. Furthermore, integration into additional algorithms, such as DMD for spatio-temporal signals, becomes less straightforward, as failure to respect the orthonormality of $U$ and $V$ matrices in each iteration rapidly degrades forecasted signal quality. The choice between these algorithms ultimately hinges on the functional requirements of the problem. Direct perturbation of the $U$ and $V$ matrices is a viable option if perturbations are sufficiently small for the context, and the problem solely necessitates pure time series signal generation (e.g., enhancing the training-testing performance of a neural network). Lastly, it should be noted that as the signal length increases, the scale of perturbations would generally need to decrease also—a natural consequence of the orthonormality constraint on $U$ and $V$.

\textbf{From Hyperspheres to Flipped Signals: Introducing \textit{SphereGen}}

Exploring further with \textit{StiefelGen}, we contemplate the scenario of: ``what if we forgo the construction of a page matrix and directly perform an SVD on the signal itself?'' Essentially, this involves extending the concept of the SVD to a 1D matrix, an unconventional approach given that SVD is traditionally employed for factorizing 2D matrices. However, since the Stiefel manifold can be defined for 1D matrices, this idea is not without merit. It's important to recall that this perspective ties back to the geometry of the hypersphere. That is $\St$ with $n=1$ is by definition a hypersphere. Thus performing \textit{StiefelGen} for the 1D case implies implicitly working with hypersphere geometry, a facet that emerges naturally in the mathematical underpinnings of the SVD.

To elaborate, consider $\mathcal{T}=U\Sigma V^{\intercal}$, where $\mathcal{T}\in\mathbb{R}^{m\times 1}$ so that $\mathcal{T}\in \St$, with $n=1$. If we assume $V=1$ for this discussion, then $U$ must encapsulate all the signal information. However, owing to orthonormality in that $\|U\|=1$, we may see that $U$ is positioned on the circumference of a unit hypersphere. Meanwhile, $\Sigma$ acts as a scalar, facilitating the restoration of the signal magnitude to its original scale. This unique corner case necessitates explicit since a lot of Stiefel geometry coding libraries do not readily work with the $n=1$. As seen in libraries like \textit{geomstats}, one would ideally opt for the explicit use of the hypersphere class.

The outcome of applying \textit{StiefelGen} directly to the 1D signal, thereby bypassing the page matrix construction (an algorithm term which we shall term as \textit{SphereGen}), is depicted in Figure \ref{fig:spheregen}.

\begin{figure}[htbp]
\centering
\subfigure[The first 5 incremental geodesic steps.]{\includegraphics[width=0.45\textwidth]{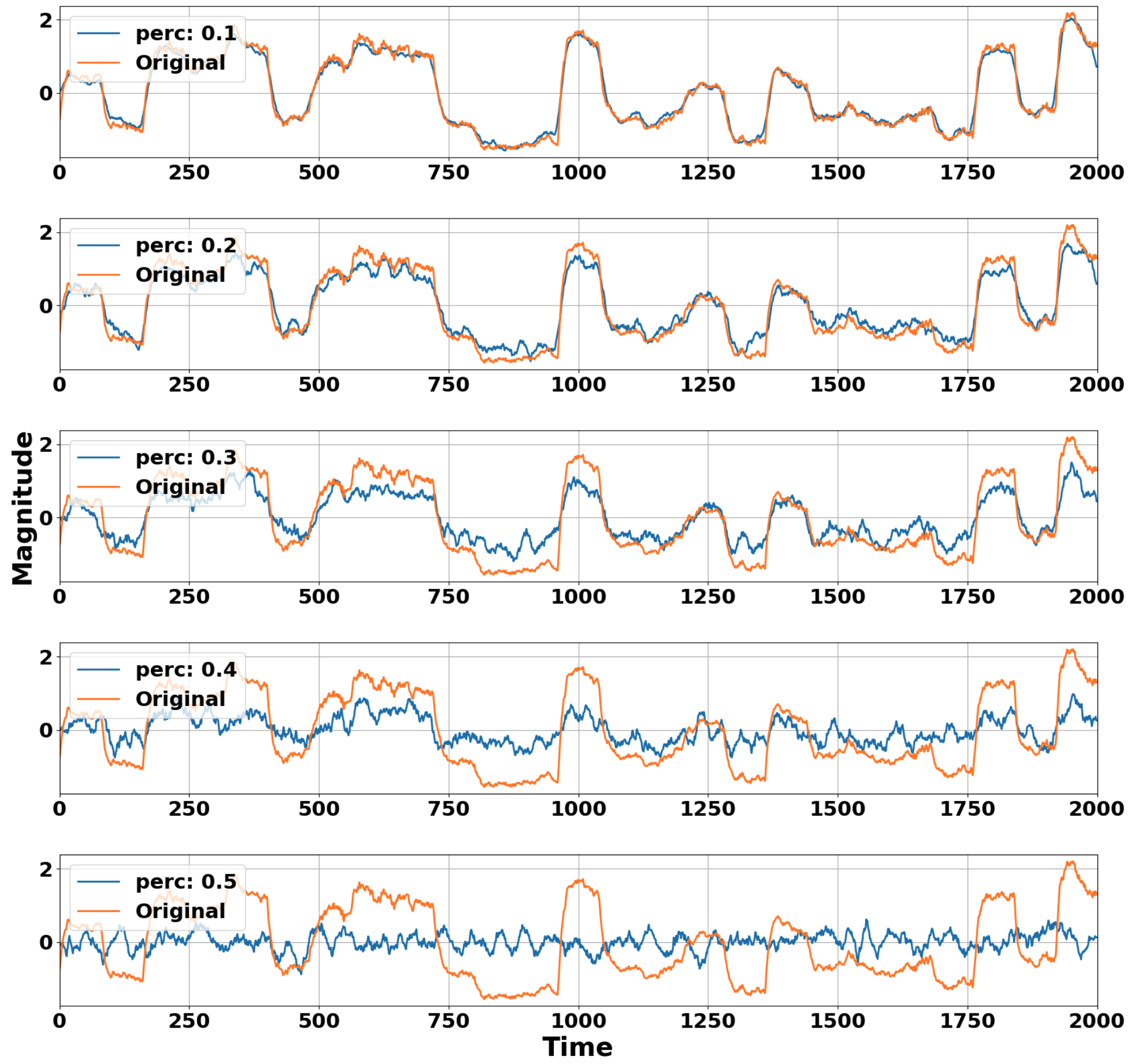}}
\hspace{0.01\textwidth}
\subfigure[The last 5 incremental geodesic steps.]{\includegraphics[width=0.45\textwidth]{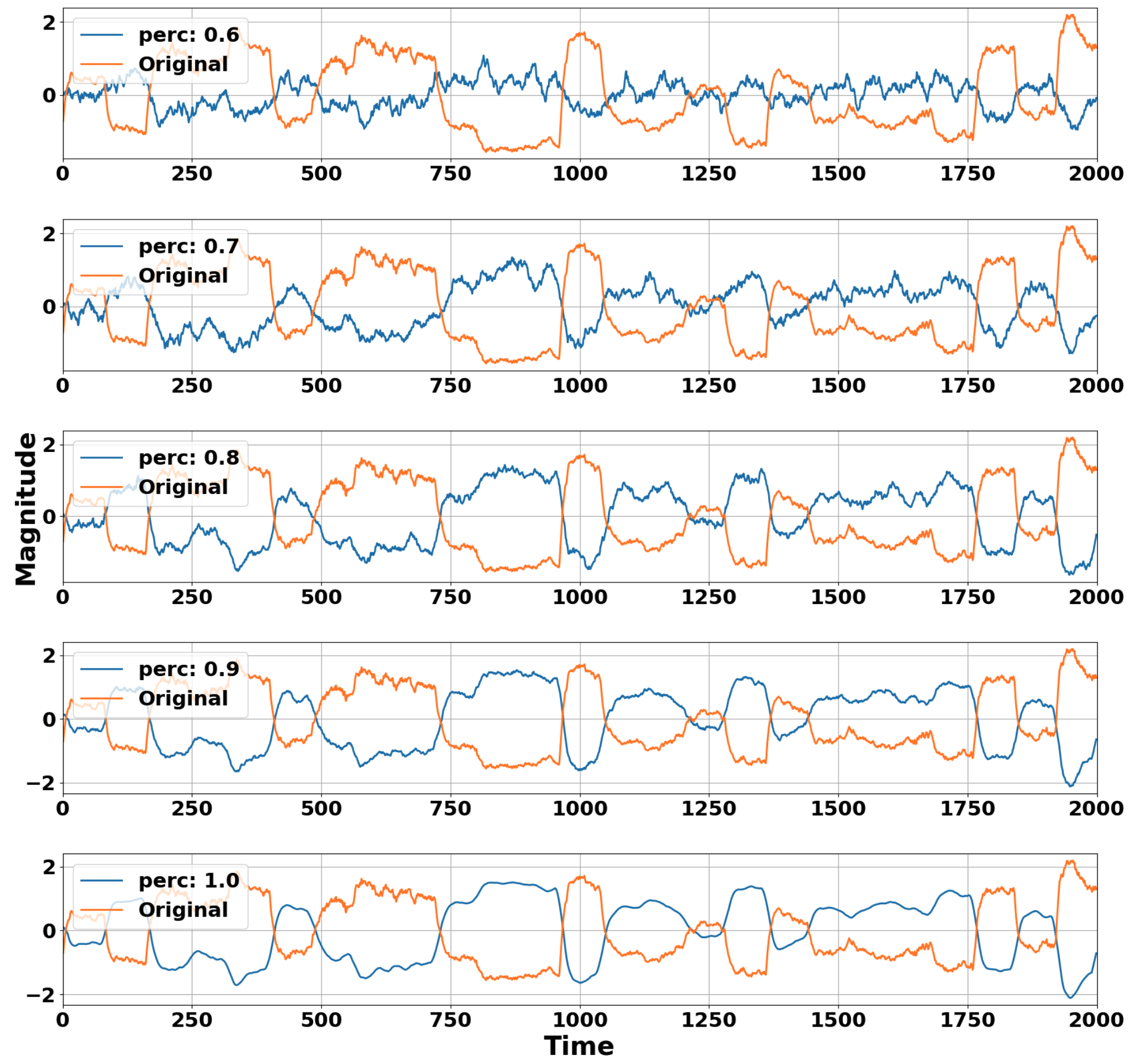}}
\caption{Incrementally deforming the generated signal with respect to a \textit{Hypersphere manifold} for the SteamGen data set. A smoothing window of $\ell=20$ was used, and no page matrix construction was performed.}
\label{fig:spheregen}
\end{figure}

Examining Figure \ref{fig:spheregen}, we observe intriguing behavior. Firstly, perturbations exhibit a significantly larger impact compared to that seen in \textit{StiefelGen}. For instance, note the level of signal deviation at a 0.3 perturbation level in Figure \ref{fig:spheregen} compared to Figure \ref{fig:intro_large} which underwent a perturbation factor of 0.9. Secondly, perturbing all the way to the radius of injectivity results in the original signal being flipped in its $y$-orientation. To elucidate these scenarios, we further delve into the geometry of the hypersphere.

\begin{figure}[htbp]
\centering
\subfigure[Diagram of the scaling that occurs in the degenerate 1D SVD case.\label{fig:hyp_explain}]{\includegraphics[width=1\textwidth]{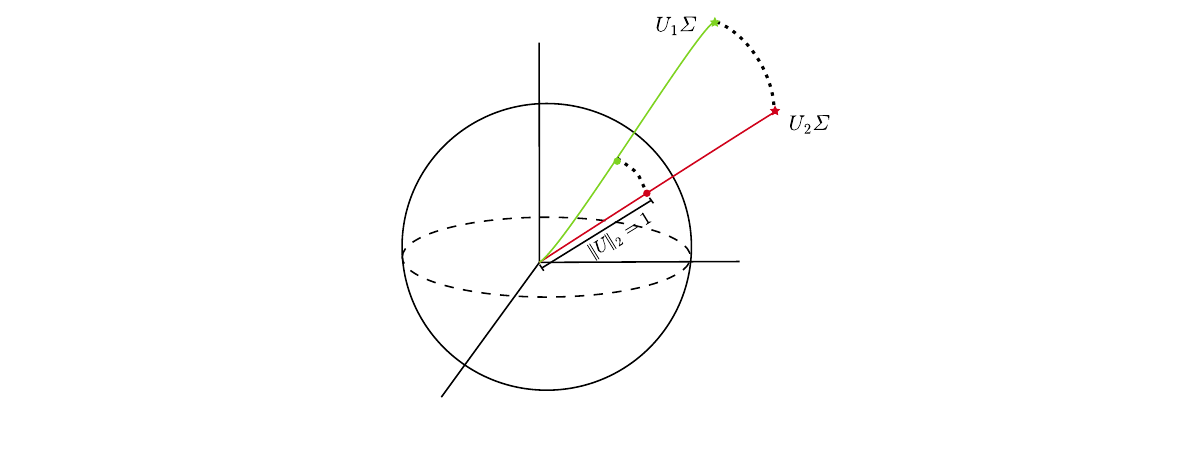}}
\hspace{0.01\textwidth}
\subfigure[Diagram explaining the negative behavior at maximal percent perturbation (at the radius of injectivity of the hypersphere).\label{fig:hyp_negative}]{\includegraphics[width=1\textwidth]{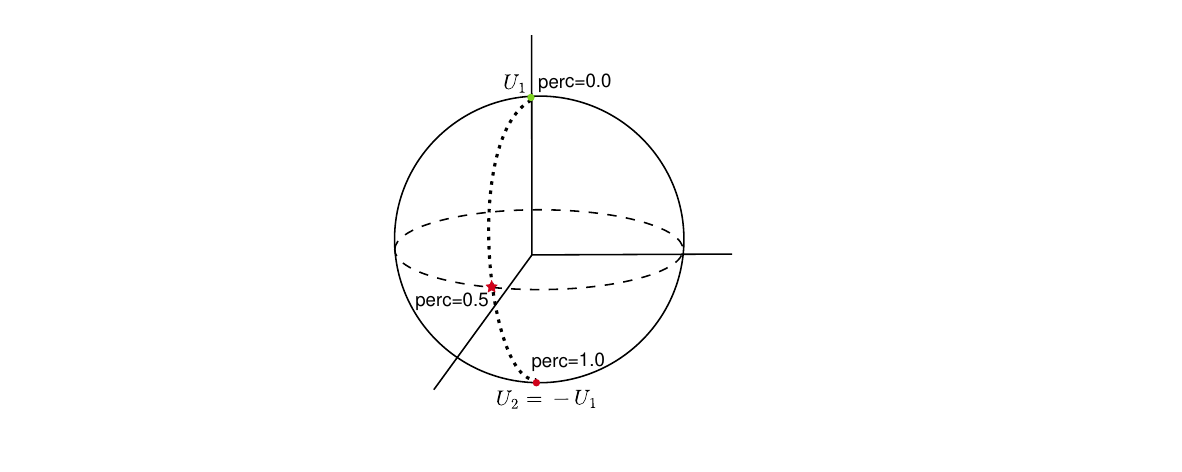}}
\caption{Illustrating the behaviors occurring geometrically for \textit{SphereGen}.}
\label{fig:hypersphere_example}
\end{figure}

The negative flipping of the signal can be attributed to reaching the edge of the injectivity radius, where one ends up at the extreme opposite end of the hypersphere, as illustrated in Figure \ref{fig:hyp_negative}. This figure also suggests why, at the 50\% perturbation mark, the signal appears excessively noisy. This is because this point maps to the equator of the sphere, the position maximally far away from either the signal itself, or its flipped orientation. Moreover, the equator has the largest radius among all horizontal slices, making accidentally arriving at a $U$ carrying useful structural information at the equator virtually a probability-zero event. It is far more likely that one would end up with a random $U$ (that satisfies orthonormality constraints) rather than anything with practical relevance. 

Due to the well known geometry of the hypersphere, the geodesic equation, also known as great circles, (as well as the exponential map) is also known in closed form

\begin{equation}
    \gamma(t) = p\cos(\|v\|t) + \frac{v}{\|v\|} \sin(\|v\|t),
\end{equation}

where $p$ is a position on the hypersphere, and $v$ is a velocity vector. Evidently the exponential map, takes on form: $\text{Exp}_p(v) = p\cos(\|v\|) + \frac{v}{\|v\|} \sin(\|v\|)$.  A simple substitution of values, $p=[1,0,0]$ (the ``North Pole''), and $v=[\pi,0,0]$, will result in $\text{Exp}_p(v) = [-1,0,0]$, confirming that radius of injectivity is in $\pi$ units in magnitude. A similar exercise with $v=[\pi/2,0,0]$, reveals that this velocity vector will map directly onto the equator, as Figure \ref{fig:hyp_exp} aims to illustrate. 

\begin{figure}[H]
  \centering
{\includegraphics[width=1\textwidth]{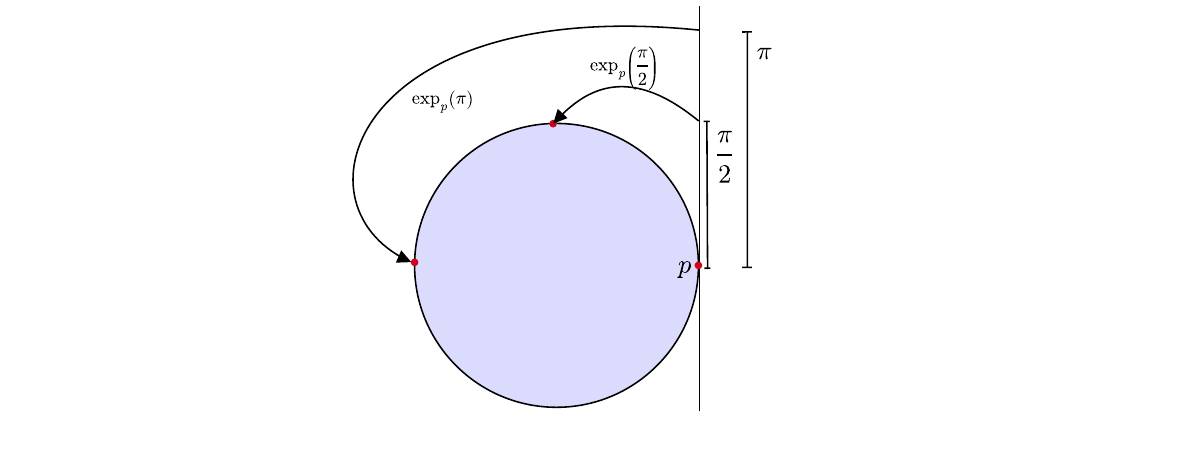}}
  \caption{Illustration of the exponential map of the tangent space elements generated at point $p$. }
  \label{fig:hyp_exp}
\end{figure}

Given this, it becomes possible to implement this \textit{StiefelGen} corner case in Python in a manner which only depends on the \textit{numpy} library, and which runs in linear $\mathcal{O}(m)$ time, 

\label{alg_app:StiefelSphere_python}
\begin{mintedbox}{python} 

import numpy as np

def SphereGen(signal, t=1, boundary=np.pi/6, smooth_fac=20):
    '''
    - signal: Input time seires (1D array)
    '''
    v = generate_random_tangent_vector(signal, boundary)
    out = smooth(sphere_geodesic(v, signal, t), smooth_fac)
    return out

def generate_random_tangent_vector(p,boundary):
    """
    Generate a random tangent vector at point p on a hypersphere with norm less than "boundary".

    Parameters:
    - p: Point on the hypersphere.
    - boundary: Limiting norm boundary for tangent vector to be samlped from.

    """
    # Dimension of the hypersphere
    dim = len(p)

    # Generate a random vector in the ambient space
    random_vector = np.random.randn(dim)

    # Project the vector onto the tangent space of the hypersphere at point p
    tangent_vector = random_vector - np.dot(p, random_vector) * p

    # Normalize the tangent vector to ensure its norm is less than boundary
    tangent_norm = np.linalg.norm(tangent_vector)
    if tangent_norm > boundary:
        tangent_vector = (boundary) * tangent_vector / tangent_norm
    return tangent_vector


def sphere_geodesic(v, p, t=1):
    """
    Exponential map for the unit sphere.

    Parameters:
    - v: Tangent vector.
    - p: Point on the sphere.
    - t: Time (t=1 defaults to exponential map)

    """
    v_norm = np.linalg.norm(v)

    if v_norm == 0:
        return np.array(p)  # Avoid division by zero for the zero vector (base case)

    exp_map_result = np.cos(v_norm*t)*np.array(p) + np.sin(v_norm*t)*np.array(v)/v_norm
    return exp_map_result
\end{mintedbox}

Similar to the \textit{StiefelGen}, \textit{SphereGen} operates without the need for model training and is entirely model-agnostic. However, \textit{SphereGen} presents certain challenges in terms of fine-tuning and controlling perturbation sizes, noise levels, and basis deformations. It lacks the straightforwardness of \textit{StiefelGen} and requires the specification of a "boundary" variable to prevent the sampling of excessively large tangent vectors. Additionally, occasional output may necessitate an extra magnitude scaling for enhanced signal variety.

One limitation of \textit{SphereGen} is its compatibility with other algorithms that rely on the non-degenerate case of the SVD, such as the DMD algorithm. Nevertheless, the absence of page matrix construction and the linear time complexity offer substantial advantages in specific contexts. Figure \ref{fig:sphere_geo_demo} illustrates examples of outputs generated by \textit{SphereGen}.

\begin{figure}[htbp]
  \centering
  \subfigure[\textit{SphereGen} being used to generate a signal and 
 travel the geodesics between the starting and ending points.]{\includegraphics[width=0.7\textwidth]{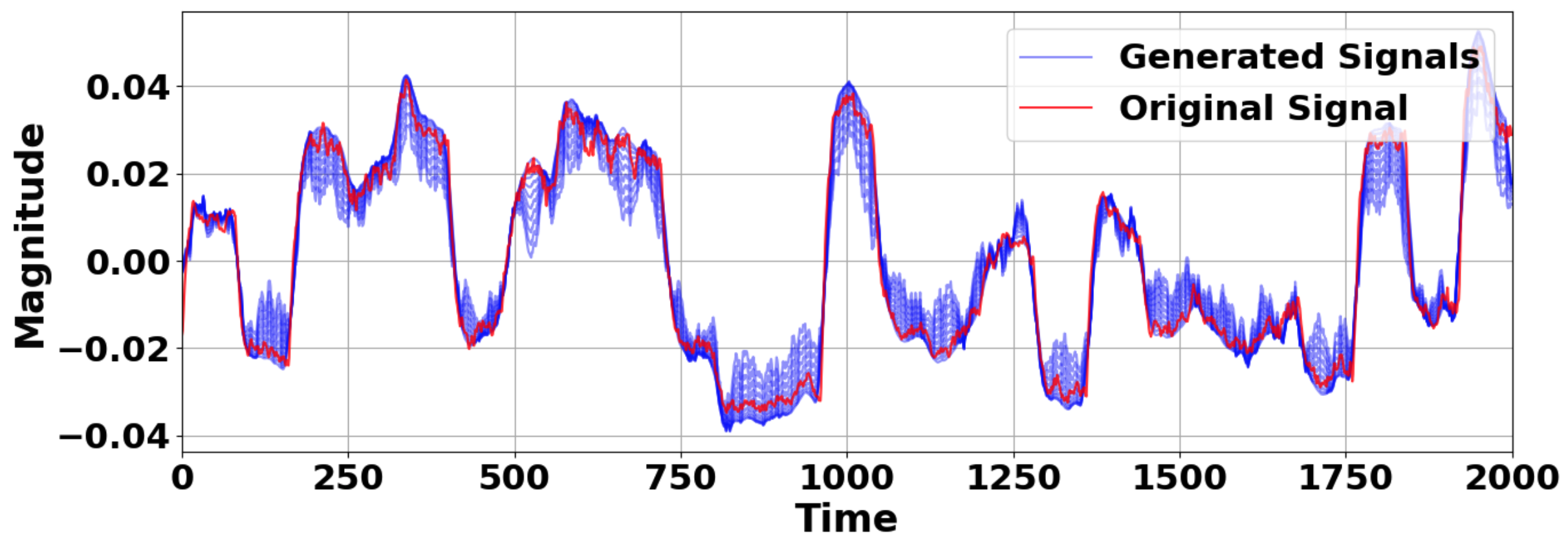}}

  \subfigure[Iterating over a range of 20 values, $(t)_{i=0}^19$, to generate multiple different signals between the two extremal flipped signal orientations, by continuously traversing the great circle of a hypersphere.]{\includegraphics[width=0.7\textwidth]{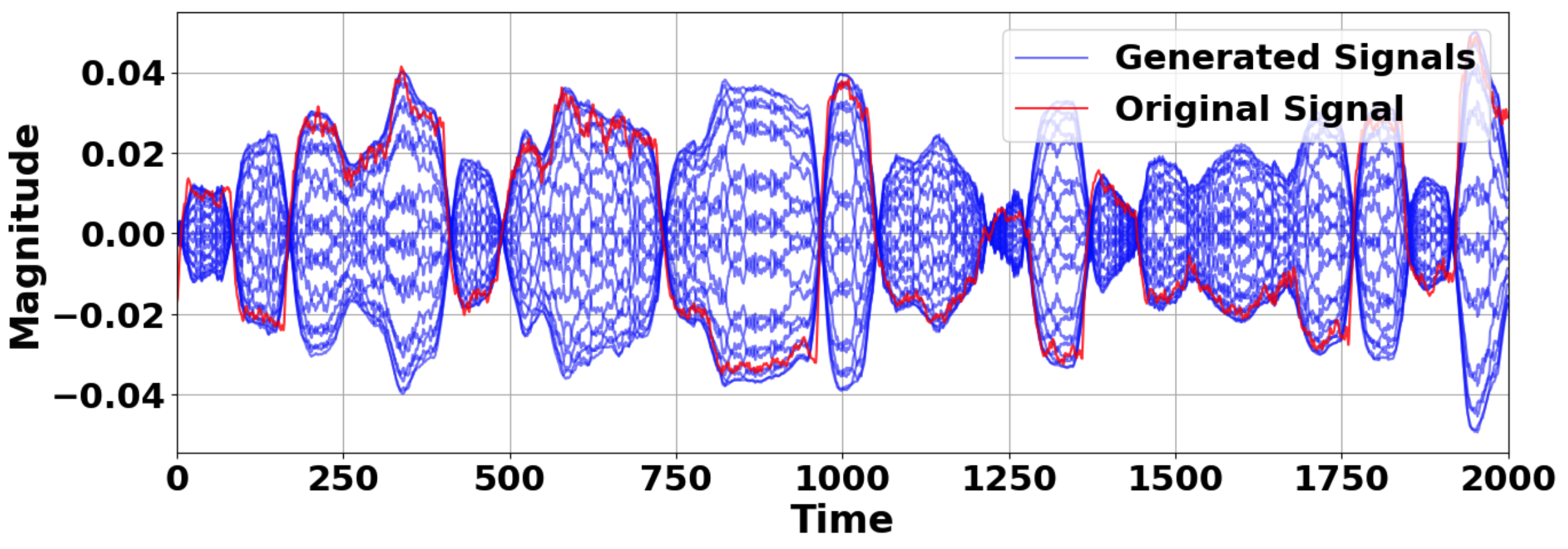}}
    \caption{Demonstration of \textit{SphereGen}.}
  \label{fig:sphere_geo_demo}
\end{figure}

\textbf{A Curse of (low) Dimensionality}

In this brief discussion, we underscore a point that, although seemingly obvious, bears significance and should be noted regardless. The majority of our research work presented thus far has focused on working with the Stiefel manifold in the sense of the orthogonal group geometries, namely $\text{St}_m^m = \mathcal{O}(m)$ for the case of $U$ and $\text{St}_n^n = \mathcal{O}(n)$ for the case of $V$. Briefly, we analyzed a corner case in the way of, $\text{St}_1^m$ and $\text{St}_1^n$, which correspond to hypersphere geometries.

Motivated by this insight, we delve into the application of the Stiefel manifold and its ability to work in a reduced rank setting. Specifically, $\text{St}_k^m$ for $k \ll m,n$. This approach becomes particularly relevant when dealing with vast data sets and databases if it is necessary to work with compressed representations of data. To illustrate this, we conduct an analysis on the Taxi dataset, considering the first 600 points (for visual clarity). The chosen dataset, known for its more periodic nature, serves to accentuate the impact of working with a lower-dimensional geometry ($\text{St}_2^30$ for $U$ and $\text{St}_2^20$ for $V$). Naturally then, the hyper parameters for the \textit{StiefelGen} reshape operation are set to $(m,n)=(30,20)$ for this analysis.

As a reference, we first explore the effect of a rank-2 signal SVD on the Taxi dataset without the application any Stiefel perturbation (Figure \ref{fig:rank2_base_case}) in order to get a base reference of the overall effect of the rank 2 reduction by itself. This initial SVD selection, where only the first two frames ($k=2$) are considered, sets the stage by highlighting how the signal is distorted before any Stiefel perturbations.

\begin{figure}[H]
\centering
\includegraphics[width=0.7\textwidth]{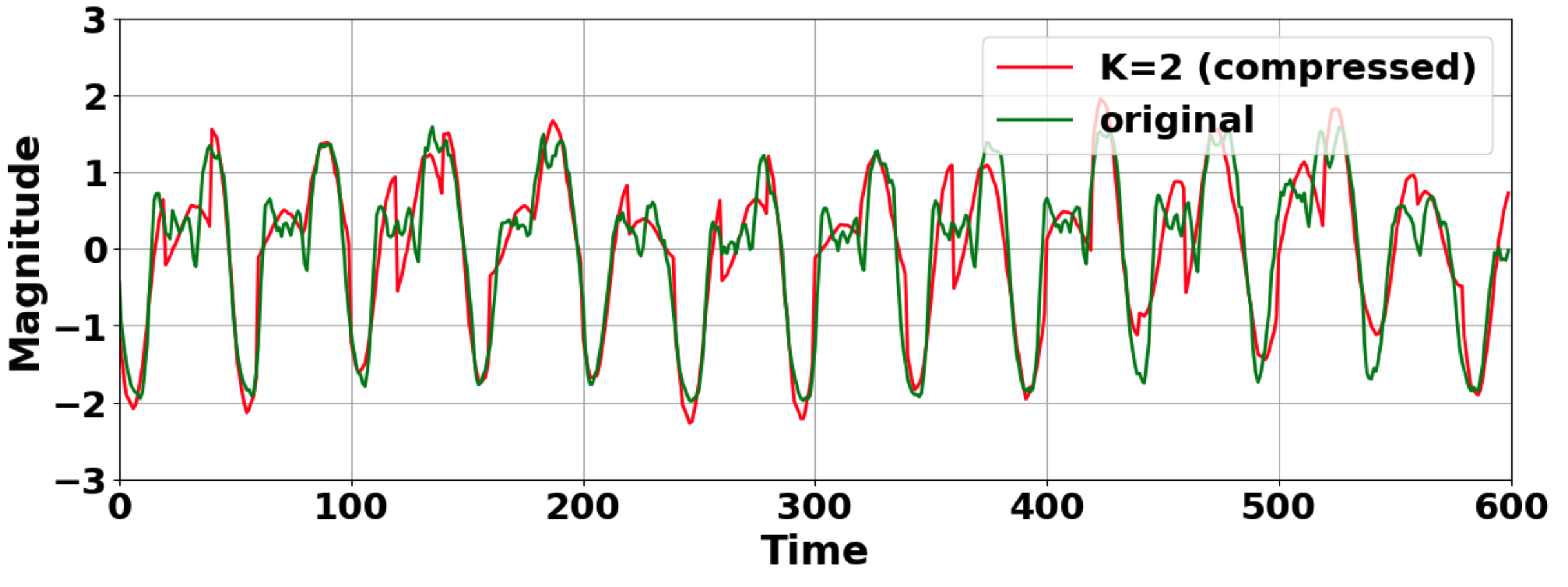}
\caption{The effect of considering a rank-2 signal SVD on the first 600 elements of the Taxi dataset with no Stiefel perturbation applied.}
\label{fig:rank2_base_case}
\end{figure}

As anticipated from Figure \ref{fig:rank2_base_case}, the rank-2 reconstruction of the original signal after the SVD introduces some deviation. We then shift our focus to a $\beta=0.2$ perturbation case, examining the impact with and without smoothing. Figure \ref{fig:rank2_no_smooth} clearly demonstrates that, at the same perturbation level, applying \textit{StiefelGen} in a rank-reduced setting results in more substantial and erratic changes compared to the full-rank system. The reasoning behind this is further clarified in Figure \ref{fig:sphere_reduced_rank}, which illustrates a 2-sphere embedded in $\mathbb{R}^3$ and a corresponding great circle curve on the sphere's surface. When perturbations over the 2-sphere are projected into a lower-rank setting (the great circle), they generally result in larger distances traveled in this low-rank setting. Additionally, more information is lost in the lower-dimensional manifold, leading to a significant reduction in finer-tuned ``per-coordinate'' control. The combination of these two factors is what leads to the increased erratic behaviour of the generated signals when working with compressed representations of the page matrix. 

\begin{figure}[htbp]
\centering
\subfigure[Comparing the impact of the generated signals with $\beta=0.2$, and with no smoothing applied ($\ell=1$).\label{fig:rank2_no_smooth}]{\includegraphics[width=0.7\textwidth]{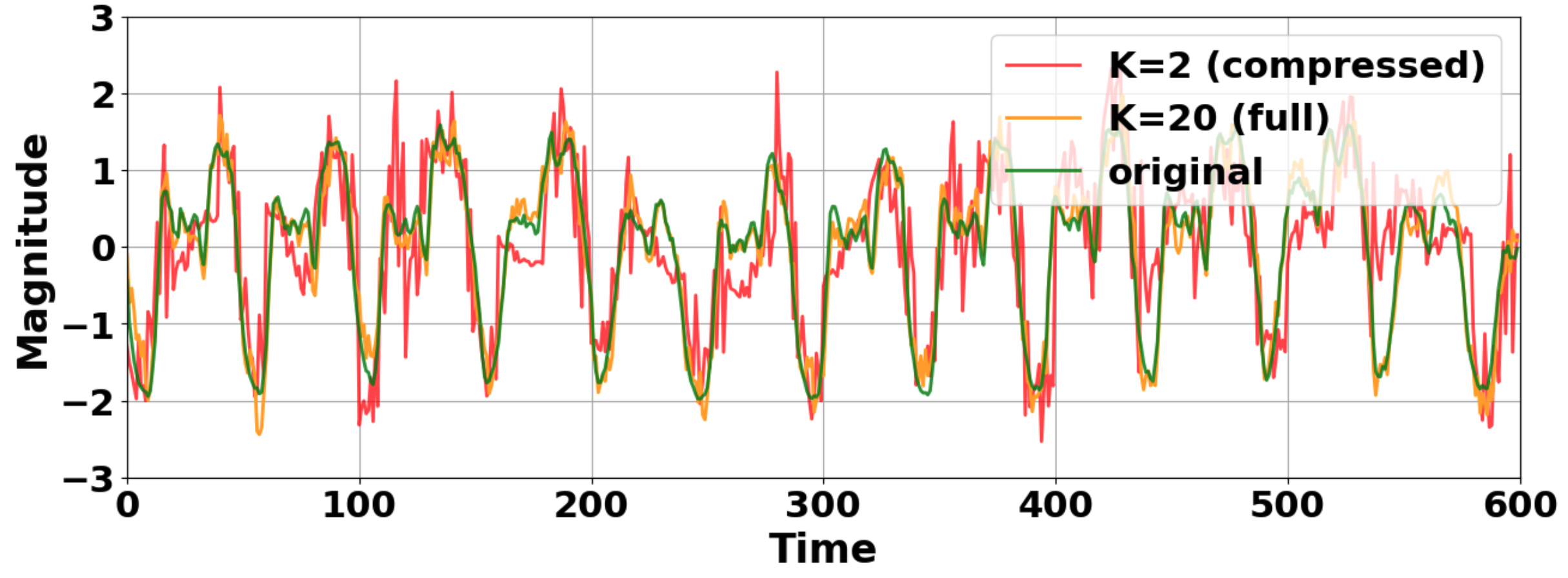}}

\subfigure[Comparing the impact of the generated signals with $\beta=0.2$, and with smoothing applied ($\ell=4$).\label{fig:rank2_smooth}]{\includegraphics[width=0.7\textwidth]{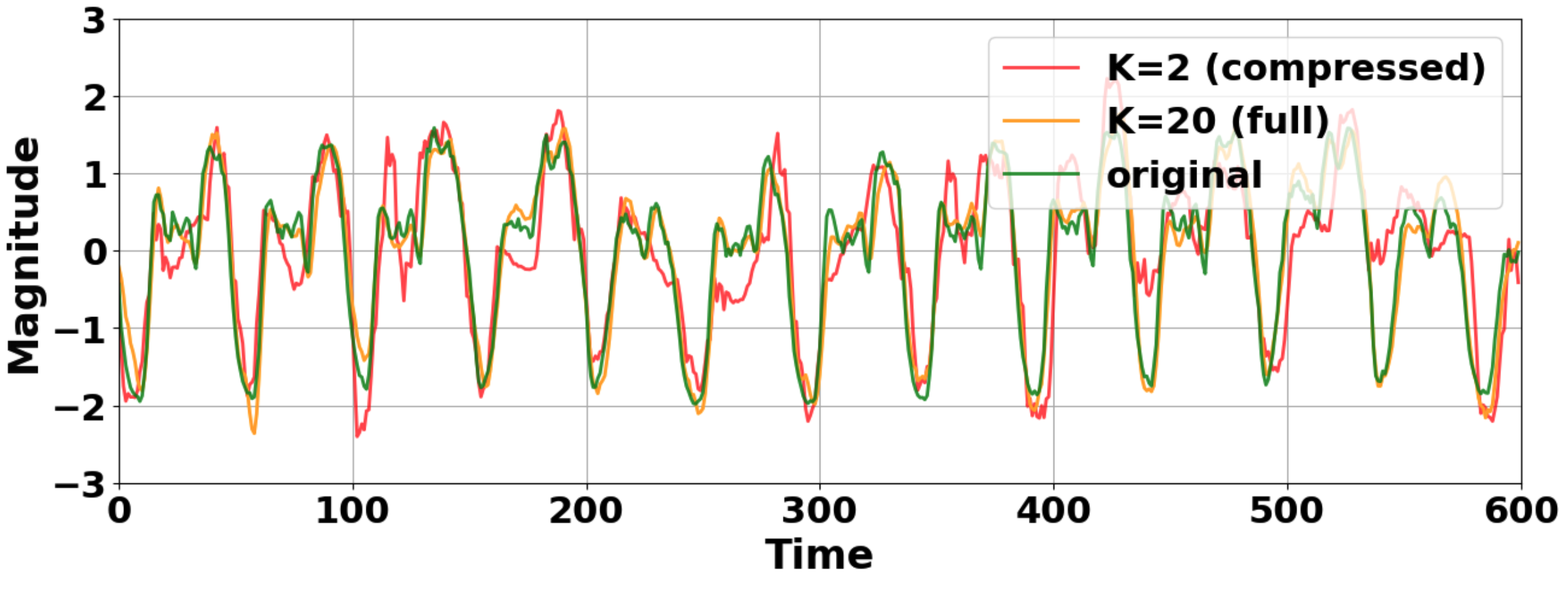}}

\subfigure[An illustration of how perturbation magnitudes are impacted when intrinsic geometric dimensions are reduced. \label{fig:sphere_reduced_rank}]{\includegraphics[width=1.15\textwidth]{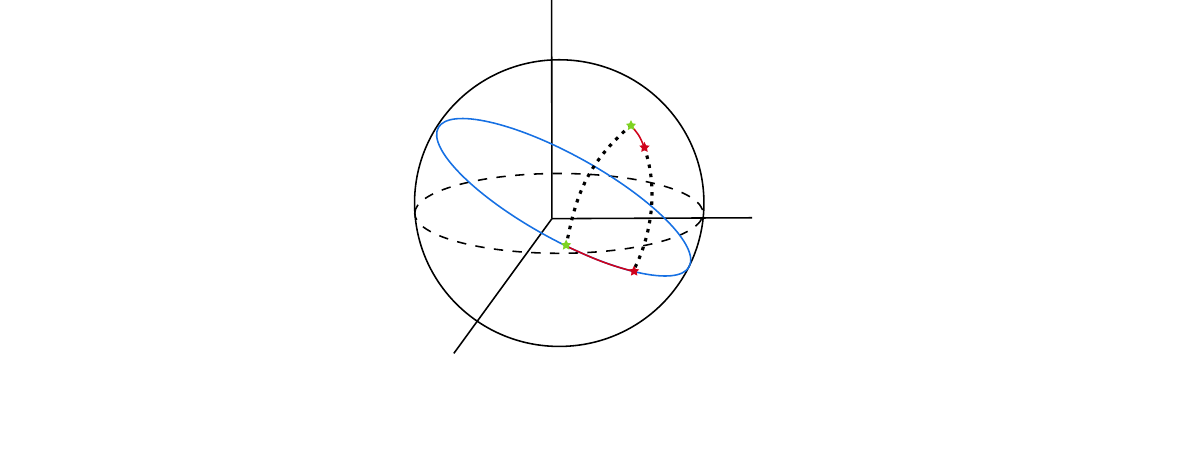}}

\caption{Comparing the impact of perturbation magnitude over the Stiefel manifold for a system that is rank 2 and full rank.}
\label{fig:rank2_cases}
\end{figure}

However, there is hope for scenarios requiring storage of rank-reduced representations. Firstly, one can simply to choose apply smaller $\beta$ perturbations. Additionally (and often recommended depending on context) with post-hoc smoothing, Figure \ref{fig:rank2_smooth} demonstrates that a reasonable signal can be generated from the original $(\beta,\ell)=(0.2,1)$ signal. This approach proves both simple and effective in mitigating the challenges associated with confining geodesics to sub-manifold structures, although naturally more complex pre, and post-hoc processing can always be applied depending on the problem functional requirements.

\textbf{Forecasting its Difficulties}

Here, we emphasize the importance of exercising caution when employing \textit{StiefelGen} in a loop for signal forecasting. It is crucial to understand that \textit{StiefelGen} is not inherently designed for time series forecasting; rather, its primary purpose lies in time series data augmentation, leveraging valuable geometric properties. For signals exhibiting weak statistical stationarity, generating signals and appending them to the current time series can be seen as a reasonable approach to forecasting to within certain short to medium term horizons. This approach serves as a naive form of data projection, allowing customization based on desired noise levels and basis function deformation within a specific geodesic distance from the source signal.

However, it is also essential to recognize the limitations of \textit{StiefelGen}, particularly when dealing with non-stationary time series. In order to demonstrate this point,   An example of this is in the publically available dataset from the ``Merchant Wholesalers, Except Manufacturers' Sales Branches and Offices'' for ``Nondurable Goods: Beer, Wine, and Distilled Alcoholic Beverages Sales'' \cite{alcohol}, which keeps track of monthly alcohol sales in millions of dollars (although here the signal has been standerdized before hand). Evidently from Figure \ref{fig:alcohol}, it is seen that the signal is highly non stationary, characterized by increasing variance and mean over time. While simple differencing schemes \cite{hyndman2018forecasting} combined with \textit{StiefelGen} augmentations may suffice for stationary signals, such an approach is ill-suited for non-stationary signals, as demonstrated in Figure \ref{fig:alcohol}.

\begin{figure}[htbp]
  \centering
  \subfigure[The misalignment of forecasted alcohol data and the appended time series highlights \textit{StiefelGen}'s inadequacy in addressing signal non-stationarity.]{\includegraphics[width=0.48\textwidth]
  {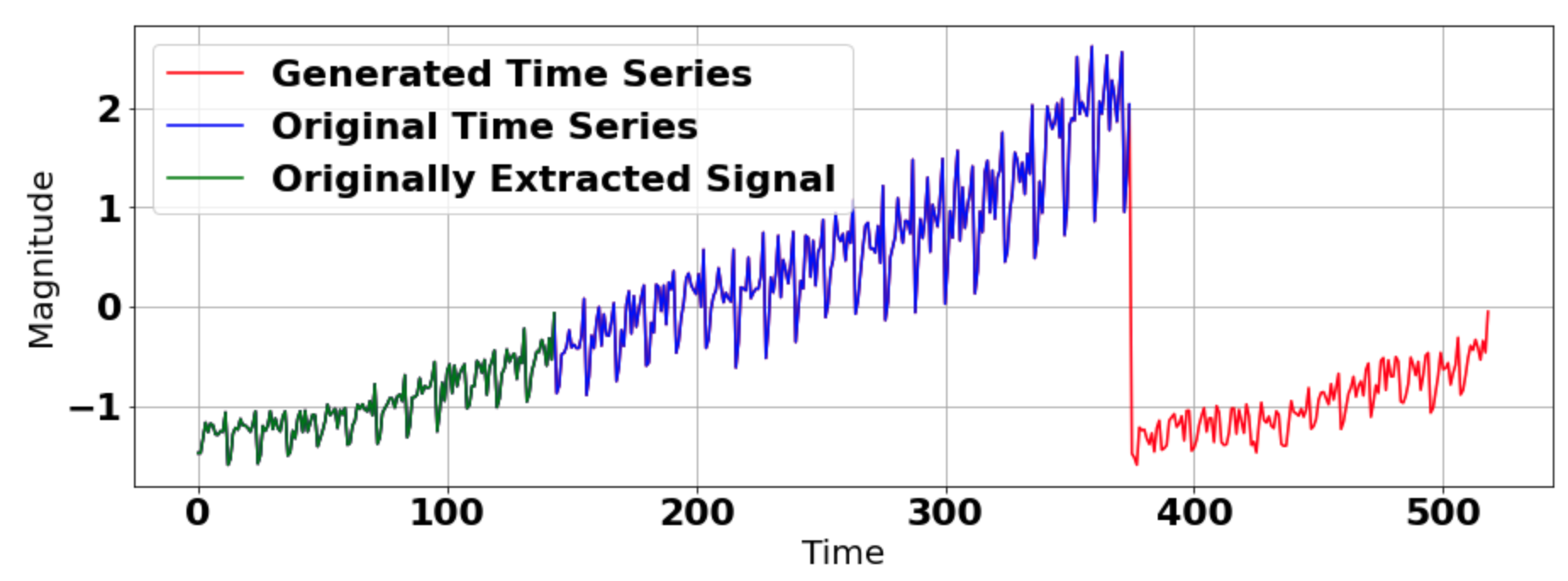}}
 \hspace{0.01\textwidth} 
  \subfigure[Efficiency verification of \textit{StiefelGen} as a signal augmentor through an overplot of the generated signal with the original portion of the alcohol data set.]{\includegraphics[width=0.48\textwidth]
  {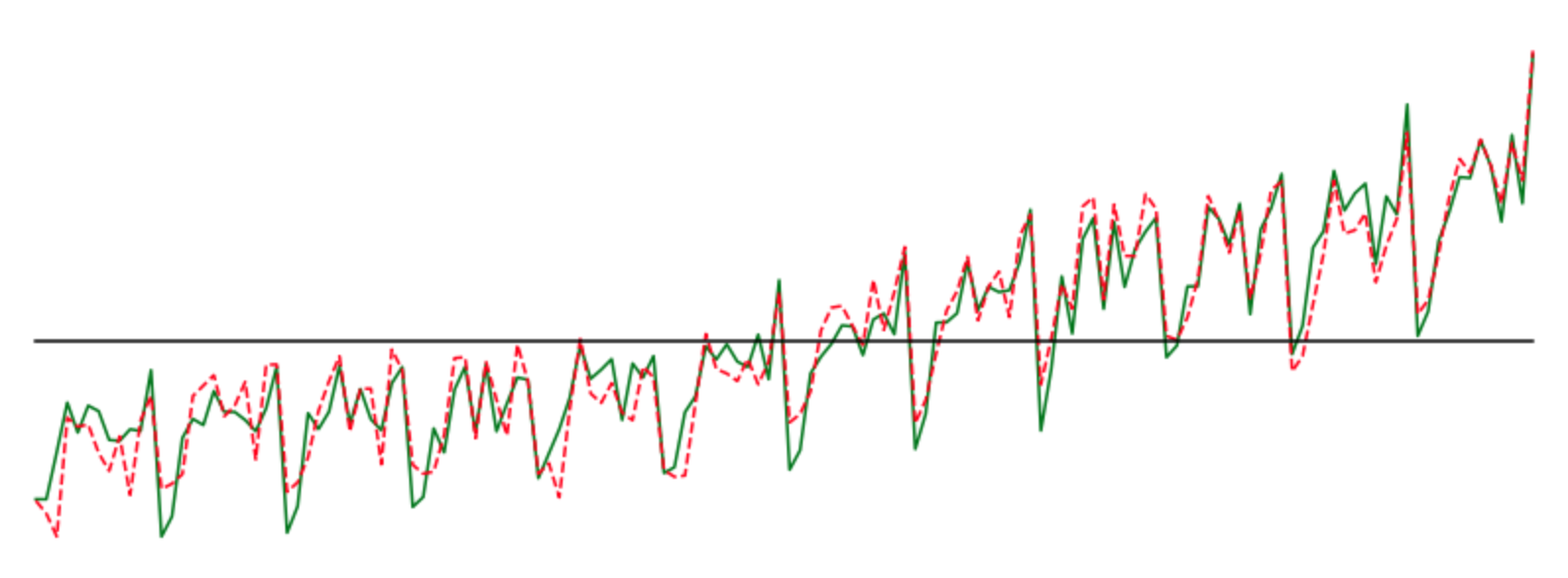}}
    \caption{Detailed investigation illustrating the careful use of \textit{StiefelGen} in forecasting problems.}
  \label{fig:alcohol}
\end{figure}

However, recall that in Section \ref{sec:UQ_DMD}, we demonstrated the capability of \textit{StiefelGen} to act as an intermediary for spatio-temporal signal forecasting. Notably, the dynamic mode decomposition (DMD) algorithm employs the singular value decomposition (SVD) in its physics loop. Thus, it is through this intermediate step that \textit{StiefelGen} seamlessly integrates to perturb the generative signal physics. Innovations building upon the na\"ive DMD algorithm presented in this paper, which utilize the SVD in its loop, suggest that \textit{StiefelGen} (or at least the principles behind its approach) could potentially serve as an invaluable tool to assist in the data augmentation of spatio-temporal signal forecasting, allowing for an additional layer of epistemic uncertainty quantification through functional data analysis across various problem settings \cite{kutz2016multiresolution, kutz2016dynamic}.

\textbf{\textit{StiefelGen} and Image Data}

In the realm of data augmentation, the spotlight has typically been shone on image data transformations. However, as this paper delves into the realm of time series data augmentation, a pertinent question emerges regarding the suitability of \textit{StiefelGen} for image data. Given that the page matrix is inherently a matrix, the theoretical possibility exists. However, practical considerations suggest that \textit{StiefelGen} may not outshine faster more suited alternatives, such as adversarial patches, hue and color changes, cropping, resizing, and rotation among others. These alternatives not only operate swiftly but also align with the natural spatial characteristics of image data.

It's crucial to note that the current implementation of the \textit{StiefelGen} algorithm operates independently along the row and column axes, a deliberate design choice enabled by the SVD's ability to separate row and column spaces. However, this independence contrasts with the spatial awareness required for image data, where the row and column surrounding a pixel should not be treated independently. To intuit this point further, from a deep learning perspective, working with images through convolutional layers is often more desirable than reshaping all images into 1D signals and stacking them, as most spatial awareness becomes lost in the process.

An illustrative example of the impact of row-column independence on applying Stiefel manifold constraints to image data is depicted in Figure \ref{fig:mnist}. The smooth transformation of the number zero into the number two over ten steps showcases the challenge, as the image quickly loses coherence and blurs between zero and two. While this unique ``blurring'' effect might eventually find applications in image data augmentation, the abundance of alternative, simple, interpretable, and spatially aware approaches questions the need to integrate \textit{StiefelGen} into the augmentation pipeline for image data in its current implementation.

\begin{figure}[H]
\centering
{\includegraphics[width=0.6\textwidth]{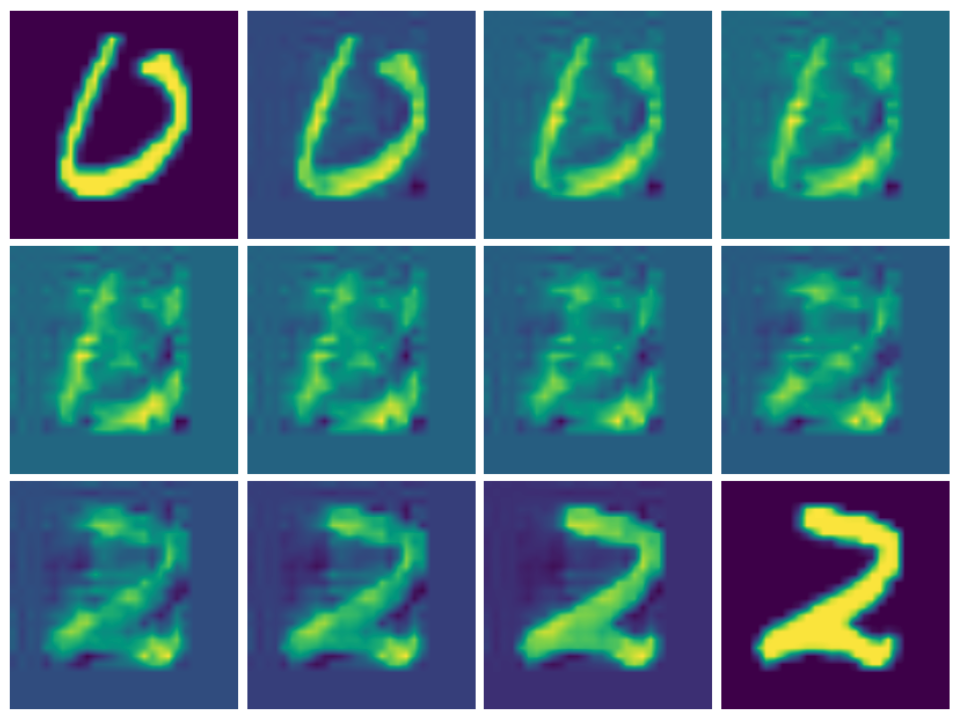}}
\caption{Following the Stiefel geodesic for two MNIST samples.}
\label{fig:mnist}
\end{figure}

On the contrary, a closely related mathematical structure, the Grassmann manifold, has found utility in image-based data due to its compatibility with symmetry. While briefly introduced in Section \ref{sec:background_math}, the Grassmann manifold will be slightly further explored in Section \ref{app_sec:Understanding_Grassmann}, as it holds a close relationship with the Stiefel manifold, acting as a defining mechanism through which the Stiefel canonical metric can be formed, and since it serves as a foundational concept for understanding matrix geometry as a whole.

In essence, the Grassmann manifold represents linear subspaces, distinguishing it from the Stiefel manifold, which instantiates basis vectors on a linear subspace. In essence, the Stiefel manifold corresponds to a ``point on a line'', whereas the Grassmann manifold represents ``the entire line'', collapsing all points on the line into a single equivalent entity. This collapse can be formulated as a mathematical equivalence relation, which ultimately proves useful for encoding shapes with general affine symmetry (such as rotational symmetry) as equivalent objects.

An application of the Grassmann manifold in image data is evident in Figure \ref{fig:turaga}, where shapes with rotational symmetry are treated as equivalent objects. In a study by Turaga et al. \cite{turaga2008statistical}, shapes from the MPEG-7 database were subjected to additive noise, and parametric Langevin distributions were trained on the Grassmann manifold for each class. Samples from this distribution represented subspaces in the form of projection matrices, and actual shapes were generated by selecting a 2-frame for the subspace using SVD of the projection matrix (Stiefel manifold). Random coordinates within the 2-frame facilitated symmetric shape generation.

\begin{figure}[H]
\centering
{\includegraphics[width=0.6\textwidth]{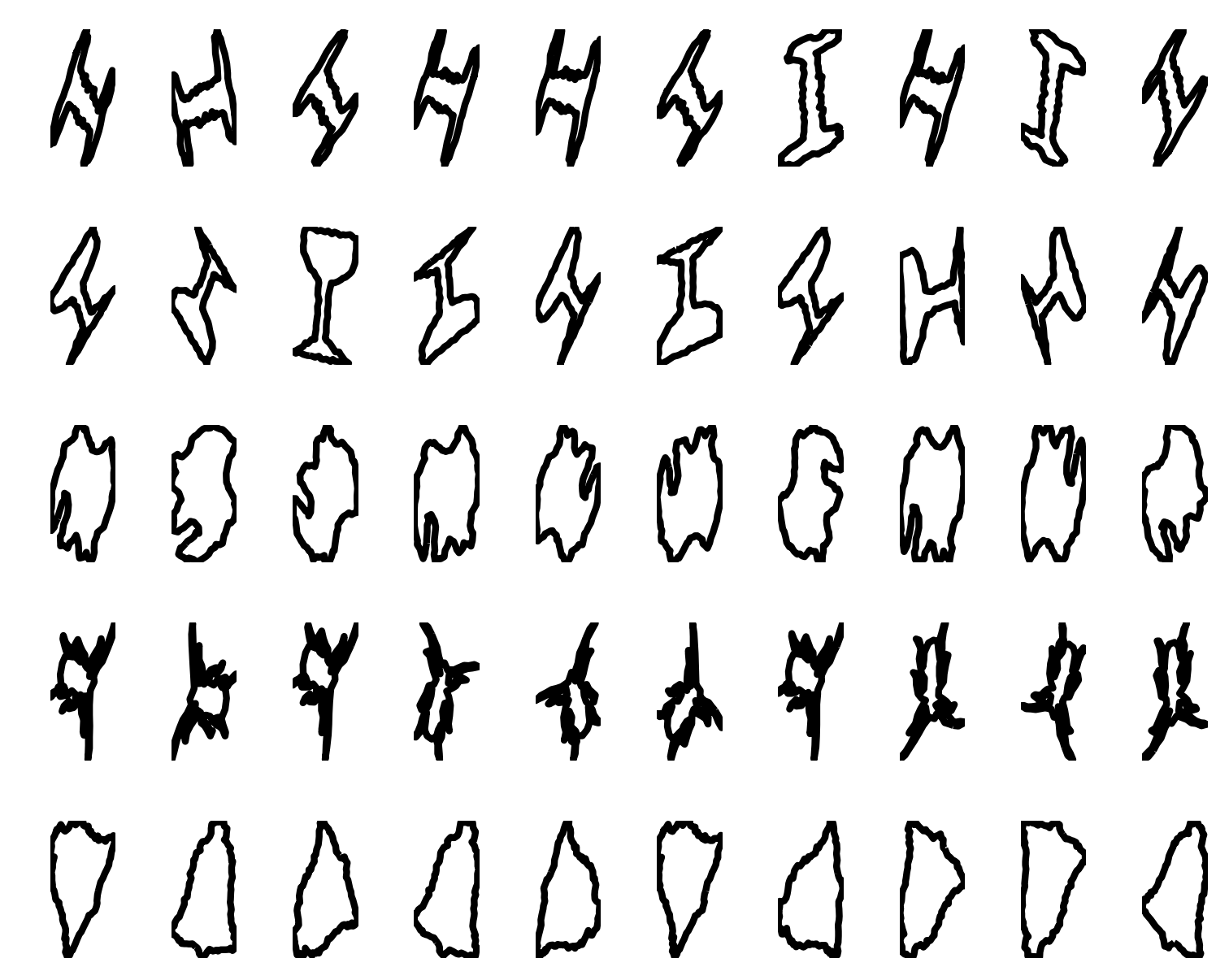}}
\caption{Samples drawn from a class conditional density using an MPEG dataset.}
\label{fig:turaga}
\end{figure}

Thus, while the Stiefel manifold excels in representing a specific orthonormal basis, suitable for scenarios where the basis is crucial for model representation and comparison, the Grassmann manifold shines when the focus is on the invariant or equivalence class of subspaces. The Grassmann abstraction proves more appropriate when the subspace or column space remains invariant regardless of the chosen basis vectors. This distinction also makes Grassmann manifolds natural candidates for applications involving spatio-temporal data modeling, shape analysis, and online visual learning in computer vision. For example, in spatio-temporal dynamical models, the Grassmann manifold is instrumental in comparing different models using subspace angles between observability matrices, complementing the Stiefel manifold's role in representing model parameters \cite{turaga2008statistical}.

Ultimately, the choice between the Grassmann and Stiefel manifolds depends on the nature of the data and the analytical goals at hand. Each manifold presents a unique set of advantages and is a valuable tool in the arsenal of geometric data analysis.

\newpage
\subsection{Full Table of Values: LSTM Experiment}\label{app_sec:LSTM_Table}

This subsection provides the quantitative summary of all the statistical values used to generate Figure \ref{fig:NN_results}. The bold values in each row indicate the experimental run which received the highest testing score for that row (which is often strongly correlated with the highest training scores for that row as well). It is clear that every single experimental augmentation setting outperformed the corresponding non-augmented experiment by a very large margin. The largest improvement was observed in the experiment with the experiment in Table \ref{tab:30times} which used a perturbation level of 10\% and the first $N=30$ data elements per class for the augmentation procedure (which coincidentally is all the available data), and then perturbing this 30 times over per class. This resulted in a mean training and test score of 96.35\% and 93.55\% respectively. This is more than double the test score of the vanilla, non-augmented model which had its best result occur similarly at $N=30$, with the training score being 50.02\%, and its testing score of 41.53\%.

\begin{table}[htbp]
\centering
\caption{The result of using the first $N$ elements from the each class of the data set, where $N=[5,10,15,20,25,30]$. This is the base reference set of training and testing accuracies which have not received any data augmentation.}
\label{tab:times1}
\resizebox{\textwidth}{!}{%
\begin{tabular}{cccccccc}
\hline
\textbf{Pert. Level {[}\%{]}} & \textbf{Accuracy {[}\%{]}} & \textbf{$5(\times 1)$} & \textbf{$10(\times 1)$} & \textbf{$15(\times 1)$} & \textbf{$20(\times 1)$} & \textbf{$25(\times 1)$} & \textbf{$30(\times 1)$} \\ \hline
\multirow{2}{*}{\textbf{0}} & \textbf{Train} & 32.78$\pm$1.45 & 38.67$\pm$1.83 & 42.26$\pm$2.43 & 49.17$\pm$1.95 & 49.13$\pm$2.40 & \textbf{50.02$\pm$2.22} \\
 & \textbf{Test} & 22.42$\pm$1.33 & 32.64$\pm$1.39 & 38.39$\pm$1.87 & 40.77$\pm$1.86 & 39.76$\pm$2.03 & \textbf{41.53$\pm$2.33} \\ \hline
\end{tabular}%
}
\end{table}

\begin{table}[htbp]
\centering
\caption{Results of applying \textit{StiefelGen} with various perturbation levels ($[5\%, 10\%, 15\%, 30\%]$, across various amounts of data taken from the dataset for generation $[5,10,15,20,25,30]$ for augmentation, assuming 5 times the amount of data is generated.}
\resizebox{\textwidth}{!}{%
\begin{tabular}{cccccccc}

\hline
\textbf{Pert. Level {[}\%{]}} & \textbf{Accuracy {[}\%{]}} & \textbf{$5(\times 5)$} & \textbf{$10(\times 5)$} & \textbf{$15(\times 5)$} & \textbf{$20(\times 5)$} & \textbf{$25(\times 5)$} & \textbf{$30(\times 5)$} \\ \hline
\multirow{2}{*}{\textbf{5}} & \textbf{Train} & 55.04$\pm$2.14 & 61.48$\pm$3.14 & \textbf{79.82$\pm$2.26} & 67.58$\pm$3.94 & 75.45$\pm$2.29 & 73.09$\pm$4.55 \\
 & \textbf{Test} & 33.38$\pm$1.65 & 45.54$\pm$2.37 & \textbf{71.80$\pm$2.14} & 58.82$\pm$3.72 & 69.07$\pm$2.70 & 69.39$\pm$3.77 \\ \hline
\multirow{2}{*}{\textbf{10}} & \textbf{Train} & 49.58$\pm$2.98 & 61.31$\pm$3.76 & \textbf{79.33$\pm$2.57} & 66.78$\pm$3.98 & 73.27$\pm$3.15 & 68.78$\pm$3.70 \\
 & \textbf{Test} & 31.73$\pm$1.75 & 47.99$\pm$2.90 & \textbf{72.51$\pm$2.32} & 62.96$\pm$4.08 & 67.99$\pm$3.01 & 66.66$\pm$3.82 \\ \hline
\multirow{2}{*}{\textbf{15}} & \textbf{Train} & 50.93$\pm$3.45 & 59.46$\pm$2.19 & 74.36$\pm$2.86 & 64.49$\pm$3.03 & 74.58$\pm$2.83 & \textbf{78.61$\pm$2.30} \\
 & \textbf{Test} & 31.19$\pm$1.74 & 47.46$\pm$1.81 & 70.30$\pm$2.72 & 64.15$\pm$3.27 & 73.14$\pm$2.58 & \textbf{78.07$\pm$2.29} \\ \hline
\multirow{2}{*}{\textbf{30}} & \textbf{Train} & 53.91$\pm$3.17 & 48.48$\pm$3.13 & 68.96$\pm$1.44 & 65.31$\pm$1.97 & 67.22$\pm$2.65 & \textbf{69.60$\pm$1.92} \\
 & \textbf{Test} & 34.15$\pm$1.48 & 41.05$\pm$2.86 & 66.80$\pm$1.67 & 65.68$\pm$2.61 & 67.18$\pm$3.14 & \textbf{70.70$\pm$1.90} \\ \hline
\end{tabular}%
}
\label{tab:5times}
\end{table}

\begin{table}[htbp]
\centering
\caption{Results of applying \textit{StiefelGen} with various perturbation levels ($[5\%, 10\%, 15\%, 30\%]$, across various amounts of data taken from the dataset for generation $[5,10,15,20,25,30]$ for augmentation, assuming 10 times the amount of data is generated.}
\resizebox{\textwidth}{!}{%
\begin{tabular}{cccccccc}
\hline
\textbf{Pert. Level {[}\%{]}} & \textbf{Accuracy {[}\%{]}} & \textbf{$5(\times 10)$} & \textbf{$10(\times 10)$} & \textbf{$15(\times 10)$} & \textbf{$20(\times 10)$} & \textbf{$25(\times 10)$} & \textbf{$30(\times 10)$} \\ \hline
\multirow{2}{*}{\textbf{5}} & \textbf{Train} & 63.39$\pm$3.66 & 85.22$\pm$2.14 & 77.24$\pm$5.26 & \textbf{84.41$\pm$4.18} & 72.50$\pm$3.88 & 82.39$\pm$4.71 \\
 & \textbf{Test} & 39.80$\pm$2.20 & 65.66$\pm$2.80 & 70.16$\pm$4.79 & \textbf{79.14$\pm$3.64} & 66.97$\pm$4.13 & 78.89$\pm$4.66 \\ \hline
\multirow{2}{*}{\textbf{10}} & \textbf{Train} & 63.19$\pm$2.46 & 80.43$\pm$3.26 & 81.37$\pm$1.72 & 79.02$\pm$2.82 & 79.12$\pm$4.58 & \textbf{87.77$\pm$1.86} \\
 & \textbf{Test} & 39.69$\pm$2.09 & 63.31$\pm$3.27 & 77.66$\pm$1.41 & 76.01$\pm$2.58 & 73.84$\pm$4.82 & \textbf{84.54$\pm$2.04} \\ \hline
\multirow{2}{*}{\textbf{15}} & \textbf{Train} & 61.12$\pm$2.20 & 80.14$\pm$2.97 & 85.63$\pm$2.05 & 81.93$\pm$2.88 & 85.50$\pm$3.34 & \textbf{86.40$\pm$2.43} \\
 & \textbf{Test} & 41.35$\pm$1.82 & 64.97$\pm$2.89 & 81.23$\pm$1.92 & 78.77$\pm$3.14 & 82.81$\pm$3.50 & \textbf{83.86$\pm$2.80} \\ \hline
\multirow{2}{*}{\textbf{30}} & \textbf{Train} & 52.02$\pm$2.37 & 71.56$\pm$3.06 & 73.16$\pm$2.10 & 68.24$\pm$3.56 & 72.23$\pm$3.58 & \textbf{79.53$\pm$2.06} \\
 & \textbf{Test} & 36.81$\pm$1.82 & 57.57$\pm$2.55 & 74.11$\pm$2.26 & 69.27$\pm$4.08 & 74.77$\pm$3.86 & \textbf{81.01$\pm$2.33} \\ \hline
\end{tabular}%
}

\label{tab:10times}
\end{table}

\begin{table}[htbp]
\centering
\caption{Results of applying \textit{StiefelGen} with various perturbation levels ($[5\%, 10\%, 15\%, 30\%]$, across various amounts of data taken from the dataset for generation $[5,10,15,20,25,30]$ for augmentation, assuming 15 times the amount of data is generated.}
\resizebox{\textwidth}{!}{%
\begin{tabular}{cccccccc}
\hline
\textbf{Pert. Level {[}\%{]}} & \textbf{Accuracy {[}\%{]}} & \textbf{$5(\times 15)$} & \textbf{$10(\times 15)$} & \textbf{$15(\times 15)$} & \textbf{$20(\times 15)$} & \textbf{$25(\times 15)$} & \textbf{$30(\times 15)$} \\ \hline
\multirow{2}{*}{\textbf{5}} & \textbf{Train} & 88.53$\pm$1.40 & 88.53$\pm$4.04 & 90.69$\pm$2.22 & 85.86$\pm$3.20 & 87.00$\pm$2.97 & \textbf{88.92$\pm$3.06} \\
 & \textbf{Test} & 55.57$\pm$2.05 & 72.95$\pm$3.07 & 83.07$\pm$1.86 & 79.54$\pm$3.02 & 82.92$\pm$2.71 & \textbf{85.66$\pm$2.97} \\ \hline
\multirow{2}{*}{\textbf{10}} & \textbf{Train} & 87.79$\pm$2.63 & 91.21$\pm$1.94 & \textbf{92.45$\pm$2.16} & 89.30$\pm$3.71 & 84.71$\pm$3.93 & 87.92$\pm$2.84 \\
 & \textbf{Test} & 52.01$\pm$2.42 & 75.39$\pm$1.68 & \textbf{85.91$\pm$1.74} & 84.88$\pm$3.78 & 82.12$\pm$3.73 & 85.16$\pm$3.04 \\ \hline
\multirow{2}{*}{\textbf{15}} & \textbf{Train} & 88.13$\pm$1.44 & 88.27$\pm$1.90 & 87.95$\pm$2.16 & 83.27$\pm$3.12 & \textbf{91.61$\pm$1.36} & 83.68$\pm$3.77 \\
 & \textbf{Test} & 56.31$\pm$2.63 & 71.76$\pm$2.52 & 84.09$\pm$1.46 & 79.43$\pm$3.08 & \textbf{88.82$\pm$1.25} & 83.27$\pm$3.31 \\ \hline
\multirow{2}{*}{\textbf{30}} & \textbf{Train} & 88.27$\pm$1.64 & 84.53$\pm$1.45 & 80.90$\pm$2.32 & 78.09$\pm$2.91 & 75.22$\pm$2.49 & \textbf{76.64$\pm$3.02} \\
 & \textbf{Test} & 52.43$\pm$1.64 & 68.91$\pm$1.72 & 79.46$\pm$2.28 & 78.69$\pm$3.29 & 75.38$\pm$2.92 & \textbf{80.34$\pm$2.66} \\ \hline
\end{tabular}%
}
\label{tab:15times}
\end{table}

\begin{table}[htbp]
\centering
\caption{Results of applying \textit{StiefelGen} with various perturbation levels ($[5\%, 10\%, 15\%, 30\%]$, across various amounts of data taken from the dataset for generation $[5,10,15,20,25,30]$ for augmentation, assuming 20 times the amount of data is generated.}
\resizebox{\textwidth}{!}{%
\begin{tabular}{cccccccc}
\hline
\textbf{Pert. Level {[}\%{]}} & \textbf{Accuracy {[}\%{]}} & \textbf{$5(\times 20)$} & \textbf{$10(\times 20)$} & \textbf{$15(\times 20)$} & \textbf{$20(\times 20)$} & \textbf{$25(\times 20)$} & \textbf{$30(\times 20)$} \\ \hline
\multirow{2}{*}{\textbf{5}} & \textbf{Train} & 88.68$\pm$2.19 & 94.61$\pm$1.45 & 89.26$\pm$3.98 & \textbf{94.28$\pm$2.14} & 89.44$\pm$3.29 & 82.52$\pm$4.10 \\
 & \textbf{Test} & 57.05$\pm$2.81 & 77.22$\pm$1.74 & 81.23$\pm$3.58 & \textbf{88.28$\pm$1.72} & 84.36$\pm$3.19 & 78.66$\pm$4.22 \\ \hline
\multirow{2}{*}{\textbf{10}} & \textbf{Train} & 92.33$\pm$2.25 & 95.45$\pm$1.47 & 92.92$\pm$1.78 & 85.67$\pm$4.23 & \textbf{91.09$\pm$1.95} & 88.95$\pm$2.20 \\
 & \textbf{Test} & 61.20$\pm$1.82 & 77.31$\pm$1.57 & 85.32$\pm$1.44 & 80.78$\pm$4.42 & \textbf{87.04$\pm$1.91} & 85.92$\pm$2.23 \\ \hline
\multirow{2}{*}{\textbf{15}} & \textbf{Train} & 84.57$\pm$3.11 & 92.15$\pm$1.79 & 90.39$\pm$1.70 & 85.23$\pm$2.79 & 81.85$\pm$3.29 & \textbf{92.99$\pm$2.02} \\
 & \textbf{Test} & 56.85$\pm$2.47 & 74.78$\pm$2.07 & 84.31$\pm$1.92 & 82.70$\pm$2.38 & 79.34$\pm$3.62 & \textbf{90.77$\pm$2.18} \\ \hline
\multirow{2}{*}{\textbf{30}} & \textbf{Train} & 84.48$\pm$2.57 & 87.91$\pm$1.29 & 86.87$\pm$1.36 & 85.56$\pm$1.79 & 82.22$\pm$2.24 & \textbf{88.07$\pm$1.28} \\
 & \textbf{Test} & 51.23$\pm$2.83 & 70.70$\pm$1.66 & 84.66$\pm$1.29 & 84.18$\pm$2.19 & 83.51$\pm$2.41 & \textbf{89.66$\pm$1.18} \\ \hline
\end{tabular}%
}
\label{tab:20times}
\end{table}

\begin{table}[htbp]
\centering
\caption{Results of applying \textit{StiefelGen} with various perturbation levels ($[5\%, 10\%, 15\%, 30\%]$, across various amounts of data taken from the dataset for generation $[5,10,15,20,25,30]$ for augmentation, assuming 25 times the amount of data is generated.}
\label{tab:25times}
\resizebox{\textwidth}{!}{%
\begin{tabular}{cccccccc}
\hline
\textbf{Pert. Level {[}\%{]}} & \textbf{Accuracy {[}\%{]}} & \textbf{$5(\times 25)$} & \textbf{$10(\times 25)$} & \textbf{$15(\times 25)$} & \textbf{$20(\times 25)$} & \textbf{$25(\times 25)$} & \textbf{$30(\times 25)$} \\ \hline
\multirow{2}{*}{\textbf{5}} & \textbf{Train} & 93.91$\pm$1.75 & 93.42$\pm$2.37 & \textbf{96.86$\pm$1.12} & 94.12$\pm$2.17 & 81.53$\pm$4.23 & 88.95$\pm$4.53 \\
 & \textbf{Test} & 60.85$\pm$1.85 & 78.19$\pm$1.98 & \textbf{89.14$\pm$1.32} & 88.03$\pm$1.78 & 75.49$\pm$4.48 & 84.65$\pm$4.37 \\ \hline
\multirow{2}{*}{\textbf{10}} & \textbf{Train} & 93.63$\pm$2.33 & 93.16$\pm$1.95 & 93.07$\pm$1.85 & 90.99$\pm$2.33 & 93.21$\pm$1.50 & \textbf{94.66$\pm$2.06} \\
 & \textbf{Test} & 62.62$\pm$2.45 & 75.86$\pm$2.36 & 87.42$\pm$1.56 & 85.34$\pm$2.54 & 89.30$\pm$1.74 & \textbf{91.20$\pm$2.03} \\ \hline
\multirow{2}{*}{\textbf{15}} & \textbf{Train} & 93.33$\pm$1.55 & 95.33$\pm$1.57 & 93.59$\pm$1.36 & 93.16$\pm$1.30 & \textbf{95.02$\pm$0.97} & 90.26$\pm$2.21 \\
 & \textbf{Test} & 61.00$\pm$2.59 & 78.35$\pm$2.20 & 87.15$\pm$1.42 & 89.70$\pm$1.27 & \textbf{92.18$\pm$1.21} & 88.65$\pm$1.97 \\ \hline
\multirow{2}{*}{\textbf{30}} & \textbf{Train} & 94.04$\pm$1.05 & 88.90$\pm$1.65 & 89.87$\pm$1.56 & 89.64$\pm$1.11 & \textbf{89.65$\pm$0.85} & 87.40$\pm$1.17 \\
 & \textbf{Test} & 56.11$\pm$2.04 & 74.22$\pm$1.55 & 85.18$\pm$1.67 & 88.77$\pm$0.84 & \textbf{90.93$\pm$0.66} & 89.19$\pm$0.97 \\ \hline
\end{tabular}%
}
\end{table}

\begin{table}[htbp]
\centering
\caption{Results of applying \textit{StiefelGen} with various perturbation levels ($[5\%, 10\%, 15\%, 30\%]$, across various amounts of data taken from the dataset for generation $[5,10,15,20,25,30]$ for augmentation, assuming 30 times the amount of data is generated.}
\label{tab:30times}
\resizebox{\textwidth}{!}{%
\begin{tabular}{cccccccc}
\hline
\textbf{Pert. Level {[}\%{]}} & \textbf{Accuracy {[}\%{]}} & \textbf{$5(\times 30)$} & \textbf{$10(\times 30)$} & \textbf{$15(\times 30)$} & \textbf{$20(\times 30)$} & \textbf{$25(\times 30)$} & \textbf{$30(\times 30)$} \\ \hline
\multirow{2}{*}{\textbf{5}} & \textbf{Train} & 94.79 $\pm$1.70 & 93.94$\pm$3.04 & \textbf{97.14$\pm$1.08} & 87.02$\pm$4.49 & 93.76$\pm$2.43 & 88.88$\pm$2.76 \\
 & \textbf{Test} & 61.12$\pm$2.12 & 76.53$\pm$3.06 & \textbf{88.49$\pm$1.43} & 80.62$\pm$4.29 & 88.05$\pm$2.02 & 85.20$\pm$2.61 \\ \hline
\multirow{2}{*}{\textbf{10}} & \textbf{Train} & 93.73$\pm$2.11 & 94.98$\pm$1.46 & 95.60$\pm$1.58 & 89.96$\pm$3.15 & 94.28$\pm$1.84 & \textbf{96.35$\pm$1.01} \\
 & \textbf{Test} & 61.39$\pm$2.58 & 78.50$\pm$1.74 & 88.74$\pm$1.63 & 84.23$\pm$3.71 & 90.61$\pm$1.76 & \textbf{93.55$\pm$0.77} \\ \hline
\multirow{2}{*}{\textbf{15}} & \textbf{Train} & 92.28$\pm$2.22 & 96.79$\pm$1.27 & 94.24$\pm$1.71 & 91.29$\pm$4.44 & \textbf{94.09$\pm$1.43} & 93.16$\pm$1.49 \\
 & \textbf{Test} & 61.97$\pm$2.56 & 80.49$\pm$2.10 & 89.46$\pm$1.13 & 87.35$\pm$4.25 & \textbf{91.43$\pm$1.48} & 90.49$\pm$1.72 \\ \hline
\multirow{2}{*}{\textbf{30}} & \textbf{Train} & 93.67$\pm$1.57 & 91.20$\pm$1.38 & 88.14$\pm$1.65 & 88.01$\pm$2.29 & 87.68$\pm$1.57 & \textbf{90.47$\pm$0.88} \\
 & \textbf{Test} & 59.42$\pm$2.22 & 73.43$\pm$1.73 & 85.14$\pm$1.32 & 87.61$\pm$2.61 & 88.09$\pm$1.31 & \textbf{91.14$\pm$0.83} \\ \hline
\end{tabular}%
}
\end{table}

\newpage
\subsection{Understanding the Stiefel Manifold} \label{app_sec:Understanding_Stiefel}

This subsection serves as an introductory point to matrix manifold theory, intending to provide a clearer understanding of the central themes explored in the paper. While it does not introduce much in the way of information, its purpose is to assist newcomers to the field as a whole. Due to it not being as widely used as other geometric structures, the subject of matrix manifolds can appear fragmented to those unfamiliar with it. Moreover, the intricate ideas from differentiable geometry may occasionally be challenging due to its extensive use of mathematical symbols and high dimensional abstractions. Our goal is to address this by consolidating fundamental concepts from various resources on matrix manifold theory, presenting them in a slightly more accessible manner. Additionally, we strive to incorporate more details and explanations in some of the mathematical derivations. The following subsections reflect a general consolidation of introductory information found in \cite{absil2008optimization, zimmermann2022computing, zimmermann2017matrix, zimmermann2019manifold, TagareNotes, bendokat2024grassmann, absil2004riemannian, baralic2011understand, guigui2023introduction}

\subsubsection{The Fundamentals: Riemannian Geometry} \label{app_sec:Riemannian_Intro}

A manifold, is a topological space that resembles a Euclidean space locally, but may have a more complicated global structure. Formally, it is a space that is locally homeomorphic (that is, it has a bijective, continuous mapping), to a Euclidean space. Broadly speaking, manifolds can be classified into different types based on additional structural requirements imposed upon them.

\begin{itemize}
    \item \textbf{Manifold: }A topological space that, locally around each point, looks like Euclidean space. It can be either smooth or not. If it is smooth, it means that it comes equipped with a smooth structure, allowing for the definition of differentiable functions on it, which leads onto differentiable manifolds.

    \item \textbf{Differentiable Manifold: }A manifold equipped with a smooth structure, meaning that it allows the definition of differentiable functions on it. Locally, around each point, the manifold is homeomorphic to an open set in Euclidean space, and the transition maps (a manner in which to compare these local Euclidean maps) between these local charts are themselves smooth. 

    \item \textbf{Riemannian Manifold :}A differentiable manifold equipped with a Riemannian metric. A Riemannian metric is a smoothly varying inner product structure defined on the tangent space at each point of the manifold. It essentially adds an additional layer of structure over the differentiable manifold, allowing more intuitive notions of geometry to be studied, such as, the measurement of distances, angles, and the definition of a notion of curvature on the manifold. 
\end{itemize}

Let's formalize this concept mathematically by introducing a differentiable manifold denoted as \(\mathcal{M}\). At any specific point \(p\in\mathcal{M}\), we can define a local tangent space, \(\mathcal{T}_p \mathcal{M}\), often referred to as the tangent plane at \(p\). This tangent space, \(\mathcal{T}_p \mathcal{M}\), serves as a vector space that offers a linear approximation to the manifold at that particular point. Essentially, it captures the local behavior of curves and vectors in the vicinity of \(p\).

In the study of differentiable manifolds, the focus often shifts to Riemannian manifolds, where a preferred metric tensor, denoted as \(g_p : \mathcal{T}_p \mathcal{M} \times \mathcal{T}_p \mathcal{M} \rightarrow \mathbb{R}\), comes into play. This metric tensor is a symmetric, positive definite, and smooth (differentiable) tensor field of rank 2, effectively representing a matrix. Notably, \(g_p\) depends on the specific point \(p\), leading to variations over \(\mathcal{M}\) based on location. It associates, at each point on the manifold, an inner product to every pair of tangent vectors local to \(\mathcal{M}\) at \(p\). While, in a more abstract sense, discussions often omit the explicit mention of the point \(p\) for notational simplicity. This crucial information desribing a Riemannian manifold is conventionally summarized as a tuple: \((\mathcal{M}, g)\).

Given the construction of a differentiable manifold, \(\mathcal{M}\), which allows for the placement of smooth curves, the introduction of a metric tensor \(g\) facilitates the analysis of these curves. It is worth noting that \(g\) plays a role in generalizing distances and angles over \(\mathcal{M}\). Consequently, this analysis leads to the definition of \textit{geodesics} over the corresponding Riemannian manifold, \((\mathcal{M}, g)\). These geodesics are determined by minimizing the functional:

\begin{align}
\mathcal{L} = \int_0^1 \|\dot{\gamma}(t)\|_g dt,
\end{align}

subject to an autoparallel condition:

\begin{align}
\nabla_{\dot{\gamma}}\dot{\gamma} = 0. \label{eqn:autoparallel}
\end{align}

In this context, the symbol \(\dot{\gamma}(t)\) represents the velocity, or the first derivative, of a smooth curve \(\gamma(t)\). The norm notation \(\|\cdot \|_g : T_p\mathcal{M} \rightarrow \mathbb{R}\) signifies the induced norm from the metric tensor \(g\) on the tangent space \(T_p\mathcal{M}\). Notice that there is also an additional geometric structure, denoted as \(\nabla\), introduced through Equation \ref{eqn:autoparallel}. Formally, \(\nabla\) defines a \textit{connection} over the manifold \((\mathcal{M}, g)\). In Riemannian geometry, a connection is a mathematical tool that facilitates the differentiation of vector fields along curves on a smooth manifold. It is an abstract concept employed to extend the notion of differentiation from Euclidean space to more generalized curved spaces.

By imposing certain assumptions and structure on \(\nabla\), particularly satisfying the Leibniz rule given by $
\nabla_X (\gamma Y) = (\gamma \nabla X) \cdot Y  +f\nabla_X Y$\footnote{This can be interpreted as the locally linear rate of change of \(Y\), scaled by the function \(\gamma\), in the direction of the vector field \(X\) (at a point \(p\) on the manifold \(\mathcal{M}\).}, for a smooth function \(\gamma\), and vector fields \(X\), and \(Y\), the concept of the \textit{covariant derivative} arises. This derivative serves as a generalization of the traditional ``directional derivative'' and notably does not explicitly depend on the metric tensor \(g\), remaining an abstract concept.

Further conditions, such as symmetry (\(\nabla_X Y = \nabla_Y X\)), a torsion-free application of \(\nabla\), and metric compatibility (\(\nabla_X g = 0\), implying that the inner product of any two vectors remains constant when parallel transported along curves on the manifold), lead to the establishment of the special Levi-Civita connection. This connection holds significance due to the uniqueness property it imposes on \(\nabla\) and provides a canonical method for comparing tangent vectors across different tangent spaces. While geodesics typically require only the condition of autoparallelism (Equation \ref{eqn:autoparallel}), stating that a tangent vector, when moved along the direction of itself, does not change, the additional metric compatibility condition (\(\nabla_X g = 0\)) ensures that these geodesics are indeed \textit{length-minimizing}. This property proves valuable in various practical settings. Figure \ref{fig:autoparallel} aids in consolidating these concepts.

\begin{figure}[htbp]
\centering
{\includegraphics[width=1.15\textwidth]{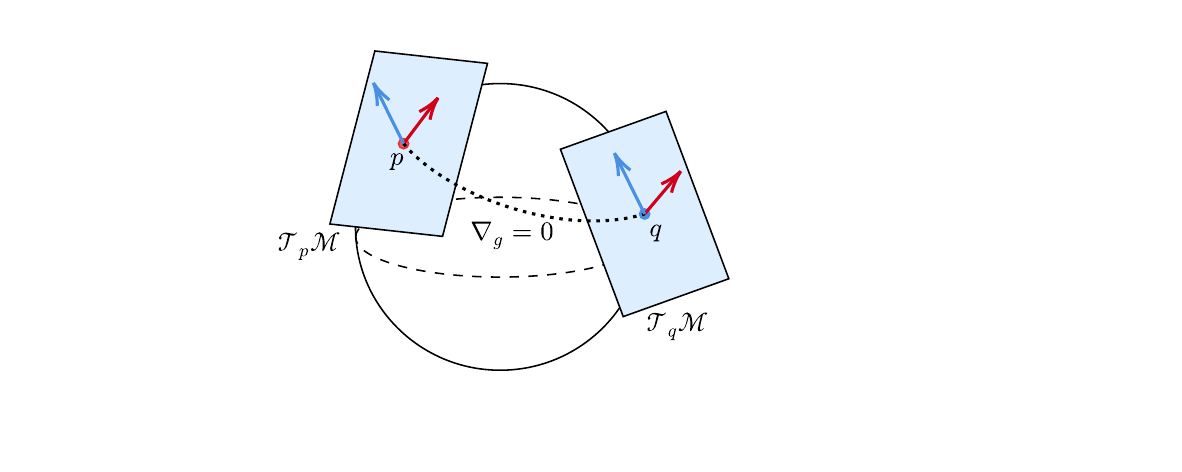}}
\caption{How to ``connect'' two separate tangent spaces together, $\mathcal{T}_p\mathcal{M}$ and $\mathcal{T}_q\mathcal{M}$ at points $p$ and $q$, by defining affine properties for the connection $\nabla$. The metric compatibility condition $\nabla g = 0$, ensures that geodesics (curves traced out via autoparallelism $\nabla_{\dot{\gamma}} \dot{\gamma}=0$ are length minimizing.}
\label{fig:autoparallel}
\end{figure}

It's crucial to note that, thus far, no coordinate system has been established over \(\mathcal{M}\), requiring an additional layer of construction for the implementation of a coordinate system. This step influences the resulting equations governing geodesics. For instance, in 2D Euclidean space, two prevalent coordinate systems exist: the canonical Cartesian system parameterized by orthogonal \((x, y)\) coordinates and the polar coordinates parameterized by orthogonal \((r, \theta)\). Depending on the chosen coordinate system, the resulting \textit{geodesic equations} will vary. To derive such geodesic equations, one must solve a set of second-order differential equations as shown in Equation \ref{eqn:geodesic_eqns}.

\begin{equation}\label{eqn:geodesic_eqns}
\frac {d^{2}x^{\lambda }}{dt^{2}}+\Gamma _{\mu \nu }^{\lambda }{\frac {dx^{\mu }}{dt}}{\frac {dx^{\nu }}{dt}}=0,
\end{equation}

Here, \(\Gamma_{\mu \nu }^{\lambda }\) represents the Christoffel symbol of the second kind corresponding to the chosen metric \(g\), and Equation \ref{eqn:geodesic_eqns} is expressed in Einstein summation notation. The theoretical solutions to the geodesic equations are often challenging to obtain, leading to computational methods that solve them incrementally in steps of \(\delta t\) over arbitrary manifolds. However, for simpler geometries like hyperspheres, these solutions are well-known (refer to \textit{SphereGen} in Section \ref{app_sec:quirks_lims}).

Under specific conditions for the manifold \(\mathcal{M}\) (such as smoothness, completeness, and prescribed initial conditions), which align with the practical uses of manifolds, the fact that Equation \ref{eqn:geodesic_eqns} constitutes a set of second-order differential equations ensures the existence and uniqueness of the resultant geodesic curve (according to the Picard-Lindel\"of theorem). Despite this, the iterative discretization and solution of geodesics can be computationally expensive and slow in various scenarios. To address this challenge, one can introduce the concept of the \textit{exponential map}.

The exponential map, denoted as \(\text{Exp}_p : \mathcal{T}_p \mathcal{M} \rightarrow \mathcal{M}\), serves as a fundamental tool in Riemannian geometry, mapping tangent vectors from the tangent space generated at point \(p\) onto the manifold \(\mathcal{M}\). Mathematically, the exponential map is defined by \(\text{Exp}_p(v) = \gamma_v(1)\), where \(v\) is a tangent vector in the tangent space \(\mathcal{T}_p \mathcal{M}\), \(\gamma_v(t)\) is a smooth curve passing through point \(p\) with \(\gamma_v(0) = p\), and \(\dot{\gamma}_v(0)=0\), indicating zero velocity at the base point \(p\). An illustrative example of the exponential map was presented in the discussion on \textit{SphereGen} in Section \ref{app_sec:quirks_lims}.

It is essential to note that the geodesics induced by the practical manifolds in this paper demand both uniqueness and existence properties. Consequently, the application of the exponential map to reach the end point of \textit{any geodesic} in \(\mathcal{M}\) is not universally possible. Instead, \(\text{Exp}_p\) can only be applied within a local radius around point \(p\), known as the \textit{radius of injectivity}. This radius is employed as a limiting factor to scale randomly sampled tangent vectors using \textit{StiefelGen}, ensuring a valid interpretation for the hyperparameter \(\beta \in [0,1]\). Mathematically, the radius of injectivity \(\mathcal{R}\) at a point \(p\) on \(\mathcal{M}\) is defined as follows:
\begin{align}
    \mathcal{R} = \sup \{ r > 0 \mid \exp_p:\mathcal{T}_p\mathcal{M}\rightarrow \mathcal{M} \text{ is a diffeomorphism on } \mathcal{B}(0, r) \subset \mathcal{T}_p\mathcal{M} \},
\end{align}
where \(\mathcal{B}(0, r)\) is an open ball centered at point \(p\), extending into the tangent space \(\mathcal{T}_p\mathcal{M}\) with radius \(r\). Figure \ref{fig:radius_injectv2} visually illustrates this process.

\begin{figure}[htbp]
\centering
{\includegraphics[width=1.15\textwidth]{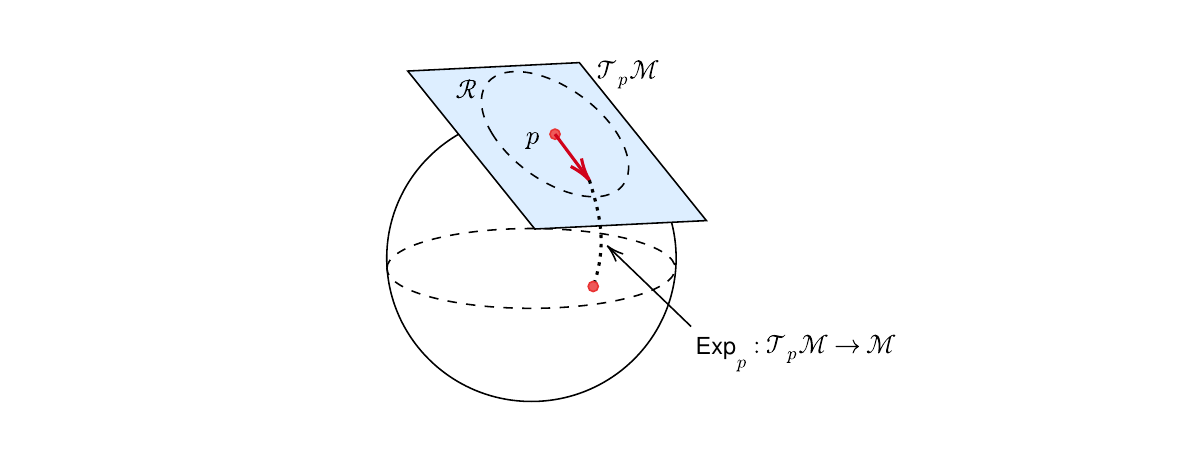}}
\caption{A point \(p\) on \(\mathcal{M}\), emitting a tangent vector extending into \(\mathcal{T}_p\mathcal{M}\) within the radius \(\mathcal{R}\) to validly apply the \(\text{Exp}_p\) operation.}
\label{fig:radius_injectv2}
\end{figure}

Throughout this paper, \(\mathcal{R}\) for the Stiefel manifold is known to have a tight global lower bound of \(0.89\pi\) (Section \ref{sec:background_math}) and is globally \( \pi \) for the hypersphere (Section \ref{app_sec:quirks_lims}).

\subsubsection{Matrix Manifolds: The Basics} 

The transition from vector spaces defined over \(\mathbb{R}^m\) to the realm of matrices in \(\mathbb{R}^{m\times n}\) introduces intriguing mathematical structures. Despite their computational portrayal as "matrices" or "arrays," these objects maintain their essence as vector space entities. This recognition allows us to extend the concept of manifolds to objects in \(\mathbb{R}^{m\times n}\), unraveling a diverse and intricate mathematical landscape.

In exploring the \(\mathbb{R}^{m\times n}\) manifold within the context of Euclidean space, we employ an inner product \(\langle V_1, V_2 \rangle\) that captures the essence of matrix interactions:
\begin{align}
    \langle V_1, V_2 \rangle = \text{Vec}(V_1)^{\intercal}\cdot \text{Vec}(V_2) = \text{tr}(V_1^{\intercal} V_2),
\end{align}
Here, \(\text{Vec} : \mathbb{R}^{m\times n}\rightarrow \mathbb{R}^{mn}\) signifies a vector stacking operation, and \(\text{tr}: \mathbb{R}^{m\times n} \rightarrow \mathbb{R}\) denotes the standard matrix trace.

One captivating facet of this exploration is the emergence of matrix manifolds, such as the \textit{Stiefel manifold}. Defined as the embedded submanifold of \(\mathbb{R}^{m\times n}\) consisting of column orthogonal rectangular matrices:
\begin{equation}
    \St \coloneqq \{U \in \mathbb{R}^{m\times n} \mid U^{\intercal} U = I_{n}\}.
\end{equation}
This specific structure, where the columns of \(U\) are orthogonal and form an orthonormal basis, presents a fascinating departure from conventional manifolds.

The tangent space \(T_U \St\) at a point \(U\in \St\) is naturally represented by:
\begin{align}
    T_U\St = \{\Delta \in\mathbb{R}^{m\times n} \mid U^{\intercal} \Delta = - U\Delta^{\intercal}\}.
\end{align}
This representation elegantly stems from differentiating \(U^{\intercal} U = I_{n}\), where the differential of \(U\) is denoted as \(\Delta\) (to be expanded upon shortly).

Comparing matrix manifolds with traditional vector space manifolds, the unique challenges and intricacies arise from the algebraic structure imposed by matrix operations. The matrix trace, determinant, and eigenvalues play pivotal roles, offering a distinct flavor to the geometry of these manifolds. Unlike vector spaces, matrix manifolds exhibit non-commutativity, influencing the development of geometric concepts.

While Euclidean metrics inherited from the ambient embedding space remain applicable, the canonical metrics for matrix manifolds often involve tailored measures that account for specific algebraic properties which arise in the matrix algebra context. The Stiefel manifold exemplifies this, showcasing the need for metrics aligned with the skew-symmetry properties inherent in certain parameterizations.

In essence, the extension to matrix manifolds not only broadens the scope of mathematical exploration but also introduces a unique interplay between geometry and algebra, unveiling a rich tapestry of structures and relationships. The contrast between traditional vector space manifolds and matrix manifolds underscores the importance of understanding the nuanced characteristics that arise in these diverse mathematical spaces.

\subsubsection{Understanding the Stiefel Manifold: Delving Deeper} 

Here we formally introduce the Stiefel manifold and delve into it some of its mathematical properties.

The Stiefel manifold, denoted as $\St$, is defined as the set of matrices $U\in\mathbb{R}^{m\times n}$ that satisfy the condition $U^{\intercal}U = I_n$, where $m\geq n$. This definition encapsulates matrices with orthonormal columns, forming a unique and structured subset of the matrix space. Based on this geometric constraint, it is possible to study some properties over the induced tangent space at point (matrix) $U$, denoted as $\Sttan$. A key Lemma is as follows:

\begin{lemma} For any $\Delta \in \Sttan$, the relation $\Delta^{\intercal}U + U^{\intercal}\Delta = 0$ holds true.
\end{lemma}\label{lem:deltaCondition}

\begin{proof} 
Consider a smooth curve $U(t): \mathbb{R} \rightarrow \mathcal{M} = \St$, where $t\in[0,1]$ and $U(0) \coloneqq U_0 \in \St$. Consider now the geometric constraint of $\St$ as it evolves over time: $U^{\intercal}(t) U(t) = I_n$. Applying now the derivative operator denoted, $D$, we arrive at the following equation, $D(U^{\intercal}(t) U(t)) = D(I_n) = 0$, since there is no change $I_n$ over time. Considering the now the LHS of the expression, $D(U^{\intercal}(t) U(t)) = D(U^{\intercal}(t)) U(t) + U^{\intercal}(t) D(U(t)) = 0$. Evaluating this expression at $t=0$, where we exist in both the tangent vector space (as it touches $\St$) and on the Stiefel manifold, gives:
\begin{equation}
    D(U(0))\in \mathcal{T}_{U_0} \St \text{ and } \St.
\end{equation}

Simultaneously, writing $D(U(0)) = \Delta$ and $D(U^{\intercal}(0)) = \Delta^{\intercal}$ completes the proof.

\end{proof}

Thus, we see that the geometric constraint of $U^{\intercal} U = I$, leads to the tangent vector property of: $\Delta^{\intercal}U + U^{\intercal}\Delta = 0$. Now, notice that by definition skew-symmetry matrices have the property that $A + A^{\intercal} = 0$. Taking $A = \Delta^{\intercal} U$ leads to the corollary:

\begin{cor}
    The matrix $\Delta^{\intercal} U$ is skew-symmetric, as $\Delta^{\intercal} U = - U^{\intercal} \Delta$.
\end{cor} 

Thus we may formalize the Stiefel manifold tangent space, $\Sttan$ as the following,

\begin{equation}
    \Sttan \coloneqq \{\Delta \in \mathbb{R}^{m\times n} \mid \Delta^{\intercal} U = -U^{\intercal} \Delta, \text{ with, } U \in \St\}.
\end{equation}

This tangent space characterization acts as the starting block for building further geometric and algebraic structures over $\St$. In particular it will be shown to be pivotal for generalizing the parametric representation of $\Sttan$ with respect to ambient matrices in $\mathbb{R}^{m \times n}$, which may not necessarily lie on $\St$ or on $\Sttan$. To begin with let's consider  some random $\UTilde \in \mathbb{R}^{m\times n}$, which is an arbitrary matrix that lies in the ambient space of $\mathbb{R}^{m\times n}$. Therefore it can be taken to be a zero probability event that $\UTilde$ will lie on $\St$ (and by consequence it is a probability zero event that the randomly sampled $\UTilde$ will be orthonormal). Write the orthogonal complement\footnote{Two finite linear subspaces $\mathcal{G}$, $\mathcal{H}$ are said to be orthogonal complements with respect to $\mathcal{I}$, if (i) $\langle g, h \rangle, \forall g\in\mathcal{G}$, $h\in\mathcal{H}$, and (ii) The direct sum of each subspace spans the construction of the full space. That is $\mathcal{G}\oplus\mathcal{H}=\text{span}(\mathcal{I})$.} of \textit{any} $\UTilde\in\mathbb{R}^{m\times n}$ as $\UTilde_{\bot}\in\mathbb{R}^{m\times (m-n)}$. Using $\UTilde$, and $\UTilde_{\bot}$, we can form the \textit{orthogonal completion} as follows, 

\begin{equation}
W= 
\begin{bNiceArray}{c|c}
\UTilde & \UTilde_{\bot}\ 
\end{bNiceArray}
\in \mathbb{R}^{m \times m},
\end{equation}

so that $W\in\mathbb{R}^{m\times m}$. Similarly, construct the matrix, $X\in\mathbb{R}^{m\times n}$, with the aim that the calculation, $WX \in \mathbb{R}^{m\times n}$, and we land back into the desired ambient space. Such an $X$ would take the following form,

\begin{equation}
 X = 
\left[
\begin{array}{c}
A \\ \hline 
B
\end{array}
\right],   
\end{equation}

so that the overall calculation of $WX$ would result in the matrix,

\begin{equation}
    WX =
\begin{bNiceArray}{c|c}
\UTilde & \UTilde_{\bot}\ 
\end{bNiceArray} 
        \cdot 
\left[
\begin{array}{c}
A \\ \hline 
B
\end{array}
\right]
=\UTilde A + \UTilde_{\bot} B \in \mathbb{R}^{m\times n}
\end{equation}

To keep matrix sizes consistent, it follows then that $A\mathbb{R}^{n\times n}$ and $B\mathbb{R}^{(m-n)\times n}$. Since $WX$ lies in $\mathbb{R}^{m\times n}$, a pertinent question now arises as: 

\textit{Under what conditions would we see} $WX$ \textit{land onto a valid member of the tangent space generated by some point,} $\UTilde$? \textit{That is, what conditions would allow for} $WX\in\Sttan$? 

Luckily, we have developed the minimal condition required for matrices to belong to $\Sttan$ in Lemma \ref{lem:deltaCondition}, which leads to the following Lemma,

\begin{lemma}
    For $WX = \UTilde A + \UTilde_{\bot}B$ to lie on $\Sttan$, we require, $\UTilde=U$, $A$ skew-symmetric, and $B$, arbitrary. 
\end{lemma}

\begin{proof}
    Recall from Lemma \ref{lem:deltaCondition}, $\DeltaT U + U^{\intercal} \Delta=0$, where $U\in\St$. Assume $WX=\UTilde^{\intercal}T A + \UTilde_{\bot}B = \Delta \in \Sttan$, and substitute this into the condition that, $\DeltaT U + \UT \Delta=0$:

    \begin{align*}
        (\UTilde A + \UTilde_{\bot}B)^{\intercal}U + U^{\intercal} (\UTilde A + \UTilde_{\bot}B) &= 0 \\ 
        \iff (A^{\T}\UTilde^{\T}  + B^{\T}\UTilde_{\bot}^{\T})U + U^{\intercal} (\UTilde A + \UTilde_{\bot}B) &= 0 \\
        \iff (A^{\T}\UTilde^{\T}U  + B^{\T}\UTilde_{\bot}^{\T}U) +  (U^{\intercal}\UTilde A + U^{\intercal}\UTilde_{\bot}B) &= 0
    \end{align*}

    Taking $\UTilde=U$ (which by consequence means $\UTilde \in \St$), we can see that the conditions $\UT\UTilde = \UTilde^{\intercal}U = \UT U = I_{n}$. Further, $\UTilde_{\bot}^{\T}U=U^{\intercal}\UTilde_{\bot}=0$, as a consequence of the property of the orthogonal complement. Then, we are left with,

    \begin{align*}
        A + A^{\T} = 0,
    \end{align*}

    which completes the proof.
\end{proof}

Thus with reference to the total space $W$, and from the geometric properties of $\Sttan$, it can be shown that the tangent space generated with respect to $U$ can indeed be represented as $\Delta = UA + U_{\bot}B$, where we require $A$ skew symmetric.

Now that we have a manner in which to parameterise $\Sttan$ with respect to matrices $(A,B)$, we can consider developing an inner product space structure, so that notions such as lengths, and angles can be generalized to $\St$, and geodesics constructed. Let us initially consider the Euclidean metric, inherited from the ambient Euclidean space $\mathbb{R}^{m\times n}$:

\begin{align*}
    \langle \cdot, \cdot \rangle_E : \mathbb{R}^{m\times n } \times \mathbb{R}^{m\times n} &\rightarrow \mathbb{R} \\ \langle \Delta_1, \Delta_2 \rangle_E &\mapsto \text{tr}(\Delta_1^{\T}\Delta_2),
\end{align*}

where $\Delta_1,\Delta_2\in\Sttan$ must span the \textit{same} tangent vector space generated with respect to $U$. For example, in Figure \ref{fig:autoparallel}, it is not valid to have $\Delta_1\in\mathcal{T}_p\mathcal{M}$, and $\Delta_2\in\mathcal{T}_q\mathcal{M}$, then calculate $\langle\Delta_1, \Delta_2\rangle$. Both $\Delta_1, \Delta_2$ must be from one tangent space or the other. Returning to the inner product space problem, we shall proceed to expand as follows, 

\begin{align*}
    \|\Delta\|_E^2  &= \langle \Delta, \Delta \rangle_E \\
    &= \text{tr}(\Delta^{\T} \Delta) \\
    &= \text{tr}((A^{\T} U^{\T} + B^{\T} U_{\bot}^{\T})(UA + U_{\bot}B) ) \\
    &= \text{tr}(A^{\T} A + B^{\T} B) \\ 
    &=\text{tr}(A^{\T} A) + \text{tr}(B^{\T} B)\\
    &= \sum_i \sum_j a_{ij}^2 + \sum_i \sum_j b_{ij}^2 \\
    &= 2 \sum_{j> i}a_{ij}^2 + \sum_i \sum_j b_{ij}^2
\end{align*}

where we notice that $\sum_i \sum_j a_{ij}^2=2 \sum_{j> i}a_{ij}^2$ arises as a consequence of the skew symmetry property over $A$.

This presents itself as a slightly unnatural result since the representational power of $A$ is twice that of $B$ when considering the Euclidean inner product. However, if one instead defines the Stiefel manifold via an equivalence class relationship with respect to the Grassmann manifold (to be defined shortly), then one arrives at an inner product result which is evenly weighted, and in general has access to easier algorithmic implementations of geometric calculations (such as the geodesic, exponential, and logarithmic maps). This is also important since many out of the box matrix geometry libraries work from the perspective of this construction, and thus use the canonical metric by default. This is particularly important for this paper, as we work to rescale the randomly sampled tangent vectors with respect to the radius of injectivity via,

\begin{align*}
    \bar{\Delta } = \frac{\Delta}{\sqrt{\langle \Delta, \Delta \rangle}\rangle_w},
\end{align*}

where $w$ refers to a particular corrective weighting factor for the inner product. Before we begin to calculate $w$ from first principles, is there perhaps a way in which we can guess of back calculate $w$? The answer is yes.

Consider working with the following weighted inner product, $\langle \cdot, \cdot \rangle_w: \langle \Delta, \Delta \rangle _w \mapsto \text{tr}(\Delta^{\T} w \Delta)$. We require $w$ to act as a corrective weighting factor to fix the double counting problem. What should $w$ be? 

Consider again the proposed parameterization of $\Delta$ as,

\begin{align*}
    \Delta = U A + U_{\bot} B.
\end{align*}

Left multiply both sides by $U^{\T}$ resulting in,

\begin{align*}
    U^{\T} \Delta = A
\end{align*}

Furthermore, once again left multiplying, but this time by $U$, 

\begin{align*}
    UA = UU^{\T} \Delta.
\end{align*}

It is very important to note that even though $U^{\T}U=I_n$, it is not true that $U U^{\T} = I_n$ (unless $U$ is a square matrix). In fact the expression for $U U^{\T}$ is often used an orthogonal projector within linear algebra (if $U\in \St$). This idea will be discussed again when considering the Grassmann manifold. Now, let's consider the following weighted expression, which is built considering that $UA = UU^{\T} \Delta$,

\begin{align*}
    w\Delta &= (I_n - \frac{1}{2}UU^{\T})\Delta\\&=\Delta -  \frac{1}{2}UU^{\T}\Delta\\ 
    &= \Delta - \frac{1}{2}UA \\
    &= (UA + U_{\bot}B) - \frac{1}{2}UA \\
    &= \frac{1}{2}UA + U_{\bot} B,
\end{align*}

Thereby introducing the additional multiplicative ``half factor'' in front of $A$ as required to balance the inner product. Proving this point by expanding this weighted inner product we arrive at, 
\begin{align*}
    \langle \Delta, \Delta \rangle_w &= \text{tr}(\Delta^{\T}w \Delta) \\ 
    &=\text{tr}(\Delta^{\T}(1/2 UA + U_{\bot} B)) \\
    &= \sum_{j>i} a_{ij}^2 + \sum_{i,j} b_{ij}^2.
\end{align*}

\subsubsection{The Canonical Metric: Quotient Manifolds and the Grassmannian}\label{app_sec:Understanding_Grassmann}

The Grassmann manifold,$\Gr$, is intricately linked to the Stiefel manifold through the orthogonal group $O(m)$. This manifold exhibits notable flexibility, primarily attributable to its nature as a quotient manifold. This characteristic allows it to be defined over subspaces, resulting in each ``point'' on the manifold representing an entire equivalence class. The manifold's flexibility manifests through its multiple interpretations and applications within various mathematical contexts. But before delving into the multiple possible interpretations of $Gr$, we quickly go over the notion of equivalence classes.

Equivalence relations hold significant importance in mathematics as they capture isomorphisms (a bijective mapping that preserves the structure, relationships, and operations defined on the sets) between elements belonging to distinct sets. To illustrate this, consider the set \(\mathbb{Z}_5\), which represents integers modulo 5. One example of such a set is naturally, \(\{0, 1, 2, 3, 4\}\). Equivalently, another valid set can be expressed as \(\{5, 6, 7, 8, 9\}\) since \(0 \mod 5 = 5 \mod 5\) and \(1 \mod 5 = 6 \mod 5\), and so on and so forth. Mathematically, we write denote this equivalence structure as, \(0 \sim 5\) and \(1 \sim 6\), ultimately leading to the set representations of equivalence classes, often denoted through square bracket notation as: \([0] \coloneqq \{ k \in \mathbb{Z} \mid (k \cong 0) \mod 5 \}\) and similarly \([1] \coloneqq \{ k \in \mathbb{Z} \mid (k \cong 1) \mod 5 \}\). In essence, we can now express \(\mathbb{Z}_5\) in a very general form as \(\{[0], [1], [2], [3], [4]\}\), where all such alternate set representations of \(\mathbb{Z}_5\) have been ``collapsed'' into a single equivalence class of relations. Alternatively, this collapsing could have been articulated notationally as, \(\mathbb{Z}_5 = \mathbb{Z} / \sim\), where this operation is also to be understood to be modulo 5. What this notation is trying to imply is that we begin with all the integers $\mathbb{Z}$, and then ``divide out'' or collapse this set by considering all such equivalence elements. Due to this notation, such a set is often termed as the quotient set, or in terms of manifolds, a structure defined as $\mathcal{M}/\sim$ is known as a \textit{quotient manifold.} 

For the sake of intuition, we shall quickly go over a simple quotient manifold known as the real projective space  \(\mathbb{RP}^1\). We shall break down its construction in a three step process as follows.

\begin{enumerate}
    \item \textbf{Initial Space:} Start with the vector space \(\mathbb{R}^2\), which consists of all vectors with two real components.
    \item \textbf{Define Equivalence Relation:} Define an equivalence relation \(\sim\) on \(\mathbb{R}^2\) by considering two vectors \((x, y)\) and \((x', y')\) to be equivalent if one is a nonzero scalar multiple of the other. That is, \((x, y) \sim (x', y')\) if and only if there exists a \(\lambda \neq 0\) such that \((x, y) = \lambda(x', y')\).
    \item \textbf{Generate Quotient Space:} The real projective space \(\mathbb{RP}^1\) is then defined as the set of equivalence classes of vectors in \(\mathbb{R}^2\) under this equivalence relation. Each equivalence class represents a line through the origin in \(\mathbb{R}^2\), and \(\mathbb{RP}^1\) can be represented as the set of all such lines. It can be thought of collapsing each line through the origin as a single point in some collapsed (that is, quotient) space. 
\end{enumerate}

In this way, \(\mathbb{RP}^1\) consists of equivalence classes of lines in \(\mathbb{R}^2\), where lines that are scalar multiples of each other are in the same equivalence class. To be more explicit, the equivalence relation \(\sim\) is defined on \(\mathbb{R}^2\) as follows:

\[ (x, y) \sim (\lambda x, \lambda y) \]

for all \((x, y) \in \mathbb{R}^2\) and \(\lambda \neq 0\). This relation identifies all vectors on the same line through the origin. Thus, the real projective space \(\mathbb{RP}^1\) is the set of equivalence classes of vectors in \(\mathbb{R}^2\) under the equivalence relation:

\[ \mathbb{RP}^1 = \{ [(x, y)] \mid (x, y) \in \mathbb{R}^2 \} \]

Here, \([(x, y)]\) denotes the equivalence class containing all vectors that are scalar multiples of each other. Each equivalence class corresponds to a line through the origin in \(\mathbb{R}^2\). A diagram summarizing this idea is represented in Figure \ref{fig:equiv_relation}.

\begin{figure}[htbp]
\centering
{\includegraphics[width=1.15\textwidth]{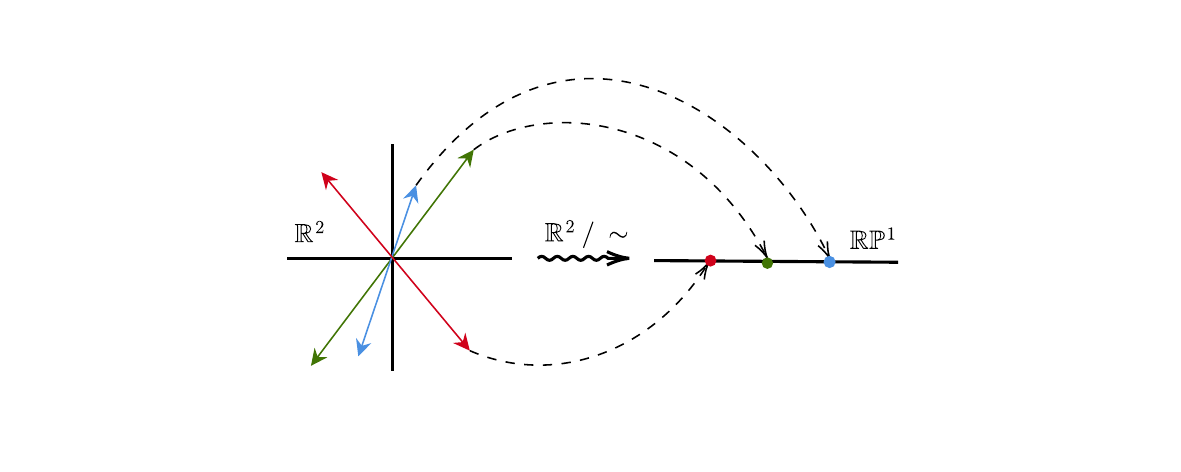}}
\caption{Demonstrating the construction of the real projective space $\mathbb{R}\mathbb{P}^1$ as a quotient manifold structure.}
\label{fig:equiv_relation}
\end{figure}

With this foundational context established, we now delve into the introduction of the Grassmann manifold from various perspectives, shedding light on its multifaceted nature.

\textbf{As a Quotient Manifold}

This perspective is important for Grassmann manifolds since they can also be expressed as quotient manifolds. Specifically, the Grassmann manifold $\Gr$ can be expressed as the quotient:

\begin{equation}
\Gr = \frac{\mathcal{O}(m)}{\mathcal{O}(n) \times \mathcal{O}(m-n)}
\end{equation}

where $\mathcal{O}(m)$ denotes the orthogonal group of $m \times m$ orthogonal matrices:

\begin{equation}
\mathcal{O}(m) = {U \in \mathbb{R}^{m \times m} : U^{\intercal} U = U U^{\intercal} = I_m}
\end{equation}

The orthogonal group $\mathcal{O}(m)$ can be viewed as the group of rotations in $\mathbb{R}^m$. Thus, the quotient structure divides out free rotations in the two subspaces of dimensions $n$ and $m-n$.

This quotient relationship is made clearer by comparing to the Stiefel manifold $\St$, which has an analogous quotient structure:

\begin{equation}
\St = \frac{\mathcal{O}(m)}{\mathcal{O}(m-n)}
\end{equation}

Notice then that there is an explicit relationship between $\Gr$ and $\St$, as,

\begin{equation}
    \Gr = \frac{\St}{\mathcal{O}(n)}
\end{equation}

which clarifies the idea the $\St$ forms an element of the equivalence set, $\Gr$. In other words, if one points on $\St$ represents an orthonormal basis, then one point of $\Gr$ would be the entire hyper plane spanned by this basis, and ultimately house a collection of possible $\St$'s. This will be further clarified when $\Gr$ is seen through the perspective of a linear algebraic subspace.

The key benefit of the quotient perspective is that it provides a precise mathematical link between the Grassmann and Stiefel manifolds in terms of the orthogonal group $\mathcal{O}(m)$. This allows properties of the Grassmannian to be understood in terms of rotational symmetries.

In particular, the quotient structure gives rise to the canonical Riemannian metric on the Stiefel manifold. Since points on the Grassmannian can be represented as points on the Stiefel manifold, this induces a metric structure on the Grassmannian. The quotient perspective also enables computations on the Grassmannian in terms of Stiefel matrix representatives.

Overall, the quotient manifold perspective elucidates the rotational symmetry underlying the Grassmannian structure. It links the Grassmannian geometrically to the better understood Stiefel manifold, providing a pathway for endowing the Grassmannian with metric properties. By revealing these connections, the quotient view facilitates both theoretical analysis and practical computations involving Grassmann manifolds.

\textbf{As a linear algebraic subspace:}

The Grassmannian $\Gr$ can be defined as the set of all $n$-dimensional linear subspaces of $\mathbb{R}^m$:

\begin{equation}
\Gr = { \text{span}(U) \mid U \in \mathbb{R}^{m\times n}, U^\intercal U = I_n}
\end{equation}

Where $\text{span}(U)$ denotes the subspace spanned by the columns of $U$. This is the most direct definition, as it simply identifying a subspace via orthonormal bases. While conceptually simple, this perspective does not readily lend to a  geometric structure of the Grassmannian. It tends to be more used a factual point of reference or representation to the best of the authors' understanding.

\textbf{As an orthogonal projector:}

The Grassmann manifold can alternatively be defined as the set of rank-$n$ orthogonal projectors in $\mathbb{R}^m$:

\begin{equation}
\Gr = { P \in \mathbb{R}^{m\times m} \mid P^\intercal = P, P^2 = P, \text{rank}(P)=n }
\end{equation}

Where the key properties of an orthogonal projector, $P$ are:

\begin{itemize}
\item \textbf{Idempotent}: $P^2 = P$
\item \textbf{Symmetric}: $P = P^\top$
\item \textbf{Rank $n$}: Projects onto an $n$-dimensional subspace of $\mathbb{R}^m$
\end{itemize}

The tangent vectors from this perspective are now characterized by the symmetry condition:

\begin{equation}
\Delta P + P \Delta = \Delta, \quad \Delta \in \text{sym}(\mathbb{R}^n)
\end{equation}

Where $\text{sym}(\mathbb{R}^n$ denotes the space of $n\times n$ symmetric matrices. Given any $U \in \St$, one can define a $P$ which satisfies these projector properties via, $P = UU^\top$. In turn this defines a point in $\Gr$. 

\textbf{As a Lie group:}

Although this view point of Grassmannians is not employed in this paper, it is kept here for completeness. However the explanation here will be largely qualitative in nature. 

A Lie group can be thought of as a smooth manifold endowed with a group structure. This group structure involves a set equipped with a binary operation satisfying properties such as closure, associativity, the existence of an identity element, and the presence of an inverse for each element. The smoothness requirement ensures that the group operations are differentiable. This structure provides a seamless integration of algebraic and geometric properties. Associated with a Lie group is its Lie algebra, a vector space that captures the group's infinitesimal behavior around the identity element. The Lie algebra consists of tangent vectors at the identity element, forming a vector space with a Lie bracket, a binary operation reflecting the group commutation relations. This Lie algebra serves as a linear approximation to the group, facilitating the study of local properties and transformations within the manifold.

Central to the understanding of Lie groups is the concept of the exponential map. This map establishes a crucial link between elements in the Lie algebra and the corresponding elements in the Lie group. It expedites the translation of algebraic operations to geometric transformations, providing a powerful tool to navigate between the algebraic and geometric realms of Lie theory. The exponential map plays a pivotal role in analyzing the global and local structure of Lie groups within the manifold framework.

In the context of the Grassmann manifold ($\Gr$) and its connection with the orthogonal group ($\mathcal{O}(m)$), the Lie algebra $\mathfrak{o}(m)$ of $\mathcal{O}(m)$ consists of $m\times m$ skew-symmetric matrices. Although the Lie algebra may not be directly employed in the present paper, its understanding is essential for delving into the deeper mathematical intricacies. For further exploration, the interested reader is encouraged to refer to \cite{bendokat2024grassmann}, providing comprehensive insights into the Grassmann manifold and its Lie algebraic structure.

\textbf{Deriving the Canonical Metric via Quotient Geometry}

One of the primary objectives in introducing the Grassmann manifold was to establish the canonical metric relationship that finds widespread use in various Stiefel matrix algorithms and coding implementations. This stems from the evident quotient relationship existing between $\St$ and $\Gr$ concerning the total space $\mathcal{O}(m)$, as elucidated in the preceding subsections.

Let us commence by defining an equivalence class over elements in $\St$:
\begin{align*}
    [U] &= \left\{W \begin{bmatrix} I_n & 0 \\ 0 & Q \end{bmatrix} \middle| Q \in \mathcal{O}(m-n) \right\} \\
    &= \left\{\left[U \mid U_{\perp}\right] \begin{bmatrix} I_n & 0 \\ 0 & Q \end{bmatrix} \middle| Q \in \mathcal{O}(m-n) \right\}
\end{align*}

Here, the equivalence class denotes a construction wherein we initially form $W=\left[U \mid U_{\perp}\right]\in\mathbb{R}^{m\times m}$ with respect to the total space $\mathcal{O}(m)$ (which was a construction used earlier in this subsection). The equivalence relation $U_1\sim U_2$ is said to hold if the first $n$ columns are identical (while the last $m-n$ columns are arbitrary). Consequently, the expression for $[U]$ embodies the quotient manifold definition of $\St$, as being that of $\St = \mathcal{O}(m)/\mathcal{O}(m-n)$ (since we have define a structure with respect to total space $\mathcal{O}(m)$, and do not ``care as much'' about the last $m-n$ colums). For the sake of formality, the isotropy subgroup action (``the group stabilizer'') of $\begin{bmatrix} I_n & 0 \\ 0 & Q_{\perp} \end{bmatrix}$ is presented in the definition of $[U]$. It represents the possible set of rotations within the total space $\mathcal{O}(m)$ that can be performed without affecting the equivalence class structure, $[U]$. Without loss of generality, these are simply relegated to the $m-n$ columns of $[U]$. However, as per \cite{edelman1998geometry}, from here onwards we shall omit this for notational brevity, and because, as will be demonstrated soon, its inclusion is ultimately dispensable.

An intriguing aspect of indirectly working with $\St$ as a quotient of $\mathcal{O}(m)$ is the division of the total tangent space at some point $W\in\mathcal{O}(m)$ into two complementary linear subspaces: the \textit{horizontal} and \textit{vertical} tangent spaces. In essence, given the tangent space at $W$ denoted as $\mathcal{T}_W \mathcal{O}(m)$, it can be decomposed into two mutually orthogonal albeit linear tangent subspaces, expressed via direct summation as:
\begin{align}
    \mathcal{T}_W \mathcal{O}(m) = \textrm{Vert}^W_{[Q]} \oplus \textrm{Horz}^W_{[Q]}.
\end{align}

The identification of these spaces involves first considering $\textrm{Vert}^W_{[Q]}$ as those vectors which move tangentially \textit{within} the equivalence relation structure. Once this is determined, $\textrm{Horz}^W_{[Q]}$ can be determined as the vector space pointing orthogonally to $\textrm{Vert}^W_{[Q]}$. To illustrate this idea further, consider the original quotient space demonstration of $\mathbb{Z}_5$. Recall that we had $\mathbb{Z}_5 = \{[0],[1],[2],[3],[4]\}$. There are essentially two tangential directions to move within the $\mathbb{Z}_5$ structure — one can either transition between equivalence classes horizontally:
\begin{equation}
    [0] \longrightarrow [1] \longrightarrow [2] \longrightarrow [3] \longrightarrow [4]
\end{equation}
known as moving tangentially in the horizontal direction, $\textrm{Horz}^W_{[Q]}$. Alternatively, one can move within the equivalence structure, creating a an effective vertical tangential movement, as demonstrated in the following flow diagram, denoted as $\textrm{Vert}^W_{[Q]}$.

\begin{centering}
\resizebox{\textwidth}{!}{%
\begin{tikzpicture}
    \node (node) at (0, 0) {[0]};
    \draw[->] (node) to [out=30, in=150, looseness=8] node[pos=0.15, above=3mm] {[5]} node[pos=0.85, above=3mm] {[10] ...} (node);

    \begin{scope}[xshift=3.5cm]
        \node (node2) at (0, 0) {[1]};
        \draw[->] (node2) to [out=30, in=150, looseness=8] node[pos=0.15, above=3mm] {[6]} node[pos=0.85, above=3mm] {[11] ...} (node2);
    \end{scope}

    \begin{scope}[xshift=7cm]
        \node (node3) at (0, 0) {[2]};
        \draw[->] (node3) to [out=30, in=150, looseness=8] node[pos=0.15, above=3mm] {[7]} node[pos=0.85, above=3mm] {[12] ...} (node3);
    \end{scope}

    \begin{scope}[xshift=10.5cm]
        \node (node4) at (0, 0) {[3]};
        \draw[->] (node4) to [out=30, in=150, looseness=8] node[pos=0.15, above=3mm] {[8]} node[pos=0.85, above=3mm] {[13] ...} (node4);
    \end{scope}

    \begin{scope}[xshift=14cm]
        \node (node5) at (0, 0) {[4]};
        \draw[->] (node5) to [out=30, in=150, looseness=8] node[pos=0.15, above=3mm] {[9]} node[pos=0.85, above=3mm] {[14] ...} (node5);
    \end{scope}
\end{tikzpicture}
}
\end{centering}

In this way, the tangent space $\mathcal{T}_W\mathcal{O}(m)$ with respect to total space, $\mathcal{O}(m)$ can similarly be broken down into two complementary tangent spaces, which move in different directions within the quotient space structure. It can be shown that the parametric forms of $\textrm{Vert}^W_{[Q]}$ and $\textrm{Horz}^W_{[Q]}$ are as follows \cite{edelman1998geometry,bendokat2024grassmann},
\begin{align}
    \textrm{Horz}^W_{[Q]} = \begin{bNiceArray}{c|c}U & U_{\bot}\ \end{bNiceArray}\begin{bmatrix} A & -B^{\intercal} \\ B & 0 \end{bmatrix}
\end{align}
\begin{align}
    \textrm{Vert}^W_{[Q]} = \begin{bNiceArray}{c|c}U & U_{\bot}\ \end{bNiceArray}\begin{bmatrix} 0 & 0 \\ 0 & C \end{bmatrix}
\end{align}

where $A$, $B$ have been defined earlier via the standard tangent space parameterization of $\Delta = UA + U_{\bot}B$, and $C$ is defined to be an arbitrary skew symmetric matrix. The component related to deriving the Stiefel canonical metric would be  $ \textrm{Horz}^W_{[Q]} = [\Delta]$\footnote{Recall that we are omitting the effect of post multiplying by isotropy subgroup, so the square bracket notation isn't the most precise here. In full it should be written as $\textrm{Horz}^W_{[Q]} = \left\{\begin{bNiceArray}{c|c}U & U_{\bot}\ \end{bNiceArray}\begin{bmatrix} A & -B^{\intercal} \\ B & 0 \end{bmatrix}\begin{bmatrix} I & 0 \\ 0 & Q \end{bmatrix}\middle| Q \in \mathcal{O}(m-n)\right\}$.} This information can be summarized in Figure \ref{fig:Proj_image}, which has been adapted from \cite{bendokat2024grassmann}, with additional annotations relevant to this discussion.

\begin{figure}[htbp]
\centering
{\includegraphics[width=0.45\textwidth]{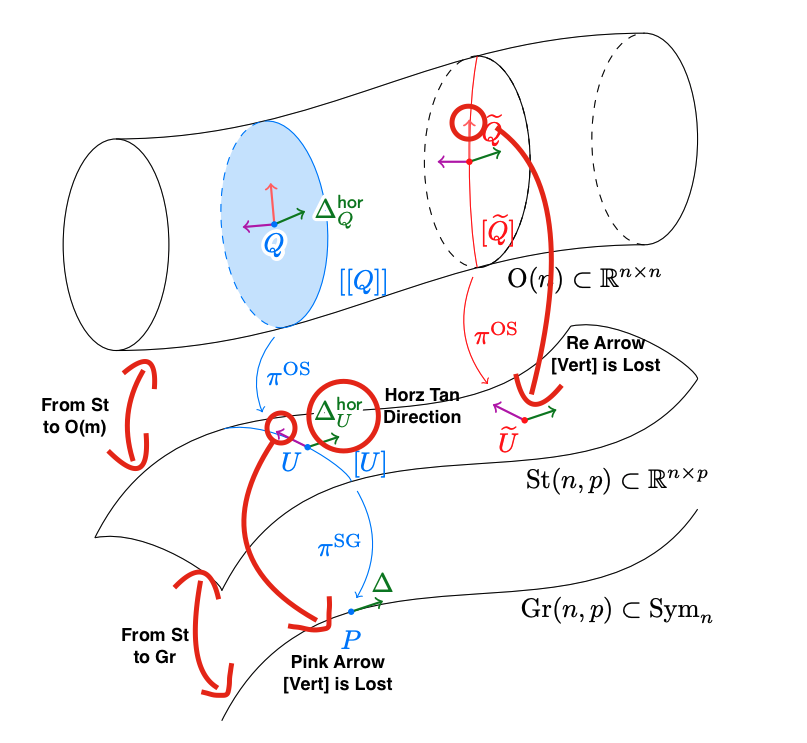}}
\caption{Demonstrating the jumps which occur as one moves between the quotient structures of $\St$ and $\Gr$ as defined relative to total space $\mathcal{O}(m)$. Notice that if one is defining $\Gr$ via a quotient manifold relative to $\mathcal{O}(m)$, then $\textrm{Horz}^W_{[[Q]]}$ would include the red, and purple arrows (and notice that this is a double-quotient space as two projections are required). Image adapted from \cite{bendokat2024grassmann}.}
\label{fig:Proj_image}
\end{figure}

With this in mind, let us proceed to evaluate the inner product with respect to the lifted horizontal tangent vectors, $\textrm{Horz}^W_{[Q]} = [\Delta]$, noting that an extra factor of $1/2$ is at the front for the purpose of scaling, which will be made clear towards the end,

\begin{align*}
\frac{1}{2}\text{tr}\left([\Delta]^\intercal [\Delta] \right) &= \text{tr}\left( \begin{bmatrix} A^{\intercal}U^{\intercal} + B^\intercal U_{\bot}^{\intercal} \\ 
-BU^\intercal \end{bmatrix} \left[UA + U_{\bot}B \mid -U B^{\intercal} \right] \right)\\
&= \text{tr}\left(\begin{bmatrix}(A^{\intercal}U^{\intercal} + B^\intercal U_{\bot}^{\intercal}) (UA + U_{\bot}B) & \ldots \\ \ldots &BU^\intercal U B^\intercal \end{bmatrix} \right)\\
&=\frac{1}{2}\text{tr}\left((A^{\intercal}U^{\intercal} + B^\intercal U_{\bot}^{\intercal}) (UA + U_{\bot}B) \right) + \text{tr}\left(B B^\intercal \right)\\
&= \frac{1}{2}\text{tr}\left(A^\intercal A\right) + 2\text{tr}\left(B^\intercal B\right)\\
&= \frac{1}{2}\cdot 2(\sum_{j>i} a_{ij}^2 + \sum_{i,j} b_{ij}^2)\\
&= \sum_{j>i} a_{ij}^2 + \sum_{i,j} b_{ij}^2,
\end{align*}

which is \textit{perfectly balanced, as all things should be.} One can also show that this works for both ``corner cases'', which are that of the hyper sphere ($\text{St}_1^m$), and the orthogonal group ($\text{St}_m^m$) respectively. We shall use the weighted representation of the metric in these examples. 

\textbf{Applying Canonical Metric to Hypersphere}

Consider $U\in\mathbb{R}^{m\times 1}$. Write this in lower-case as $u$ to denote it's a vector in $\mathbb{R}^m$, and note that $u^{\intercal}u=1$, but $uu^{\intercal}\in\mathbb{R}^{m\times m}$, with $\text{rank}(uu^{\intercal})=1$. Thus,
\begin{align*}
   \text{tr}\left(\Delta^{\intercal}\left(I - \frac{1}{2}uu^{\intercal} \right) \tilde{\Delta}\right) &= \text{tr}(\Delta^{\intercal}\tilde{\Delta}) - \frac{1}{2}\text{tr}\left(\Delta^{\intercal} uu^{\intercal}\tilde{\Delta} \right)\\
   &= \text{tr}(\Delta^{\intercal}\tilde{\Delta})- \frac{1}{2}\text{tr}\left(uu^{\intercal} \tilde{\Delta}\Delta^{\intercal}\right) \\
   &= \text{tr}(\Delta^{\intercal}\tilde{\Delta})- \frac{1}{2}\text{tr}\left((u^{\intercal}u)^{\intercal} \tilde{\Delta}\Delta^{\intercal}\right) \\
   &=\text{tr}(\Delta^{\intercal}\tilde{\Delta})- \frac{1}{2}\text{tr}\left(\tilde{\Delta}\Delta^{\intercal}\right)\\
   &=\text{tr}(\Delta^{\intercal}\tilde{\Delta})- \frac{1}{2}\text{tr}\left(\Delta^{\intercal}\tilde{\Delta}\right) \\
   &=\frac{1}{2}\text{tr}\left(\Delta^{\intercal}\tilde{\Delta}\right),
\end{align*}

where interestingly it can be seen that once again the factor of a $1/2$ is introduced out front. 

\textbf{Applying Canonical Metric to Orthogonal Group}

If $U\in\mathcal{O}(m)$, then $U^{\intercal}U =UU^{\intercal}=I $ the,
\begin{align*}
     \text{tr}\left(\Delta^{\intercal}\left(I - \frac{1}{2}UU^{\intercal} \right) \tilde{\Delta}\right) &= \text{tr}\left(\Delta^{\intercal}\left(I - \frac{1}{2}I \right) \tilde{\Delta}\right)\\
     &= \text{tr}\left(\Delta^{\intercal}\left(\frac{1}{2}I\right) \tilde{\Delta}\right) \\
     &=\frac{1}{2}\text{tr}\left(\Delta^{\intercal}\tilde{\Delta}\right),
\end{align*}

which also necessitates the half scaling factor. 

\textbf{Stiefel $\alpha$-metrics}

As a final note we will make note of a recent development by Zimmerman \cite{zimmermann2022computing} that has seen a unification of a set of Stiefel metrics parameterisied by $\alpha\in\mathbb{R}\-1$. The corresponding generalized $\alpha$-metric over the Stiefel manifold is given as follows,

\begin{align*}
  \langle \Delta, \tilde{\Delta} \rangle_U^{\alpha} = \text{tr} \left( \Delta^T \left(I - \frac{2\alpha + 1}{ 2(\alpha + 1)}UU^T\right) \tilde{\Delta}  \right) = \frac{1}{2(\alpha + 1)} \text{tr} (A^T \tilde{A}) + \text{tr} (B^T \tilde{B}). 
\end{align*}

Evidently, if $\alpha = 0$ one arrives at the canonical metric, and if $\alpha=-\frac{1}{2}$, one arrives at the Euclidean metric. Additionally, if one takes the limit of $\alpha \rightarrow \infty$, the $\alpha$-metric quasi-geodesic curves (smooth curves close to geodesics in a sense, but perhaps may have ill-defined logarithmic maps, or poor tangent space retractions). Moreover, if one considers the case of $\alpha < -1$, one arrives at a set of pseudo-Riemannian metrics, of which there is currently no known use case \cite{huper2021lagrangian}. 

The generalization of the $\alpha$-metric has also lead to the generalization of the exponential maps (and Stiefel logarithms, although they were not used here). The generalized $\alpha$-exponential map over $\St$ is defined as, 

\begin{equation*}
\text{Exp}_U^{\alpha}(\Delta) = \exp_m\left(-\frac{2\alpha + 1}{\alpha + 1}UAU^\intercal + \Delta U^\intercal - U\Delta^\intercal\right) \exp_m\left(\frac{\alpha}{\alpha + 1}A\right)
\end{equation*}

where $\exp_m(U) = \sum_{n=0}^{\infty} \frac{U^n}{n!}$, denotes matrix exponentiation. An algorithm for computing the exponential retraction with such metrics are provided in \cite{zimmermann2022computing, edelman1998geometry}, and often rely on the efficiency gains imposed by the QR-decomposition. If $\Delta\in\mathbb{R}^{m\times n}$, then the corresponding time complexity is then typically taken $O(mn^2)$, for $m>n$ (and is naturally then, $O(m^2n)$, for $m<n$) \cite{zimmermann2022computing,edelman1998geometry}.

We now assess the impact of various $\alpha$-metrics on the \textit{StiefelGen} data augmentation algorithm, examining their effects on the SteamGen and Taxi datasets (Figure \ref{fig:alpha_steam} and Figure \ref{fig:alpha_taxi}). Both datasets, with shapes $(m, n) = (20, 50)$, were subjected to moderate to large perturbation levels of 0.6 and smoothed with a 3-unit window. Notably, the $\alpha=-0.5$ metric, resembling the conventional ambient Euclidean metric, generally exhibited less perturbation behavior than the Stiefel canonical metric $\alpha=0$ (Figure \ref{fig:alpha}). This observation encourages further investigation, offering a compelling avenue for future research into the nuanced impact of differently weighting the $A$ and $B$ parameters via one's choice of $\alpha$-metric.

\begin{figure}[htbp]
\centering
\subfigure[Applying $\alpha$-metric geometry to the SteamGen data set.\label{fig:alpha_steam}]{\includegraphics[width=0.7\textwidth]{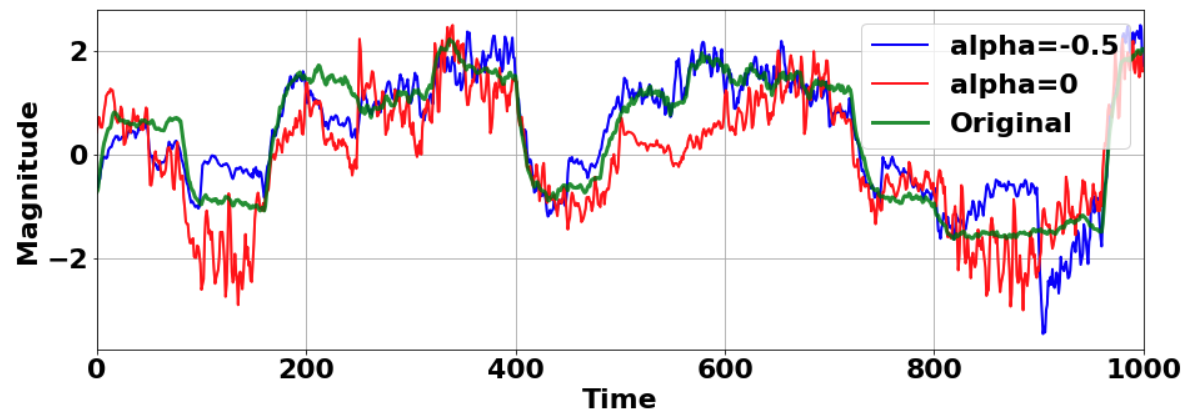}}

\subfigure[Applying $\alpha$-metric geometry to the Taxi data set.\label{fig:alpha_taxi}]{\includegraphics[width=0.7\textwidth]{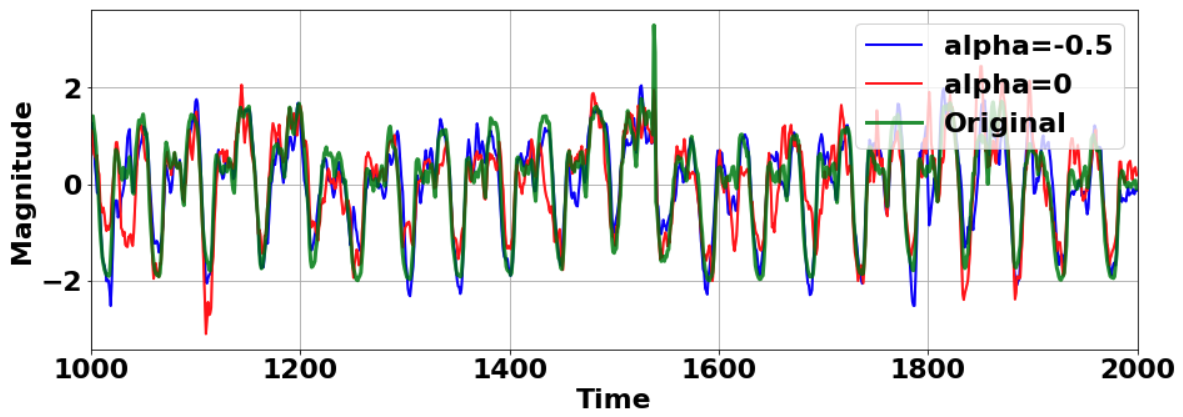}}

\caption{Applying $\alpha=[-0.5,0.0]$ to the SteamGen and Taxi datasets, to see the effect of different $\alpha$-metrics on the \textit{StiefelGen algorithm}.}
\label{fig:alpha}
\end{figure}

\end{document}